%% file: main.tex
\documentclass{article} 
\usepackage{iclr2026_arxiv,times}

\usepackage{microtype}
\usepackage{graphicx}
\usepackage{subfig}
\usepackage{booktabs} 

\input{math_commands}

\usepackage{amsmath}
\usepackage{amssymb}
\usepackage{mathtools}
\usepackage{amsthm}
\usepackage{wrapfig}
\usepackage[utf8]{inputenc} 
\usepackage[T1]{fontenc}    
\usepackage{amsfonts}       
\usepackage{nicefrac}       
\usepackage{xcolor}
\usepackage{tikz}
\usetikzlibrary{shapes, arrows.meta, positioning, decorations.markings}
\usepackage{subcaption,siunitx}
\usepackage[colorlinks=true,
            linkcolor=red,
            citecolor=blue1,
            filecolor=blue1,
            urlcolor=blue1]{hyperref}
\usepackage{url}
\usepackage{MnSymbol}
\usepackage{tcolorbox}

\usepackage{dsfont}
\usepackage{enumitem}
\usepackage{longtable}
\usepackage{makecell}
\usepackage{multirow}

\usepackage[capitalize,noabbrev]{cleveref}

\newtheorem{theorem}{Theorem}
\newtheorem{lemma}[theorem]{Lemma}

\newtheorem{definition}{Definition}
\newtheorem{assumption}{Assumption}
\newtheorem{remark}{Remark}
\newtheorem{proposition}{Proposition}

\newtheorem{principle}{Principle}

\usepackage{algorithm}
\usepackage{algpseudocode}
\usepackage{caption}

\usepackage[textsize=tiny]{todonotes}

\usepackage{minitoc}

\title{Observationally Informed Adaptive Causal Experimental Design}

\author{
  Erdun Gao\textsuperscript{1}, Liang Zhang\textsuperscript{2,$\dagger$}, Jake Fawkes\textsuperscript{3}, Aoqi Zuo\textsuperscript{4}, Wenqin Liu\textsuperscript{5}, Haoxuan Li\textsuperscript{6}, \\ \textbf{Mingming Gong\textsuperscript{5} \& Dino Sejdinovic\textsuperscript{1}} \\
  \textsuperscript{1}Australian Institute for Machine Learning, Adelaide University \\ \textsuperscript{2}Hong Kong University of Science and Technology (Guangzhou) \\ 
  \textsuperscript{3}University College London \quad
  \textsuperscript{4}The University of Sydney \\
  \textsuperscript{5}The University of Melbourne \quad
  \textsuperscript{6}Peking University \\
  \texttt{\{erdun.gao,dino.sejdinovic\}@adelaide.edu.au} \\
  \texttt{liangzhang@hkust-gz.edu.cn}, \texttt{j.fawkes@ucl.ac.uk} \\ \texttt{aoqi.zuo@sydney.edu.au}, \texttt{wenqinl@student.unimelb.edu.au} \\ \texttt{hxli@stu.pku.edu.cn}, \texttt{mingming.gong@unimelb.edu.au}
}

\iclrfinalcopy 
\begin{document}

\maketitle
\begingroup
  \renewcommand{\thefootnote}{$\dagger$}
  \footnotetext{Corresponding author.}
\endgroup

\doparttoc 
\faketableofcontents

\input{Pages/Abstract}
\input{Pages/Introduction}
\input{Pages/Preliminaries}
\input{Pages/Problem}
\input{Pages/Framework}
\input{Pages/Theory}
\input{Pages/Experiments}
\input{Pages/Related_works}

\input{Pages/Conclusion}

\clearpage
\bibliography{reference}
\bibliographystyle{iclr2026_conference}

\newpage
\appendix
\onecolumn
\part{Appendix} 
\parttoc

\input{Pages/Appendix/Additional_related_works}
\input{Pages/Appendix/Models}
\input{Pages/Appendix/Acq_functions}
\input{Pages/Appendix/Convergence}
\input{Pages/Appendix/Results}
\input{Pages/Appendix/Discussions}

\end{document}

%% file: math_commands.tex
\usepackage{amsmath,amsfonts,bm}

\def\eqref#1{equation~\ref{#1}}

\def\1{\bm{1}}

\def\rr{{\textnormal{r}}}
\def\rs{{\textnormal{s}}}
\def\rt{{\textnormal{t}}}

\def\ry{{\textnormal{y}}}

\def\rvu{{\mathbf{i}}}

\def\rvu{{\mathbf{u}}}

\def\rvx{{\mathbf{x}}}

\def\vp{{\bm{p}}}

\def\vx{{\bm{x}}}

\def\mI{{\bm{I}}}

\def\mK{{\bm{K}}}

\def\mX{{\bm{X}}}

\DeclareMathAlphabet{\mathsfit}{\encodingdefault}{\sfdefault}{m}{sl}
\SetMathAlphabet{\mathsfit}{bold}{\encodingdefault}{\sfdefault}{bx}{n}

\def\gA{{\mathcal{A}}}
\def\gB{{\mathcal{B}}}
\def\gC{{\mathcal{C}}}
\def\gD{{\mathcal{D}}}

\def\gH{{\mathcal{H}}}

\def\gM{{\mathcal{M}}}
\def\gN{{\mathcal{N}}}
\def\gO{{\mathcal{O}}}
\def\gP{{\mathcal{P}}}

\def\gS{{\mathcal{S}}}

\def\gU{{\mathcal{U}}}

\def\gX{{\mathcal{X}}}
\def\gY{{\mathcal{Y}}}
\def\gZ{{\mathcal{Z}}}

\def\sI{{\mathbb{I}}}

\def\sP{{\mathbb{P}}}

\def\sV{{\mathbb{V}}}

\newcommand{\E}{\mathbb{E}}

\newcommand{\R}{\mathbb{R}}

\newcommand{\entropy}{\mathrm{H}}
\newcommand{\mi}{\mathrm{I}}

\newcommand{\tar}{\textup{tar}}

\DeclareMathOperator*{\argmax}{arg\,max}
\DeclareMathOperator*{\argmin}{arg\,min}

\usepackage{aligned-overset}

\usepackage{fontawesome5}
\usepackage[svgnames]{xcolor}
\usepackage[most]{tcolorbox}

\definecolor{mygreen}{HTML}{79BF41}
\definecolor{myblue}{HTML}{4DBCC9}
\definecolor{myred}{named}{FireBrick} 
\definecolor{codegreen}{rgb}{0,0.6,0}
\definecolor{codegray}{rgb}{0.5,0.5,0.5}
\definecolor{codepurple}{rgb}{0.58,0,0.82}
\definecolor{backcolour}{rgb}{0.95,0.95,0.92}
\definecolor{blue1}{HTML}{2E86AB}

\lstdefinestyle{mystyle}{
    backgroundcolor=\color{backcolour},   
    commentstyle=\color{codegreen},
    keywordstyle=\color{magenta},
    numberstyle=\tiny\color{codegray},
    stringstyle=\color{codepurple},
    basicstyle=\ttfamily\scriptsize,
    breakatwhitespace=false,         
    breaklines=true,                 
    captionpos=b,                    
    keepspaces=true,                                 
    numbersep=1pt,                  
    showspaces=false,                
    showstringspaces=false,
    showtabs=false,                  
    tabsize=1
}
\tcbset {
  base/.style={
    arc=0mm, 
    bottomtitle=0.5mm,
    boxrule=0mm,
    colbacktitle=black!10!white, 
    coltitle=black, 
    fonttitle=\bfseries, 
    left=1mm,
    leftrule=1mm,
    right=1mm,
    top=1mm,
    bottom=1mm,
    title={#1},
    toptitle=0.75mm, 
    width=\linewidth
  }
}

\newtcolorbox{greybox}[1][]{%
    colback=blue!3!white,    
    colframe=blue!3!white,   
    sharp corners,           
    #1
}

\newtcolorbox{leftbarbox}[1][]{%
    blanker,                 
    left=3mm,                
    borderline west={1mm}{0pt}{blue!40!black}, 
    #1
}

\newtcolorbox{cleanframe}[1][]{%
    colback=white,           
    colframe=black!50,       
    boxrule=0.5pt,           
    sharp corners,           
    #1
}

\newtcolorbox{promptbox}[1]{
  colframe=black!15!white,
  base={#1},
  leftrule=0mm,
  breakable,
}

\newtcolorbox{bluebox}[1]{
  colframe=myblue!50!white,
  colback=myblue!15!white,
  base={#1},
  breakable
}

\newtcolorbox{greenbox}[1]{
  colframe=mygreen!50!white,
  colback=mygreen!15!white,
  base={#1},
  breakable
}

\newtcolorbox{redbox}[1]{
  colframe=myred!50!white,
  colback=myred!15!white,
  base={#1},
  breakable
}

\newcommand{\circlednumorange}[1]{%
  \tikz[baseline=(char.base)]{
    \node[shape=circle,draw=orange!50!black,fill=orange!10,inner sep=1pt] (char) {\small \textbf{#1}};
  }%
}
\newcommand{\circlednumblue}[1]{%
  \tikz[baseline=(char.base)]{
    \node[shape=circle,draw=blue!50!black,fill=blue!10,inner sep=1pt] (char) {\small \textbf{#1}};
  }%
}

%% file: Pages/Abstract.tex
\begin{abstract}
Randomized Controlled Trials (RCTs) represent the gold standard for causal inference yet remain a scarce resource. While large-scale observational data is often available, it is utilized only for retrospective fusion, and remains discarded in prospective trial design due to bias concerns. We argue this "tabula rasa" data acquisition strategy is fundamentally inefficient. In this work, we propose \textit{Active Residual Learning}, a new paradigm that leverages the observational model as a foundational prior. This approach shifts the experimental focus from learning target causal quantities from scratch to efficiently estimating the residuals required to correct observational bias. To operationalize this, we introduce the R-Design framework. Theoretically, we establish two key advantages: (1) a structural efficiency gap, proving that estimating smooth residual contrasts admits strictly faster convergence rates than reconstructing full outcomes; and (2) information efficiency, where we quantify the redundancy in standard parameter-based acquisition (e.g., BALD), demonstrating that such baselines waste budget on task-irrelevant nuisance uncertainty. We propose R-EPIG (Residual Expected Predictive Information Gain), a unified criterion that directly targets the causal estimand, minimizing residual uncertainty for estimation or clarifying decision boundaries for policy. Experiments on synthetic and semi-synthetic benchmarks demonstrate that R-Design significantly outperforms baselines, confirming that repairing a biased model is far more efficient than learning one from scratch.
\end{abstract}

%% file: Pages/Introduction.tex
\section{Introduction}
\label{sec:introduction}

Precise estimation of individual treatment effects~\citep{wager2018estimation} is fundamental to personalized decision-making in domains ranging from economics~\citep{heckman2000causal} and recommendation systems~\citep{gao2024causal} to healthcare~\citep{foster2011subgroup}. While the Average Treatment Effect (ATE) offers a coarse summary, effective policy requires the granularity of the Conditional Average Treatment Effect (CATE), which quantifies how interventions vary across individuals~\citep{kunzel2019metalearners}. However, accurately estimating CATE is fundamentally constrained by the data dichotomy. Large-scale observational studies offer population representativeness~\citep{athey2019estimating} but suffer from hidden confounding~\citep{pearl2009causality}, whereas Randomized Controlled Trials (RCTs) guarantee unbiasedness but are crippled by prohibitive costs and limited sample sizes~\citep{frieden2017evidence, degtiar2023review,feuerriegel2024causal}. This trade-off has motivated a growing literature on \textit{post-hoc data fusion}~\citep{kallus2018removing, colnet2024causal, zhou2025two, fawkes2025the}. While these methods successfully demonstrate the benefits of integrating both modalities, they remain fundamentally \textit{retrospective}~\citep{song2024ace, cha2025abc3, zhang2025active}. They treat the experimental data collection as a fixed step, attempting to repair statistical power after the fact. Consequently, this passive paradigm fails to optimize the data acquisition process itself, leaving the potential synergy between observational priors and experimental design largely unexplored.

\begin{figure*}[t]
\centering
\subfloat{
    \includegraphics[width=0.23\textwidth]{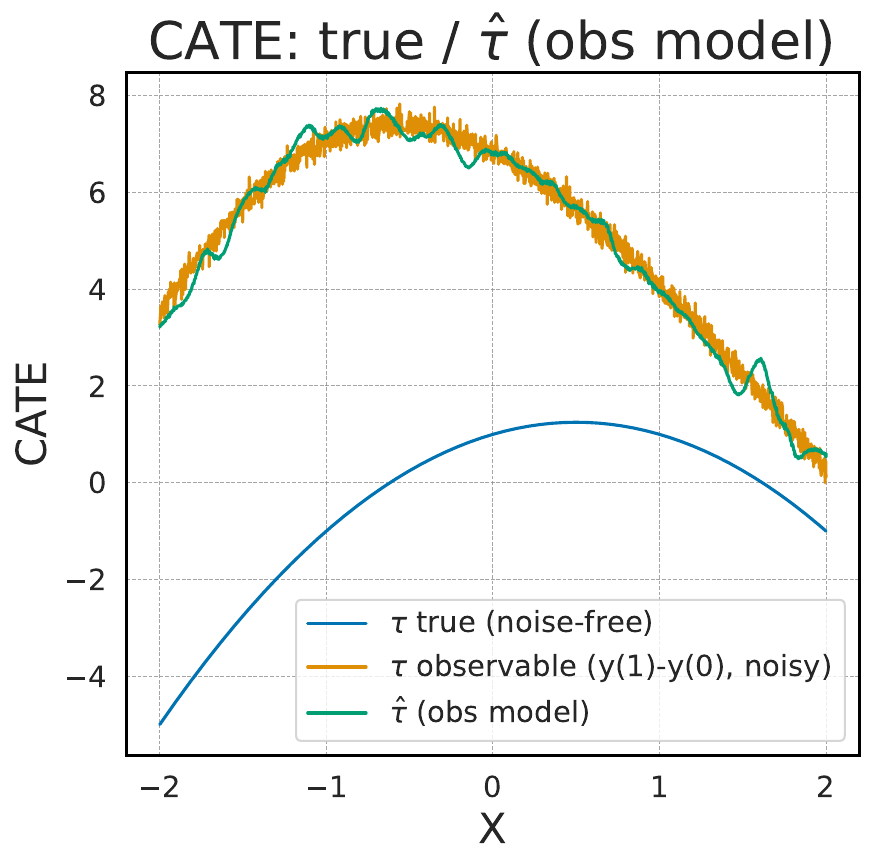}
    \label{fig:vs1}
}
\hfill
\subfloat{
    \includegraphics[width=0.23\textwidth]{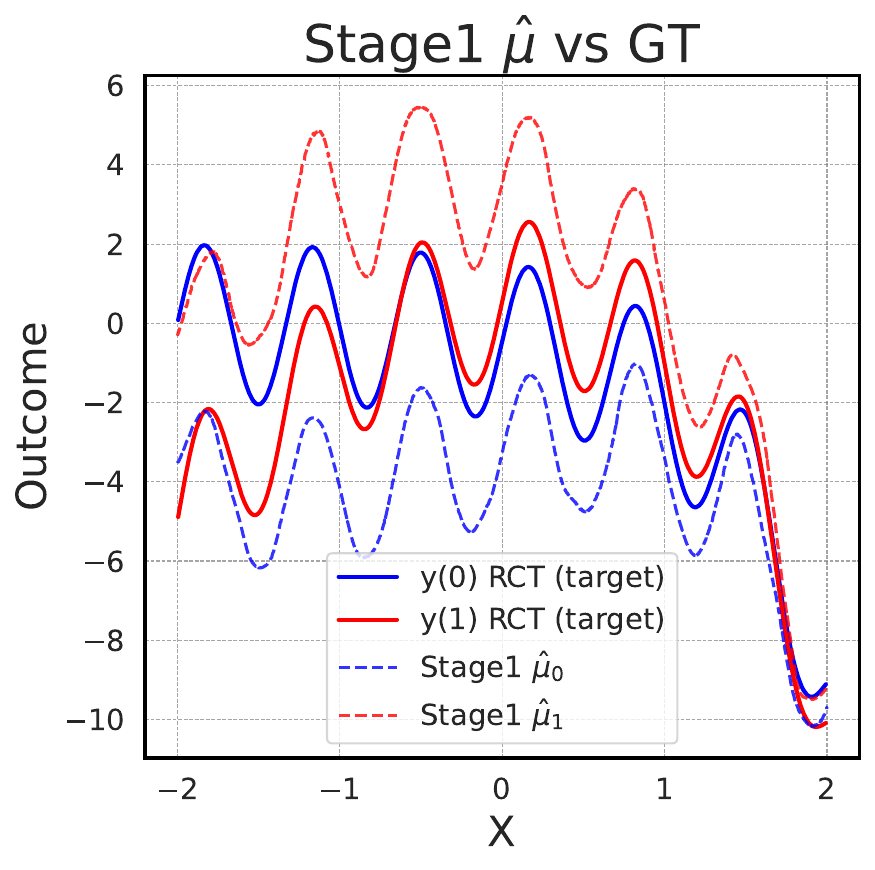}
    \label{fig:vs2}
}
\hfill
\subfloat{
    \includegraphics[width=0.23\textwidth]{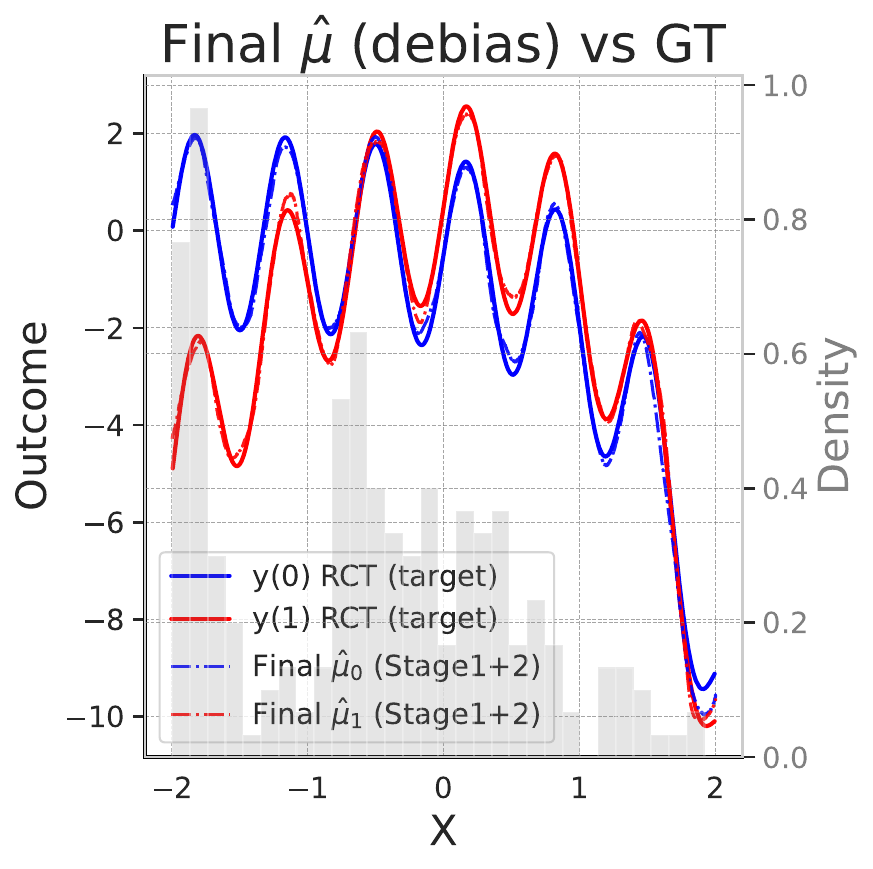}
    \label{fig:vs3}
}
\hfill
\subfloat{
    \includegraphics[width=0.23\textwidth]{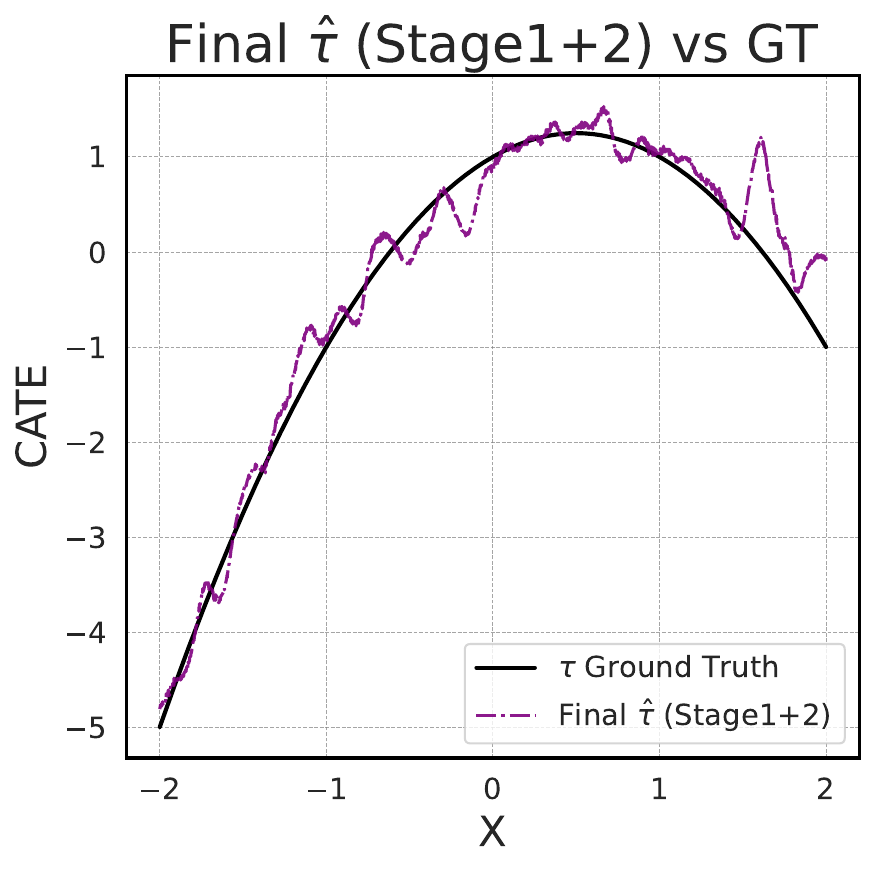}
    \label{fig:vs4}
}
\caption{R-Design intuition on 1D synthetic data. The observational prior is biased (Left) but captures high-frequency structure (Middle-Left). Treating it as a fixed offset, R-Design learns a simpler residual function (Middle-Right), exploiting the complexity gap to recover the true CATE with few samples (Right), while the naive observational model fails due to uncorrected bias.}
\label{fig:intuition}
\end{figure*}

\textbf{Motivation.} In many real-world application domains, large amounts of observational data are frequently available prior to experimentation~\citep{rosenman2021designing, epanomeritakis2025choosing}. Yet, standard causal experimental design for CATE estimation typically adopts a tabula rasa approach, ignoring this prior information to learn causal mechanisms from scratch~\citep{addanki2022sample, ghadiri2023finite}. We argue this is statistically inefficient for two reasons. First, it overlooks the target population distribution delineated by the large-scale observational data, which characterizes the real-world covariate support where the causal estimand is ultimately evaluated. Second, it discards the observational model itself, which, despite hidden confounding, often successfully captures the global baseline structure of the outcomes~\citep{kallus2018removing, hatt2022combining}. This suggests a critical opportunity to transform experimental design. Rather than treating the confounded observational model as a nuisance to be discarded, it could serve as a foundational informative prior (as illustrated in Fig.~\ref{fig:intuition}). This perspective motivates the core research question:
\begin{bluebox}{}
\faThumbtack \ \textbf{Key Question:} How can we leverage the confounded observational prior to guide causal experimental design, shifting the goal from exploring outcomes from scratch to adaptively learning the necessary corrections for the observational bias?
\label{key_question}
\end{bluebox}
\textbf{Challenges.} However, operationalizing this adaptive correction strategy is non-trivial and faces three fundamental challenges. \circlednumblue{1} \textit{Disentangling Signal from Bias.} The observational prior intrinsically conflates rich structural information (signal) with hidden confounding (bias). The difficulty lies in designing a mechanism that selectively transfers the global structural knowledge without inheriting the local confounding errors, effectively isolating the useful prior shape from the biases we seek to correct. \circlednumblue{2} \textit{Objective Misalignment.} There is an inherent discrepancy between identifying the global observational bias and maximizing downstream performance. The regions where the observational model is most biased (high error magnitude) may not overlap with the target population or the decision boundaries. A naive design might waste budget correcting biases in irrelevant regions. \circlednumblue{3} \textit{Lack of Residual-Aware Acquisition.} Standard causal experimental design metrics are ill-suited here. While Bayesian strategies can incorporate priors, they typically target the uncertainty of the full treatment effect. Consequently, they fail to focus the budget strictly on learning the residual contrast required for bias correction.

\textbf{Contributions.} To address these gaps, this paper makes the following contributions. 
\circlednumorange{1} \textit{A New Paradigm.} We formally define the problem of observationally informed adaptive causal experimental design. We argue that when informative (albeit biased) observational data exists, the design objective should fundamentally shift from reconstructing outcome surfaces from scratch to strictly identifying and correcting observational biases. \circlednumorange{2} \textit{The R-Design Framework.} We propose a unified methodology comprising two core components: R-EPIG (Sec.~\ref{sec:acquisition_function}), a task-aware information-theoretic criterion that adapts to estimation or policy objectives, and the Two-Stage Residual (TSR) strategy (Sec.~\ref{sec:method_tsr}), which structurally decouples base effect estimation from uncertainty quantification to ensure scalability. \circlednumorange{3} \textit{Theoretical Foundations.} We provide a rigorous theoretical analysis (Sec.~\ref{sec:theory}) anchored by four key results. We formally establish the \textit{Structural Efficiency Gap} (Lemma~\ref{lemma:complexity_decomp_hoelder}), proving that learning the residual function is structurally more sample-efficient than learning from scratch. We demonstrate the \textit{Objective Alignment} (Prop.~\ref{prop:objective_alignment}), verifying the equivalence between our acquisition objective and Bayesian CATE risk minimization. Crucially, we quantify the \textit{Information Redundancy} of standard baselines (Prop.~\ref{prop:redundancy}), proving that parameter-based methods waste budget on nuisance uncertainty. Finally, we derive uniform convergence rates for our greedy strategy (Thm.~\ref{thm:epig_bound}). \circlednumorange{4} \textit{Extensive Empirical Validation.} We conduct comprehensive experiments across synthetic and semi-synthetic benchmarks (Sec.~\ref{sec:experiments}). Our results demonstrate that R-Design consistently outperforms state-of-the-art baselines, providing strong empirical evidence that repairing a biased model is significantly more efficient than tabula rasa acquisition.

%% file: Pages/Preliminaries.tex
\section{Preliminaries and Problem Setup}
\label{sec:preliminaries}

In this section, we establish the mathematical foundations of our framework. We first formalize the causal inference problem and the CATE estimand within the potential outcomes framework. Subsequently, we review the necessary background on disparate data integration and the specific active learning setup used in adaptive causal experimental design.

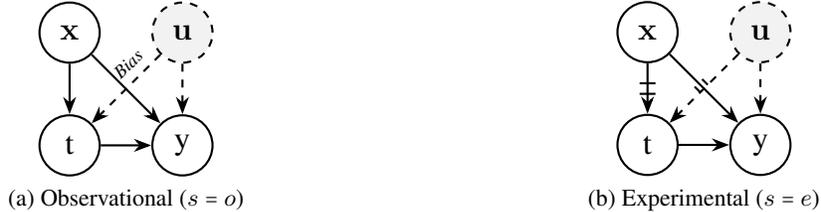
\begin{figure}[t]
    \centering
    \tikzset{
        severed/.style={
            decoration={markings, mark=at position 0.5 with {
                \draw[thick, -] (-2pt, -3pt) -- (-2pt, 3pt);
                \draw[thick, -] (2pt, -3pt) -- (2pt, 3pt);
            }},
            postaction={decorate}
        }
    }

    \begin{minipage}{0.45\linewidth}
        \centering
        \begin{tikzpicture}[scale=0.6,
            node_style/.style={circle, draw=black, thick, minimum size=0.8cm, font=\large},
            dashed_node/.style={circle, draw=black, dashed, thick, minimum size=0.8cm, fill=gray!10, font=\large},
            arrow/.style={->, >=Stealth, thick},
            dashed_arrow/.style={->, >=Stealth, thick, dashed}
        ]
            \node[node_style] (T) at (0,0) {$\rt$};
            \node[node_style] (Y) at (2.5,0) {$\ry$};
            \node[node_style] (X) at (0,2.5) {$\rvx$};
            \node[dashed_node] (U) at (2.5,2.5) {$\rvu$};

            \draw[arrow] (X) -- (T);
            \draw[dashed_arrow] (U) -- (Y);
            \draw[arrow] (T) -- (Y);
            \draw[arrow] (X) -- (Y);
            \draw[dashed_arrow] (U) -- node[pos=0.3, sloped, above, font=\scriptsize] {\textit{Bias}} (T);
        \end{tikzpicture}
        \vspace{0.2em}
        \centerline{\small (a) Observational ($s=o$)}
    \end{minipage}
    \hfill 
    \begin{minipage}{0.45\linewidth}
        \centering
        \begin{tikzpicture}[scale=0.6,
            node_style/.style={circle, draw=black, thick, minimum size=0.8cm, font=\large},
            dashed_node/.style={circle, draw=black, dashed, thick, minimum size=0.8cm, fill=gray!10, font=\large},
            arrow/.style={->, >=Stealth, thick},
            dashed_arrow/.style={->, >=Stealth, thick, dashed}
        ]
            \node[node_style] (T) at (0,0) {$\rt$};
            \node[node_style] (Y) at (2.5,0) {$\ry$};
            \node[node_style] (X) at (0,2.5) {$\rvx$};
            \node[dashed_node] (U) at (2.5,2.5) {$\rvu$};

            \draw[dashed_arrow] (U) -- (Y);
            \draw[arrow] (T) -- (Y);
            \draw[arrow] (X) -- (Y);

            \draw[arrow, severed] (X) -- (T);
            \draw[dashed_arrow, severed] (U) -- (T);
            
        \end{tikzpicture}
        \vspace{0.2em}
        \centerline{\small (b) Experimental ($s=e$)}
    \end{minipage}

    \caption{Causal diagrams of data-generating processes.}
    \label{fig:causal_graph_comparison}
\end{figure}

\subsection{Causal Estimands and Targets}

\textbf{Potential Outcome Framework.} Our analytical framework is grounded in the Neyman-Rubin potential outcomes model~\citep{rubin2005causal,ding2024first}. Let $\rvx \in \gX$ denote the covariates, $\rt \in \{0,1\}$ the binary treatment assignment, and $\ry \in \gY$ the observed outcome. We use $\vx$, $t$, and $y$ to denote realizations of these variables. The two potential outcomes are $\ry(0)$ and $\ry(1)$, corresponding to the potential outcome under control and treatment, respectively. To accommodate diverse downstream objectives, we define a general target quantity of interest, denoted by $\Phi(\vx)$, evaluated over a target population distribution $p_{\text{tar}}(\vx)$ (typically the observational population $p_o(\vx)$~\citep{kallus2018removing, colnet2024causal}). We use a hat to denote an estimator of a target quantity, e.g., $\hat{\Phi}$. In this paper, we focus on two primary specifications:

\textbf{CATE Estimation.} When the primary objective is to quantify the heterogeneity of treatment effects across the population, we target the Conditional Average Treatment Effect (CATE). Formally, let $\mu(\vx, t) \coloneqq \E[\ry(t) \mid \rvx=\vx]$ denote the expected potential outcome surface for arm $t$. The CATE is defined as the pointwise difference between the treated and control response surfaces:
\begin{equation}
    \tau(\vx) \coloneqq \mu(\vx, 1) - \mu(\vx, 0).
\end{equation}
This estimand captures how the treatment effect varies as a function of the covariates $\vx$. The estimation quality is measured by the Precision in Estimation of Heterogeneous Effects (PEHE)~\citep{hill2011bayesian}, which quantifies the expected squared error (or $L_2$ risk) of the estimator over the target distribution:
\begin{equation}
    \epsilon_{\text{PEHE}}(\hat\tau) = \E_{\rvx \sim p_{\text{tar}}(\cdot)}\left[(\hat{\tau}(\rvx) - \tau(\rvx))^2\right].
    \label{eq:pehe}
\end{equation}
\textbf{Decision Making.} Alternatively, the objective may be policy optimization, where the goal is to determine the optimal treatment assignment for each unit to maximize the total population welfare~\citep{fernandez2022causal}. This is equivalent to identifying the \textit{sign} of the treatment effect rather than estimating its continuous magnitude. The target estimand is the binary optimal policy $\Phi(\vx) = \pi(\vx) = \sI(\tau(\vx) > 0)$, which prescribes the treatment arm yielding the higher expected outcome. Unlike CATE estimation, this task implies that high precision is primarily required near the decision boundary (where $\tau(\vx) \approx 0$). The performance is measured by the Average Policy Error (APE)~\citep{gao2024variational}, defined as the probability of assigning a sub-optimal treatment (misclassification rate) over the target population:
\begin{equation}
    \epsilon_{\text{APE}}(\hat\pi) = \E_{\rvx \sim p_{\text{tar}}(\cdot)}\left[ \sI(\hat\pi(\rvx) \neq \pi(\rvx)) \right].
    \label{eq:ape}
\end{equation}

\subsection{Disparate Data Sources}
We consider a scenario involving two distinct data sources: a large observational study (OS) and a smaller experimental dataset, as shown in Fig.~\ref{fig:causal_graph_comparison}. We denote the source using a random variable $\rs \in \{o, e\}$. For each source $s$, we define the source-specific conditional means as $\mu_s(\vx,t) \coloneqq \E[\ry \mid \rvx=\vx, \rt=t, \rs=s]$. The source-specific contrast for source $s$ is $\tau_s(\vx) \coloneqq \mu_s(\vx,1) - \mu_s(\vx,0)$. This quantity measures the observed conditional mean difference. For the OS, the contrast $\tau_o(\vx)$ is generally biased due to hidden confounding~\citep{fawkes2025the}. For the experimental source, the controlled assignment mechanism ensures that $\tau_e(\vx)$ identifies the true CATE within the experimental support $\gX_e \subseteq \gX$:
\begin{equation}
    \tau(\vx) = \tau_e(\vx), \quad \text{for } \vx \in \gX_e.
\end{equation}
To integrate these disparate sources, we characterize the structural relationship between the biased observational contrast $\tau_o(\vx)$ and the true causal effect $\tau(\vx)$ through the following decomposition:
\begin{equation}
\label{eq:residual_decomp}
    \tau(\vx) = \tau_o(\vx) + \tau_\delta(\vx), \quad \text{where } \tau_\delta(\vx) \coloneqq \delta(\vx, 1) - \delta(\vx, 0).
\end{equation}
Here, $\tau_\delta(\vx)$ represents the residual contrast (debiasing correction) that captures the discrepancy between the observational correlations and the underlying causal mechanism. To ensure that this decomposition yields a valid and identifiable causal target, we impose the following assumptions on the data generation mechanisms.
\begin{assumption}
\label{ass:identifiability}
\textit{Stable Unit Treatment Value Assumption:} $\ry = \rt \ry(1) + (1-\rt)\ry(0)$, and no interference between units. \textit{Unconfoundedness by Design:} The treatment assignment in the experiment is controlled by the design policy and depends solely on observed information, ensuring $\rt \perp \{\ry(0), \ry(1)\} \mid (\rvx, \rs=e)$. \textit{Positivity of the Design Space:} The experimental design allows for the exploration of both treatment arms within the region of interest, $0 < p(\rt=1 \mid \rvx, \rs=e) < 1$. \textit{Effect Invariance (Transportability):} The underlying causal mechanism is invariant across sources, i.e.,
$\forall s \in \{o, e\}, \E[\ry(1) - \ry(0) \mid \rvx=\vx, \rs=s] = \tau(\vx)$. This rules out unnatural shifts in CATE where the experimental protocol itself, rather than the treatment, interferes with the outcome mechanism. \textit{Support Inclusion:} The support of the target population is contained within the feasible experimental region, $\text{supp}(p_{\text{tar}}) \subseteq \gX_e$.
\end{assumption}
Assump.~\ref{ass:identifiability} ensures that $\tau_\delta(\vx)$ is a structurally consistent quantity representing pure observational bias. Transportability guarantees that $\tau_\delta(\vx)$ captures the confounding mechanism rather than a shift in the underlying CATE, rendering the causal target identifiable through residual correction.

\subsection{Problem Setup}
\label{sec:problem_statement}

Unlike retrospective fusion where data collection is fixed, we address the dynamic allocation of a limited experimental budget $n_B$. We assume access to a static, labeled observational dataset $\gD_O = \{(\vx_i, t_i, y_i, s_i=o)\}_{i=1}^{n_O}$. We aim to optimize the CATE estimator for a target population, characterized by the marginal covariate distribution $p_{\text{tar}}(\vx)$. Following standard conventions, we assume the observational covariates are representative of this target (i.e., the empirical distribution of $\gD_O$ approximates $p_{\text{tar}}$)~\citep{kallus2018removing, colnet2024causal,yang2025data}. For the experimental phase, we have a finite candidate pool of recruitable units $\gD_P = \{\vx_j\}_{j=1}^{n_P}$. The experimental dataset $\gD_E$ is initially empty. The design process proceeds in sequential stages $k=1, \dots, n_B$. At each stage: \circlednumblue{1} The learner jointly selects a unit $\vx_k \in \gD_P$ and a treatment $t_k \in \{0, 1\}$; \circlednumblue{2} The learner observes the experimental outcome $y_k$; \circlednumblue{3} The datasets are updated: $\gD_E^{(k)} \leftarrow \gD_E^{(k-1)} \cup \{(\vx_k, t_k, y_k, \rs_k=e)\}$ and the unit is removed from the candidate pool $\gD_P$.
\begin{bluebox}{}
\faThumbtack \ \textbf{Design Objective.}
The central challenge of this task is to design a principled utility function, $U(\cdot \mid \gD_O, \gD_E^{(k-1)})$, that quantifies the information gain of a specific query $(\vx, t)$. The acquisition strategy at stage $k$ maximizes this utility over the joint product space:
\begin{equation}
(\vx_k^*, t_k^*) = \argmax_{(\vx, t) \in \gD_P \times \{0, 1\}} U(\vx, t \mid \gD_O, \gD_E^{(k-1)}).
\label{eq:utility_function}
\end{equation}
\label{key_objective}
\end{bluebox}

%% file: Pages/Problem.tex
\section{Principles for Observationally Informed Causal Experimental Design}
\label{sec:principles}

\textbf{From Outcome Exploration to Residual Correction.} Standard experimental design typically adopts a \textit{tabula rasa} way, expending budget to learn the full outcome surface from scratch. We argue this is structurally inefficient. Given the decomposition $\tau(\vx) = \tau_o(\vx) + \tau_\delta(\vx)$ where the observational contrast $\tau_o(\vx)$ is estimated prior to experimentation, the conditional entropy of the total effect is equivalent to that of the residual (treating $\tau_o$ as a fixed offset): $\entropy(\tau(\vx) \mid \gD_E) \equiv \entropy(\tau_\delta(\vx) \mid \gD_E)$. Consequently, any acquisition that re-learns the pre-estimated signal $\tau_o(\vx)$ yields zero marginal utility. This structural equivalence dictates our first principle:
\begin{principle}[The Residual Uncertainty Principle]
\label{principle:fusion}
Experimental design should target the epistemic uncertainty of the residual contrast $\tau_\delta(\vx)$. The acquisition function should treat the estimated observational contrast $\tau_o(\vx)$ as a fixed mean function, utilizing the budget to resolve the structural discrepancy between the biased prior and the interventional truth, rather than re-learning the full outcome surface.
\end{principle}
\textbf{From Global Exploration to Estimand Alignment.} While Principle~\ref{principle:fusion} identifies the \textit{structural target} (the residual), it remains silent on the \textit{spatial allocation} of the budget. Classical active learning criteria typically aim to minimize model variance globally, implicitly assuming a uniform utility over the domain~\citep{houlsby2011bayesian}. However, causal inference is intrinsically objective-driven; its validity is defined solely relative to a specific estimand (e.g., CATE or Policy) and a target population $p_{\text{tar}}(\vx)$~\citep{hernan2023causal, smith2023prediction}. Consequently, reducing residual uncertainty in regions that are irrelevant to the target density or inconsequential for the decision boundary provides negligible value. Pure epistemic uncertainty does not equate to utility.
\begin{principle}[Estimand-Aligned Acquisition]
\label{principle:alignment}
Optimal experimental design must be risk-centric rather than uncertainty-centric. The utility of a design point $(\vx, t)$ is not intrinsic, but is quantified solely by the expected reduction in the target-specific risk. This mandates that the experimental budget be concentrated non-uniformly on regions where the residual uncertainty is both high and consequential for the specified estimand.
\end{principle}

%% file: Pages/Framework.tex
\section{The R-Design Framework}
\label{sec:R-Design}

In this section, we introduce R-Design, a framework that transforms causal experimental design into active residual learning. Leveraging the decomposition $\tau(\vx) = \hat{\tau}_o(\vx) + \tau_\delta(\vx)$, where the observational contrast $\hat{\tau}_o$ acts as a pre-computed offset, R-Design explicitly targets the identification of the residual $\tau_\delta$. We first derive the residual-based acquisition criterion, R-EPIG, and then present the Active Learning (AL) strategy that enables its tractable computation.

\subsection{The R-Design Criterion: R-EPIG}
\label{sec:acquisition_function}

Our framework operationalizes Principle~\ref{principle:alignment} by explicitly minimizing the epistemic uncertainty of the specific downstream target $\Phi(\cdot)$. Critically, since the observational base $\hat{\mu}_o(\vx, t)$ is treated as a fixed mean function, the information content of an experimental query $(\vx, t)$ resides solely in the \textit{residual outcome}, denoted as the random variable $\rr \coloneqq \ry - \hat{\mu}_o(\vx, t)$. We propose the Residual Expected Predicted Information Gain (R-EPIG). Unlike standard criteria operating on raw outcomes~\citep{smith2023prediction}, R-EPIG selects experiments to maximize the expected information gain regarding the target $\Phi(\cdot)$ specifically via the observed residual $\rr$. Let $\gH_k$ be the history at step $k$. The general acquisition objective is:
\begin{bluebox}{}
    \begin{equation}
        \alpha_{\text{R-EPIG}}(\vx, t) \coloneqq \E_{\rvx^* \sim p_{\text{tar}}(\cdot)} \Big[ \mi\big(\rr;\, \Phi(\rvx^*) \mid \vx, t, \gH_k \big) \Big].
    \label{eq:repig_general}
    \end{equation}
\end{bluebox}
Here, the randomness in the target $\Phi(\cdot)$ is induced solely by the posterior of the residual processes. The outer expectation averages this gain over the target population.
\begin{remark}
While $\mi(\rr; \Phi) \equiv \mi(\ry; \Phi)$, formulating the objective over $\rr$ aligns with our Residual GP. This structurally decouples inference from $\gD_O$, ensuring complexity scales solely with the experimental budget $n_B$ rather than the observational size $n_O$.
\end{remark}
By specifying $\Phi(\cdot)$, this framework yields two task-specific strategies: \textit{R-EPIG-Est} for estimation and \textit{R-EPIG-Policy} for policy.

\subsubsection{Design for Estimation (R-EPIG-Est)}
When the goal is minimizing PEHE, we target the magnitude of the correction terms. We propose two specifications:

\textbf{Joint-Outcome Residuals (R-EPIG-$\mu$).} 
To recover the full potential outcome surfaces, we define the target as the \textit{residual vector} $\Phi(\vx) = \boldsymbol{\delta}(\vx) \coloneqq [\delta(\vx, 0), \delta(\vx, 1)]^\top$. Since the observational base $\hat{\mu}_o$ is treated as a fixed offset, reducing uncertainty in $\boldsymbol{\delta}(\vx)$ directly reduces uncertainty in the absolute potential outcomes $\mu_e(\vx, t)$. The criterion is:
\begin{equation}
    \alpha_{\text{R-EPIG}}^{\mu}(\vx, t) \coloneqq \E_{\rvx^* \sim p_{\text{tar}}(\cdot)} \Big[ \mi\big(\rr;\, \boldsymbol{\delta}(\rvx^*) \mid \vx, t, \gH_k \big) \Big].
\label{eq:repig_mu}
\end{equation}
This strategy is conservative but robust: it identifies regions where the observational prior is unreliable for either arm.

\textbf{Direct CATE Residuals (R-EPIG-$\tau$).} 
To learn the CATE, we can define the target strictly as the residual contrast $\Phi(\vx) = \tau_\delta(\vx) \coloneqq \delta(\vx, 1) - \delta(\vx, 0)$. The criterion is:
\begin{equation}
    \alpha_{\text{R-EPIG}}^{\tau}(\vx, t) \coloneqq \E_{\rvx^* \sim p_{\text{tar}}(\cdot)} \Big[ \mi\big(\rr;\, \tau_\delta(\rvx^*) \mid \vx, t, \gH_k \big) \Big].
\end{equation}
This variant is theoretically more sample-efficient. By targeting the difference, it naturally ignores shared baseline errors that cancel out in the subtraction, avoiding budget waste on variations irrelevant to the treatment effect.

\subsubsection{Design for Policy (R-EPIG-Policy)}
When minimizing APE, the estimand collapses to the binary decision boundary. We define the target as the policy indicator $\Phi(\vx) = \pi(\vx) \coloneqq \sI(\hat{\tau}_o(\vx) + \tau_\delta(\vx) > 0)$. The criterion targets the information gain regarding this sign:
\begin{equation}
    \alpha_{\text{R-EPIG}}^{\pi}(\vx, t) \coloneqq \E_{\rvx^* \sim p_{\text{tar}}(\cdot)} \Big[ \mi\big(\rr;\, \pi(\rvx^*) \mid \vx, t, \gH_k \big) \Big].
\label{eq:repig_pi}
\end{equation}
Unlike variance reduction, R-EPIG-Policy is \textit{decision-aware}: it ignores regions where the sign of the total effect is determined with high confidence, concentrating the budget strictly on resolving ambiguities near the decision boundary.
\begin{remark}[Information Gain vs. Direct Error Reduction]
While minimizing APE is the ultimate goal, direct error reduction suffers from sparse signals: for candidates far from the current boundary, the probability of altering the policy is negligible, yielding zero-utility plateaus. In contrast, entropy reduction on $\pi$ serves as a smooth surrogate, providing informative guidance that drives the learner toward the boundary even before a sign flip is probable.
\end{remark}

\subsection{Model Architecture: The TSR Strategy}
\label{sec:method_tsr}

To instantiate the probabilistic terms in R-EPIG, we introduce the Two-Stage Regression (TSR) strategy. TSR structurally decouples the \textit{capacity} requirements from the \textit{uncertainty} requirements.

\textbf{Stage 1: The Observational Base Learner.}
We first exploit the abundant observational data $\gD_O$ to learn a high-capacity base estimate of the potential outcomes. This stage is model-agnostic, permitting the use of state-of-the-art estimators $\gM_O$ (e.g., TabPFN~\citep{hollmann2025tabpfn}, CausalPFN~\citep{balazadeh2025causalpfn}) to capture complex, global structural signals. The resulting outcome surfaces $\hat{\mu}_o(\vx, t)$ are treated as pre-computed functional offsets. This base model acts as a fixed mean function for the subsequent residual model, relieving the active learner from reconstructing the global outcome surface and allowing it to focus its limited capacity on correcting local deviations.

\textbf{Stage 2: Bayesian Residual Learning.}
The sequential phase strictly models the epistemic uncertainty of the residual vector $\boldsymbol{\delta}(\vx)$. Learning $\boldsymbol{\delta}(\vx)$ faces the fundamental counterfactual challenge: for any unit, we observe only the factual residual $r_k = y_k - \hat{\mu}_o(\vx_k, t_k)$.
To address this, we adopt a Multi-task GP (MTGP)~\citep{bonilla2007multi} as the backend inference engine. We choose MTGP because its kernel structure explicitly encodes causal correlations, allowing information sharing across treatment arms to infer missing counterfactual residuals:
\begin{equation}
    \boldsymbol{\delta}(\vx) \sim \mathcal{GP}(\mathbf{0}, \mK(\vx, \vx')), \quad r \mid \vx, t \sim \mathcal{N}(\delta(\vx, t), \sigma^2).
    \label{eq:tsr_mtgp}
\end{equation}
By decomposing the kernel $\mK$ into task-specific and input-specific components (e.g., via CMGP~\citep{alaa2017bayesian}, NSGP~\citep{alaa2018limits}), the model captures both the smoothness of the bias and the correlation between potential outcome residuals. \textit{Scalability.} Standard GPs scale cubically, $\mathcal{O}(n^3)$. TSR resolves this bottleneck: the GP fits \textit{only} the residuals on the small experimental set $\gD_E$. The complexity scales with the budget $n_E \ll n_O$, not the massive observational size, ensuring real-time feasibility for sequential acquisition.

\begin{remark}[The Intuition of Structural Efficiency]
\label{rmk:complexity_gap}
The efficiency of TSR relies on a key structural insight: while the full outcome mechanism (e.g., biological response) is often complex and data-hungry, the confounding bias (e.g., clinical guidelines) typically varies more smoothly across the population. By offloading the complex base learning to $\gD_O$, TSR focuses the expensive experimental budget solely on this simpler residual, enabling faster convergence.
\end{remark}

\subsection{Budgeted Acquisition Algorithm}
The R-EPIG utility drives our active learning strategy for CATE estimation under a fixed budget (Alg.~\ref{alg:r_design_loop}). The process employs a two-stage architecture. First, observational models are trained on $\gD_O$ and frozen to serve as functional offsets. Second, in the adaptive residual learning phase, we iteratively select batches of high-utility points. Unlike greedy approaches, we use softmax-weighted sampling on the utility scores to ensure diversity in the acquired batch. Experimental outcomes are queried, immediately transformed into residuals by subtracting the frozen offsets, and used to update the residual model $\hat{\boldsymbol{\delta}}$. This cycle continues until the budget is exhausted, yielding a final estimator that superimposes the learned residual contrast onto the observational baseline.

%% file: Pages/Theory.tex
\section{Theoretical Analysis}
\label{sec:theory}

We justify our proposed R-Design framework by establishing four theoretical pillars. 
First, we identify a fundamental \textit{Structural Advantage} (Lemma~\ref{lemma:complexity_decomp_hoelder}), proving that the residual contrast admits a strictly faster minimax convergence rate than the full outcome surface. 
Second, to exploit this advantage, we derive an \textit{Objective Alignment} result (Prop.~\ref{prop:objective_alignment}), verifying that minimizing residual uncertainty is mathematically equivalent to minimizing the Bayesian PEHE risk. 
Third, we demonstrate the \textit{Information Efficiency} of our approach (Prop.~\ref{prop:redundancy}), formally proving that parameter-based baselines (e.g., BALD) optimize a loose bound containing task-irrelevant nuisance information, whereas R-EPIG targets the estimand directly. Finally, we provide an \textit{Algorithmic Uncertainty Convergence} guarantee (Thm.~\ref{thm:epig_bound}), showing that our greedy strategy effectively realizes the accelerated rates predicted by our complexity analysis.

\subsection{Sample Complexity: The Structural Efficiency Gap}
\label{sec:theory_convergence}
We quantify the efficiency of R-Design via minimax convergence rates over Hölder spaces~\citep{yang2015minimax}. Adopting a joint-function perspective on $\gZ = \gX \times \{0,1\}$, we decompose the true potential outcomes into a fixed observational base and a learnable residual: $\mu_e(\vx, t) = \mu_o(\vx, t) + \delta(\vx, t)$. Here, $\mu_o$ captures the correlations accessible via observational data, while $\delta$ represents the causal correction required to recover $\mu_e$. Let $\alpha_{\cdot, t}$ and $d_{\cdot, t}$ denote the smoothness and effective dimension for arm $t$.

\begin{lemma}
\label{lemma:complexity_decomp_hoelder}
Assume the observational base learner $\hat{\mu}_o$ and the experimental residual learner $\hat{\delta}$ employ priors that are optimally adapted to the smoothness of their respective target functions. Under standard assumptions (overlap and experimental unconfoundedness), the expected PEHE risk $\psi(\hat{\tau})$ is bounded by:
\begin{equation}
\label{eq:excess_risk_bound_hoelder}
\small
    \psi(\hat{\tau}) \leq \, C_1 \underbrace{\max_{t \in \{0,1\}} n_E^{-\frac{2\alpha_{\delta, t}}{2\alpha_{\delta, t} + d_{\delta, t}}}}_{\text{Residual Convergence}} + C_2 \underbrace{\max_{t \in \{0,1\}} n_O^{-\frac{2\alpha_{\mu_o, t}}{2\alpha_{\mu_o, t} + d_{\mu_o, t}}}}_{\text{Observational Baseline}} + \epsilon_{\textup{approx}},
\end{equation}
where $C_1, C_2 > 0$ are domain-dependent constants and $\epsilon_{\textup{approx}}$ denotes the irreducible approximation error.
\end{lemma}

\textbf{Interpretation.} Lemma~\ref{lemma:complexity_decomp_hoelder} governs the risk via residual complexity. \textit{Regime Dominance:} Since $n_E \ll n_O$, the observational error acts as a static baseline, making the convergence dominated by Term I. \textit{Efficiency Gap:} Standard "Tabula Rasa" methods are bottlenecked by the raw outcome complexity $\gO(n_E^{-\frac{2\alpha_{\mu_e}}{2\alpha_{\mu_e} + d_{\mu_e}}})$. In contrast, R-Design exploits the structural property that residuals are typically smoother ($\alpha_{\delta} > \alpha_{\mu_e}$), yielding a strictly faster decay rate and lower sample complexity. The detailed proof, adapting the bounds from~\citet[Thm.~2]{alaa2018limits} to our residual decomposition, is provided in App.~\ref{app:proof_lemma_complexity_hoelder}.

\subsection{Objective: Optimality of Residual Learning}
\label{sec:theory_objective}

Lemma~\ref{lemma:complexity_decomp_hoelder} establishes that the residual is the theoretically superior target for learning. We now show that explicitly targeting the reduction of residual uncertainty is a valid proxy for minimizing the causal error. Since the observational contrast $\hat{\tau}_o(\vx)$ is deterministic conditioned on $\gD_O$, the posterior uncertainty of $\tau = \hat{\tau}_o + \tau_\delta$ is identical to that of the residual contrast: $\sV[\tau \mid \gD_E, \gD_O] \equiv \sV[\tau_\delta \mid \gD_E]$.

\begin{proposition}[Objective Alignment]
\label{prop:objective_alignment}
Assume the CATE estimator is the posterior mean $\hat{\tau}_k(\vx) = \E[\tau(\vx) \mid \gH_{k-1}]$. For active CATE estimation targeting the population $p_{\textup{tar}}$, the optimal choice to minimize the expected Bayesian PEHE risk simplifies, under the R-Design decomposition, to minimizing the integrated posterior variance of the residual contrast $\tau_\delta(\vx)$:
\begin{equation}
\label{eq:ivr_objective}
\begin{aligned}
    (\vx_k^\dagger, t_k^\dagger) &= \argmin_{(\vx, t)} \E_{\ry} \Big[ \E_{\rvx^* \sim p_{\textup{tar}}} \big[ (\hat{\tau}_{k}(\rvx^*) - \tau(\rvx^*))^2 \mid \gD_+ \big] \Big] \\
    &\equiv \argmin_{(\vx, t)} \E_{\ry} \Big[ \int \sV[\tau_\delta(\vx^*) \mid \gD_+] \, d p_{\textup{tar}}(\vx^*) \Big],
\end{aligned}
\end{equation}
where $\gD_+ = \gD_E^{(k-1)} \cup \{(\vx, t, y)\}$ is the updated history, and the outer expectation is over the predictive posterior $p(y \mid \vx, t, \gH_{k-1})$.
\end{proposition}
\textbf{Practical Surrogate.} While Eq.~\ref{eq:ivr_objective} (Integrated Variance Reduction) is the gold standard, it is computationally intractable. R-Design instead maximizes the information gain $\mi(\rr; \tau_\delta)$, which serves as a computationally efficient surrogate. This relaxation is justified by the monotonicity between entropy reduction and variance minimization in Gaussian processes (see App.~\ref{app:entropy_proxy}).

\subsection{Surpassing Parameter-Based Acquisition}
\label{sec:theory_efficiency_gap}

Given the structural advantage of learning residuals (Lemma \ref{lemma:complexity_decomp_hoelder}), a natural question arises regarding the acquisition strategy: could standard parameter-based active learning methods (e.g., BALD targeting the residual model parameters $\theta$) achieve optimal efficiency? We provide a negative answer by quantifying the information-theoretic gap between parameter learning and estimand learning.

\begin{proposition}[Information Redundancy]
\label{prop:redundancy}
Let $\theta$ be the parameters of the residual model and $\tau_\delta(\vx^*)$ be the residual contrast at a target $\vx^*$. The parameter-oriented utility (Residual-BALD) decomposes into the target utility (R-EPIG) plus a non-negative redundancy term:
\begin{equation}
    \underbrace{\mi(y; \theta \mid \gD)}_{\text{Residual-BALD}} \;=\; \underbrace{\mi(y; \tau_\delta(\vx^*) \mid \gD)}_{\text{Target Utility (R-EPIG)}} \;+\; \underbrace{\mi(y; \theta \mid \tau_\delta(\vx^*), \gD)}_{\text{Nuisance Redundancy}}.
\end{equation}
\end{proposition}

\textbf{The Cost of Over-Modeling.} The term $\Delta = \mi(y; \theta \mid \tau_\delta(\vx^*), \gD)$ represents \textit{nuisance information}: bits acquired about the model's internal structure (e.g., high-frequency basis coefficients or weight permutations) that effectively cancel out when computing the contrast $\tau_\delta$. 
By maximizing the LHS, parameter-based methods inevitably waste a portion of the experimental budget on reducing this irrelevant uncertainty. In contrast, R-EPIG targets the first term directly, ensuring that every acquired sample contributes exclusively to refining the causal estimand on the target population. We provide the rigorous proof based on the chain rule of mutual information in App.~\ref{app:theory_comparisons}.

\subsection{Information-Theoretic Convergence}
\label{sec:theory_epig_convergence}

Having established the efficiency of the R-EPIG objective, we finally analyze whether it realizes the accelerated convergence predicted in Lemma~\ref{lemma:complexity_decomp_hoelder}. We invoke the transductive AL bound linking error reduction to the maximum information capacity $\gamma_{n}(\cdot)$~\citep{hubotter2024transductive}.

\begin{theorem}[Uncertainty Convergence]
\label{thm:epig_bound}
Let $\gamma_{n_B}(\tau_\delta)$ denote the maximum information capacity of the residual contrast function. Under the greedy R-EPIG strategy (assuming $\alpha$-weak submodularity), the posterior variance of the residual contrast converges uniformly. Specifically, with high probability, for all target points $\vx^* \in \gX$:
\begin{equation}
    \sV[\tau_\delta(\vx^*) \mid \gD_{n_B}] \le \eta^2_{\textup{feas}}(\vx^*) + \tilde{\gO}\left( \frac{\gamma_{n_B}(\tau_\delta)}{\sqrt{n_B}} \right),
\end{equation}
where $\eta^2_{\textup{feas}}(\cdot)$ represents the feasible irreducible uncertainty. Since the rate term depends only on $n_B$ and capacities, this implies uniform convergence to the irreducible floor.
\end{theorem}

\begin{figure*}[t]
\centering
\subfloat{
    \includegraphics[width=0.48\linewidth]{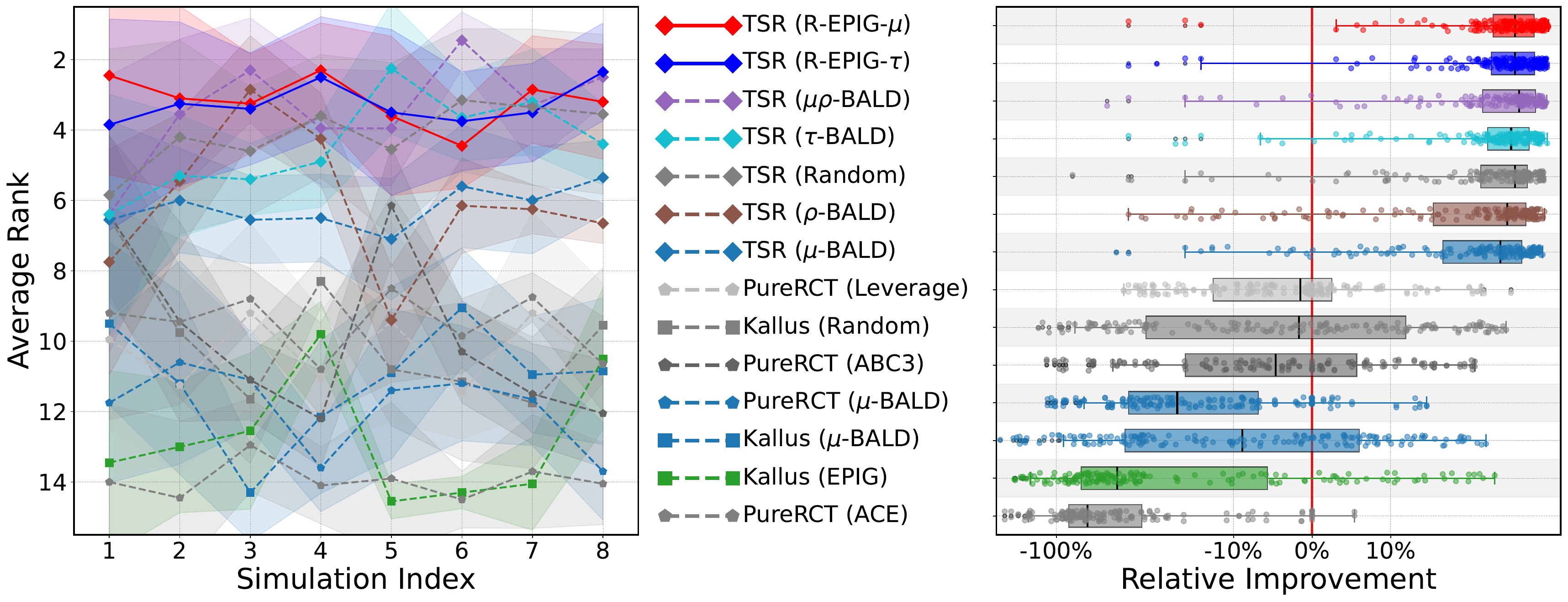}
}
\hfill
\subfloat{
    \includegraphics[width=0.48\linewidth]{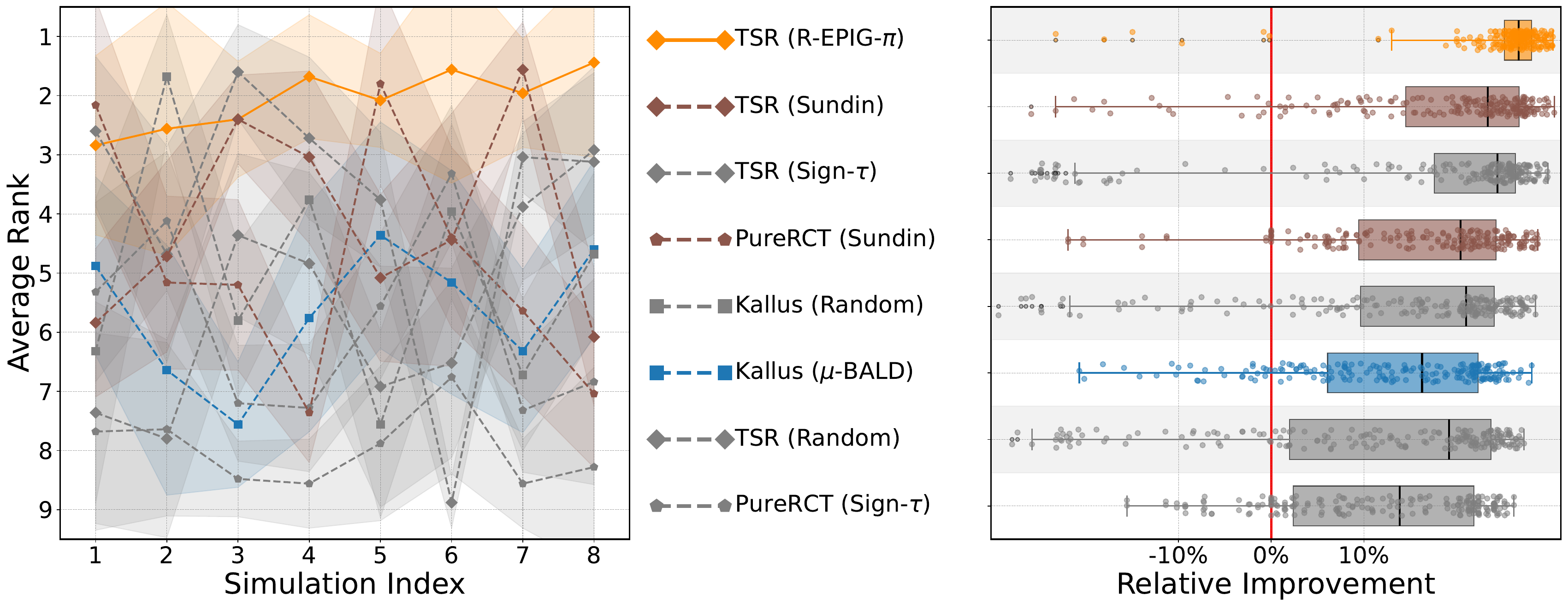}
}
\caption{Comparison of acquisition functions using the average rank metric and relative performance improvements for effect estimation and policy learning tasks over eight combinations of base effect functions, bias functions, and CATE functions.}
\label{fig:univariate_results}
\end{figure*}

\textbf{Efficiency Gap.} A standard active learner targeting the full surface $\tau$ obeys an identical bound governed by the larger capacity $\gamma_{n_B}(\tau)$. Crucially, since $\tau$ and $\tau_\delta$ share the exact same irreducible uncertainty floor ($\sV[\tau|\gD] \equiv \sV[\tau_\delta|\gD]$), the performance gap is strictly determined by the capacity ratio. As established in Lemma~\ref{lemma:complexity_decomp_hoelder}, since $\alpha_\delta > \alpha_{\mu_e}$, the residual capacity $\gamma_{n_B}(\tau_\delta)$ grows significantly slower than the full outcome capacity $\gamma_{n_B}(\tau)$ (e.g., logarithmic vs.\ polynomial). This structurally guarantees that R-Design achieves a superior convergence rate. We provide the detailed proof, establishing uniform convergence via spectral decay and treating approximate submodularity, in App.~\ref{app:proof_thm_epig}.

%% file: Pages/Experiments.tex
\section{Experimental Results}
\label{sec:experiments}

In this section, we report the empirical results to verify the effectiveness of our framework in acquiring data points, across multiple synthetic datasets, two semi-synthetic datasets, specifically the Infant Health and Development Program (IHDP)~\citep{hill2011bayesian} and The AIDS Clinical Trials Group Study 175 (ACTG-175)~\citep{hammer1996trial}. Details of all datasets are provided in App.~\ref{app:datasets}.

\begin{figure*}[t]
\centering
\subfloat{
    \includegraphics[width=0.23\textwidth]{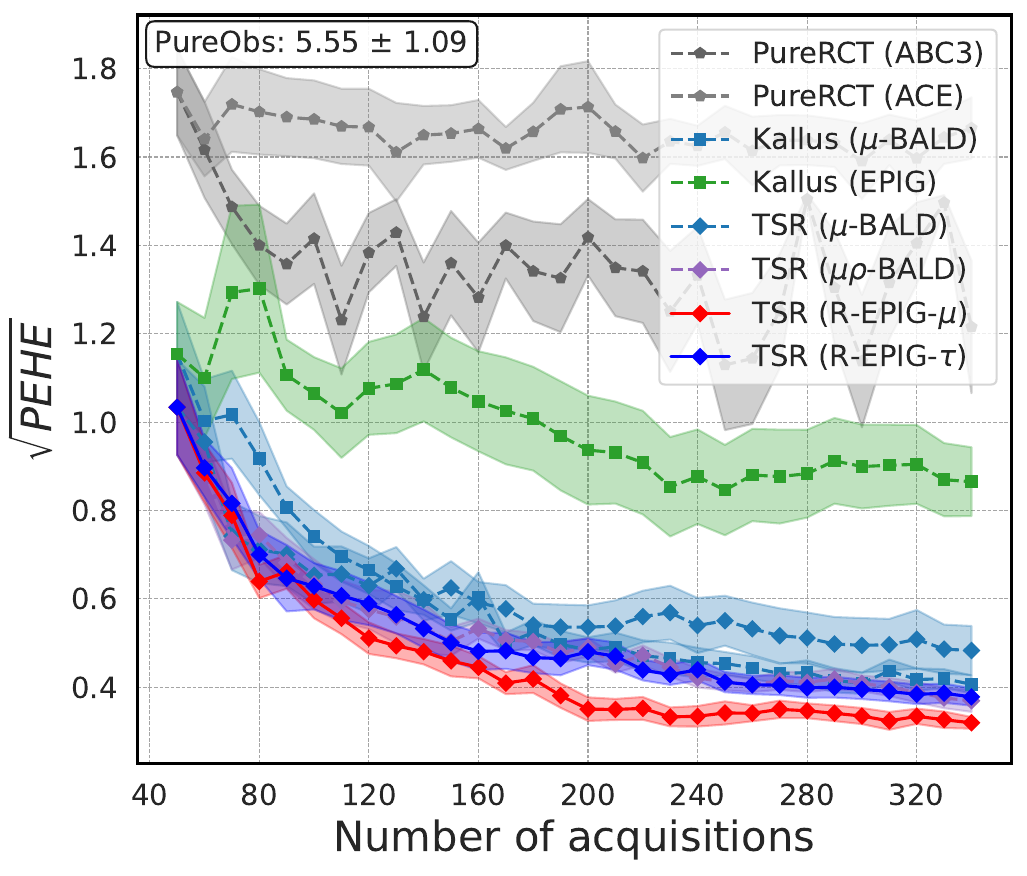}
}
\hfill
\subfloat{
    \includegraphics[width=0.23\textwidth]{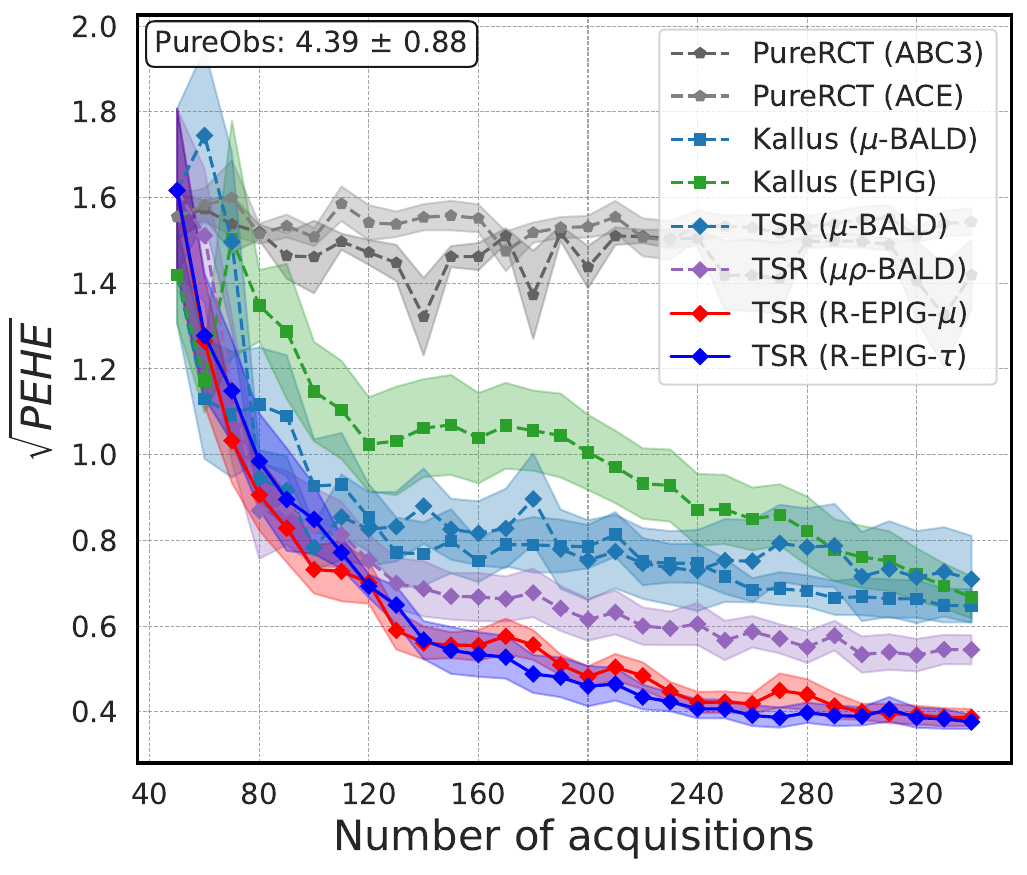}
}
\hfill
\subfloat{
    \includegraphics[width=0.23\textwidth]{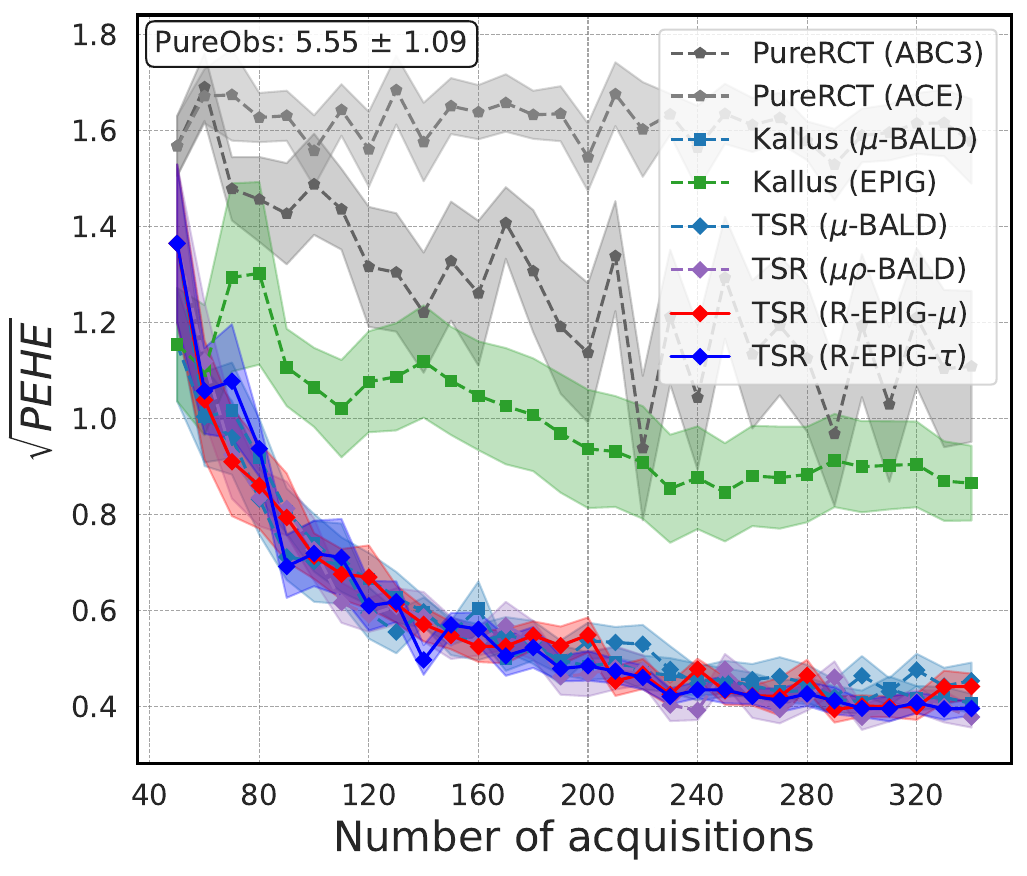}
}
\hfill
\subfloat{
    \includegraphics[width=0.23\textwidth]{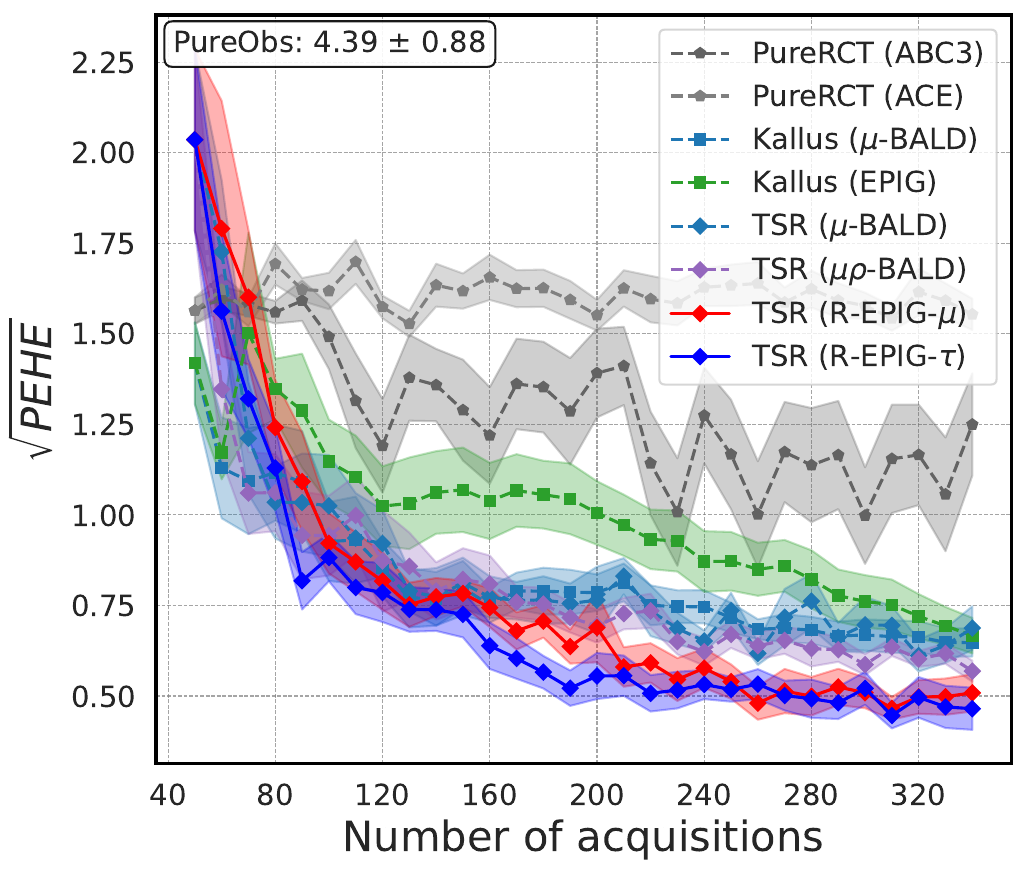}
}
\caption{Comparison of $\sqrt{\text{PEHE}}$ on simulation dataset of two trial estimators (CMGP and NSGP). (1,3) CMGP and NSGP, (2,4) CMGP and NSGP with heavy covariate shift.}
\label{fig:main_results}
\end{figure*}

\textbf{Baselines Acquisition Fuctions.} We benchmark our method against a diverse set of acquisition strategies, ranging from random sampling to advanced baselines: Leverage~\citep{addanki2022sample}, ACE~\citep{song2024ace}, ABC3~\citep{cha2025abc3}, Causal-BALD~\citep{jesson2021causal}. For decision-making settings, we additionally include Sundin~\citep{sundin2019active}, DEIG~\citep{filstroff2024targeted}, and RFAN~\citep{klein2025towards}.

\textbf{Inference Architectures and Backbones.} We compare our proposed R-Design (TSR) framework against two baseline architectures: (1) PureRCT, a standard approach that trains causal estimators solely on experimental data, ignoring observational priors; and (2) Kallus~\citep{kallus2018removing}, a data fusion method that pools observational and experimental data but corrects for confounding via importance weighting. All frameworks are instantiated using the same set of robust Bayesian learners. For the observational base (Stage 1), we primarily utilize the TabPFN-v2.5~\citep{grinsztajn2025tabpfn} due to its strong performance on tabular data, while also evaluating CausalPFN~\citep{balazadeh2025causalpfn}. For the experimental/residual learning phase (Stage 2), we employ established Bayesian estimators capable of providing valid joint posteriors, including CMGP~\citep{alaa2017bayesian}, NSGP~\citep{alaa2018limits}, BART~\citep{hill2011bayesian}, BCF~\citep{hahn2020bayesian}, and CMDE~\citep{jiang2023estimating}.

\textbf{Metrics.} We evaluate performance using two primary criteria. First, we report the PEHE (Eq.~\ref{eq:pehe}), which quantifies the accuracy of estimated treatment effects relative to the ground truth. Second, for decision-making tasks, we report the APE (Eq.~\ref{eq:ape}) and the Average Regret over the observational population~\citep{gao2024variational, fernandez2022causal}. All results are reported as mean $\pm$ standard deviation over $10$ independent runs.

\subsection{Synthetic Data}
We evaluate different acquisition functions using two simulation studies: a univariate and a multivariate setting. The univariate study consists of $8$ datasets spanning simple and complex combinations of base functions, bias functions, and CATE functions, and evaluates performance on both CATE estimation and policy learning tasks. Additional simulation results are provided in App.~\ref{app:more_results}, demonstrating the robustness and competitive performance of R-Design across a wide range of settings.

\textbf{Results.} Figs~\ref{fig:main_results} and~\ref{fig:synthetic_acquisition_comparison} summarize the main results on synthetic datasets with $\text{dim}=6$ covariates, demonstrating the strong performance of our proposed R-EPIG acquisition criteria within the TSR framework across both effect estimation and policy learning tasks.

\textit{Effect Estimation Performance.} As shown in Fig.~\ref{fig:main_results}, our proposed R-EPIG criteria consistently achieve substantial PEHE reductions compared to all baseline methods across the entire acquisition trajectory. Among all $15$ methods evaluated, R-EPIG-$\mu$ achieves the best average rank of $3.00$ and R-EPIG-$\tau$ ranks second with an average rank of $3.60$ in the CD diagram based on the Nemenyi post-hoc test. The critical difference threshold ($\text{CD} = 3.03$ at $\alpha = 0.05$) confirms that both R-EPIG variants are statistically indistinguishable from each other while being significantly better than all non-TSR baselines. Notably, all TSR-based methods (ranks $1$--$7$) substantially outperform Kallus-based methods (ranks $8$--$10$) and PureRCT methods (ranks $11$--$15$), validating the effectiveness of our two-stage framework that leverages observational data for target distribution estimation. Pairwise Wilcoxon signed-rank tests further confirm that R-EPIG-$\mu$ significantly outperforms $14$ of $15$ baselines ($p < 0.05$) with an $84.5\%$ win rate, and R-EPIG-$\tau$ outperforms $13$ baselines with an $80.5\%$ win rate. Importantly, neither R-EPIG variant loses significantly to any baseline method.

\begin{figure}[t]
\centering
\begin{minipage}{0.48\textwidth}
    \centering
    \subfloat{
        \includegraphics[width=0.45\linewidth]{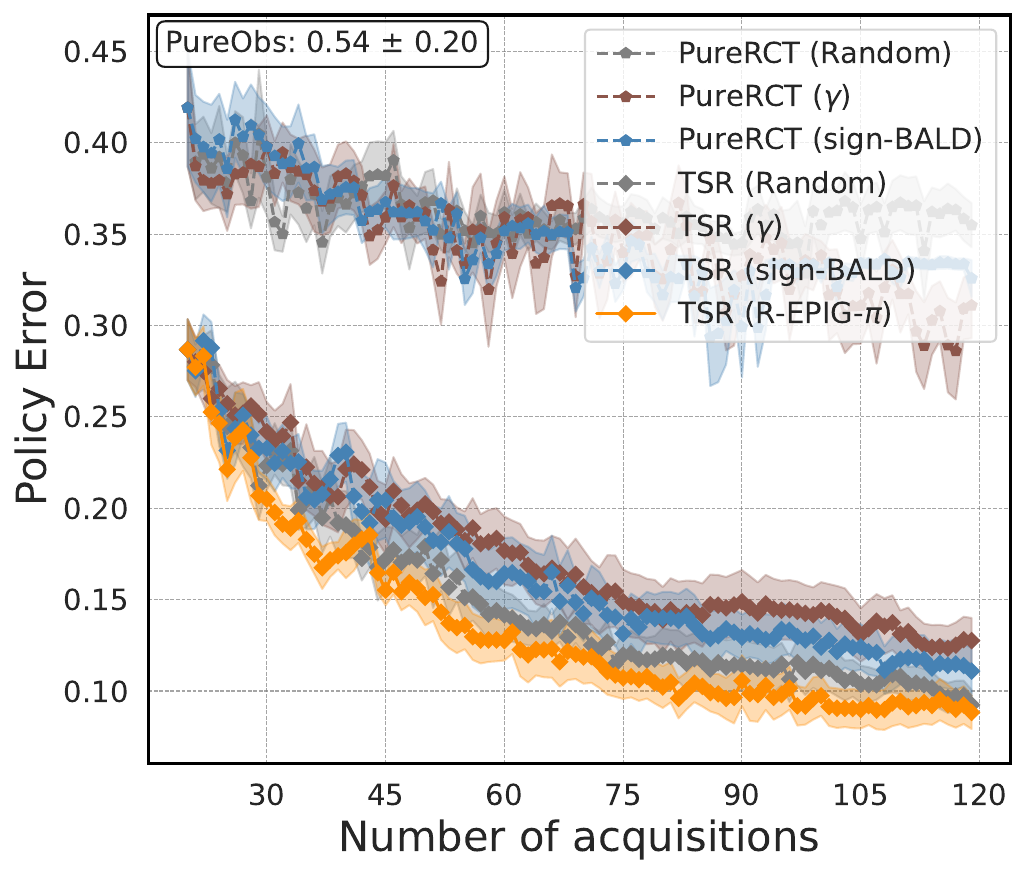}
    }
    \hfill
    \subfloat{
        \includegraphics[width=0.45\linewidth]{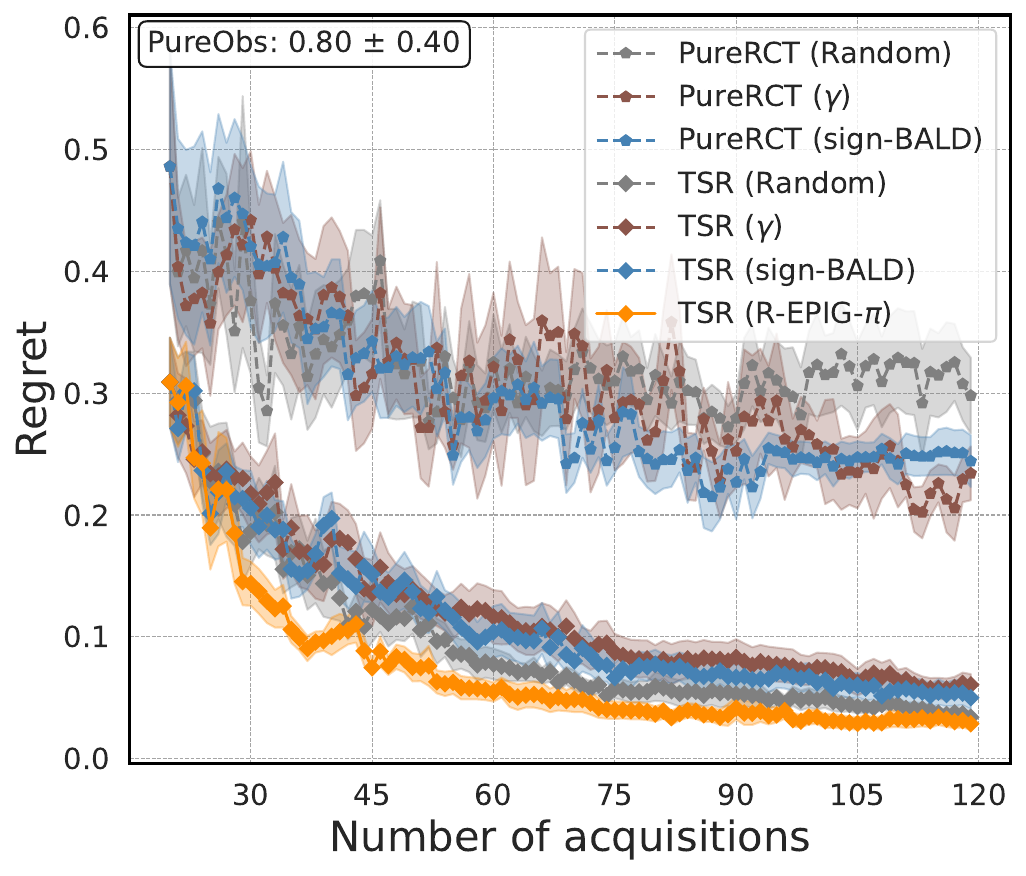}
    }
    \caption{Performance comparison of all methods on policy learning task (a) APE (b) AR.}
    \label{fig:policy_results}
\end{minipage}
\hfill
\begin{minipage}{0.48\textwidth}
    \centering
    \subfloat{
        \includegraphics[width=0.45\linewidth]{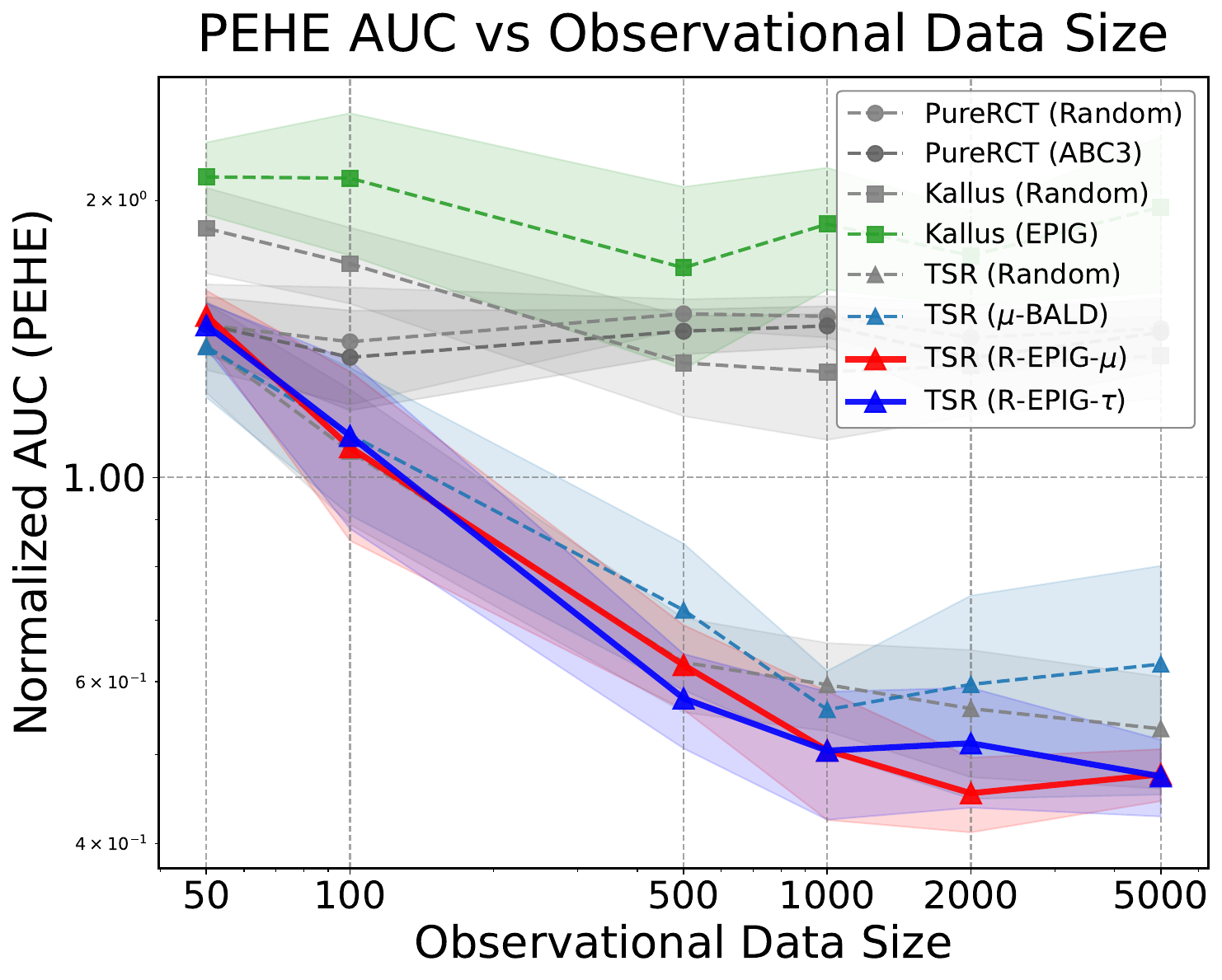}
    }
    \hfill
    \subfloat{
        \includegraphics[width=0.45\linewidth]{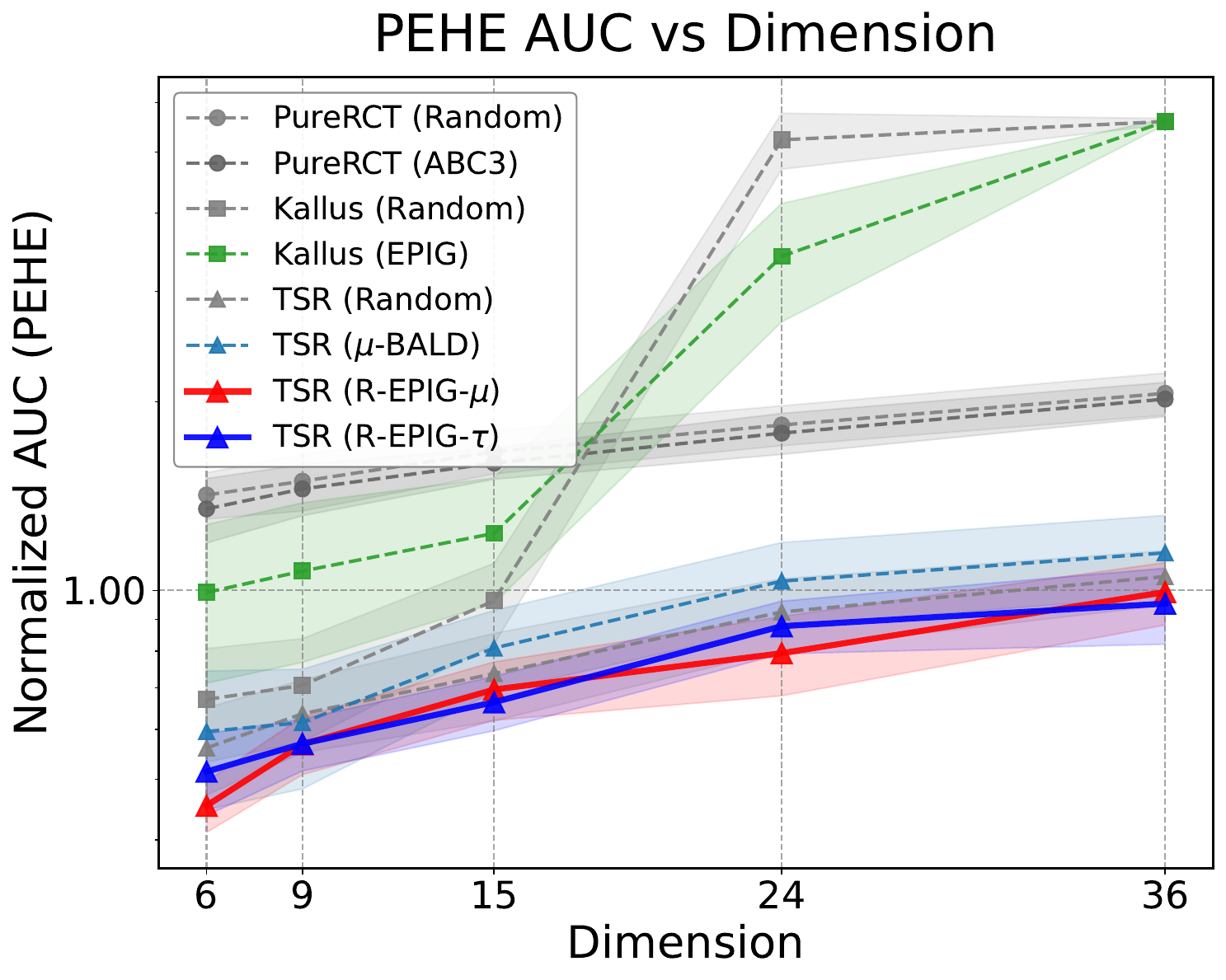}
    }
    \caption{AUC under varying prior sizes and covariate dims.}
    \label{fig:different_size_dim}
\end{minipage}
\end{figure}

\textit{Policy Learning Performance.} For treatment policy optimization, we evaluate methods using APE and Average Regret (AR) metrics. As illustrated in Fig.~\ref{fig:policy_results}, R-EPIG-$\pi$, our policy-aware acquisition function specifically designed for decision-making, outperforms all baseline acquisition functions on both metrics. This demonstrates the benefit of directly targeting the policy-relevant information gain rather than relying on CATE estimation accuracy as a proxy for policy quality.

\textit{Scalability Analysis.} Fig.~\ref{fig:different_size_dim} examines method performance as the covariate dimensionality and observational dataset size vary. As the dimension increases from $6$ to $36$, the performance gap between TSR methods and baselines widens, with R-EPIG-$\mu$ maintaining its top rank (avg.~rank $3.00$) across all dimensions. This robustness stems from R-EPIG's explicit awareness of the target distribution, which becomes increasingly important under severe distribution shift in higher-dimensional spaces. When varying the observational data size from $50$ to $5000$ samples, TSR methods demonstrate consistent improvements with more observational data, while the relative ordering of acquisition functions remains stable.

\textit{Robustness Across Data Generating Processes.} Fig.~\ref{fig:univariate_results} evaluates the robustness of R-EPIG criteria across $27$ diverse combinations of base effect functions (linear, polynomial, sinusoidal), confounding bias patterns (uniform, moderate, strong), and CATE structures (constant, heterogeneous, interaction). Across all configurations, R-EPIG-$\mu$ and R-EPIG-$\tau$ consistently rank among the top performers, demonstrating that their effectiveness is not limited to specific functional forms but generalizes across a broad spectrum of causal data generating processes. Comprehensive numerical results, additional convergence plots, and detailed ablation studies examining the sensitivity to hyperparameters, kernel choices, and model specifications are provided in App.~\ref{app:more_results}.
\begin{table}[h]
\centering
\caption{Index of additional ablation studies.}
\label{tab:ablation_index}
\small
\setlength{\tabcolsep}{3pt}
\begin{tabular}{@{}l p{5.8cm} c@{}}
\toprule
\textbf{Dimension} & \textbf{Variations Evaluated} & \textbf{Ref.} \\ \midrule

\multicolumn{3}{@{}l}{\textit{\textbf{1. Active Learning Configuration}}} \\
\quad Batch Size & $n_b \in \{2, 3, 5, 10, 20, 30\}$ & \S\ref{fig:different_step_sizes} \\
\quad Pool Size & $n_P \in \{500, 1000, 1500, 2000, 2500\}$ & \S\ref{fig:different_pool_sizes} \\\midrule

\multicolumn{3}{@{}l}{\textit{\textbf{2. Model Agnosticism (Robustness)}}} \\
\quad Stage 1 (Prior) & TabPFN, CMGP, NSGP, CausalPFN & \S\ref{fig:different_obs_models} \\
\quad Stage 2 (Trial) & CMGP, NSGP, BART, BCF, CMDE & \S\ref{fig:different_trial_models} \\
\quad Kernel Type & RBF, Mat\'{e}rn, RQ & \S\ref{fig:different_kernels} \\ \bottomrule
\end{tabular}
\end{table}

\textit{More Ablation Studies.} Tab.~\ref{tab:ablation_index} summarizes additional ablation studies examining the sensitivity of our framework to various hyperparameters and model choices. These include AL configurations (batch size, pool size, and temperature) as well as model backbone variations for both Stage~1 (prior models) and Stage~2 (trial models). Detailed experimental setups, results, and analyses for each ablation are provided the corresponding appendix sections.

\begin{figure}[t]
\centering
\begin{minipage}{0.48\textwidth}
    \centering
    \subfloat{
        \includegraphics[width=0.45\linewidth]{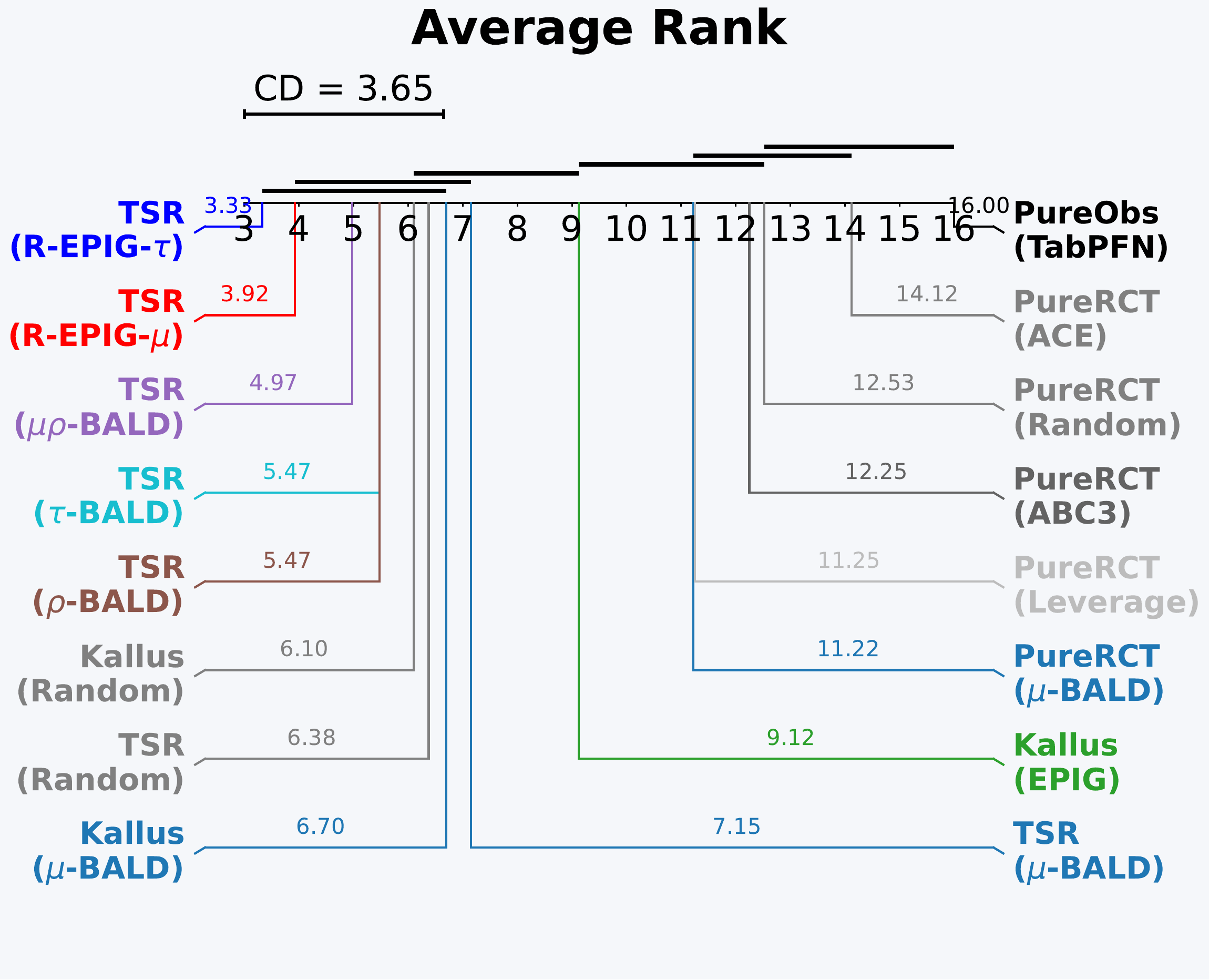}
    }
    \hfill
    \subfloat{
        \includegraphics[width=0.45\linewidth]{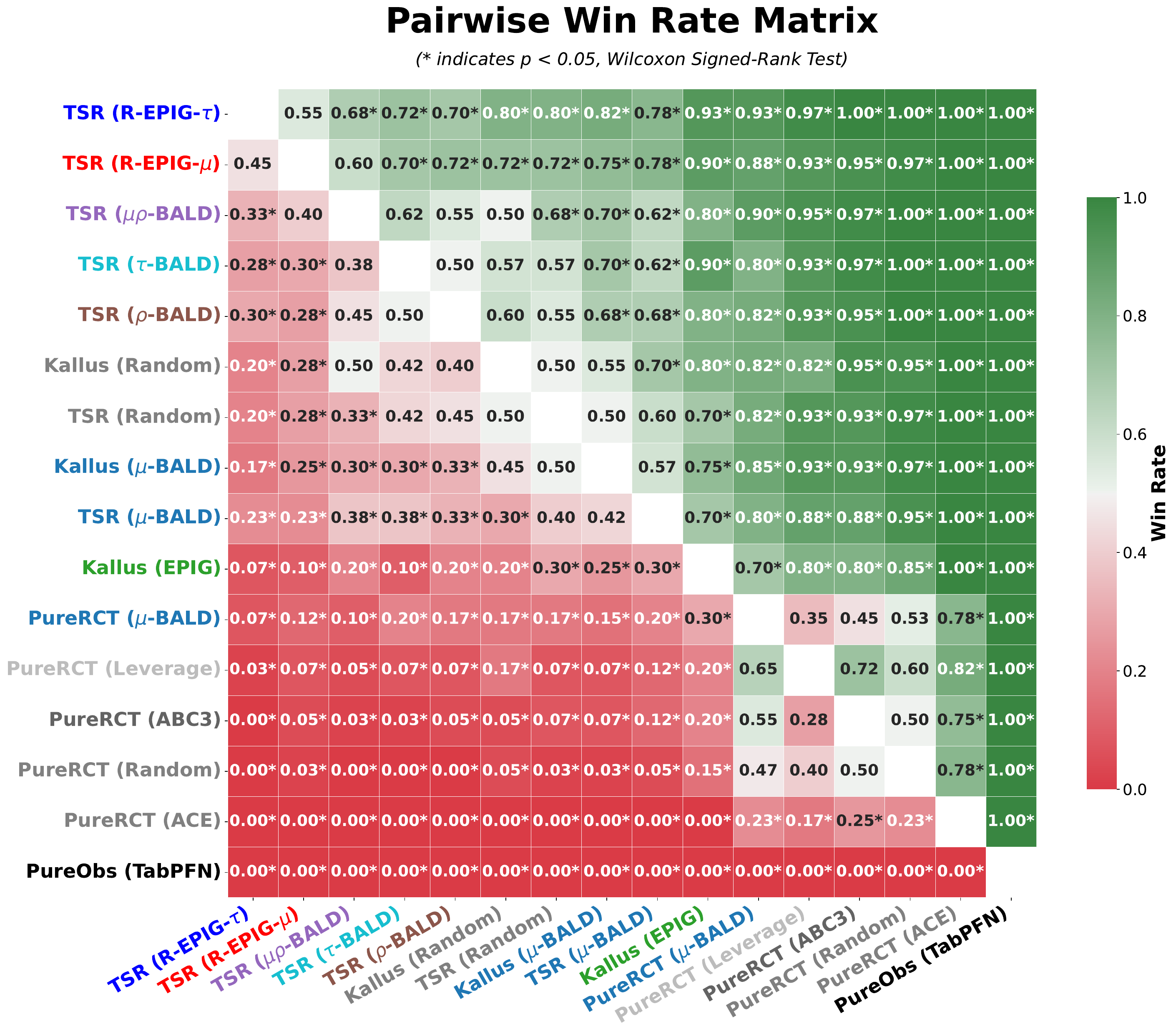}
    }
    \caption{Comparison of acquisition functions on synthetic datasets: (a) critical difference diagram across all settings, (b) pairwise win rate matrix with statistical significance}
    \label{fig:synthetic_acquisition_comparison}
\end{minipage}
\hfill
\begin{minipage}{0.48\textwidth}
    \centering
    \subfloat{
        \includegraphics[width=0.45\linewidth]{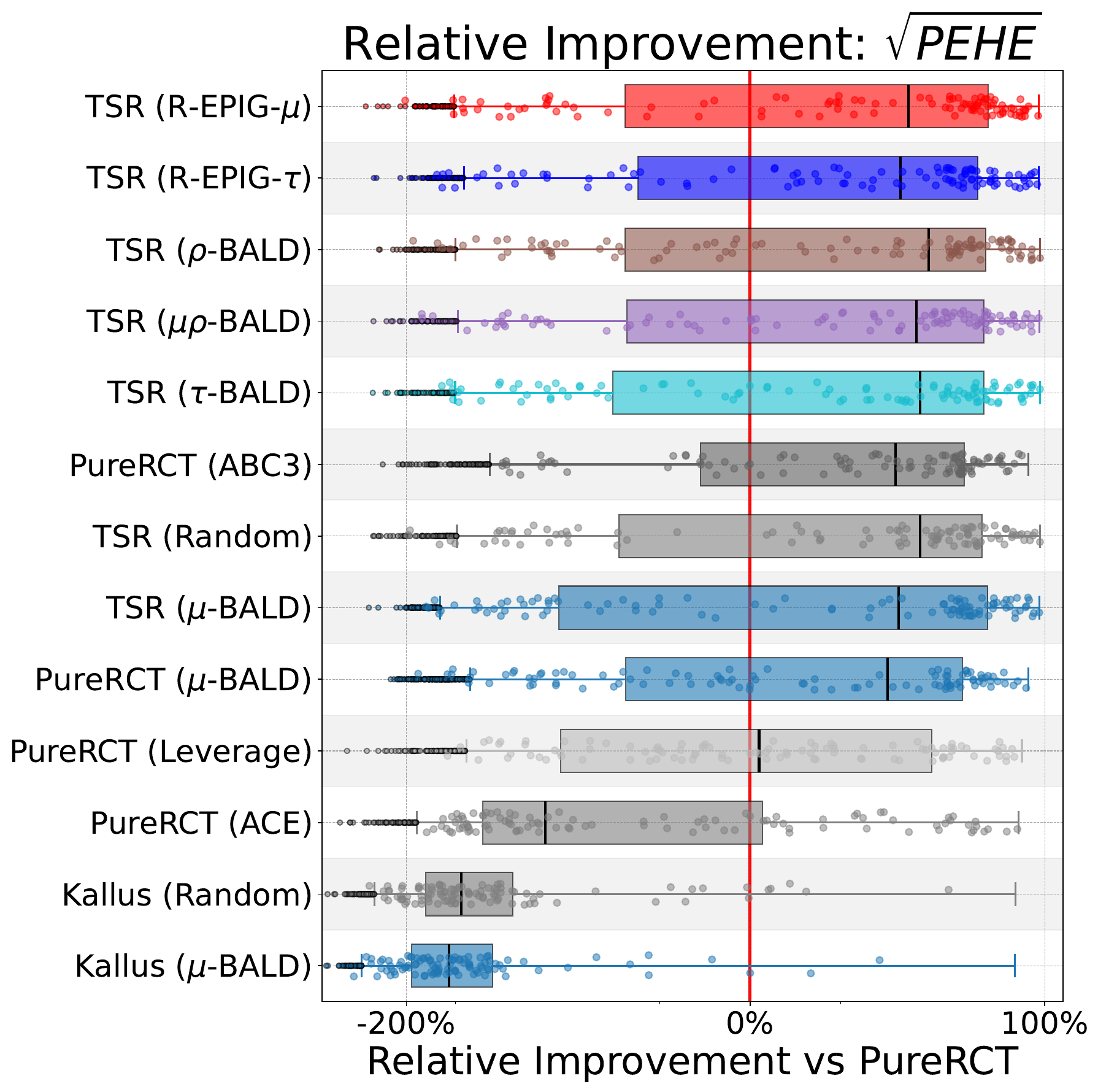}
    }
    \hfill
    \subfloat{
        \includegraphics[width=0.45\linewidth]{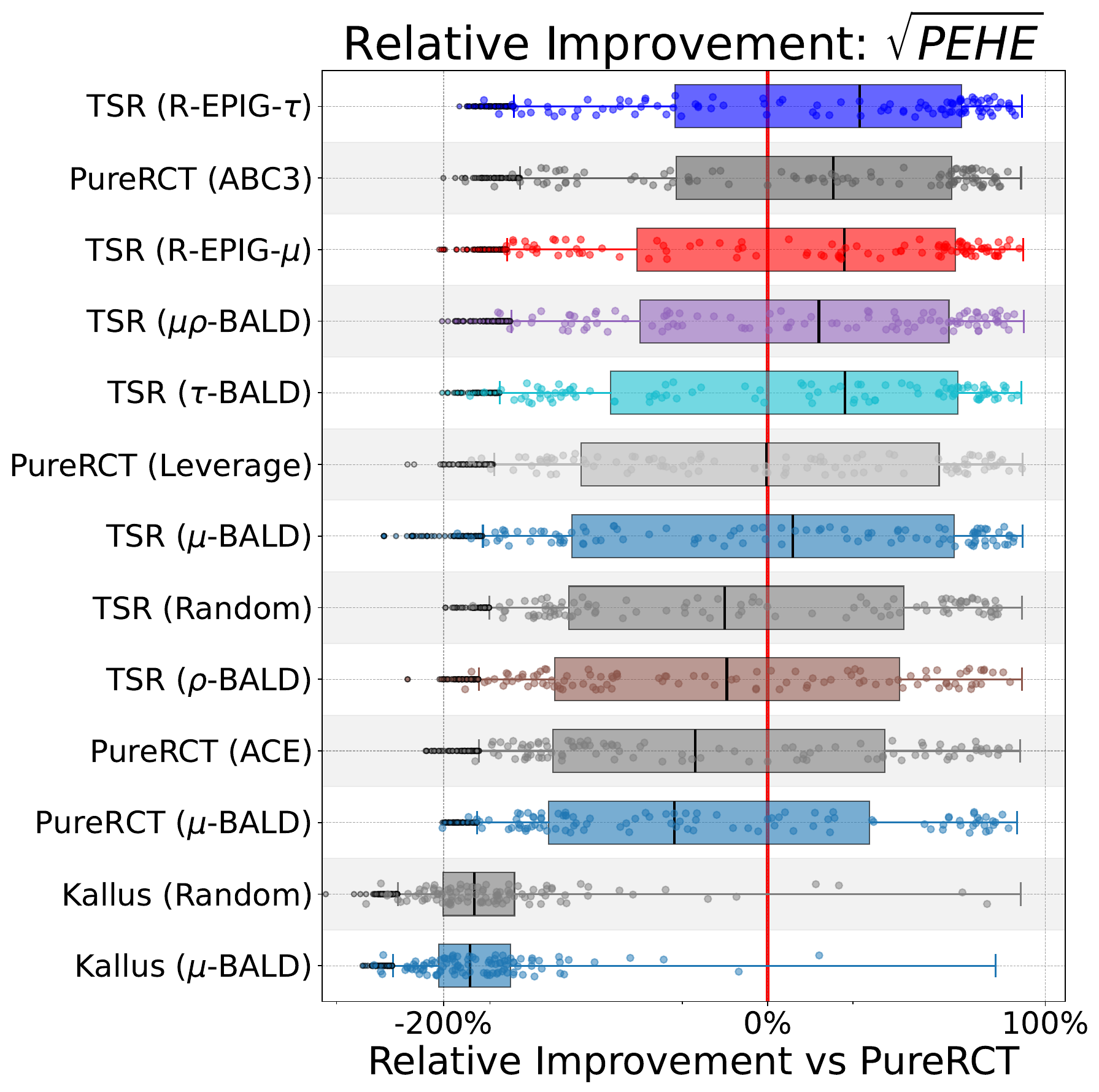}
    }
    \caption{Comparison of acquisition functions on semi-synthetic datasets. Relative performance improvement on the IHDP dataset (left) and the ACTG-175 dataset (right).}
    \label{fig:semi_synthetic_results}
\end{minipage}
\end{figure}

\subsection{Semi-synthetic Data.} 
We evaluate our framework on two well-established semi-synthetic benchmarks. The first, IHDP~\citep{hill2011bayesian} ($25$ covariates), induces selection bias by removing a non-random subset of treated units from a randomized trial. The second, ACTG-175~\citep{hammer1996trial} ($12$ covariates), constructs an observational cohort by excluding participants based on enrollment symptoms. Fig.~\ref{fig:semi_synthetic_results} reports relative performance improvements over the PureRCT strategy with random acquisition. On IHDP, TSR with R-EPIG-$\mu$ achieves the best performance, followed by R-EPIG-$\tau$. We attribute this to the higher covariate dimensionality of IHDP, where targeting the outcome surfaces provides more robust uncertainty quantification. In contrast, ACTG-175 has a lower-dimensional covariate space, under which R-EPIG-$\tau$ ranks first due to its direct focus on treatment effect heterogeneity, while R-EPIG-$\mu$ ranks third, demonstrating continued competitiveness across problem settings. Additional results are provided in the App.~\ref{app:more_results}.

%% file: Pages/Related_works.tex
\section{Related Works}
\label{sec:related_works}

Our framework lies at the intersection of AL and adaptive experimental design for causal inference. Existing literature can be broadly categorized by the level of control the learner has over the intervention mechanism. One line focuses on active CATE estimation from a fixed observational cohort~\citep{jesson2021causal, gao2025causal, gao2025active}. In this setting, treatments have already been assigned by an unknown policy, and the learner selectively queries costly outcomes to improve a CATE model. The core challenge here is a structural transduction problem: inferring pairs of unobservable potential outcomes from single factual queries. However, the learner is passive regarding the treatment assignment, limiting the ability to explore counterfactuals directly. A second, distinct line of research concerns adaptive experimental design~\citep{addanki2022sample, zhang2025active}. Here, the learner has the power to assign treatments. Recent works have explored adaptive randomization to minimize ATE variance~\citep{kato2024active}, methods for covariate balancing in a finite pool~\citep{harshaw2024balancing}, and Bayesian experimental design to reduce CATE posterior variance~\citep{cha2025abc3, klein2025towards}. However, these methods typically operate in a tabula rasa setting, designing experiments from scratch while ignoring the wealth of existing observational data. We bridge this gap by addressing a hybrid question: given a large but biased observational prior, how should we sequentially assign treatments to most efficiently learn the residual function required to debias the prior and recover the true CATE? We provide a detailed review of additional related topics in App.~\ref{app_sec:more_related_works}.

%% file: Pages/Conclusion.tex
\section{Discussion and Conclusion}
\label{sec:discussions}

We introduced R-Design, a paradigm shift that bridges the divide between large-scale observational studies and targeted experimental design. By formalizing the problem as active residual learning, we demonstrated that biased observational models should not be discarded, but rather repurposed as foundational priors. Our theoretical analysis rigorously established the structural efficiency gap, proving that estimating the smooth residual contrast requires significantly fewer samples than reconstructing complex outcome mechanisms from scratch. Furthermore, R-EPIG resolves the critical misalignment between active learning objectives and downstream causal tasks: by focusing strictly on the target population and adapting to specific goals (estimation vs. policy), it prevents budget waste on irrelevant uncertainties. Ultimately, R-Design offers a blueprint for resource-constrained causal inference: it suggests that instead of pursuing global perfection through massive experimentation, we can achieve safer, more efficient decision-making by strategically repairing the biases of the past.

\textbf{Future works.} While R-Design currently focuses on binary interventions with a single prior, future iterations will expand the scope of residual learning. We plan to extend the framework to continuous treatment regimes (e.g., dosage), where the objective shifts from estimating scalar contrasts to learning complex residual dose-response surfaces.

%% file: Pages/Appendix/Additional_related_works.tex
\section{Extended Related Works}
\label{app_sec:more_related_works}

In this section, we provide a more detailed review of the literature related to our paper, including the CATE estimation from observational data, causal data fusion (observational data and RCT data), and adaptive experimental design, identifying the specific differences and gaps that R-Design addresses.

\paragraph{CATE Estimation.}
A dominant paradigm for estimating CATE is called meta-learners, which decompose the problem into estimating nuisance components (e.g., outcome regressions and propensity scores) and then constructing a final CATE estimator from a pseudo-outcome~\citep{curth2021nonparametric}. Representative examples include the Two (T)-learner, Single (S)-learner, and X-learner~\citep{kunzel2019metalearners}, the Domain learner~\citep{shalit2017estimating}, as well as doubly robust methods such as the DR-learner~\citep{kennedy2023towards} and the Residual learner~\citep{nie2021quasi}. Beyond meta-learners, a variety of powerful alternatives have been proposed, including tree-based ensembles such as Causal Forests~\citep{wager2018estimation} and Bayesian Additive Regression Trees (BART)~\citep{hill2011bayesian}, Bayesian approaches such as Causal Gaussian Processes~\citep{alaa2017bayesian, alaa2018limits} and Bayesian causal forests~\citep{hahn2020bayesian}, CausalPFN~\citep{balazadeh2025causalpfn}, Causal Multi-task Deep Ensemble (CMDE)~\citep{jiang2023estimating}, as well as kernel-based methods~\citep{Sej2025a}. While our proposed data acquisition strategy is compatible with and can benefit all of these downstream estimators, our contribution is orthogonal to estimator design: rather than introducing a new CATE estimator, we develop a principled approach for collecting more informative data, thereby improving the sample efficiency of any chosen model.

\paragraph{Data Fusion for CATE Estimation.}
The complementary strengths of observational data and RCTs have motivated a substantial body of work on data fusion for causal inference. A prominent line of research leverages a small amount of unconfounded RCT data to correct for hidden confounding in much larger observational datasets~\citep{colnet2024causal}. One class of approaches explicitly models the bias or confounding function, ranging from strong parametric assumptions such as linearity~\citep{kallus2018removing} to more flexible nonparametric estimators~\citep{yang2025data, wu2022integrative}. A distinct paradigm is based on representation learning, typically implemented in a two-stage framework: a shared representation is first learned from the large observational dataset and subsequently refined or calibrated using unconfounded RCT data, giving rise to methods such as CorNet~\citep{hatt2022combining} and the Two-Stage Pretraining-Finetuning (TSPF)~\citep{zhou2025two}. Complementary Bayesian perspectives study the fundamental limits of validation under data fusion~\citep{fawkes2025the} or employ multi-task Gaussian processes to adaptively control the influence of biased data sources~\citep{dimitriou2024data}. While this literature provides powerful tools for the \emph{retrospective} analysis of fixed, heterogeneous datasets, our work addresses a fundamentally different and \emph{prospective} problem: how to exploit an existing observational dataset to actively guide the future collection of experimental data.

\paragraph{Adaptive Causal Experimental Design.}
A significant body of work focuses on optimizing the treatment assignment policy to minimize predictive error. This literature can be broadly categorized into four streams. 

(1) \textit{Adaptive Randomization:} One major line involves updating treatment probabilities based on accumulating data to minimize the variance of an estimator, typically the ATE. This area rests on firm theoretical foundations~\citep{van2008construction, hahn2011adaptive}, with recent works proposing refined designs that utilize online estimates of nuisance components~\citep{kato2021adaptive, tabord2023stratification}, specialized estimators like A2IPW~\citep{kato2020efficient}, or designs tailored for adaptive data collection~\citep{cook2024semiparametric}.

(2) \textit{Finite Pool Experimental Design:} A complementary line considers design choices for a fixed, finite pool of individuals. Research in this domain ranges from foundational analyses of the tradeoff between covariate balance and robustness~\citep{efron1971forcing} to modern frameworks offering finite-sample guarantees, such as leverage score sampling~\citep{addanki2022sample, ghadiri2023finite} and the Gram-Schmidt Walk~\citep{harshaw2024balancing}.

(3) \textit{Bayesian Design for CATE:} Recent advances have approached CATE estimation through the lens of Bayesian experimental design. Methods in this stream integrate regulatory constraints~\citep{klein2025towards} or address structural uncertainty~\citep{toth2022active}, while others develop GP-based acquisition functions to minimize posterior variance~\citep{cha2025abc3} or provide guarantees in network settings~\citep{zhang2025active}. Crucially, however, these methods typically operate in a \textit{tabula rasa} setting, designing experiments from scratch without leveraging prior observational data.

(4) \textit{Active CATE Estimation:} Finally, a fourth stream focuses on active outcome acquisition for CATE estimation, where treatment assignments are often observational or fixed, but outcomes are costly to acquire~\citep{nwankwo2025batch}. The goal is to selectively query outcomes from a cohort to refine CATE models. While some existing methods adapt standard active learning heuristics, such as diversity-based coreset selection~\citep{qin2021budgeted, wen2025enhancing} or targeting specific causal error components~\citep{sundin2019active}, these often rely on indirect proxies rather than directly reducing estimand uncertainty. Information-theoretic criteria offer a closer alignment: Causal-EPIG~\citep{gao2025causal} directly targets the uncertainty in the CATE estimand; Causal-EIG~\citep{fawkes2025is} and Causal-BALD~\citep{jesson2021causal} focus on model-internal parameters; and ActiveCQ~\citep{gao2025active} generalizes this principle using GP to guide acquisition across a wider range of causal quantities.

\paragraph{Observationally-Informed Experimental Design.}
While the hybrid paradigm of leveraging observational priors for experimental design has been explored, existing methods operate under fundamentally different assumptions than R-Design. For instance, \citet{rosenman2021designing} utilize observational data to inform experimental allocation, but their approach is strictly \textit{static} and \textit{parametric}. It operates within a pre-defined stratified setting, using observational data solely to construct robust confidence sets for stratum variances to optimize a single-batch minimax regret. Similarly, \citet{epanomeritakis2025choosing} propose a framework for selecting experiments to complement biased observational estimates. However, their approach is likewise limited to a static, parametric regime, designed to identify a low-dimensional parameter vector $\theta$ of a known structural model $\tau(\theta)$ via a one-shot selection. In contrast, R-Design is fully adaptive (sequential), non-parametric (compatible with universal kernels), and explicitly targets the learning of the residual function $\delta$, enabling flexible bias correction in complex, high-dimensional spaces.

%% file: Pages/Appendix/Models.tex
\begin{algorithm}[t]
\caption{R-Design: Residual-Based Adaptive Causal Experimental Design}
\label{alg:r_design_loop}
\begin{algorithmic}[1]
\Require Observational data $\gD_O$, Unlabeled candidate pool $\gD_P = \{\vx_i\}_{i=1}^n$, Budget $n_B$, Batch size $n_b$.
\Ensure CATE estimator $\hat{\tau}(\cdot) = \hat{\tau}_o(\cdot) + \hat{\tau}_\delta(\cdot)$.

\Statex \textcolor{blue!70!black}{// -- Stage 1: Observational Warm-Start --}
\State Train observational learners $\hat{\mu}_o(\cdot, 0)$ and $\hat{\mu}_o(\cdot, 1)$ on $\gD_O$.
\State Construct fixed observational contrast: $\hat{\tau}_o(\vx) \gets \hat{\mu}_o(\vx, 1) - \hat{\mu}_o(\vx, 0)$.
\State \textbf{Freeze} parameters of $\hat{\mu}_o(\cdot)$ and $\hat{\tau}_o(\cdot)$ to serve as functional offsets.

\Statex \textcolor{blue!70!black}{// -- Stage 2: Adaptive Residual Learning --}
\State Initialize experimental dataset $\gD_E \gets \emptyset$.
\State Initialize residual model $\hat{\boldsymbol{\delta}}(\cdot)$ (e.g., MTGP for vector $[\delta(\cdot, 0), \delta(\cdot, 1)]^\top$).

\While{$|\gD_E| < n_B$ \textbf{and} $\gD_P \neq \emptyset$}
    \Statex \quad \textcolor{gray}{// 1. Evaluation: Optimize Treatment Assignment per Unit}
    \State Initialize candidate scores $S \gets \emptyset$ and assignments $A \gets \emptyset$.
    \For{$\vx \in \gD_P$}
        \State Compute utilities for both arms: $u_t \gets U(\vx, t \mid \hat{\boldsymbol{\delta}}, \gD_E)$ for $t \in \{0, 1\}$.
        \State Select best arm: $t^* \gets \argmax_{t \in \{0,1\}} u_t$.
        \State Store max score: $S[\vx] \gets \max(u_0, u_1)$.
        \State Store assignment: $A[\vx] \gets t^*$.
    \EndFor
    
    \Statex \quad \textcolor{gray}{// 2. Selection: Stochastic Batch Acquisition (Softmax Trick)}
    \State Compute selection probabilities via Softmax: $\vp \gets \text{Softmax}(S / \text{temp})$.
    \State Sample batch $\gB_{x} \subset \gD_P$ of size $n_b$ based on $\vp$ (without replacement).
    \State Form final query batch: $\gB \gets \{ (\vx, A[\vx]) \mid \vx \in \gB_{x} \}$.
    \State Query experimental outcomes: $\gY_{\gB} \gets \{ y \mid (\vx, t) \in \gB \}$.
    \State $\gD_P \gets \gD_P \setminus \gB_{x}$.
    
    \Statex \quad \textcolor{gray}{// 3. Update: Residual Calculation and Posterior Update}
    \State Compute residual tuples: $\gD_{\text{new}} \gets \{ (\vx, t, r) \mid (\vx, t) \in \gB, y \in \gY_{\gB}, r = y - \hat{\mu}_o(\vx, t) \}$.
    \State Update $\gD_E \gets \gD_E \cup \gD_{\text{new}}$.
    \State Update $\hat{\boldsymbol{\delta}}(\cdot)$ given new data $\gD_E$.
\EndWhile
\State \Return $\hat{\tau}(\vx) = \hat{\tau}_o(\vx) + \underbrace{(\hat{\delta}(\vx, 1) - \hat{\delta}(\vx, 0))}_{\hat{\tau}_\delta(\vx)}$
\end{algorithmic}
\end{algorithm}

\section{Implementation Details}
\label{app:implementation}

In this section, we detail the implementation components of our R-Design framework and the baseline methods, including the adaptive experimental design loop, the baselien strategies, the Bayesian CATE estimators employed in this work, the dataset preparation procedures, and the hardware we used to run all experiments.

\subsection{Adaptive Causal Experimental Design Loop}
\label{app:active_loop}

The R-Design workflow distinguishes itself from standard tabula rasa, saying traditional adaptive causal experimental design through its reliance on a pre-computed observational. The procedure, outlined in Alg.~\ref{alg:r_design_loop}, proceeds in two phases: (1) \textit{Observational Warm-Start}, where the observational learners, usually any CATE estimators, $\mu_o(\cdot, 0)$ and $\mu_o(\cdot, 1)$ are trained on the static dataset $\gD_O$ to construct the observational CATE $\tau_o$, which is biased; and (2) \textit{Adaptive Residual (debiasing) function Learning}, where we sequentially query experimental outcomes. Crucially, during the evaluation step, we optimize the treatment assignment for each candidate $\vx$ by selecting the arm $t^* \in \{0, 1\}$ that maximizes the estimated utility, ensuring that experimental resources are directed toward the most informative counterfactuals.

\paragraph{Batch Selection Strategy.}
Constructing an optimal batch for information-theoretic acquisition functions typically requires expensive sequential greedy optimization (re-evaluating scores after each selection) to handle information redundancy~\citep{holzmuller2023framework}. To maintain computational efficiency with parallelized scoring (computing scores $S$ once per cycle), we adopt Stochastic Batch Acquisition via the Softmax trick~\citep{kirsch2023stochastic}. Instead of deterministically selecting the top-$n_b$ candidates, which sometimes leads to querying redundant clusters of points with similar high scores, we sample the batch $\gB$ from the pool with probabilities proportional to their softmax-normalized utility scores: $p(\vx) \propto \exp(S[\vx] / \tau)$. This strategy introduces necessary stochasticity to approximate batch diversity and cover the input space effectively without the computational overhead of sequential re-scoring.

\subsection{Model Architectures}
\label{app:models}

The R-Design framework is conceptually model-agnostic regarding the Stage 1 observational learner. However, guided by Lemma~\ref{lemma:complexity_decomp_hoelder}, the objective of this stage is to minimize the \textit{observational baseline error} ($n_O^{-\gamma_{\mu}}$). Consequently, we treat Stage 1 as a standard CATE estimation problem on the large observational dataset $\gD_O$, leveraging advanced estimators to maximize the accuracy of the observational contrast $\hat{\tau}_o$.

\paragraph{Stage 1: Observational Learner.}
For our primary experimental results, we employ TabPFN (v2.5)~\citep{grinsztajn2025tabpfn}, a SOTA Prior-data Fitted Network (PFN) designed for tabular data~\citep{zhang2025tabpfn}. We adopt the S-Learner meta-strategy, where a single model $\hat{\mu}_o(\vx, t)$ is trained on $\gD_O$ to predict outcomes, and the fixed observational contrast is derived as $\hat{\tau}_o(\vx) = \hat{\mu}_o(\vx, 1) - \hat{\mu}_o(\vx, 0)$. This choice effectively leverages the massive pre-training of TabPFN to capture complex covariate interactions~\citep{hollmann2023tabpfn, hollmann2025tabpfn}. To ensure robustness and comprehensive evaluation, we also benchmark other advanced estimators, including CausalPFN (a PFN variant fine-tuned for causal inference)~\citep{balazadeh2025causalpfn} and GP-based estimators such as CMGP~\citep{alaa2017bayesian} and NSGP~\citep{alaa2018limits}. Crucially, regardless of the architecture, this model is trained once on $\gD_O$ and treated as a fixed offset function during the subsequent experimental phase.

\paragraph{Stage 2: Residual Learner (Probabilistic CATE Estimator).}
Following Lemma~\ref{lemma:complexity_decomp_hoelder}, the objective of Stage $2$ is to estimate the residual process $\tau_{\delta}$. While formally a CATE estimation problem, it is structurally distinct: the target represents the bias function of the observational baseline, which typically exhibits higher smoothness than the full outcome surface~\citep{colnet2024causal}. Given that experimental data is costly and scarce ($n_E \ll n_O$), we require a model with high data efficiency. Furthermore, to drive our R-EPIG, the model must be strictly probabilistic, providing access to the joint posterior distribution (either in closed form or via sampling) to quantify epistemic uncertainty. To balance these requirements, we employ the CMGP~\citep{alaa2017bayesian} as our primary residual learner due to its stability in small-data regimes. For scalability to larger budgets, we employ Sparse Variational GPs (SVGP) and provide ablation studies. We also benchmark alternative probabilistic estimators, including NSGP~\citep{alaa2018limits}, BART~\citep{hill2011bayesian}, BCF~\citep{hahn2020bayesian}, and CMDE~\citep{jiang2023estimating}.

\begin{itemize}[leftmargin=*]
    \item \textbf{Kernel Architecture:} We mainly utilize the RBF kernel with automatic relevance determination. This choice allows the model to learn anisotropic lengthscales, effectively capturing the variable sensitivity of the residual bias across different covariate dimensions, as required in the Lemma.~\ref{lemma:complexity_decomp_hoelder}. We also provided the ablation studies of using different kernels.
    \item \textbf{Hyperparameter Optimization:} The kernel parameters (lengthscales, signal variance) and likelihood noise are optimized via Maximum Likelihood (marginal likelihood maximization) after each batch acquisition.
    \item \textbf{Implementation:} All GP models are implemented using \texttt{GPy}~\citep{gpy2014}\footnote{\url{https://github.com/SheffieldML/GPy}}.
\end{itemize}

\subsubsection{Baseline Architectures}
\label{app:baseline_architectures}

To show the contributions of our TSR paradigm, we benchmark against two fundamental architectural baselines: the Tabula Rasa approach (standard practice) and the parametric correction (marked as Kallus in our experiments) approach~\citep{kallus2018removing} (the structural precursor to our method).

\paragraph{1. The Tabula Rasa Framework (Direct Estimation).}
This represents the standard "Direct Learner" approach in experimental design~\citep{song2024ace,cha2025abc3}.
\begin{itemize}[leftmargin=*]
    \item \textbf{Structure:} It ignores the observational data $\gD_O$ entirely (or uses it only for covariate distribution estimation), adhering to the principle that observational correlations are unreliable due to hidden confounding. The learner attempts to reconstruct the full CATE function surface $\tau(\vx)$ from scratch using only the experimental data $\gD_E$.
    \item \textbf{Implementation:} We employ the same CMGP backbones as R-Design, but train them directly on the raw experimental outcomes $y \in \gD_E$ to predict $\mu_e(\vx, t)$.
    \item \textbf{Limitation:} As noted by \citet{kallus2018removing}, this approach is theoretically consistent but statistically inefficient. It discards the "global signal" contained in the large observational sample, forcing the learner to expend budget on relearning basic outcome structures that are likely shared between $\gD_O$ and $\gD_E$.
\end{itemize}

\paragraph{2. The Kallus Correction Framework (Passive Residuals).}
Our R-Design framework is structurally inspired by the "Recursive Partitioning for CATE" and specifically the bias-correction strategy introduced by~\citet{kallus2018removing}.
\begin{itemize}[leftmargin=*]
    \item \textbf{Structure:} Similar to R-Design, this method adopts a two-stage decomposition.
    \begin{itemize}
        \item \textit{Stage 1:} Learn a biased observational CATE $\omega(\vx)$ (equivalent to our $\hat{\tau}_o(\vx)$) from $\gD_O$.
        \item \textit{Stage 2:} Learn a correction term $\eta(\vx)$ (equivalent to our $\tau_\delta(\vx)$) using the experimental data $\gD_E$. The final estimator is constructed as $\hat{\tau}(\vx) = \hat{\omega}(\vx) + \hat{\eta}(\vx)$.
    \end{itemize}
    \item \textbf{Original vs. Our Implementation:} The original work by \citet{kallus2018removing} assumes a parametric (typically linear) form for the confounding function $\eta(\vx) = \theta^\top \vx$ to ensure identifiability with limited data. To provide a stronger, non-linear baseline that matches the capacity of our R-Design backbone, we upgrade the Stage 2 learner from a linear regression to a standard GP.
    \item \textbf{Distinction:} While structurally similar, the "Kallus Baseline" in our experiments represents the \textit{passive} (random sampling) or \textit{standard uncertainty} ($\tau$-BALD) version of this architecture. R-Design replaces the standard regression target with a Residual-Aware AL objective (R-EPIG), explicitly targeting the epistemic uncertainty of the correction term rather than the full outcome. Moreover, Kallus focuses on static CATE estimation. In contrast, our key contribution is to adaptively learn the residual term within the two-stage framework, enabling dynamic refinement of the CATE estimator as additional information is acquired.
\end{itemize}

\subsubsection{CATE Model Details}
\label{app:cate_model_details}

In this part, we introduce three families of established Bayesian CATE estimators. We explicitly detail how each architecture models the treatment effect and, crucially, how we extract the joint posterior distribution $p(y, \tau(\vx) \mid \vx, t, \gD)$ or $p(y, \boldsymbol{\delta}(\vx) \mid \vx, t, \gD)$ required for our information-theoretic acquisition functions. \textit{Note: While described here in their general form (modeling outcomes $y$ and effects $\tau$), when deployed in R-Design Stage $r$, these models are applied to residuals $r$ and residual contrasts $\tau_\delta$.}

\paragraph{1. Bayesian Tree Ensembles (BART \& BCF).}\footnote{\url{https://github.com/socket778/XBCF}}
Tree-based methods serve as robust baselines for tabular data. We utilize BART~\citep{hahn2011adaptive} and its causal extension BCF~\citep{hahn2020bayesian, krantsevich2023stochastic}.
\begin{itemize}[leftmargin=*]
    \item \textbf{CATE Estimation:} BCF explicitly parametrizes the conditional expectation as $\mathbb{E}[y \mid \vx, t] = \mu(\vx) + \tau(\vx)t$. Here, both the prognostic term $\mu(\vx)$ and the treatment effect $\tau(\vx)$ are modeled as independent sums-of-trees priors. To mitigate regularization-induced bias (regularization-induced confounding), BCF employs a propensity-based orthogonalization, fitting $\mu(\vx)$ on the propensity-adjusted outcomes.
    \item \textbf{Joint Posterior via Gaussian Approximation:} Inference in BART/BCF is performed via MCMC, yielding a set of posterior samples $\{ (\hat{y}^{(s)}, \hat{\tau}^{(s)}) \}_{s=1}^S$ rather than a closed-form distribution. To compute differential entropy for acquisition functions, we follow the standard approximation in causal active learning~\citep{jesson2021causal, kirsch2023blackbox, gao2025causal}. We treat the MCMC samples as draws from an implicit Multivariate Gaussian distribution. For a candidate $(\vx, t)$, we compute the empirical mean vector and covariance matrix of the concatenated vector $[\hat{y}, \hat{\tau}]^\top$ from the $S$ samples. This Gaussian proxy allows us to apply closed-form mutual information formulas efficiently.
\end{itemize}

\paragraph{2. GP-based Estimators (CMGP \& NSGP).}\footnote{\url{https://github.com/vanderschaarlab/mlforhealthlabpub/tree/main/alg/causal_multitask_gaussian_processes_ite}}
Gaussian Processes serve as the primary probabilistic backbone for R-Design (Stage 2) due to their analytical tractability~\citep{kanagawa2025gaussian}. We employ CMGP~\citep{alaa2017bayesian} and NSGP~\citep{alaa2018limits}.
\begin{itemize}[leftmargin=*]
    \item \textbf{CATE Estimation:} 
    CMGP treats potential outcomes as a multi-task learning problem, placing a joint GP prior over the vector of conditional means $\boldsymbol{\mu}(\vx) = [\mu(\vx, 0), \mu(\vx, 1)]^\top$. It employs a Linear Model of Coregionalization (LMC) kernel to capture correlations between treatment arms.
    NSGP handles heterogeneity by modeling the joint space $\gX \times \{0,1\}$ with a non-stationary kernel. It places a GP prior directly over the response function $\mu(\vx, t)$, where the kernel hyperparameters are treatment-dependent. In both cases, the CATE is derived as the linear transformation $\tau(\vx) = \mu(\vx, 1) - \mu(\vx, 0)$.
    
    \item \textbf{Joint Posterior (Closed-Form):} 
    Unlike trees or ensembles, GPs provide the \textit{exact} joint posterior. Since GP posteriors are closed under linear operations, the joint distribution of the observed outcome $y$ (at $\vx, t$) and the unobserved estimand $\tau(\vx)$ is a Multivariate Gaussian. We derive the mean vector and covariance matrix analytically from the GP posterior kernel, enabling precise and efficient computation of acquisition scores without sampling approximations.
\end{itemize}

\paragraph{3. Causal Multi-task Deep Ensemble (CMDE).}\footnote{\url{https://github.com/jzy95310/ICK/tree/main/experiments/causal_inference}}
For high-dimensional tasks where GPs scale poorly, we can employ CMDE~\citep{jiang2023estimating}, which adapts the Deep Ensemble insight to causal inference.
\begin{itemize}[leftmargin=*]
    \item \textbf{CATE Estimation:} CMDE utilizes a multi-head architecture with three components: a shared representation network $\phi(\vx)$ (capturing common features) and two treatment-specific heads $h_0(\phi(\vx))$ and $h_1(\phi(\vx))$. The potential outcomes are modeled as $\hat{y}_t = h_t(\phi(\vx))$. We train an ensemble of $M$ independently initialized networks to quantify epistemic uncertainty. The CATE is estimated as the average difference between the heads: $\hat{\tau}(\vx) = \frac{1}{M} \sum_{m=1}^{M} (h_1^{(m)}(\vx) - h_0^{(m)}(\vx))$.
    \item \textbf{Joint Posterior via Ensemble Aggregation:} Deep Ensembles do not provide an analytical posterior. We interpret the ensemble members as approximate samples from the posterior distribution of the \textit{potential outcome means}. For a query point $\vx$, each ensemble member $m \in \{1, \ldots, M\}$ outputs predictions for the mean vector $\hat{\boldsymbol{\mu}}^{(m)}(\vx) = [\hat{\mu}^{(m)}(\vx, 0), \hat{\mu}^{(m)}(\vx, 1)]^\top$. To construct the joint posterior, we compute the empirical moments across the ensemble predictions. This yields a Gaussian approximation:
    \begin{equation}
        p(\boldsymbol{\mu}(\vx) \mid \gD) \approx \gN\left( \bar{\boldsymbol{\mu}}(\vx), \hat{\mathbf{\Sigma}}(\vx) \right),
    \end{equation}
    where $\bar{\boldsymbol{\mu}}(\vx) = \frac{1}{M} \sum_{m=1}^{M} \hat{\boldsymbol{\mu}}^{(m)}(\vx)$ is the ensemble mean vector, and $\hat{\mathbf{\Sigma}}(\vx)$ is the empirical covariance matrix. The posterior for the CATE $\tau(\vx)$ (and its covariance with any candidate outcome $y$) can then be derived via linear transformations of this joint Gaussian (e.g., $\tau = \mu_1 - \mu_0$).
\end{itemize}

\subsection{Datasets}
\label{app:datasets}

To comprehensively evaluate R-Design across varying degrees of complexity, we employ a multiple validation strategies. We begin with \textit{Univariate Synthetic} scenarios to dissect the impact of specific functional forms (e.g., non-linearity, confounding frequency) on residual learning. We then progress to \textit{Multivariate Synthetic} environments to assess robustness to high dimensionality and covariate shift. Finally, we benchmark performance on \textit{Semi-Synthetic} datasets (IHDP, ACTG-175) that combine real-world covariate structures with simulated outcomes, providing a rigorous test ground with ground-truth causal estimands.

\subsubsection{Univariate Synthetic Dataset}

We adapt the simulation framework from Causal-ICM~\citep{dimitriou2024data} to construct a suite of $8$ scenarios that systematically vary three structural components: the treatment effect heterogeneity $\tau(\vx)$, the baseline outcome $\mu(\vx, 0)$, and the observational confounding mechanism $\eta(\vx)$.

\paragraph{Data Generation Process.}
For the experimental (RCT) source $\rs=e$, we simulate a trial participation model: candidate covariates are first drawn from $\tilde{\rvx} \sim \gU[-2, 2]$, then filtered through a logistic participation probability $p(\rs=e \mid \tilde{\rvx}) = \sigma(\beta_0 + \beta_1 \tilde{\rvx})$. Only candidates with $\rs=e$ enter the trial pool. Treatment assignment is balanced: $\rt \sim \text{Bernoulli}(0.5)$, and outcomes follow $\ry = \mu(\rvx, \rt) + \epsilon$ with $\epsilon \sim \gN(0, 1)$. For the observational source $\rs=o$, covariates are drawn from $\rvx \sim \gU[-2, 2]$, with confounded treatment assignment $e(\vx) = \sigma(\vx)$. The observational outcomes are structurally biased:
\begin{equation}
    \mu_o(\vx, t) = \mu(\vx, t) + (2t - 1)\eta(\vx),
\end{equation}
where $(2t - 1)\eta(\vx)$ introduces a treatment-dependent shift simulating unobserved confounding.

\paragraph{Component Functions.}
We define two complexity levels for each structural component:

\begin{itemize}
    \item \textbf{Treatment Effect} $\tau(\vx) = \mu(\vx, 1) - \mu(\vx, 0)$:
    \begin{itemize}
        \item Simple: $\tau(\vx) = 1 - \vx$ (Linear).
        \item Complex: $\tau(\vx) = 1 + \vx - \vx^2$ (Non-linear, Non-monotonic).
    \end{itemize}
    
    \item \textbf{Baseline Outcome} $\mu(\vx, 0)$:
    \begin{itemize}
        \item Simple: $\mu(\vx, 0) = 0.8\vx^3 - \vx$ (Polynomial).
        \item Complex: $\mu(\vx, 0) = 2\sin(3\pi \vx) - 1.5\exp(1.5(\vx - 0.8))$ (High-frequency).
    \end{itemize}
    
    \item \textbf{Confounding Bias} $\eta(\vx)$:
    \begin{itemize}
        \item Simple: $\eta(\vx) = 2\exp(-0.8(\vx + 2))$ (Smooth decay).
        \item Complex: $\eta(\vx) = 5\exp\left(-\frac{(\vx+3)^2}{12.5}\right)(1 + 0.2\cos(\vx))$ (Oscillatory).
    \end{itemize}
\end{itemize}

\paragraph{Simulation Scenarios.}
We evaluate the full factorial combination ($2^3$) of these components, resulting in 8 distinct scenarios (Table~\ref{tab:univariate_sims}).
\begin{table}[h]
\centering
\caption{Univariate simulation scenarios defined by component complexity.}
\label{tab:univariate_sims}
\small
\begin{tabular}{c|ccc}
\toprule
Scenario & Confounding $\eta(\vx)$ & Baseline $\mu(\vx, 0)$ & Treatment Effect $\tau(\vx)$ \\
\midrule
1 & Simple & Simple & Simple \\
2 & Simple & Simple & Complex \\
3 & Simple & Complex & Simple \\
4 & Simple & Complex & Complex \\
5 & Complex & Simple & Simple \\
6 & Complex & Simple & Complex \\
7 & Complex & Complex & Simple \\
8 & Complex & Complex & Complex \\
\bottomrule
\end{tabular}
\end{table}

\paragraph{Sample Sizes.} The RCT candidate pool $\gD_P$ contains $n_P = 1000$ samples, and the observational dataset $\gD_O$ contains $n_O = 2000$ samples. Results are averaged over 10 random seeds.

\subsubsection{Multivariate Synthetic Dataset}

We extend the simulation to a multivariate setting to introduce challenges such as prognostic noise, complex confounding, and covariate shift.

\paragraph{Feature Structure.}
We define a $d$-dimensional covariate space $\vx \in \R^d$, partitioning features into three groups:
\begin{itemize}
    \item \textbf{Confounders $\gC$:} Features affecting both $\rt$ and $\ry$ (primary bias source).
    \item \textbf{Prognostic Factors $\gP$:} Features affecting only $\ry$ (variance source).
    \item \textbf{Noise Features $\gN$:} Irrelevant features (robustness test).
\end{itemize}
We use $d=6$ for main experiments ($2$ per group) and scale up to $d \in \{9, 15, 24, 36\}$.

\paragraph{Data Generation.}
For the RCT pool, $\vx \sim \gU[-2, 2]^d$ with balanced assignment. For observational data, $\vx \sim \gU[-2, 2]^d$, but assignment is confounded via $\pi_o(\vx) = \sigma( f_{\text{prop}}(\vx_{\gC}) + 0.2 \tanh(\sum \vx_{\gC}) )$.
Outcomes are generated as $\ry = \mu(\vx, \rt) + \epsilon$ (RCT) and $\ry = \mu_o(\vx, \rt) + \epsilon$ (Obs), where $\mu_o$ includes the bias term.

\textbf{Baseline Outcome} $\mu(\vx, 0)$:
\begin{equation}
    \mu(\vx, 0) = f_{\mu}(\vx_{\gC}, \vx_{\gP}) + 0.5 \sin(3\pi \vx_{\gC, 1}) \cos(\pi \vx_{\gC, 2}).
\end{equation}
This combines a random MLP $f_{\mu}$ with a high-frequency component to challenge the learner.

\textbf{Treatment Effect} $\tau(\vx)$:
\begin{equation}
    \tau(\vx) = 1 + 2\sin(2\pi \vx_1) + 0.5\cos(\pi \vx_2).
\end{equation}

\textbf{Observational Bias} $\eta(\vx)$:
The observational mean is shifted by $\mu_o(\vx, t) = \mu(\vx, t) + (2t - 1)\eta(\vx)$, where:
\begin{equation}
    \eta(\vx) = 1.5 + (\vx_{\gC, 1} + \vx_{\gC, 2}) + 0.5 \vx_{\gC, 1}^2.
\end{equation}
This bias is smooth (quadratic), allowing the residual learner to recover $\tau_\delta(\vx)$ effectively.

\paragraph{Covariate Shift.}
In shift experiments, $\gD_P$ is drawn from $\gU[-4, 2]^d$ while $\gD_O$ remains $\gU[-2, 2]^d$, creating regions where the experimental policy must extrapolate.

\paragraph{Protocol.}
$n_P = 1000$, $n_O = 2000$. Active learning: warm-start 50, batch size 10, 30 iterations. For policy evaluation (Decision Making), we use warm-start 20, batch size 1, 100 iterations.

\subsubsection{Semi-Synthetic Datasets}

We validate R-Design on two benchmarks derived from real-world studies, providing realistic covariate correlations while maintaining ground-truth estimands.

\paragraph{IHDP Dataset.}
Based on the Infant Health and Development Program~\citep{hill2011bayesian}, pre-processed to induce selection bias ($n=747$, $d=25$).

\textit{Response Surface.} The potential outcomes are generated via a non-linear exponential model:
\begin{align}
    \mu(\vx, 0) &= \exp\left((\vx + 0.5)^\top \boldsymbol{\beta}\right), \\
    \mu(\vx, 1) &= (\vx + 0.5)^\top \boldsymbol{\beta} - \omega,
\end{align}
where $\boldsymbol{\beta}$ is a sparse coefficient vector.

\textit{Confounding Mechanism.}
For the observational split, we inject treatment-dependent bias $\eta(\vx)$ such that $\mu_o(\vx, t) = \mu(\vx, t) + (2t - 1)\eta(\vx)$:
\begin{equation}
    \eta(\vx) = \gamma \cdot \left(1 + c_1 \vx_1 + c_2 \vx_2 + c_3 \vx_1^2\right).
\end{equation}
Here, $\vx_1, \vx_2$ correspond to standardized `bw` and `b.head`. We use $\gamma=1.5$ and randomized coefficients $c_i$.

\paragraph{ACTG-175 Dataset.}
Derived from an AIDS clinical trial~\citep{hammer1996trial} ($n=813$, $d=12$).

\textit{Response Surface.} We define a complex outcome surface combining a random neural network with interpretable interactions:
\begin{equation}
\begin{aligned}
\mu(\vx, 0)
= {} & f_{\text{MLP}}(\vx) + 6 + 0.3 \vx_{\texttt{wtkg}}^2 \\
& {} - \sin(\vx_{\texttt{age}})\,(\vx_{\texttt{gender}} + 1)
+ 0.6 \vx_{\texttt{hemo}} \vx_{\texttt{race}}
- 0.2 \vx_{\texttt{z30}} .
\end{aligned}
\end{equation}
The treatment effect is $\tau(\vx) = 1.5 \sin(\vx_{\texttt{wtkg}}) (\vx_{\texttt{karnof\_hi}} + 1) + 2 \vx_{\texttt{age}}$.

\textit{Confounding \& Shift.}
We apply the same confounding bias function $\eta(\vx)$ as in IHDP. Furthermore, we induce covariate shift by under-sampling older patients in the RCT pool (probabilistic rejection sampling based on median age), simulating eligibility bias.

\paragraph{Experimental Protocol.}
Data is partitioned 4:1:2 (Obs Train / Obs Test / RCT Pool). Active learning proceeds with 30 warm-start samples and 150 single-sample acquisitions, averaged over 10 seeds.

\subsection{Hardware}

\paragraph{Computational Resources.}
All experiments were conducted on a workstation with the following specifications:
\begin{itemize}
    \item \textbf{CPU:} Intel Core i9-13900KF (24 cores, 32 threads)
    \item \textbf{RAM:} 128\,GB
    \item \textbf{GPU:} NVIDIA GeForce RTX 4090 (24\,GB VRAM)
    \item \textbf{OS:} Ubuntu 24.04 LTS
    \item \textbf{Software:} Python 3.10, PyTorch 2.3.1 (CUDA 12.1)
\end{itemize}

%% file: Pages/Appendix/Acq_functions.tex
\section{Acquisition Functions Details}
\label{app:acq_details}

This section provides the detailed derivations and computational strategies for the acquisition functions used in R-Design. We focus on two primary objectives: minimizing the estimation error of the CATE function (\textbf{R-EPIG-Est}) and minimizing the uncertainty of the optimal treatment decision (\textbf{R-EPIG-Policy}). Since the observational contrast $\hat{\tau}_o(\vx)$ is treated as a fixed offset during the experimental phase, the uncertainty in the full CATE $\tau(\vx)$ is entirely driven by the uncertainty in the residual contrast $\tau_\delta(\vx)$. Specifically, because $\tau(\vx) = \hat{\tau}_o(\vx) + \tau_\delta(\vx)$, it holds that:
\begin{equation}
\begin{aligned}
    \sV[\tau(\vx) \mid \gD_E] =& \sV[\tau_\delta(\vx) \mid \gD_E] \\
    \entropy(\tau(\vx) \mid \gD_E) =& \entropy(\tau_\delta(\vx) \mid \gD_E) + \text{const}.
\end{aligned}
\end{equation}
Thus, maximizing information gain regarding the residual $\tau_\delta(\vx)$ is mathematically equivalent to maximizing information gain regarding the true CATE $\tau(\vx)$.

\subsection{CATE Baseline Acquisition Strategies}
\label{app:cate_baselines}

We compare R-Design against a diverse set of acquisition strategies, ranging from random sampling to advanced information-theoretic criteria. To ensure a fair comparison, all baselines utilize the same underlying GP hyperparameters as our method (unless specific to a different model class) and select experimental points $(\vx, t)$ from the candidate pool $\gD_P$.

\paragraph{Random Sampling.}
The standard non-adaptive benchmark selects candidates uniformly at random from the pool:
\begin{equation}
    \alpha_{\text{Random}}(\vx, t) \coloneqq \gU(0, 1).
\end{equation}
This serves as a fundamental lower bound to quantify the benefits of active adaptation.

\paragraph{Leverage Score Sampling.}
A classical experimental design criterion from randomized linear algebra~\citep{addanki2022sample}. To adapt this to the causal setting, we define the kernel matrix $\mK$ over the augmented space $\gX \times \{0, 1\}$. The leverage score for a candidate design point $(\vx, t)$ is proportional to its diagonal entry in the regularized hat matrix:
\begin{equation}
    \alpha_{\text{Leverage}}(\vx, t) \coloneqq \left[ \mK (\mK + \sigma^2 \mI)^{-1} \mK \right]_{(\vx, t), (\vx, t)}.
\end{equation}
In the GP context, this score is monotonically related to the posterior variance of the potential outcome $\mu(\vx, t)$. Maximizing this criterion selects design points that exert the greatest influence on the model fit, prioritizing regions (in the joint covariate-treatment space) where the current estimation is least stable.

\paragraph{Causal-BALD series acquisition functions~\citep{jesson2021causal}.}

\paragraph{$\mu$-BALD.}
A heuristic that selects points where the current model is most uncertain about the potential outcomes, ignoring the correlation between arms. For a probabilistic model predicting $\mu(\vx, t)$:
\begin{equation}
    \alpha_{\mu\text{-BALD}}(\vx, t) \coloneqq \max_{t' \in \{0,1\}} \sV[\mu(\vx, t') \mid \gD_E].
\end{equation}
This method focuses on reducing global outcome uncertainty rather than targeting the CATE specifically.

\paragraph{$\tau$-BALD.}
Targets uncertainty about the treatment effect directly by maximizing the posterior variance of the contrast $\tau(\vx) = \mu(\vx, 1) - \mu(\vx, 0)$:
\begin{equation}
    \alpha_{\tau\text{-BALD}}(\vx) \coloneqq \sV[\tau(\vx) \mid \gD_E].
\end{equation}
Since $\tau$ is invariant to the specific treatment arm chosen for the query (in the model space), the arm $t$ is typically assigned randomly or to maximize marginal outcome variance.

\paragraph{$\rho$-BALD.}
Based on the derivation by \citet{jesson2021causal}, this criterion prioritizes points where the model is uncertain about the treatment effect (high $\sV[\tau]$) but confident about the counterfactual outcome (low $\sV[\mu_{1-t}]$). For a candidate arm $t$, the score is proportional to the mutual information lower bound:
\begin{equation}
    \alpha_{\rho\text{-BALD}}(\vx, t) \coloneqq \frac{1}{2} \log \left( 1 + \frac{\sV[\tau(\vx) \mid \gD_E]}{\sV[\mu(\vx, 1-t) \mid \gD_E]} \right).
\end{equation}
Note that unlike $\tau$-BALD, this score depends on the specific treatment arm $t$ being queried, penalizing candidates where the counterfactual $1-t$ is highly uncertain (as this makes disentangling the effect difficult).

\paragraph{$\mu$-$\rho$ BALD.}
This strategy combines the global exploration of $\mu$-BALD with the targeted refinement of $\rho$-BALD. It is defined as a multiplicative scaling of the outcome uncertainty by the signal-to-noise ratio:
\begin{equation}
    \alpha_{\mu\text{-}\rho}(\vx, t) \coloneqq \sV[\mu(\vx, t) \mid \gD_E] \cdot \frac{\sV[\tau(\vx) \mid \gD_E]}{\sV[\mu(\vx, 1-t) \mid \gD_E]}.
\end{equation}
This criterion aggressively selects points that are simultaneously uncertain in the predictive outcome (high $\sV[\mu_t]$), informative for the CATE (high $\sV[\tau]$), and anchored by a confident counterfactual (low $\sV[\mu_{1-t}]$).

\paragraph{ABC3 (Active Bayesian Causal Covariate-Conditioning).}
A variance reduction criterion that measures the expected reduction in CATE uncertainty across the target population $\gD_O$ from acquiring a single observation~\citep{cha2025abc3}. For a candidate $(\vx_i, t_i)$, the score is the total expected variance reduction in $\tau$ across all target points $\vx_j \in \gD_O$:
\begin{equation}
    \alpha_{\text{ABC3}}(\vx_i, t_i) \coloneqq \sum_{\vx_j \in \gD_O} \frac{\textup{Cov}[y(\vx_i, t_i), \tau(\vx_j) \mid \gD_E]^2}{\sV[y(\vx_i, t_i) \mid \gD_E]}.
\end{equation}
The arm $t_i \in \{0, 1\}$ yielding the higher total reduction is selected.

\paragraph{ACE (Active Covariance Exploitation).}
Prioritizes points with strong cross-arm covariance, targeting regions where observing one arm provides significant information about the counterfactual arm~\citep{song2024ace}:
\begin{equation}
    \alpha_{\text{ACE}}(\vx_i, t_i) \coloneqq \frac{|\textup{Cov}[\mu(\vx_i, 0), \mu(\vx_i, 1) \mid \gD_E]|}{\sV[\mu(\vx_i, t_i) \mid \gD_E]}.
\end{equation}

\paragraph{EPIG (Expected Predicted Information Gain).}
An information-theoretic baseline~\citep{smith2023prediction} that maximizes the mutual information between a candidate observation $(\vx_i, t_i)$ and the predictions on the observational target set $\gD_O$. It encourages acquiring points that best reduce the joint entropy of the target predictions:
\begin{equation}
    \alpha_{\text{EPIG}}(\vx_i, t_i) \coloneqq \frac{1}{|\gD_O|} \sum_{\vx_j \in \gD_O} \mi(y(\vx_i, t_i); y(\vx_j) \mid \gD_E),
\end{equation}
where $y(\vx_j)$ represents the predictive distribution over outcomes in the target set. Under Gaussian assumptions, this mutual information has a closed form:
\begin{equation}
    \mi(y_i; y_j) = -\frac{1}{2} \log \left( 1 - \frac{\textup{Cov}[y(\vx_i, t_i), y(\vx_j)]^2}{\sigma_i^2 \sigma_j^2} \right).
\end{equation}

\subsection{Decision-Making Acquisition Strategies}
\label{app:dm_baselines}

For the decision-making task, we focus on acquisition functions that target policy optimization rather than global CATE estimation accuracy. These methods aim to minimize the APE by reducing uncertainty specifically about the sign of the treatment effect, $\text{sign}(\tau(\vx))$.

\paragraph{Random Sampling.}
Identical to the strategy in CATE acquisition baselines, selecting candidates $(\vx, t)$ uniformly at random. This serves as the non-adaptive baseline for policy learning.

\paragraph{Sign-BALD.}
Targets uncertainty about the binary treatment decision $\pi(\vx) = \sI(\tau(\vx) > 0)$. It maximizes the variance of the decision indicator across the posterior~\citep{jesson2021causal, klein2025towards}:
\begin{equation}
    \alpha_{\text{Sign}}(\vx) \coloneqq \sV\left[\sI(\tau(\vx) > 0) \mid \gD_E\right].
\end{equation}
In practice, for a model producing $K$ posterior samples $\{\tau^{(k)}(\vx)\}_{k=1}^K$, we compute the empirical variance of the binary indicators $z^{(k)} = \sI(\tau^{(k)}(\vx) > 0)$. This score is maximized (at 0.25) when the posterior probability $\sP(\tau(\vx) > 0) = 0.5$, identifying points precisely at the decision boundary where the model is most ambiguous about the optimal treatment. The query arm $t$ is typically selected uniformly at random.

\paragraph{Type-S Error.}
Selects points with the highest probability of making a "Type S" (Sign) error in the treatment decision~\citep{sundin2019active}. Assuming a Gaussian posterior for the CATE, $\tau(\vx) \mid \gD_E \sim \gN(\hat{\mu}_\tau(\vx), \hat{\sigma}_\tau^2(\vx))$, the acquisition score is defined as the probability that the true effect sign differs from the estimated mean sign:
\begin{equation}
    \alpha_{\text{Type-S}}(\vx) \coloneqq \Phi\left(-\frac{|\hat{\mu}_\tau(\vx)|}{\hat{\sigma}_\tau(\vx)}\right),
\end{equation}
where $\Phi(\cdot)$ is the standard normal CDF. This criterion prioritizes points where the signal-to-noise ratio $|\hat{\mu}_\tau|/\hat{\sigma}_\tau$ is low, which occurs either when the effect is close to zero (decision boundary) or when the uncertainty $\hat{\sigma}_\tau$ is extremely high.

\subsection{R-EPIG: Residual Expected Predicted Information Gain}
\label{sec:repig}

In the main paper, we introduce the R-EPIG framework, adapting the information-theoretic criterion from EPIG~\citep{smith2023prediction} and Causal-EPIG~\citep{gao2025causal} to the new setting. R-EPIG operates in a natural shift mode: the candidate set comprises the experimental pool $\gD_P$, while the target set is usually the observational population $\gD_O$ where estimation accuracy matters most. Formally, our goal is to maximize the expected reduction in posterior entropy of the residual contrast $\tau_\delta(\vx^*)$ over the target population. Since the observational contrast $\hat{\tau}_o$ is fixed, uncertainty in the CATE $\tau$ is entirely affected by the residual function: $\entropy(\tau \mid \gD_E) \equiv \entropy(\tau_\delta \mid \gD_E)$. The acquisition score for a candidate $(\vx, t)$ is derived as the expected information gain averaged over the target set:
\begin{align}
\alpha_{\text{R-EPIG}}(\vx, t) 
&\coloneqq 
\E_{\vx^* \sim \gD_O} \, \E_{y \sim p(y \mid \vx, t, \gD_E)} 
\Big[
\entropy(\tau_\delta(\vx^*) \mid \gD_E) \notag\\
&\qquad - \entropy(\tau_\delta(\vx^*) \mid \gD_E \cup \{(\vx, t, y)\}) 
\Big] \\
&= \frac{1}{|\gD_O|} \sum_{\vx^* \in \gD_O} \mi\left(y; \tau_\delta(\vx^*) \mid \gD_E\right).
\end{align}

We present two variants of this criterion targeting different aspects of the residual uncertainty.

\subsubsection{R-EPIG-$\tau$: Targeting Residual Effect Uncertainty}

R-EPIG-$\tau$ directly maximizes the mutual information between the candidate experimental outcome $y$ (at $\vx, t$) and the residual contrast $\tau_\delta(\vx^*)$ at target points. 
Under the GP posterior, the joint distribution of the observed outcome and the target residual is Multivariate Gaussian. The mutual information admits a closed-form expression:
\begin{equation}
    \mi(y; \tau_\delta(\vx^*)) = \frac{1}{2} \log \left( \frac{\sigma_{y}^2 \sigma_{\tau_\delta^*}^2}{\sigma_{y}^2 \sigma_{\tau_\delta^*}^2 - \textup{Cov}(y, \tau_\delta(\vx^*))^2} \right) = \frac{1}{2} \log \left( \frac{1}{1 - \rho_{y, \tau_\delta^*}^2} \right).
\end{equation}
Intuitively, this criterion selects candidates whose outcomes are maximally correlated ($\rho^2 \to 1$) with the uncertainty in the target residual contrasts. 
To implement this efficiently, we expand the residual covariance terms using the GP kernel properties. For a target $\vx^*$, the residual contrast is defined as $\tau_\delta(\vx^*) = \delta(\vx^*, 1) - \delta(\vx^*, 0)$. Thus:
\begin{equation}
    \sigma_{\tau_\delta^*}^2 = \sV[\delta(\vx^*, 1)] + \sV[\delta(\vx^*, 0)] - 2\textup{Cov}(\delta(\vx^*, 1), \delta(\vx^*, 0)).
\end{equation}
The cross-covariance between the candidate outcome $y$ (under treatment $t$) and the target residual contrast is:
\begin{equation}
    \textup{Cov}(y, \tau_\delta(\vx^*)) = \textup{Cov}(\delta(\vx, t), \delta(\vx^*, 1)) - \textup{Cov}(\delta(\vx, t), \delta(\vx^*, 0)).
\end{equation}
Crucially, these terms are directly retrievable from the posterior covariance matrix of the CMGP~\citep{alaa2017bayesian} or the NSGP~\citep{alaa2018limits}. If we use other base CATE estimators, such as BART~\citep{hill2011bayesian}, BCF~\citep{hahn2020bayesian}, and CMDE~\citep{jiang2023estimating}, we can also approximate the joint posterior as MultiNormal distribution and use totally the same calculation method to get the score.

\subsubsection{R-EPIG-$\mu$: Targeting Joint Residual Outcomes}

Alternatively, R-EPIG-$\mu$ maximizes information gain regarding the joint vector of residual potential outcomes $\boldsymbol{\delta}(\vx^*) = [\delta(\vx^*, 0), \delta(\vx^*, 1)]^\top$. This variant recognizes that reducing uncertainty in the individual response surfaces implicitly constrains the effect difference:
\begin{equation}
    \alpha_{\text{R-EPIG-}\mu}(\vx, t) \coloneqq \frac{1}{|\gD_O|} \sum_{\vx^* \in \gD_O} \mi\left(y; \boldsymbol{\delta}(\vx^*) \mid \gD_E\right).
\end{equation}
The mutual information for this vector target is computed via the Schur complement. Let $\mathbf{\Sigma}_{\boldsymbol{\delta}^*}$ be the $2 \times 2$ covariance matrix of the target residuals. The information gain is:
\begin{equation}
    \mi(y; \boldsymbol{\delta}(\vx^*)) = \frac{1}{2} \log \frac{\sigma_{y}^2}{\sigma_{y | \boldsymbol{\delta}^*}^2}, \quad \text{where } \sigma_{y | \boldsymbol{\delta}^*}^2 = \sigma_{y}^2 - \mathbf{c}^\top \mathbf{\Sigma}_{\boldsymbol{\delta}^*}^{-1} \mathbf{c},
\end{equation}
and $\mathbf{c} = [\textup{Cov}(y, \delta(\vx^*, 0)), \textup{Cov}(y, \delta(\vx^*, 1))]^\top$. While computationally slightly more expensive than R-EPIG-$\tau$ due to the matrix inversion per target, this variant captures correlations in the joint output space that the difference operator might cancel out.

\paragraph{Summary and Implementation Strategy}

\begin{table*}[t]
\centering
\caption{Comparison of R-EPIG variants.}
\label{tab:repig_comparison}
\small
\begin{tabular}{l|cc}
\toprule
Aspect & R-EPIG-$\tau$ & R-EPIG-$\mu$ \\
\midrule
\textbf{Target Estimand} & Residual Contrast $\tau_\delta$ & Residual Vector $\boldsymbol{\delta}$ \\
\textbf{Information Metric} & Scalar Mutual Information & Vector Mutual Information \\
\textbf{Complexity} & $\gO(n_P \cdot n_O)$ & $\gO(n_P \cdot n_O)$ (with small constant overhead) \\
\textbf{Intuition} & Maximizes correlation with CATE & Maximizes correlation with Outcomes \\
\bottomrule
\end{tabular}
\end{table*}

\paragraph{Treatment Arm Selection.}
For a chosen candidate $\vx$, we compute the score for both arms $t \in \{0, 1\}$ and select the one maximizing the utility: $t^* = \argmax_t \alpha_{\text{R-EPIG}}(\vx, t)$. This enables the model to request specific counterfactuals. For instance, observing the control arm if the residual uncertainty is dominated by $\delta(\cdot, 0)$.

\paragraph{Computational Efficiency.}
A distinct advantage of our GP-based implementation is that all variance and covariance terms ($\sigma^2, \textup{Cov}$) are computed analytically from the updated posterior kernel $\mathbf{K}_{\text{post}}$ without Monte Carlo sampling. In the batched setting, we utilize the "Parallel Scoring" described in App.~\ref{app:active_loop}, computing the covariance map between the entire candidate pool $\gD_P$ and target set $\gD_O$ in a single vectorized operation on the GPU.

\subsubsection{R-EPIG-$\pi$: Targeting Policy Uncertainty}
\label{sec:repig_pi}

For decision-making tasks where the goal is to optimize the treatment policy $\pi(\vx) = \sI(\tau(\vx) > 0)$, standard variance reduction is inefficient as it over-invests in regions where the sign of the effect is already certain. We propose R-EPIG-$\pi$, which maximizes the expected reduction in the \textit{binary entropy} of the optimal policy decision at target points:
\begin{equation}
    \alpha_{\text{R-EPIG-}\pi}(\vx, t) \coloneqq \frac{1}{|\gD_O|} \sum_{\vx^* \in \gD_O} \mi\left(y; \pi(\vx^*) \mid \gD_E\right).
\end{equation}

\paragraph{Derivation with Gaussian-Bernoulli Variables.}
Unlike the estimation case, this mutual information involves a continuous Gaussian outcome $y$ and a discrete Bernoulli variable $\pi^*$. We express it as the difference between the current policy entropy and the expected posterior entropy:
\begin{equation}
    \mi(y; \pi) = \entropy(\pi(\vx^*) \mid \gD_E) - \E_{y \sim p(y \mid \vx, t)} \Big[ \entropy(\pi(\vx^*) \mid y, \gD_E) \Big].
\end{equation}
Recall that in R-Design, the CATE is decomposed as $\tau(\vx^*) = \hat{\tau}_o(\vx^*) + \tau_\delta(\vx^*)$. Since the observational contrast $\hat{\tau}_o$ is fixed, the probability of the optimal action being treatment ($t=1$) is derived purely from the residual GP posterior:
\begin{equation}
    \sP(\pi=1) = \sP(\tau_\delta(\vx^*) > -\hat{\tau}_o(\vx^*)) = \Phi\left(\frac{\mu_{\tau_\delta}(\vx^*) + \hat{\tau}_o(\vx^*)}{\sigma_{\tau_\delta}(\vx^*)}\right) \coloneqq p_{\text{prior}}.
\end{equation}
The first term is simply the binary entropy $\entropy_{\text{Bern}}(p_{\text{prior}})$.

\paragraph{Expected Posterior Entropy via Quadrature.}
The second term requires integrating over the unobserved experimental outcome $y$. Since the GP update for the posterior mean $\mu_{\tau_\delta|y}$ is linear in $y$, while the posterior variance $\sigma^2_{\tau_\delta|y}$ is deterministic, the conditional probability $p_{\text{post}}(y)$ becomes a function of $y$.
We approximate this intractable integral using Gauss-Hermite quadrature. By performing the change of variables $y = \mu_y + \sqrt{2}\sigma_y z$, we map the expectation to a weighted sum over $K$ standard normal nodes:
\begin{equation}
    \E_{y} [\entropy_{\text{Bern}}(p_{\text{post}}(y))] \approx \frac{1}{\sqrt{\pi}} \sum_{k=1}^{K} w_k \cdot \entropy_{\text{Bern}}\left( \Phi\left( \frac{\mu_{\text{total}} + \frac{\textup{Cov}(y, \tau_\delta)}{\sigma_y}\sqrt{2}z_k}{\sigma_{\tau_\delta \mid y}} \right) \right),
\end{equation}
where $\mu_{\text{total}} = \mu_{\tau_\delta}(\vx^*) + \hat{\tau}_o(\vx^*)$. In our implementation, we use $K=96$ nodes, providing high numerical precision with negligible computational overhead compared to the $\gO(n^3)$ GP inference.

\paragraph{Intuition and Design Rationale.}
R-EPIG-$\pi$ exhibits a distinct behavior from its estimation counterparts. As shown in Tab.~\ref{tab:repig_pi_comparison}, it naturally focuses on the decision boundary ($\tau(\vx) \approx 0$). In these regions, $p_{\text{prior}} \approx 0.5$ (maximum entropy), and a small information gain from $y$ can significantly resolve the sign ambiguity. Conversely, in regions with strong positive or negative effects, the entropy is near zero, suppressing acquisition.

\begin{table*}[t]
\centering
\caption{R-EPIG variants: CATE estimation vs.\ policy optimization.}
\label{tab:repig_pi_comparison}
\small
\begin{tabular}{l|ccc}
\toprule
Aspect & R-EPIG-$\tau$ & R-EPIG-$\mu$ & R-EPIG-$\pi$ \\
\midrule
\textbf{Target Quantity} & CATE Magnitude & Outcomes  & Policy \\
\textbf{Objective} & Minimize PEHE & Minimize PEHE & Minimize APE \\
\textbf{Information Metric} & Gaussian MI & Gaussian Vector MI & Bernoulli MI \\
\textbf{Computation} & Analytic (Closed-form) & Analytic (Closed-form) & Gauss-Hermite Quad. \\
\textbf{Focus Region} & High Epistemic Unc. (CATE) & High Epistemic Unc. (Outcome) & Decision Boundary \\
\bottomrule
\end{tabular}
\end{table*}

\subsection{Computational Complexity}
\label{app:complexity}

A critical design choice in R-EPIG is the use of the mean-marginal formulation ($\sum_{\vx^*} \mi(y; \cdot)$) rather than the global joint information gain.
\begin{itemize}[leftmargin=*]
    \item \textbf{Global Formulation (Intractable):} Calculating $\mi(y; \boldsymbol{\delta}_{\text{global}})$ would require computing the determinant of the updated covariance matrix for the entire target set ($\gD_O$) for \textit{each} candidate query, leading to a prohibitive cost of $\gO(|\gD_P| \cdot |\gD_O|^3)$.
    \item \textbf{Our Summation Approach (Efficient):} By assuming additivity of information across target points (a lower bound on total information), we reduce the complexity to $\gO(|\gD_P| \cdot |\gD_O|)$. 
\end{itemize}
Crucially, since we use GP, we only need the predictive cross-covariance between the candidate pool and the target set. We implement this efficiently via "Parallel Scoring": we compute the full kernel block $\mathbf{K}_{PT}$ between $\gD_P$ and $\gD_O$ in a single batched GPU operation, allowing R-Design to scale.

%% file: Pages/Appendix/Convergence.tex
\section{Theoretical Analysis Details}
\label{app:convergence_details}

In this section, we provide detailed proofs for the theoretical results presented in Section~\ref{sec:theory}, establishing the rigorous foundations of the R-Design framework. We begin by proving the \textit{Structural Efficiency Gap} (Lemma~\ref{lemma:complexity_decomp_hoelder}), formally demonstrating that learning the residual contrast $\tau_\delta$ is strictly more sample-efficient than the tabula rasa approach under smoothness assumptions. Next, we establish the \textit{Objective Alignment} (Prop.~\ref{prop:objective_alignment}), proving the mathematical equivalence between minimizing the Bayesian PEHE risk and minimizing the posterior variance of the residual contrast. Finally, we provide the complete \textit{Convergence Analysis} for the R-EPIG acquisition strategy (Thm.~\ref{thm:epig_bound}), bridging our causal active learning problem with the general Transductive Active Learning (TAL) framework~\citep{hubotter2024transductive}.

\subsection{Mapping R-Design to the TAL Framework}
\label{app:tal_mapping}

To invoke the theoretical machinery of TAL, we formally map our residual-based experimental design problem to the GP AL setting.

\textbf{1. The Residual Process.}
We model the latent residual function $\delta(\vx, t)$ as a GP over the augmented domain $\tilde{\gX} = \gX \times \{0, 1\}$:
\begin{equation}
    \delta(\tilde{\vx}) \sim \mathcal{GP}(0, k_\delta(\tilde{\vx}, \tilde{\vx}')), \quad \text{where } \tilde{\vx} \coloneqq (\vx, t).
\end{equation}
Since the observational base $\hat{\mu}_o(\vx, t)$ is treated as a fixed functional offset during the experimental phase, observing an outcome $y$ is mathematically equivalent to observing a noisy realization of the residual process:
\begin{equation}
    y_i = \hat{\mu}_o(\vx_i, t_i) + \delta(\vx_i, t_i) + \epsilon_i \iff r_i \coloneqq y_i - \hat{\mu}_o(\vx_i, t_i) = \delta(\vx_i, t_i) + \epsilon_i.
\end{equation}
Then, minimizing the posterior variance of the residual contrast $\tau_\delta$ directly minimizes the posterior variance of the CATE $\tau$, as $\sV[\tau(\vx) \mid \gD_E] = \sV[\tau_\delta(\vx) \mid \gD_E]$ (where $\tau_\delta(\vx) \coloneqq \delta(\vx, 1) - \delta(\vx, 0)$).

\textbf{2. Spaces and Constraints.}
We define the relevant sets for the AL formulation:
\begin{itemize}[leftmargin=*]
    \item \textbf{Augmented Pool ($\tilde{\gD}_P$):} The set of all potential experimental interventions. $\tilde{\gD}_P = \{ (\vx_i, t) \mid i \in \{1 \dots n\}, t \in \{0, 1\} \}$.
    \item \textbf{Valid Experimental Designs ($\mathbb{S}_{\text{valid}}$):} Due to the fundamental problem of causal inference, we cannot observe both potential outcomes for the same unit. A selected design $\gA \subseteq \tilde{\gD}_P$ is valid if and only if:
    \begin{equation}
        (\vx_i, 0) \in \gA \implies (\vx_i, 1) \notin \gA, \quad \forall i \in \{1 \dots n\}.
    \end{equation}
    This forms a partition matroid constraint on the selection.
    \item \textbf{Augmented Target ($\tilde{\mX}_{\tar}$):} The set of residuals we wish to predict to estimate CATE. $\tilde{\mX}_{\tar} = \{ (\vx^*, t) \mid \vx^* \in \mX_{\text{eval}}, t \in \{0, 1\} \}$.
\end{itemize}

\textbf{3. Weak Submodularity Assumption.}
To account for potential distribution shifts between the observational and experimental distributions, we adopt a relaxed submodularity assumption consistent with recent theoretical advances.

\begin{assumption}[$\alpha$-Weak Submodularity]
\label{ass:submodularity}
The utility function $\psi(\gA) = \mi(\tau_{\delta, \tilde{\mX}_{\tar}}; \ry_{\gA})$ is a monotone $\alpha$-weakly submodular set function with parameter $\alpha \in (0, 1]$. That is, for any $\gA \subseteq \gB \subseteq \tilde{\gD}_P$ and $v \in \tilde{\gD}_P \setminus \gB$:
\begin{equation}
    \psi(\gA \cup \{v\}) - \psi(\gA) \ge \alpha \left( \psi(\gB \cup \{v\}) - \psi(\gB) \right).
\end{equation}
\end{assumption}

\textit{Justification:} While mutual information is strictly submodular ($\alpha=1$) in inductive settings where Pool $\subseteq$ Target, this property may relax to weak submodularity ($\alpha < 1$) in transductive settings under distribution shift due to potential information synergies~\citep[App. C.4]{hubotter2024transductive}. Assuming $\alpha > 0$ is a standard relaxation that preserves the constant-factor approximation guarantees of greedy strategies.

\subsection{The Feasible Irreducible Uncertainty}
\label{app:feasible_limit}

A unique challenge in the adaptive causal experimental design is that the "irreducible" uncertainty is not a static property of the sample space, but depends on the treatment assignment choices. Since we cannot observe both potential outcomes for the same individual~\citep{gao2025causal}, the information we ultimately acquire depends on which potential outcome we choose to reveal. Unlike standard TAL where $\eta^2_{\gS}$ is fixed (observing the entire sample space $\gS$), here we can only observe a subset $\gD \subset \tilde{\gD}_P$ subject to causal constraints. To rigorously define the convergence limit, we consider the \textit{best possible} variance reduction achievable if we were to experiment on the entire pool optimally.

\begin{definition}[Maximal Valid Design]
Let $\tilde{\gD}_P$ be the augmented pool of size $2n$ (all potential outcomes). A design $\gD \subset \tilde{\gD}_P$ is a \textit{maximal valid design} if it contains exactly one outcome for every unit in the pool:
\begin{equation}
\begin{aligned}
    \gD \in \mathbb{S}_{\text{max}} \iff |\gD| = n \\ \forall i \in \{1 \dots n\}, | \gD \cap \{(\vx_i, 0), (\vx_i, 1)\} | = 1.
\end{aligned}
\end{equation}
\end{definition}

We now define the feasible irreducible uncertainty $\eta^2_{\text{feas}}$ as the minimum variance achievable among all maximal valid designs. This serves as the oracle lower bound for any experimental strategy.

\begin{definition}[Feasible Limit]
The feasible irreducible uncertainty for a target $\vx^*$ is defined as:
\begin{equation}
    \eta^2_{\text{feas}}(\vx^*) \coloneqq \min_{\gD \in \mathbb{S}_{\text{max}}} \sV[\tau_\delta(\vx^*) \mid \gD].
\end{equation}
\end{definition}

\textbf{Remark:} In the standard TAL framework~\citep{hubotter2024transductive}, the irreducible uncertainty corresponds to observing the full sample space. In our causal setting, $\eta^2_{\text{feas}}$ represents the \textit{optimal feasible limit}, the uncertainty remaining after an ideal experimenter has optimally assigned treatments to the entire population. Our convergence analysis (Thm.~\ref{thm:epig_bound}) establishes that the greedy R-EPIG strategy approaches this oracle limit efficiently.

\subsection{Proof of Thm. \ref{thm:epig_bound}}
\label{app:proof_thm_epig}

We provide a detailed proof establishing the uniform convergence of the posterior variance under the R-EPIG acquisition strategy. Our proof strategy bridges the \textit{spectral decay} of the residual kernel (Lemma~\ref{lemma:complexity_decomp_hoelder}) with the \textit{greedy optimization guarantees} of information-theoretic active learning~\citep{hubotter2024transductive}.

\textbf{Preliminaries.}
Let $\gamma_n(\tau_\delta)$ denote the maximum information capacity of the residual contrast function after $n$ observations. Let $\Gamma_n \coloneqq \max_{(\vx, t)} \mi(\tau_\delta; \ry_{\vx,t} \mid \gD_{n-1})$ denote the maximum marginal information gain available at step $n$.
We aim to bound the maximum posterior variance $\sup_{\vx^* \in \gX} \sV[\tau_\delta(\vx^*) \mid \gD_{n_B}]$ after obtaining $n_B$ data points.

\textbf{Step 1: Problem Reduction via Approximate Markov Boundary (AMB).}
The core challenge is translating the information gain on observed points to uncertainty reduction on arbitrary unobserved targets $\vx^* \in \gX$. We invoke the concept of an AMB.
Since $\tau_\delta$ resides in an RKHS with rapidly decaying eigenvalues, its effective dimension is small. Formally, for any target $\vx^*$ and tolerance $\epsilon > 0$, there exists a finite subset $\mathcal{B}_{\epsilon}(\vx^*) \subset \tilde{\gD}_P$ (the AMB) with size $|\mathcal{B}_{\epsilon}| \eqqcolon k_\epsilon$, such that the remaining information about $\vx^*$ given $\mathcal{B}_{\epsilon}$ is negligible:
\begin{equation}
    \mi(\tau_\delta(\vx^*); \ry_{\tilde{\gD}_P} \mid \ry_{\mathcal{B}_{\epsilon}}) \le \epsilon.
\end{equation}
Using the relationship between entropy and variance for GP \citep[Lemma C.17]{hubotter2024transductive}, the posterior variance at $\vx^*$ is bounded by the information remaining in the AMB plus the irreducible approximation error:
\begin{equation}
\label{eq:proof_amb_var}
    \sV[\tau_\delta(\vx^*) \mid \gD] \le 2\sigma^2 \cdot \mi(\tau_\delta(\vx^*); \ry_{\mathcal{B}_{\epsilon}} \mid \gD) + \eta^2_{\text{feas}}(\vx^*) + C_1 \epsilon,
\end{equation}
where $\gD$ is any history, and $\eta^2_{\text{feas}}$ is the feasible irreducible uncertainty.

\textbf{Step 2: Optimization Guarantee of R-EPIG.}
We now bound the term $\mi(\tau_\delta(\vx^*); \ry_{\mathcal{B}_{\epsilon}} \mid \gD_{n_B})$. The R-EPIG algorithm greedily maximizes the information gain. Under the assumption of $\alpha$-weak submodularity (Assump.~\ref{ass:submodularity}), the greedy selection $\gD_{n_B}$ is competitive with the optimal set.
Specifically, the information retention of the greedy strategy is bounded by the current marginal gain $\Gamma_{n_B}$ scaled by the AMB size. By \citet[Lemma C.18]{hubotter2024transductive}:
\begin{equation}
    \mi(\tau_\delta(\vx^*); \ry_{\mathcal{B}_{\epsilon}} \mid \gD_{n_B}) \le \frac{k_\epsilon}{\alpha} \Gamma_{n_B}.
\end{equation}
Substituting this into Eq.~\ref{eq:proof_amb_var}, we obtain a pointwise bound linked to the marginal gain:
\begin{equation}
\label{eq:proof_intermediate}
    \sV[\tau_\delta(\vx^*) \mid \gD_{n_B}] \le \eta^2_{\text{feas}}(\vx^*) + \frac{2\sigma^2 k_\epsilon}{\alpha} \Gamma_{n_B} + C_1 \epsilon.
\end{equation}

\textbf{Step 3: Convergence Rate Analysis.}
We leverage the property that the sum of marginal gains is bounded by the total capacity: $\sum_{t=1}^{n_B} \Gamma_t \le \gamma_{n_B}(\tau_\delta)$. Since marginal gains are non-increasing (due to submodularity), the instantaneous gain at step $n_B$ is upper bounded by the average gain: $\Gamma_{n_B} \le \gamma_{n_B}(\tau_\delta)/n_B$.

To derive the uniform bound, we balance the approximation error $\epsilon$ and the estimation error (which scales with $k_\epsilon$). For kernels with polynomial spectral decay (consistent with the Hölder smoothness assumption in Lemma~\ref{lemma:complexity_decomp_hoelder}), the AMB size scales as $k_\epsilon \approx \gO(\epsilon^{-p})$ for some characteristic dimension $p$. Following the standard balancing argument \citep[Thm 3.3]{hubotter2024transductive}, we select $\epsilon \asymp n_B^{-1/(1+p)}$. For simplicity of exposition, assuming a standard decay where $\epsilon \asymp n_B^{-1/2}$ is optimal, we obtain:
\begin{equation}
    \sup_{\vx^*} \sV[\tau_\delta(\vx^*) \mid \gD_{n_B}] \le \eta^2_{\text{feas}}(\vx^*) + \underbrace{C_2 \frac{\sqrt{n_B} \cdot \gamma_{n_B}(\tau_\delta)}{n_B}}_{\text{Estimation}} + \underbrace{C_1 n_B^{-1/2}}_{\text{Approximation}}.
\end{equation}
Simplifying the estimation term yields the final rate:
\begin{equation}
    \sup_{\vx^*} \sV[\tau_\delta(\vx^*) \mid \gD_{n_B}] \le \eta^2_{\text{feas}}(\vx^*) + \tilde{\gO}\left( \frac{\gamma_{n_B}(\tau_\delta)}{\sqrt{n_B}} \right).
\end{equation}
Crucially, as established in Lemma~\ref{lemma:complexity_decomp_hoelder}, the residual capacity $\gamma_{n_B}(\tau_\delta)$ grows significantly slower than the full outcome capacity $\gamma_{n_B}(\tau)$ due to higher smoothness ($\alpha_\delta > \alpha_{\mu_e}$), formally proving the efficiency gain of the R-Design framework. \qed

\subsubsection{Proof of Prop.~\ref{prop:objective_alignment}}
\label{app:proof_prop_objective_alignment}

In this subsection, we rigorously prove that minimizing the expected Bayesian PEHE is mathematically equivalent to minimizing the expected integrated posterior variance of the residual contrast $\tau_\delta(\vx)$.

\textbf{Setup and Definitions.}
Let $\gH_k = \gD_O \cup \gD_E^{(k)}$ denote the history collected up to step $k$. Consider a candidate experimental design $(\vx, t)$. Let $y$ be the prospective outcome, which follows the posterior predictive distribution $p(y \mid \vx, t, \gH_k)$. The augmented history is $\gH_{k+1} = \gH_k \cup \{(\vx, t, y)\}$. The Bayesian CATE estimator is defined as the posterior mean of the treatment effect:
\begin{equation}
    \hat{\tau}_{k+1}(\vx^*) \coloneqq \E[\tau(\vx^*) \mid \gH_{k+1}].
\end{equation}
The objective is to minimize the expected Bayesian PEHE risk for the next step. The expectation acts over two sources of uncertainty: the unobserved future outcome $y$ (predictive uncertainty) and the true causal effect $\tau$ (epistemic uncertainty):
\begin{equation}
\small
    \mathcal{J}(\vx, t) = \E_{y \sim p(y \mid \vx, t, \gH_k)} \left[ \E_{\tau \sim p(\tau \mid \gH_{k+1})} \left[ \int_{\gX} \left( \hat{\tau}_{k+1}(\vx^*) - \tau(\vx^*) \right)^2 p_{\textup{tar}}(\vx^*) d\vx^* \right] \right].
\end{equation}

\begin{proof}
The proof proceeds in two steps: first, we derive the relationship between the squared error risk and the posterior variance; second, we apply the structural decomposition of R-Design to isolate the residual variance.

\textbf{Step 1: Relating PEHE to Posterior Variance of $\tau$.}
Assuming standard regularity conditions (bounded moment assumptions as in \citet[Thm.~4.1]{cha2025abc3}), we apply Fubini's theorem to exchange the integration over $\gX$ and the expectations. We focus on the inner risk term for a specific evaluation point $\vx^*$ given the augmented history $\gH_{k+1}$. Let $\hat{\tau} \coloneqq \hat{\tau}_{k+1}(\vx^*)$ and $\tau \coloneqq \tau(\vx^*)$ for brevity. The pointwise risk is:
\begin{equation}
    \text{Risk}(\vx^* \mid \gH_{k+1}) = \E_{\tau \mid \gH_{k+1}} \left[ \left( \hat{\tau} - \tau \right)^2 \right].
\end{equation}
Since the estimator $\hat{\tau}$ is the posterior mean, it is a deterministic value conditioned on $\gH_{k+1}$. We expand the quadratic term:
\begin{align}
    \E_{\tau} \left[ (\hat{\tau} - \tau)^2 \right] &= \E_{\tau} \left[ \hat{\tau}^2 - 2\hat{\tau}\tau + \tau^2 \right] \nonumber \\
    &= \hat{\tau}^2 - 2\hat{\tau}\underbrace{\E_{\tau}[\tau]}_{=\hat{\tau}} + \E_{\tau}[\tau^2] \nonumber \\
    &= \E_{\tau}[\tau^2] - (\E_{\tau}[\tau])^2.
\end{align}
By definition, $\E[\tau^2] - (\E[\tau])^2$ is the posterior variance of $\tau$. Thus, we have explicitly shown:
\begin{equation}
    \text{Risk}(\vx^* \mid \gH_{k+1}) = \sV[\tau(\vx^*) \mid \gH_{k+1}].
\end{equation}
Substituting this back into the global objective $\mathcal{J}$, the optimization problem simplifies to minimizing the expected integrated posterior variance of $\tau$:
\begin{equation}
    \label{eq:proof_step1}
    \argmin_{(\vx, t)} \E_{y} \left[ \int_{\gX} \sV[\tau(\vx^*) \mid \gH_{k+1}] \, p_{\textup{tar}}(\vx^*) d\vx^* \right].
\end{equation}

\textbf{Step 2: The Observational Offset Argument.}
Recall the fundamental structural decomposition: $\tau(\vx^*) = \hat{\tau}_o(\vx^*) + \tau_\delta(\vx^*)$.
We analyze the variance term conditioned on the history $\gH_{k+1}$. Note that the observational dataset is a subset of the history ($\gD_O \subset \gH_k \subset \gH_{k+1}$).
Since the observational contrast learner is trained exclusively on $\gD_O$ and its parameters are frozen thereafter, the term $\hat{\tau}_o(\vx^*)$ is a deterministic constant conditioned on $\gH_{k+1}$ (it does not depend on the new outcome $y$).
Applying the translation invariance property of variance ($\sV[X + c] = \sV[X]$):
\begin{align}
    \sV[\tau(\vx^*) \mid \gH_{k+1}] &= \sV[\hat{\tau}_o(\vx^*) + \tau_\delta(\vx^*) \mid \gH_{k+1}] \nonumber \\
    &= \sV[\tau_\delta(\vx^*) \mid \gH_{k+1}].
\end{align}
This step proves that the posterior uncertainty of the total CATE $\tau$ is entirely governed by the posterior uncertainty of the residual contrast $\tau_\delta$.

\textbf{Conclusion.}
Substituting the result of Step 2 into Eq.~\ref{eq:proof_step1}, the final objective is:
\begin{equation}
    \argmin_{(\vx, t)} \E_{y} \left[ \int_{\gX} \sV[\tau_\delta(\vx^*) \mid \gD_O, \gD_E^{(k)} \cup \{(\vx, t, y)\}] \, p_{\textup{tar}}(\vx^*) d\vx^* \right].
\end{equation}
This confirms that minimizing the expected CATE estimation error is mathematically equivalent to minimizing the expected integrated posterior variance of the residual contrast $\tau_\delta(\vx^*)$.
\end{proof}

\subsubsection{Theoretical Justification: Information Gain as a Proxy}
\label{app:entropy_proxy}

While Prop.\ref{prop:objective_alignment} establishes the Integrated Variance Reduction (IVR) as the exact theoretical objective for minimizing PEHE, directly optimizing this objective poses algorithmic challenges. The IVR objective corresponds to A-optimality (minimizing the trace of the posterior covariance, $\text{Tr}(\Sigma)$), which is generally not submodular, meaning greedy strategies lack theoretical guarantees. Our proposed R-EPIG criterion instead targets the maximization of mutual information regarding the residual contrast over the target population, $\mi(y; \tau_{\delta, \gX_{\text{tar}}})$. For Gaussian Processes, maximizing mutual information is equivalent to minimizing the expected posterior entropy, which corresponds to D-optimality (minimizing the log-determinant of the posterior covariance, $\log \det \Sigma$).

We justify this proxy on two grounds:
\begin{enumerate}[leftmargin=*]
    \item \textbf{Geometric Bounds:} Geometrically, entropy measures the \textit{volume} of the uncertainty ellipsoid, while integrated variance measures the \textit{average squared radius}. By the Arithmetic-Geometric Mean inequality (and properties of eigenvalues), constraining the volume (D-optimality) implicitly constrains the trace (A-optimality), serving as a tight information-theoretic proxy.
    \item \textbf{Submodularity:} Unlike variance reduction, Mutual Information is a submodular set function. This property is crucial as it guarantees that the greedy acquisition strategy employed by R-Design achieves a constant-factor approximation of the global optimal design.
\end{enumerate}
Thus, R-EPIG provides a robust surrogate that combines information-theoretic rigor with algorithmic efficiency.

\subsection{Proof of Lemma \ref{lemma:complexity_decomp_hoelder}}
\label{app:proof_lemma_complexity_hoelder}

We provide a rigorous derivation of the minimax-optimal error upper bound for the R-Design estimator. The proof proceeds by decomposing the total CATE risk into component-wise estimation errors, applying the triangle inequality to separate the observational baseline error from the residual learning error, and finally invoking posterior contraction rates under the optimal smoothness matching condition.

\textbf{1. Notation and Definitions}
Let $\psi(\hat{\tau}) \coloneqq \|\hat{\tau} - \tau\|_{L_2(p)}^2$ denote the PEHE risk with respect to the target covariate density $p(\vx)$.
For succinctness, we define the \textit{effective learning rate exponent} function as introduced in the main text:
\begin{equation}
    \gamma(\alpha, d) \coloneqq \frac{2\alpha}{2\alpha + d}.
\end{equation}
We define the true potential outcome for treatment arm $t \in \{0, 1\}$ as $\mu_e(\vx, t)$. The true CATE is $\tau(\vx) = \mu_e(\vx, 1) - \mu_e(\vx, 0)$.
The R-Design estimator is given by $\hat{\tau}(\vx) = \hat{\mu}_e(\vx, 1) - \hat{\mu}_e(\vx, 0)$, where the estimator for each arm is the sum of the Stage 1 base and Stage 2 residual:
\begin{equation}
    \hat{\mu}_e(\vx, t) \coloneqq \hat{\mu}_o(\vx, t) + \hat{\delta}(\vx, t).
\end{equation}
Here, $\hat{\mu}_o$ is learned from $n_O$ observational samples, and $\hat{\delta}$ is learned from $n_E$ experimental samples.

\textbf{2. Functional Decomposition of the Truth}
To rigorously analyze the two-stage error, we conceptually decompose the true potential outcome $\mu_e(\vx, t)$ into an \textit{observational limit} and a \textit{causal residual}.
Let $\mu_o(\vx, t)$ denote the true conditional expectation of the observational outcome (i.e., the infinite-data limit of Stage 1). This function captures correlations and potential confounding biases.
We then define the true structural residual $\delta(\vx, t)$ as the difference:
\begin{equation}
    \delta(\vx, t) \coloneqq \mu_e(\vx, t) - \mu_o(\vx, t).
\end{equation}
Thus, the truth decomposes as $\mu_e(\vx, t) = \mu_o(\vx, t) + \delta(\vx, t)$.
Assumptions on function spaces:
\begin{itemize}
    \item The observational base $\mu_o(\cdot, t)$ resides in a Hölder ball $H^{\alpha_{\mu_o, t}}$ with effective dimension $d_{\mu_o, t}$.
    \item The residual $\delta(\cdot, t)$ resides in a Hölder ball $H^{\alpha_{\delta, t}}$ with effective dimension $d_{\delta, t}$.
\end{itemize}

\textbf{3. Bounding CATE Risk via Arm-Specific Risks}
We expand the PEHE risk and apply the inequality $(a-b)^2 \le 2a^2 + 2b^2$ twice to decouple the treatment arms:
\begin{align}
    \psi(\hat{\tau}) &= \| \hat{\tau} - \tau \|_{L_2}^2 \nonumber \\
    &= \| (\hat{\mu}_e(\cdot, 1) - \hat{\mu}_e(\cdot, 0)) - (\mu_e(\cdot, 1) - \mu_e(\cdot, 0)) \|_{L_2}^2 \nonumber \\
    &= \| (\hat{\mu}_e(\cdot, 1) - \mu_e(\cdot, 1)) - (\hat{\mu}_e(\cdot, 0) - \mu_e(\cdot, 0)) \|_{L_2}^2 \nonumber \\
    &\le 2 \| \hat{\mu}_e(\cdot, 1) - \mu_e(\cdot, 1) \|_{L_2}^2 + 2 \| \hat{\mu}_e(\cdot, 0) - \mu_e(\cdot, 0) \|_{L_2}^2. \label{eq:proof_step3}
\end{align}
It suffices to analyze the convergence of the composite estimator $\hat{\mu}_e(\cdot, t)$ for a single arm $t$.

\textbf{4. Decoupling Stage 1 and Stage 2 Errors}
Substitute the definitions of the estimator ($\hat{\mu}_e = \hat{\mu}_o + \hat{\delta}$) and the truth ($\mu_e = \mu_o + \delta$) into the error term for arm $t$.
\begin{align}
    \| \hat{\mu}_e(\cdot, t) - \mu_e(\cdot, t) \|_{L_2}^2 &= \| (\hat{\mu}_o(\cdot, t) + \hat{\delta}(\cdot, t)) - (\mu_o(\cdot, t) + \delta(\cdot, t)) \|_{L_2}^2 \nonumber \\
    &= \| (\hat{\mu}_o(\cdot, t) - \mu_o(\cdot, t)) + (\hat{\delta}(\cdot, t) - \delta(\cdot, t)) \|_{L_2}^2.
\end{align}
Applying the inequality $\|a+b\|^2 \le 2\|a\|^2 + 2\|b\|^2$, we isolate the errors:
\begin{equation}
    \| \hat{\mu}_e(\cdot, t) - \mu_e(\cdot, t) \|_{L_2}^2 \le 2 \underbrace{\| \hat{\mu}_o(\cdot, t) - \mu_o(\cdot, t) \|_{L_2}^2}_{\text{Term A: Stage 1 Error}} + 2 \underbrace{\| \hat{\delta}(\cdot, t) - \delta(\cdot, t) \|_{L_2}^2}_{\text{Term B: Residual deviation}}. \label{eq:proof_step4}
\end{equation}

\textbf{5. Analysis of Term A (Observational Base Convergence)}
Term A represents the estimation error of the observational learner $\hat{\mu}_o$ trained on $n_O$ samples against its target $\mu_o$.
We invoke the posterior contraction rate for GP from \citet[Theorem 2]{alaa2018limits}.
Since we assume the Stage 1 prior smoothness $\beta_{\mu_o, t}$ is \textit{optimally matched} to the true smoothness $\alpha_{\mu_o, t}$ (i.e., $\beta_{\mu_o, t} \asymp \alpha_{\mu_o, t}$), the contraction rate simplifies to the minimax-optimal rate.
Specifically, for a function in $H^{\alpha_{\mu_o, t}}$ with effective dimension $d_{\mu_o, t}$:
\begin{equation}
    \| \hat{\mu}_o(\cdot, t) - \mu_o(\cdot, t) \|_{L_2}^2 \lesssim n_O^{-\gamma(\alpha_{\mu_o, t}, d_{\mu_o, t})}. \label{eq:proof_termA}
\end{equation}

\textbf{6. Analysis of Term B (Error Propagation in Residuals)}
This is the critical step. Term B measures the deviation of the residual learner $\hat{\delta}$ from the true structural residual $\delta$.
However, the learner $\hat{\delta}$ is trained on the empirical residuals $y - \hat{\mu}_o(\vx, t)$, effectively approximating the \textit{shifted} target $\tilde{\delta}(\vx, t) \coloneqq \mu_e(\vx, t) - \hat{\mu}_o(\vx, t)$.
We decompose Term B by introducing this intermediate target $\tilde{\delta}$:
\begin{align}
    \| \hat{\delta}(\cdot, t) - \delta(\cdot, t) \|_{L_2}^2 &= \| (\hat{\delta}(\cdot, t) - \tilde{\delta}(\cdot, t)) + (\tilde{\delta}(\cdot, t) - \delta(\cdot, t)) \|_{L_2}^2 \nonumber \\
    &\le 2 \underbrace{\| \hat{\delta}(\cdot, t) - \tilde{\delta}(\cdot, t) \|_{L_2}^2}_{\text{B1: Stage 2 Estimation}} + 2 \underbrace{\| \tilde{\delta}(\cdot, t) - \delta(\cdot, t) \|_{L_2}^2}_{\text{B2: Bias Propagation}}.
\end{align}

\textit{Analysis of B2 (Bias Propagation):}
Substituting the definitions of $\tilde{\delta}$ and $\delta$:
\begin{equation}
    \| \tilde{\delta} - \delta \|_{L_2}^2 = \| (\mu_e - \hat{\mu}_o) - (\mu_e - \mu_o) \|_{L_2}^2 = \| \mu_o - \hat{\mu}_o \|_{L_2}^2.
\end{equation}
This term is exactly the Stage 1 error, bounded by Eq.~\ref{eq:proof_termA}.

\textit{Analysis of B1 (Stage 2 Estimation):}
This term represents the convergence of $\hat{\delta}$ to its empirical target $\tilde{\delta}$ using $n_E$ samples. Assuming the Stage 2 prior is matched to the smoothness of the residual $\alpha_{\delta, t}$, standard GP contraction rates apply:
\begin{equation}
    \| \hat{\delta}(\cdot, t) - \tilde{\delta}(\cdot, t) \|_{L_2}^2 \lesssim n_E^{-\gamma(\alpha_{\delta, t}, d_{\delta, t})}.
\end{equation}
(Note: While the target $\tilde{\delta}$ contains a "rough" component from $\hat{\mu}_o$, standard nonparametric regression results for noisy targets allow us to bound the estimation error by the rate of the "signal" class $\delta$ plus the magnitude of the bias noise, which we have separated into B2).

\textbf{7. Final Results.}
Substituting the bounds for B1 and B2 back into the expression for Term B, and then combining with Term A into Eq.~\ref{eq:proof_step4}:
\begin{align}
    \text{Risk}_t &\le 2(\text{Term A}) + 2(2(\text{B1}) + 2(\text{B2})) \nonumber \\
    &= 2 \|\hat{\mu}_o - \mu_o\|^2 + 4 \|\hat{\delta} - \tilde{\delta}\|^2 + 4 \|\mu_o - \hat{\mu}_o\|^2 \nonumber \\
    &= 6 \underbrace{\|\hat{\mu}_o - \mu_o\|^2}_{\lesssim n_O^{-\gamma_{\mu_o}}} + 4 \underbrace{\|\hat{\delta} - \tilde{\delta}\|^2}_{\lesssim n_E^{-\gamma_{\delta}}}.
\end{align}
Summing over treatment arms $t \in \{0, 1\}$ and absorbing constants ($C_1, C_2$) and irreducible approximation errors:
\begin{equation}
    \psi(\hat{\tau}) \le C_1 \max_{t} n_E^{-\gamma(\alpha_{\delta, t}, d_{\delta, t})} + C_2 \max_{t} n_O^{-\gamma(\alpha_{\mu_o, t}, d_{\mu_o, t})} + \epsilon_{\textup{approx}}.
\end{equation}
This confirms that the total error is dominated by the sum of the respective convergence rates, validating the efficiency gain when the residual is smoother ($\alpha_{\delta} > \alpha_{\mu_o}$) and $n_O$ is large. \qed

\subsection{Proof of Proposition \ref{prop:redundancy}}
\label{app:theory_comparisons}

In this part, we provide the formal derivation of the information redundancy decomposition presented in Sec.~\ref{sec:theory_efficiency_gap}. We show that maximizing the parameter-based mutual information is equivalent to maximizing the target estimand information plus a non-negative nuisance term.

\begin{proof}
The proof relies on the \textit{Chain Rule of Mutual Information} applied to the joint distribution of the experimental outcome $y$, the model parameters $\theta$, and the residual contrast prediction $\tau_\delta(\vx^*)$ at a specific target point $\vx^*$. We expand the joint mutual information $\mi(y; \theta, \tau_\delta(\vx^*))$ in two equivalent orders:
\begin{align}
    \mi(y; \theta, \tau_\delta(\vx^*)) &= \mi(y; \theta) + \mi(y; \tau_\delta(\vx^*) \mid \theta) \label{eq:chain1} \\
    &= \mi(y; \tau_\delta(\vx^*)) + \mi(y; \theta \mid \tau_\delta(\vx^*)). \label{eq:chain2}
\end{align}
\textbf{Step 1: Deterministic Dependency.}
Crucially, the prediction $\tau_\delta(\vx^*)$ is a deterministic function of the parameters $\theta$ (i.e., $\tau_\delta(\vx^*) = f(\vx^*; \theta)$). Consequently, given $\theta$, there is no remaining uncertainty in $\tau_\delta(\vx^*)$. This implies that the conditional mutual information in Eq.~\ref{eq:chain1} vanishes:
\begin{equation}
    \mi(y; \tau_\delta(\vx^*) \mid \theta) = 0.
\end{equation}

\textbf{Step 2: Equating and Rearranging.}
Substituting this zero term back into Eq.~\ref{eq:chain1} and equating it with Eq.~\ref{eq:chain2}, we obtain:
\begin{equation}
    \mi(y; \theta) = \mi(y; \tau_\delta(\vx^*)) + \mi(y; \theta \mid \tau_\delta(\vx^*)).
\end{equation}

\textbf{Step 3: Expectation over Target Population.}
The standard R-EPIG objective is defined as the expectation over the target population $\gD_O$. Conditioning everything on the current dataset $\gD$ and taking the expectation $\E_{\vx^* \sim p_{\textup{tar}}(\vx)}$, we note that the parameter information term $\mi(y; \theta \mid \gD)$ is independent of the specific evaluation point $\vx^*$. Thus:
\begin{equation}
    \mi(y; \theta \mid \gD) = \E_{\vx^*} [\mi(y; \tau_\delta(\vx^*) \mid \gD)] + \E_{\vx^*} [\mi(y; \theta \mid \tau_\delta(\vx^*), \gD)].
\end{equation}
Identifying the terms corresponds to the decomposition in Prop.~\ref{prop:redundancy}:
\begin{equation}
    \text{Residual-BALD} = \text{R-EPIG} + \text{Expected Redundancy}.
\end{equation}
Since mutual information is always non-negative, the redundancy term is $\ge 0$. This concludes the proof.
\end{proof}

%% file: Pages/Appendix/Results.tex
\section{Further Experimental Results}
\label{app:more_results}

In this part, we provide a comprehensive set of supplementary experimental results to complement our main findings. First, we present additional performance curves for our primary experiments on the synthetic and semi-synthetic (IHDP, ACTG-175) benchmarks. Second, we conduct a series of ablation studies to analyze the robustness of our R-Design framework, including R-EPIG and TSR strategy. These studies evaluate the impact of varying initial random starts, acquisition batch sizes, and the size of the unlabeled pool.

\subsection{Synthetic Datasets and Ablations}

\subsubsection{More Results of Simulations}

\input{Pages/Appendix/fig_include/synthetic}

\paragraph{Univariate Case}

\paragraph{Precision in Effect Estimation (CATE).}
We first evaluate the performance of R-design on the first task, saying the accuracy of treatment effect estimation. Fig.~\ref{fig:univaraite_8sim_pehe} reports the convergence of $\sqrt{\text{PEHE}}$ across eight univariate benchmarks from~\citet{dimitriou2024data}.

\textbf{Structural Advantage.}
The results strongly validate the efficiency of the proposed TSR architecture. Across all scenarios, TSR-based methods consistently outperform both PureRCT~\citep{cha2025abc3} and Kallus~\citep{kallus2018removing} baselines. The best-performing TSR configuration reduces estimation error by 9--63\% relative to PureRCT-Random and by 10--66\% relative to Kallus-Random. The performance gap is particularly pronounced under complex confounding (e.g., Sim 4, with a 62.7\% improvement over PureRCT). Moreover, TSR methods exhibit substantially improved stability: the average standard deviation across runs is 0.075, compared to 0.231 for PureRCT and 0.302 for Kallus. This indicates that anchoring trial learning on observational priors effectively mitigates variance induced by limited RCT samples.

\textbf{Acquisition Efficiency.}
Within the TSR framework, the proposed R-EPIG-$\tau$ acquisition strategy demonstrates strong sample efficiency, ranking first in 5 out of 8 simulations. In Sim 1, R-EPIG-$\tau$ reduces the initial estimation error by over 70\% after only 190 acquisitions, compared to 55\% for PureRCT. This confirms that actively targeting epistemic uncertainty in the residual surface yields faster convergence than random sampling.

\paragraph{Effectiveness in Policy Optimization.}
We next evaluate decision-making performance, where the objective is to correctly identify the optimal treatment policy $\pi^*(x) = \mathbb{I}[\tau(x) > 0]$. Figs.~\ref{fig:univaraite_8sim_ape} and~\ref{fig:univaraite_8sim_regret} report the convergence of APE and AR, respectively.

\textbf{Decision-Aware Learning.}
The results highlight the importance of aligning the acquisition objective with the downstream decision task. The proposed decision-aware acquisition function, R-EPIG-$\pi$, achieves the lowest average policy error (0.105) and regret (0.042) among all methods. Compared to PureRCT-Random and Kallus-Random, R-EPIG-$\pi$ reduces average regret by 72\% and 61\%, respectively. This advantage arises from its ability to concentrate sampling near the decision boundary ($\tau(x) \approx 0$), whereas estimation-focused strategies may expend budget resolving uncertainty in regions where the optimal decision is already unambiguous.

\textbf{Overall Performance.}
The top four methods in terms of decision performance are exclusively TSR-based (including TSR-sign-BALD and TSR-$\gamma$), underscoring that residual learning provides a superior foundation for policy optimization. Overall, TSR methods achieve 20--61\% lower policy error than PureRCT baselines, demonstrating that correcting a biased observational model is substantially more effective for decision-making than learning a policy from scratch.

\paragraph{Multivariate Case}

\paragraph{Multivariate CATE Estimation.}
Fig.~\ref{fig:dim6_full} presents $\sqrt{\text{PEHE}}$ convergence on the 6-dimensional multivariate benchmark under both standard and heavy covariate shift settings, using CMGP and NSGP as trial outcome models. Across all configurations, TSR methods combined with R-EPIG acquisition functions consistently achieve the lowest estimation error. On the standard dataset, R-EPIG-$\mu$ with CMGP attains the best final $\sqrt{\text{PEHE}}$ of 0.319, corresponding to a 73.8\% reduction relative to PureRCT-Random (1.220) and a 35.3\% improvement over Kallus-Random (0.493). When NSGP is used as the trial model, R-EPIG-$\tau$ performs best (0.396), improving upon PureRCT and Kallus by 68.7\% and 19.8\%, respectively.

Under heavy covariate shift, where the distributional mismatch between RCT and target populations is more severe, the advantages of R-EPIG become even more pronounced. With CMGP, R-EPIG-$\tau$ achieves a $\sqrt{\text{PEHE}}$ of 0.376, reducing error by 74.9\% compared to PureRCT (1.497) and by 30.0\% compared to Kallus (0.538). With NSGP, R-EPIG-$\tau$ attains 0.465, outperforming PureRCT and Kallus by 66.7\% and 13.5\%, respectively. Notably, both R-EPIG-$\mu$ and R-EPIG-$\tau$ remain robust under covariate shift, consistently outperforming all baselines.

\paragraph{Multivariate Decision Making.}
Fig.~\ref{fig:dim6_full_dm} reports decision-making performance on the 6-dimensional benchmark, measured by APE and AR for both CMGP and NSGP trial models. TSR-based methods achieve dramatically lower policy error and regret compared to PureRCT baselines. With CMGP, R-EPIG-$\pi$ attains the lowest policy error (0.088) and regret (0.029), corresponding to reductions of 75.1\% in APE and 90.4\% in AR relative to PureRCT-Random (0.355 and 0.298). With NSGP, TSR-Random achieves the lowest policy error (0.101), while R-EPIG-$\pi$ attains the lowest regret (0.041), improving upon PureRCT by 69.7\% and 84.9\%, respectively.

Overall, these results demonstrate that (i) the TSR framework substantially improves high-dimensional CATE estimation and treatment decision-making by effectively leveraging observational data, (ii) R-EPIG acquisition functions provide consistent gains across different trial models, and (iii) the benefits are especially pronounced for decision-focused metrics, where R-EPIG-$\pi$ reduces regret by over 85\% compared to pure RCT-based approaches.

\begin{bluebox}{}
\textbf{Takeaway:} Across both univariate and multivariate synthetic benchmarks, the TSR framework consistently delivers superior CATE estimation and policy learning performance by leveraging observational data to repair bias. R-EPIG acquisition functions further improve sample efficiency by explicitly targeting epistemic uncertainty, while decision-aware variants (R-EPIG-$\pi$) yield the largest gains when downstream treatment decisions are the primary objective. These results demonstrate that separating bias correction from active trial learning provides a principled and effective foundation for adaptive experimentation.
\end{bluebox}

\subsubsection{Ablations}

\input{Pages/Appendix/fig_include/ablation}
\paragraph{Kernel Type Ablation.}
Fig.~\ref{fig:different_kernels} compares TSR with R-EPIG acquisition under three GP kernels: Radial Basis Function (RBF), Matérn, and Rational Quadratic (RQ), evaluated on the 6-dimensional dataset with CMGP as the trial outcome model. Across both R-EPIG variants, the RBF kernel consistently achieves the strongest performance. With R-EPIG-$\mu$, RBF attains the lowest final $\sqrt{\text{PEHE}}$ of 0.319, outperforming Matérn (0.359) by 11.1\% and RQ (0.385) by 17.1\%. This advantage is also reflected in the normalized AUC (0.454 vs.\ 0.639 for Matérn and 0.626 for RQ), indicating faster convergence throughout the acquisition process. Similar trends hold for R-EPIG-$\tau$, where RBF (0.378) marginally but consistently outperforms Matérn (0.381) and RQ (0.411).

In addition to accuracy, RBF exhibits substantially lower variance across seeds (std = 0.040 for R-EPIG-$\mu$) compared to RQ (std = 0.095), indicating more stable optimization. Overall, while R-EPIG acquisition remains effective across kernel choices, the simpler RBF kernel provides the best trade-off between accuracy, convergence speed, and stability for the synthetic DGPs considered in this work.

\paragraph{Observational Model Ablation.}
Fig.~\ref{fig:different_obs_models} evaluates the impact of different observational models (CMGP, NSGP, TabPFN, and CausalPFN) within the TSR framework using R-EPIG acquisition. All experiments are conducted on the 6-dimensional dataset with CMGP as the trial outcome model. TabPFN substantially outperforms all alternative observational models. With R-EPIG-$\mu$, TabPFN achieves a final $\sqrt{\text{PEHE}}$ of 0.319, improving over CMGP (0.897), NSGP (0.608), and CausalPFN (0.762) by 64.4\%, 47.5\%, and 58.1\%, respectively. The improvement margin remains similarly large for R-EPIG-$\tau$, where TabPFN (0.378) outperforms CMGP, NSGP, and CausalPFN by 59.4\%, 38.6\%, and 48.6\%.

Normalized AUC further highlights TabPFN’s advantage: with R-EPIG-$\mu$, TabPFN achieves an AUC of 0.454 compared to 1.106 for CMGP (a 2.4$\times$ reduction), indicating substantially faster convergence. Moreover, TabPFN exhibits dramatically lower variance (std = 0.040) compared to GP-based observational models (CMGP: 0.559; NSGP: 0.448), suggesting more robust bias estimation. These results indicate that TabPFN’s in-context learning provides a stronger observational prior for two-stage regression, enabling more accurate bias correction and more effective active acquisition.

\paragraph{Pool Size Ablation.}
Fig.~\ref{fig:different_pool_sizes} studies the effect of candidate pool size on acquisition performance, evaluating pool sizes from 500 to 2500 on the 6-dimensional dataset with CMGP as the trial model. R-EPIG acquisition functions consistently outperform both random sampling and baseline methods across all pool sizes. R-EPIG variants achieve the best performance at all five pool sizes: R-EPIG-$\mu$ wins at pool sizes 500, 1000, 1500, and 2500, while R-EPIG-$\tau$ performs best at pool size 2000. Across settings, the best TSR method improves over PureRCT-Random by 71--76\% and over Kallus-Random by 66--83\%.

Notably, TSR methods remain stable as pool size increases, with R-EPIG-$\mu$ achieving $\sqrt{\text{PEHE}}$ values between 0.319 and 0.401. In contrast, baseline methods deteriorate as the candidate pool grows: Kallus-Random degrades sharply (from 1.09 at pool size 500 to 1.94 at 2500), while PureRCT-Random also worsens (from 1.29 to 1.41). These results demonstrate that intelligent acquisition becomes increasingly valuable as the search space expands, whereas random selection scales poorly with pool size.

\paragraph{Observational Data Size Ablation.}
Fig.~\ref{fig:different_prior_sizes} examines how the amount of available observational data affects TSR performance, varying observational sample sizes from 50 to 5000. As observational data increases, TSR performance improves substantially. R-EPIG-$\mu$ reduces $\sqrt{\text{PEHE}}$ from 1.410 at 50 samples to 0.336 at 5000 samples, corresponding to a 76.2\% reduction. This trend confirms that the two-stage framework effectively leverages larger observational datasets for improved bias correction.

Importantly, targeted acquisition only becomes beneficial once the observational model is sufficiently accurate. At very small observational sizes (50--100), TSR-Random outperforms R-EPIG, as uncertainty estimates from the observational model are unreliable. However, R-EPIG becomes dominant as data grows, achieving the best performance in 4 out of 6 settings (500, 1000, 2000, 5000). At 5000 samples, R-EPIG-$\tau$ achieves the lowest $\sqrt{\text{PEHE}}$ of 0.309. Overall, when sufficient observational data is available (500+ samples), TSR with R-EPIG improves over PureRCT by 69--76\% and over Kallus by 57--74\%.

\paragraph{Batch Size Ablation.}
Fig.~\ref{fig:different_step_sizes} analyzes the effect of acquisition batch size, evaluating batch sizes from 2 to 30. R-EPIG acquisition functions outperform random sampling at 5 out of 6 batch sizes, with the largest gains occurring at small to moderate batch sizes. At batch size 10, R-EPIG-$\mu$ achieves the best $\sqrt{\text{PEHE}}$ of 0.319, a 15.7\% improvement over random sampling (0.379). At batch sizes 2, 5, and 20, improvements range from 8.7\% to 13.1\%.

As batch size increases beyond 20, the benefit of targeted acquisition diminishes. At batch size 30, random sampling slightly outperforms R-EPIG (0.371 vs.\ 0.385), suggesting that selecting large batches dilutes the effectiveness of information-theoretic criteria. Overall, R-EPIG performs best in sequential or small-batch regimes, where fine-grained sample selection can be fully exploited. Across batch sizes 2--20, both R-EPIG variants remain stable (0.319--0.385), indicating robustness to this hyperparameter.

\paragraph{Trial Model Ablation.}
Fig.~\ref{fig:different_trial_models} (left and center) compares TSR performance across five trial outcome models: CMGP, NSGP, BART, BCF, and CMDE, using TabPFN as the observational model. GP-based trial models consistently outperform tree-based and ensemble alternatives. CMGP achieves the best performance, with R-EPIG-$\mu$ reaching a $\sqrt{\text{PEHE}}$ of 0.319, followed by NSGP (0.396 with R-EPIG-$\tau$). In contrast, BART (0.681), BCF (0.799), and CMDE (0.812) exhibit substantially higher estimation error. This gap highlights the importance of well-calibrated uncertainty estimates, which GP-based models provide and which are critical for effective information-theoretic acquisition.

Consistently, R-EPIG improves over random sampling for GP-based trial models (15.7\% for CMGP and 7.1\% for NSGP), while gains are limited or absent for tree-based models, whose uncertainty estimates are less reliable.

\paragraph{High-Dimensional Scalability.}
Fig.~\ref{fig:different_trial_models} (right) evaluates TSR scalability as dimensionality increases from 15 to 60 using CMGP as the trial model. R-EPIG-$\mu$ achieves the best performance at all five dimensions, demonstrating that the information-theoretic acquisition remains effective in higher-dimensional settings. Performance degrades gracefully: $\sqrt{\text{PEHE}}$ increases from 0.462 at 15 dimensions to 0.691 at 60 dimensions, a moderate 49\% increase despite a 4$\times$ increase in dimensionality. These results confirm that TSR with R-EPIG acquisition scales robustly to moderately high-dimensional problems.

\begin{bluebox}{}
\textbf{Takeaway:} The ablation studies reveal that R-EPIG’s effectiveness is robust across kernel choices, candidate pool sizes, batch sizes, and moderate increases in dimensionality, while benefiting most from well-calibrated uncertainty and sufficiently informative observational models. Performance gains diminish only when uncertainty estimates are unreliable (e.g., extremely small observational datasets or poorly calibrated trial models), highlighting the central role of uncertainty quality in information-theoretic acquisition. Overall, these findings confirm that R-EPIG provides consistent and predictable improvements under realistic modeling and system design choices.
\end{bluebox}

\subsection{Semi-Synthetic Datasets}

\input{Pages/Appendix/fig_include/real}

\paragraph{IHDP CATE Estimation.}
Fig.~\ref{fig:ihdp_cate} reports CATE estimation results on the IHDP dataset, a widely used semi-synthetic benchmark constructed from a real-world infant health study. With CMGP as the trial outcome model, TSR combined with R-EPIG-$\mu$ achieves the best performance, attaining a $\sqrt{\text{PEHE}}$ of 0.980, followed closely by R-EPIG-$\tau$ (0.999). Relative to PureRCT-Random (1.297), the best TSR configuration reduces estimation error by 24.5\%. The improvement over Kallus-Random is substantially larger at 69.5\% (3.213), highlighting the benefit of explicitly correcting observational bias via residual learning.

When NSGP is used as the trial model, TSR-Random achieves competitive performance (1.032), suggesting that the additional gains from targeted acquisition depend on the calibration quality of the trial model’s uncertainty. Nonetheless, across both trial models, TSR-based methods consistently outperform PureRCT and Kallus baselines, demonstrating the effectiveness of leveraging observational data for bias correction in realistic, non-synthetic settings.

\paragraph{IHDP Decision Making.}
Fig.~\ref{fig:ihdp_dm} evaluates treatment policy learning on IHDP using APE and AR. With CMGP, the proposed decision-aware acquisition function R-EPIG-$\pi$ achieves the lowest policy error (0.018) and regret (0.015), improving upon PureRCT-Random by 47.0\% in policy error (0.034) and 61.9\% in regret (0.040). With NSGP, TSR-Random and R-EPIG-$\pi$ achieve comparable policy error (0.020), while TSR-Random attains slightly lower regret (0.012 vs.\ 0.017).

Overall, these results validate that decision-focused acquisition remains effective on real-world data, yielding substantial improvements over pure RCT-based approaches, while also revealing that the magnitude of gains depends on the reliability of uncertainty estimates provided by the trial model.

\paragraph{ACTG-175 CATE Estimation.}
Fig.~\ref{fig:actg_cate} reports CATE estimation performance on the ACTG-175 dataset, derived from an AIDS clinical trial and characterized by more complex and heterogeneous treatment effects. With CMGP, TSR combined with R-EPIG-$\tau$ achieves the best $\sqrt{\text{PEHE}}$ of 1.647, representing a 13.7\% improvement over PureRCT-Random (1.908). With NSGP, R-EPIG-$\tau$ (1.778) slightly outperforms TSR-Random (1.880), indicating that targeted acquisition remains beneficial, though the margin is smaller than on IHDP.

Notably, Kallus-based methods perform poorly on this dataset: Kallus-$\mu$ reaches a $\sqrt{\text{PEHE}}$ of 5.993, over three times worse than TSR-based approaches. This gap highlights the limitations of methods that rely solely on trial data without explicitly correcting for observational bias, particularly in challenging real-world settings.

\paragraph{ACTG-175 Decision Making.}
Fig.~\ref{fig:actg_dm} presents treatment policy learning results on ACTG-175. With CMGP, decision-focused TSR methods achieve the strongest performance. TSR-sign-BALD attains the lowest policy error (0.138) and regret (0.113), followed closely by TSR-$\gamma$ (PE: 0.139, Reg: 0.122) and R-EPIG-$\pi$ (PE: 0.138, Reg: 0.130). All TSR variants substantially outperform PureRCT-Random (PE: 0.190, Reg: 0.195), yielding improvements of 27--42\% in policy error and 33--42\% in regret. With NSGP, TSR-sign-BALD again achieves the best overall performance (PE: 0.131, Reg: 0.116).

Taken together, these results demonstrate that decision-aware acquisition functions (sign-BALD, $\gamma$, and R-EPIG-$\pi$) provide consistent benefits for treatment policy learning on real-world datasets. Across both IHDP and ACTG-175, TSR-based methods reliably outperform PureRCT baselines, confirming that residual learning combined with principled acquisition leads to more effective and robust decision-making in practice.

\begin{bluebox}{}
\textbf{Takeaway:} On real-world benchmarks (IHDP and ACTG-175), TSR-based methods consistently outperform pure RCT baselines in both CATE estimation and treatment decision-making, despite the absence of oracle assumptions. Decision-aware acquisition functions yield the strongest improvements when uncertainty is well-calibrated, while TSR-Random remains a robust fallback in noisier regimes. These results demonstrate that residual learning combined with principled acquisition is not only effective in controlled synthetic settings, but also practical and reliable for real-world causal decision problems.
\end{bluebox}

%% file: Pages/Appendix/fig_include/synthetic.tex
\begin{figure*}[h]
    \centering   
    \begin{minipage}{0.24\linewidth}
        \centering
        \includegraphics[width=\linewidth]{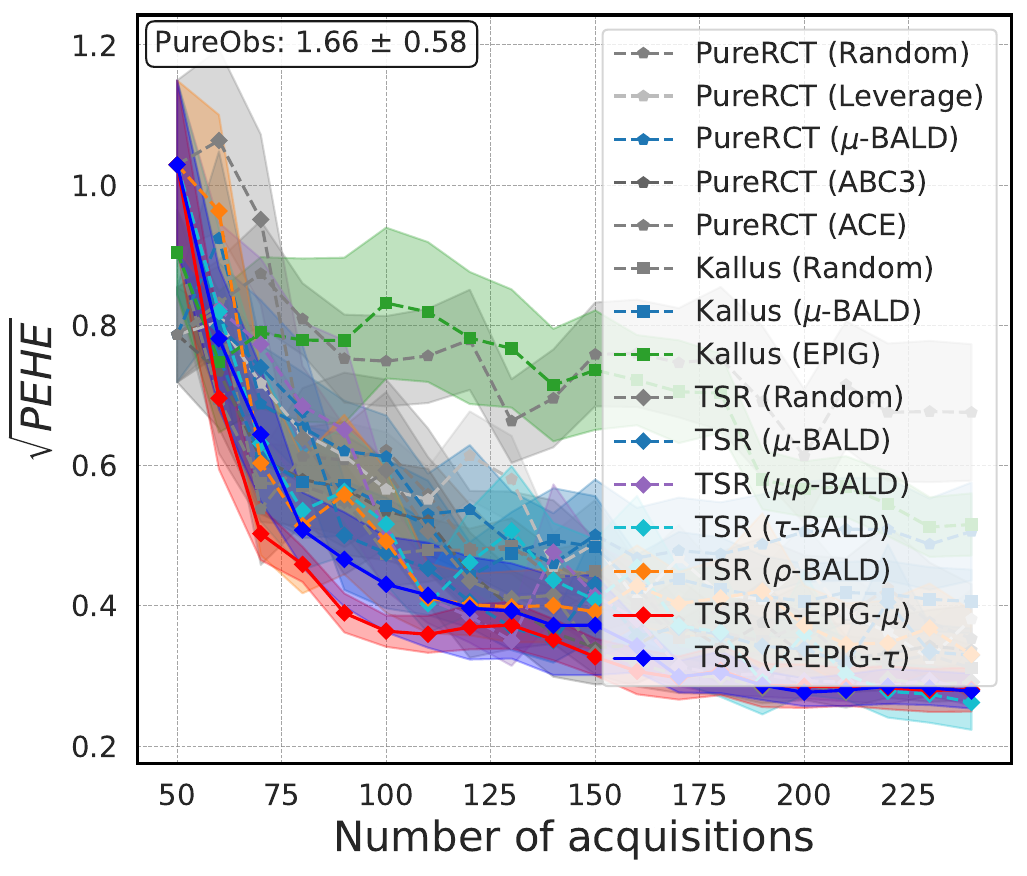}
    \end{minipage}
    \begin{minipage}{0.24\linewidth}
        \centering
        \includegraphics[width=\linewidth]{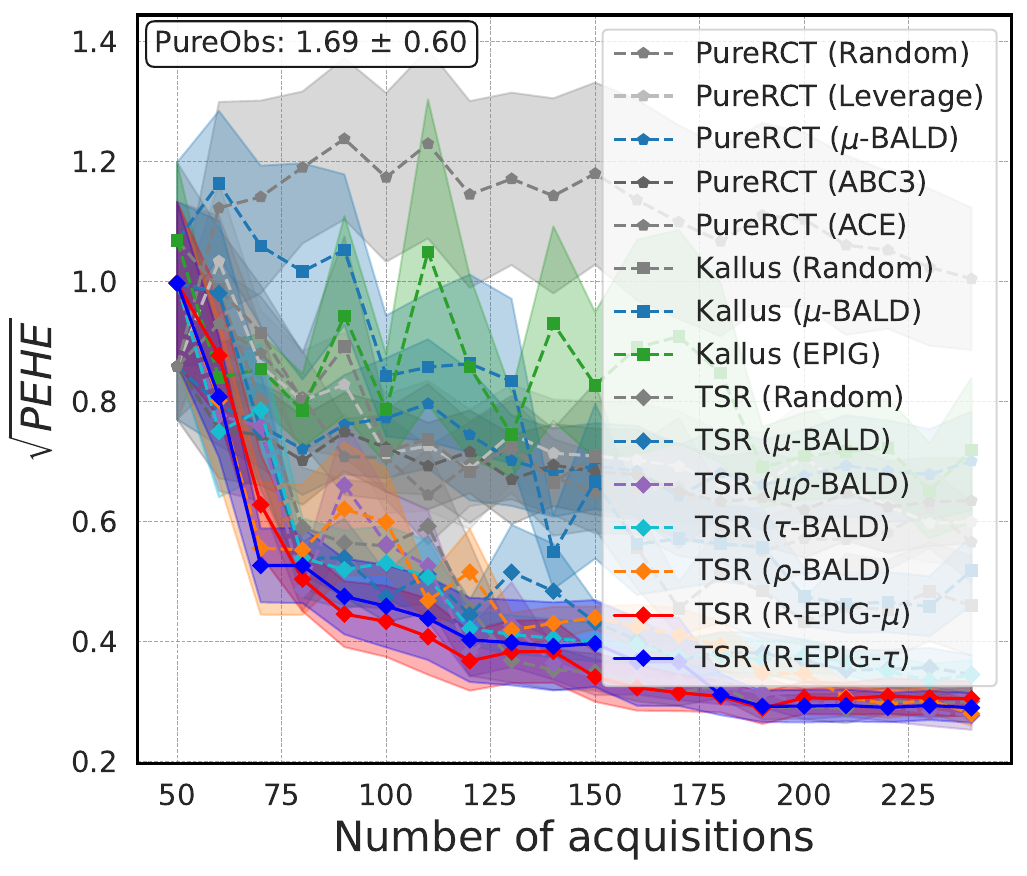}
    \end{minipage}
    \begin{minipage}{0.24\linewidth}
        \centering
        \includegraphics[width=\linewidth]{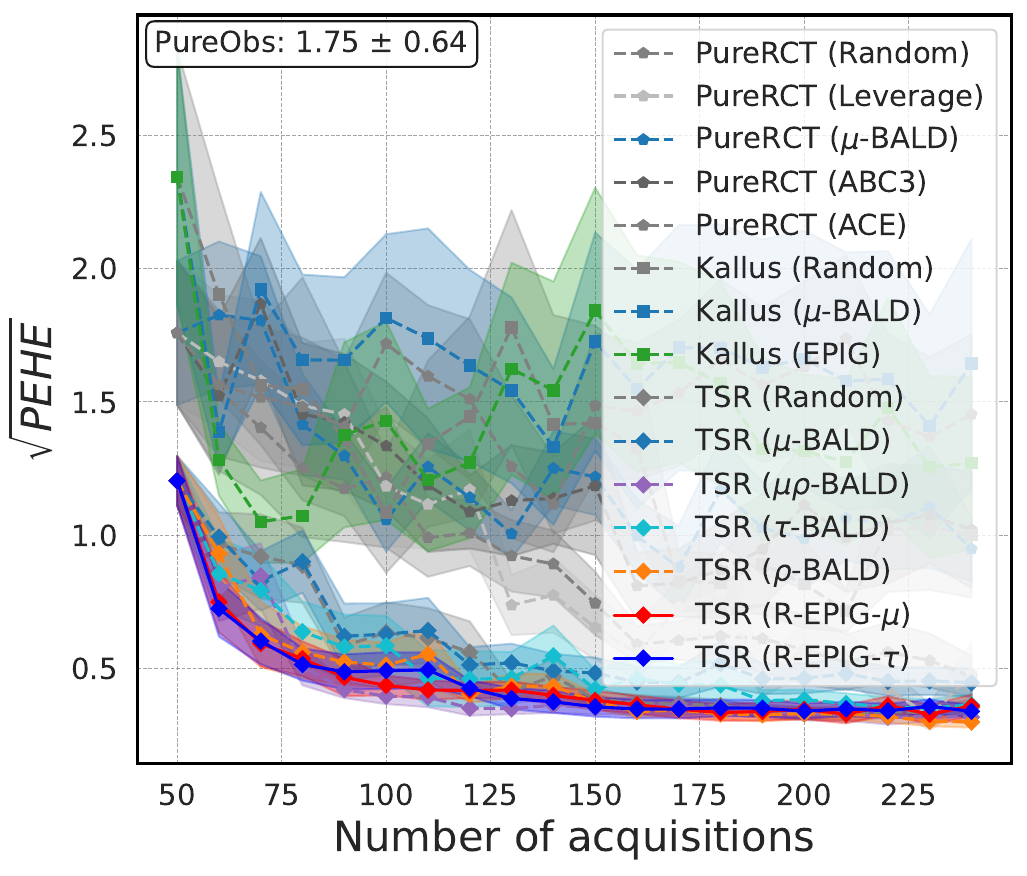}
    \end{minipage}
    \begin{minipage}{0.24\linewidth}
        \centering
        \includegraphics[width=\linewidth]{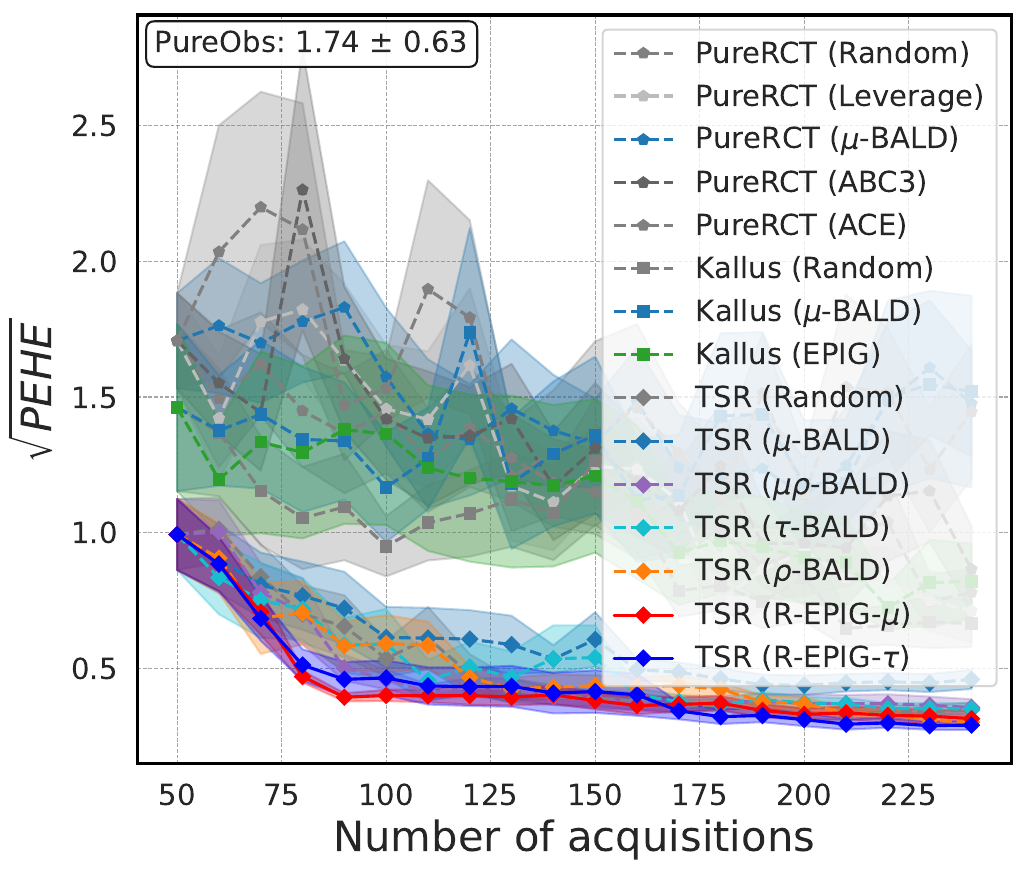}
    \end{minipage} \\
    
    \begin{minipage}{0.24\linewidth}
        \centering
        \includegraphics[width=\linewidth]{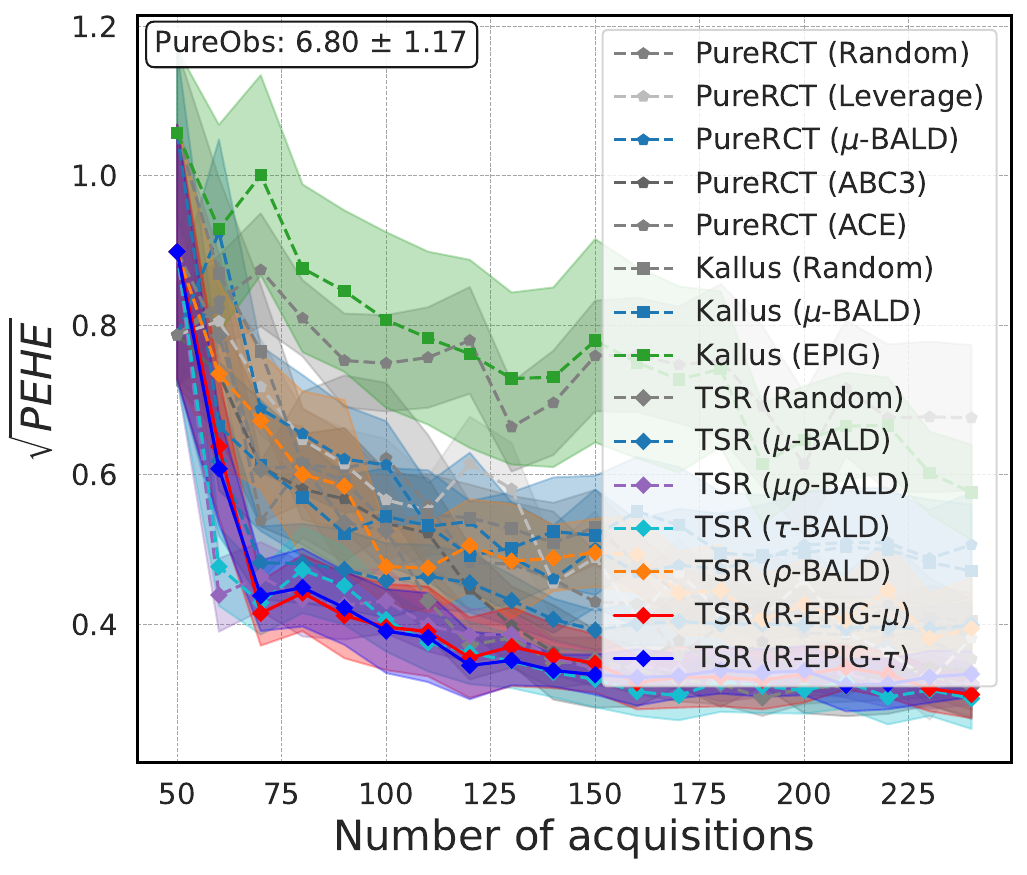}
    \end{minipage}
    \begin{minipage}{0.24\linewidth}
        \centering
        \includegraphics[width=\linewidth]{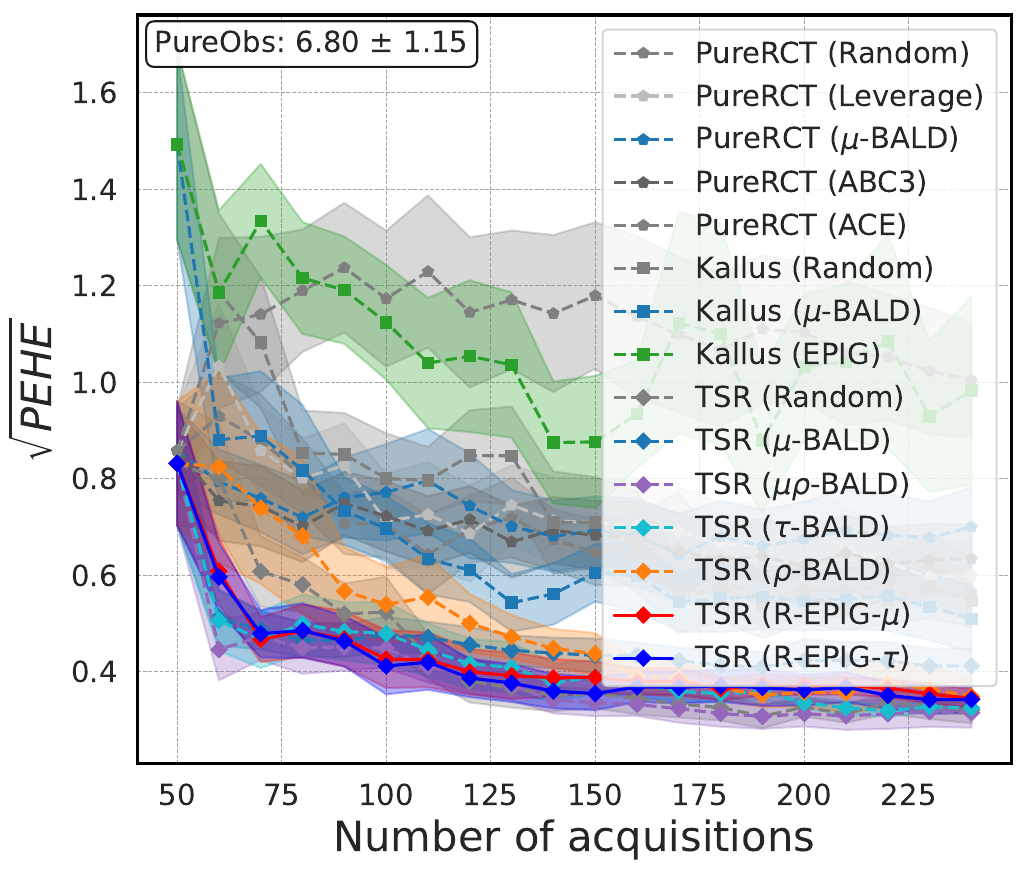}
    \end{minipage}
    \begin{minipage}{0.24\linewidth}
        \centering
        \includegraphics[width=\linewidth]{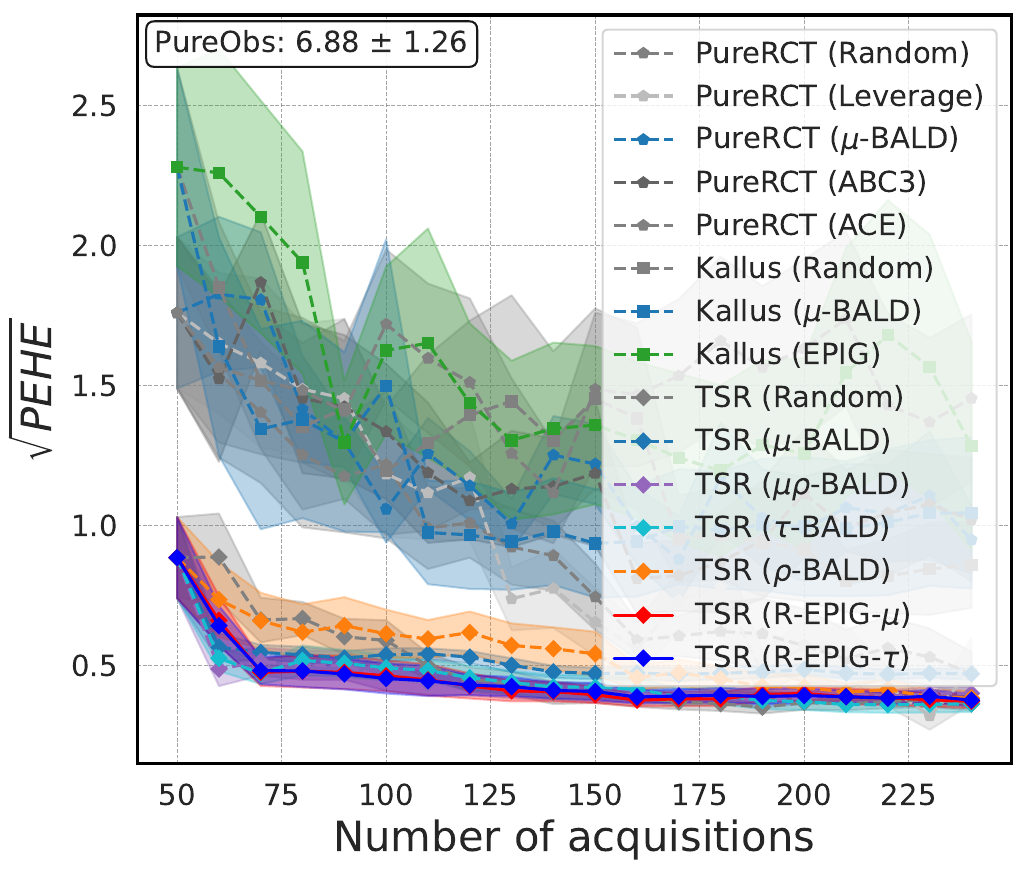}
    \end{minipage}
    \begin{minipage}{0.24\linewidth}
        \centering
        \includegraphics[width=\linewidth]{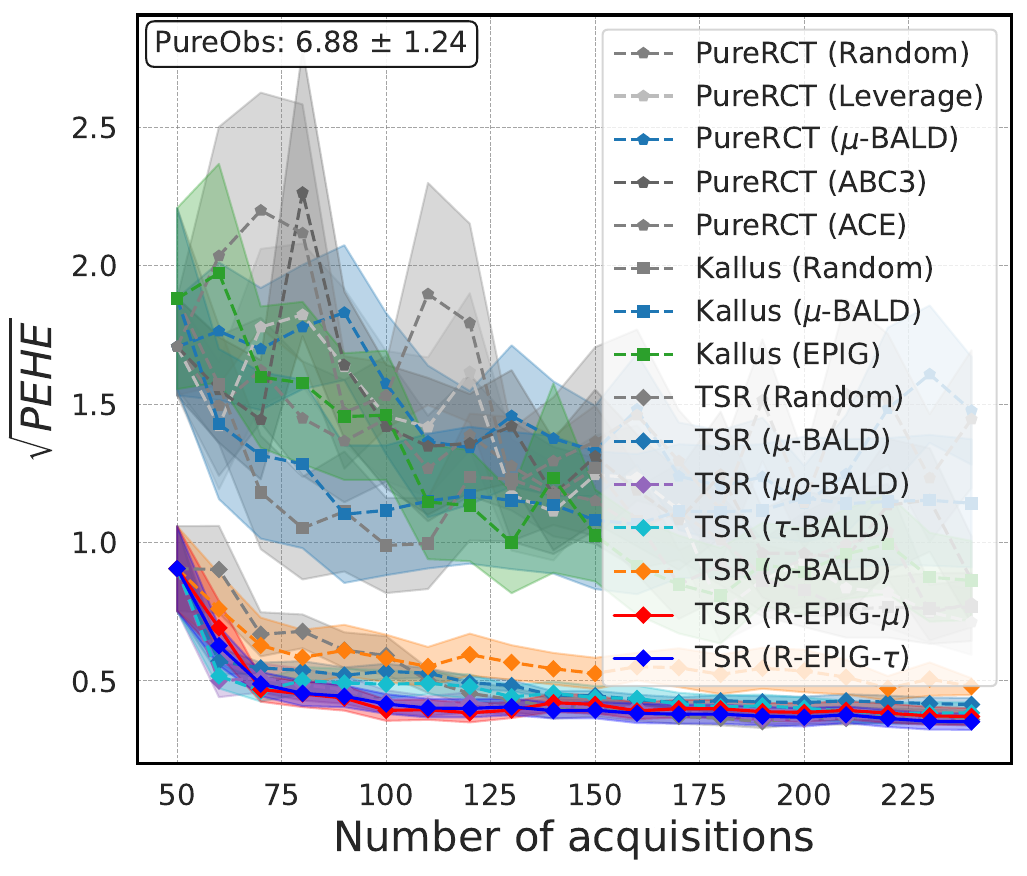}
    \end{minipage} 
    
    \caption{Comparison of $\sqrt{\text{PEHE}}$ univaraite simulation setup. The first row shows the results of simulation setups indexed from 1-4 and the second row shows the results of 5-8.}
    \label{fig:univaraite_8sim_pehe}
\end{figure*}

\begin{figure*}[h]
    \centering   
    \begin{minipage}{0.24\linewidth}
        \centering
        \includegraphics[width=\linewidth]{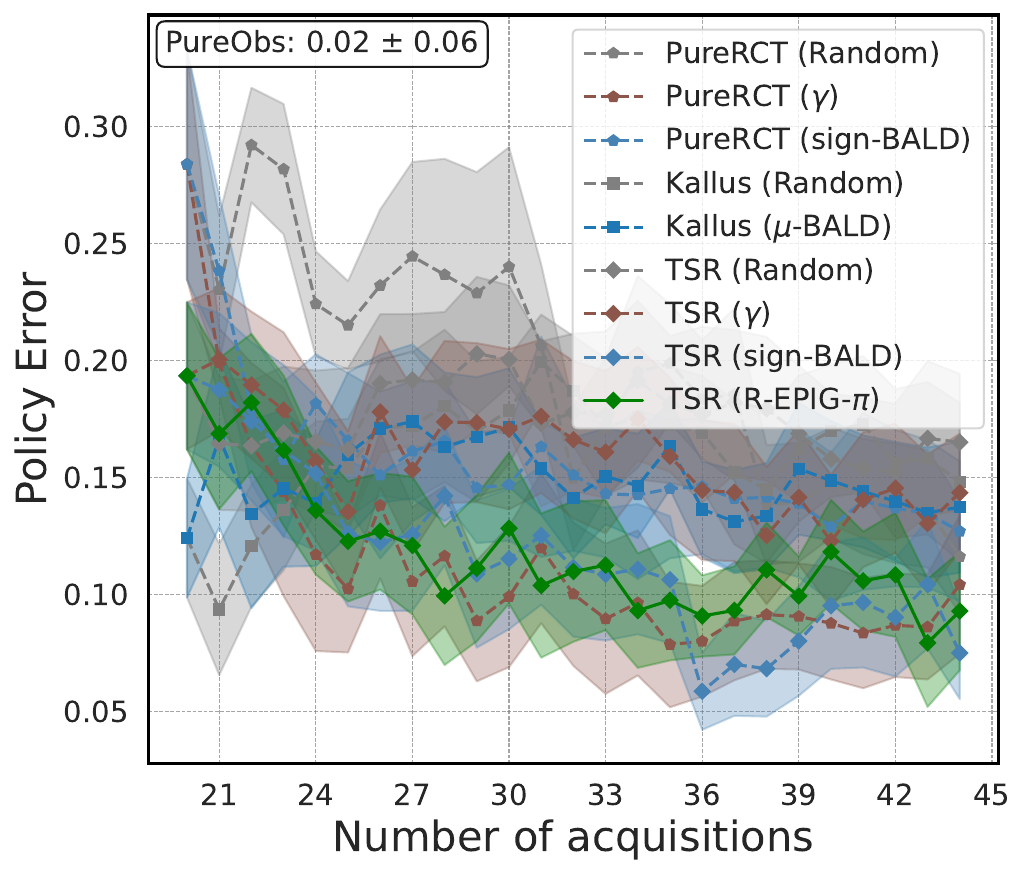}
    \end{minipage}
    \begin{minipage}{0.24\linewidth}
        \centering
        \includegraphics[width=\linewidth]{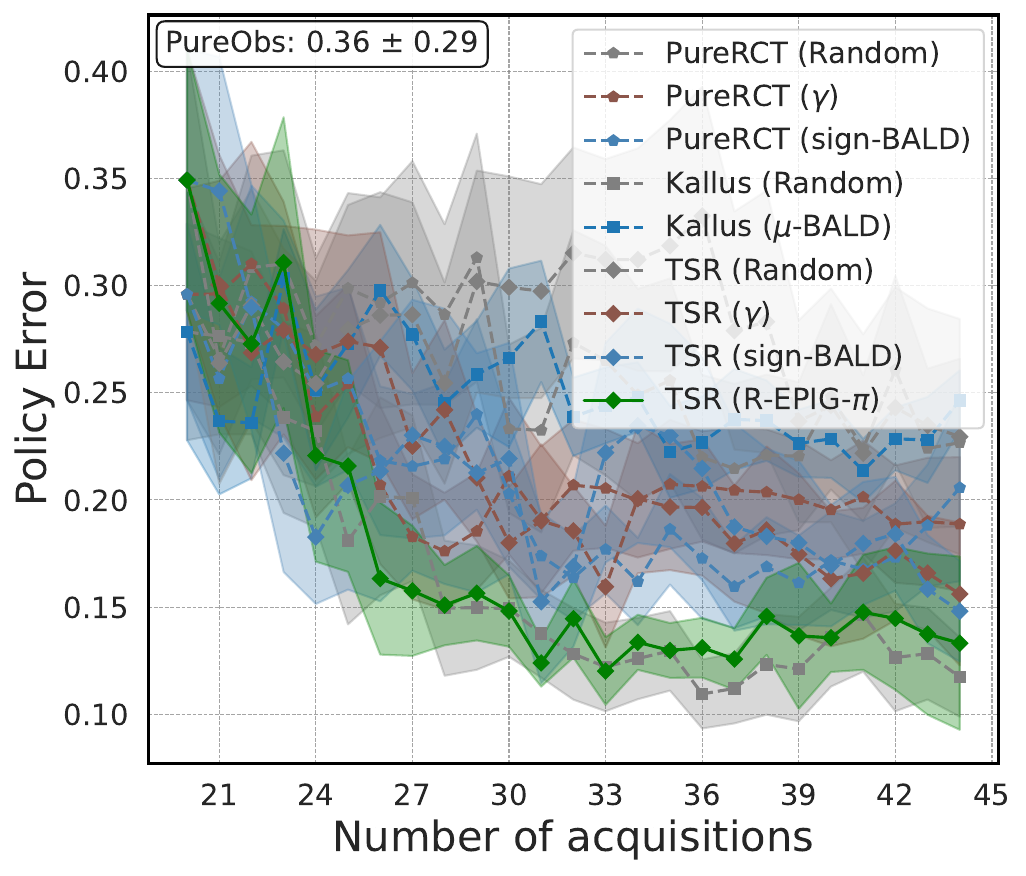}
    \end{minipage}
    \begin{minipage}{0.24\linewidth}
        \centering
        \includegraphics[width=\linewidth]{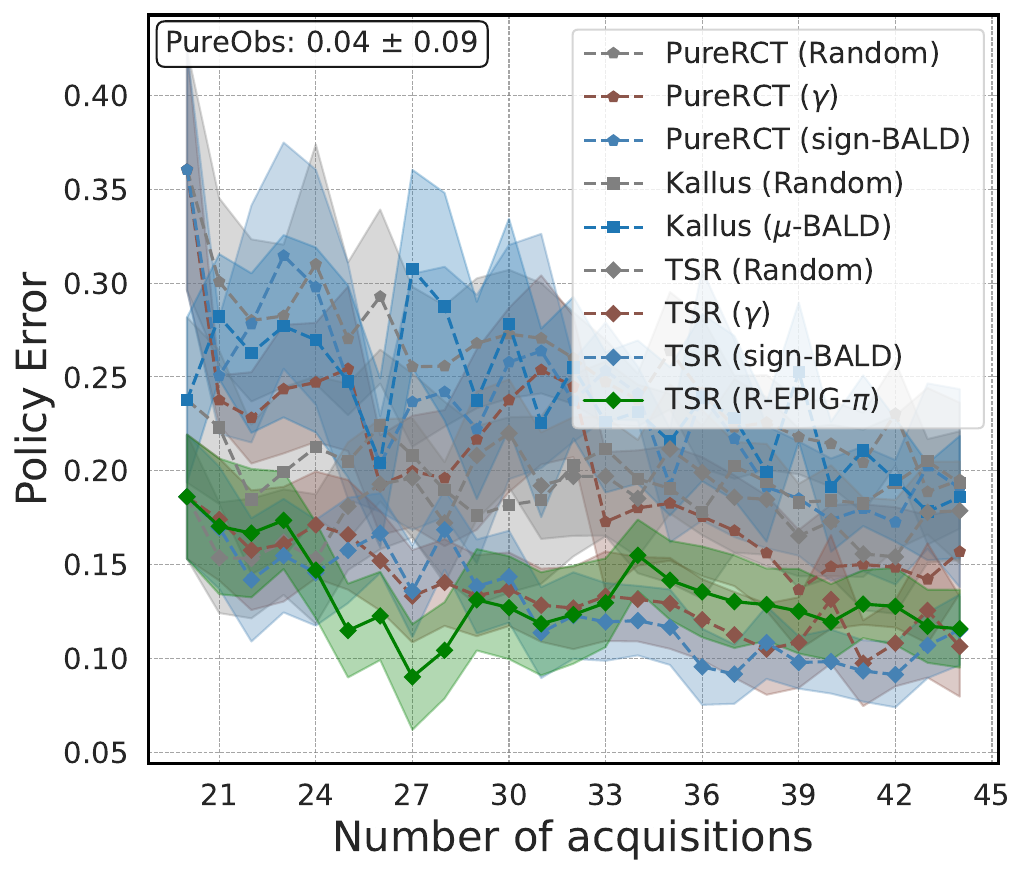}
    \end{minipage}
    \begin{minipage}{0.24\linewidth}
        \centering
        \includegraphics[width=\linewidth]{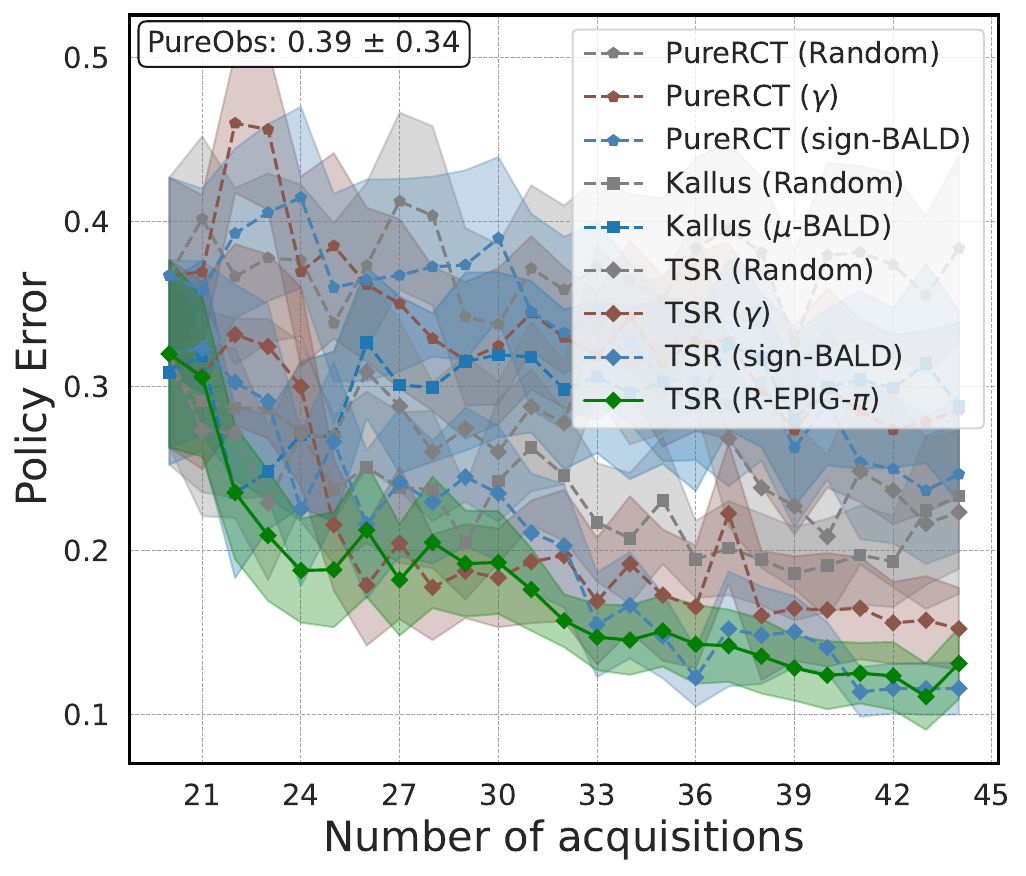}
    \end{minipage} \\
    
    \begin{minipage}{0.24\linewidth}
        \centering
        \includegraphics[width=\linewidth]{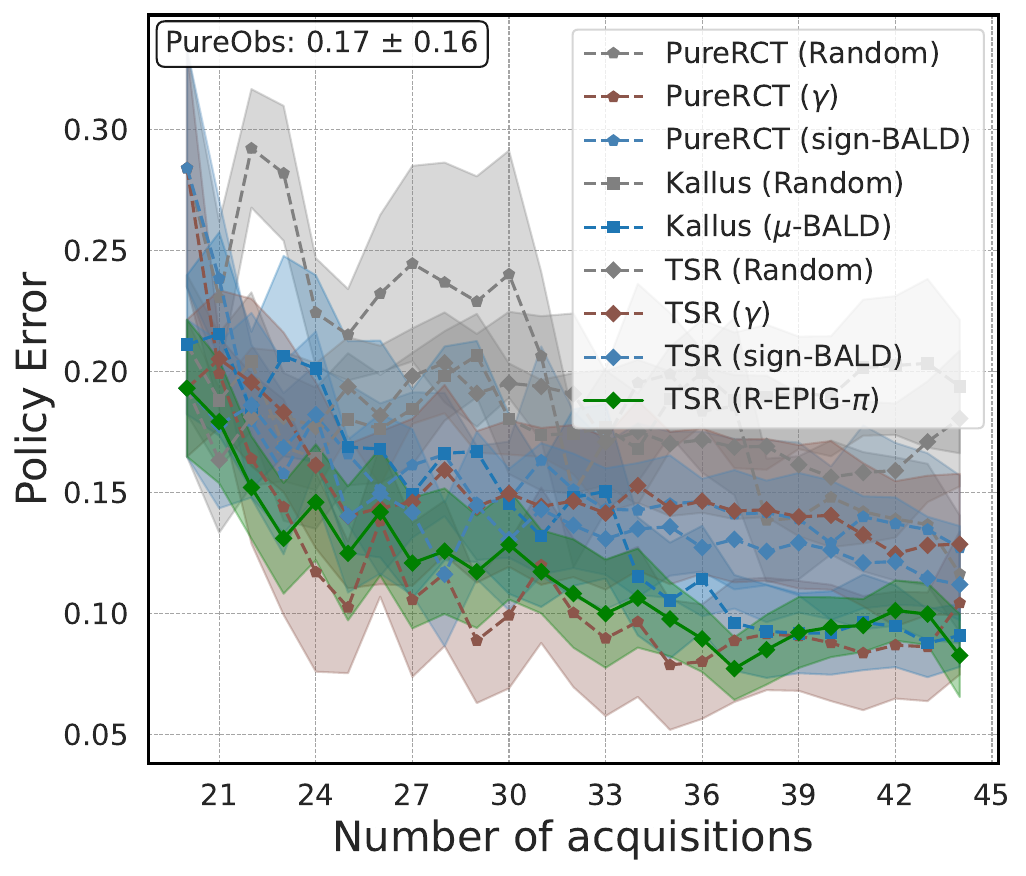}
    \end{minipage}
    \begin{minipage}{0.24\linewidth}
        \centering
        \includegraphics[width=\linewidth]{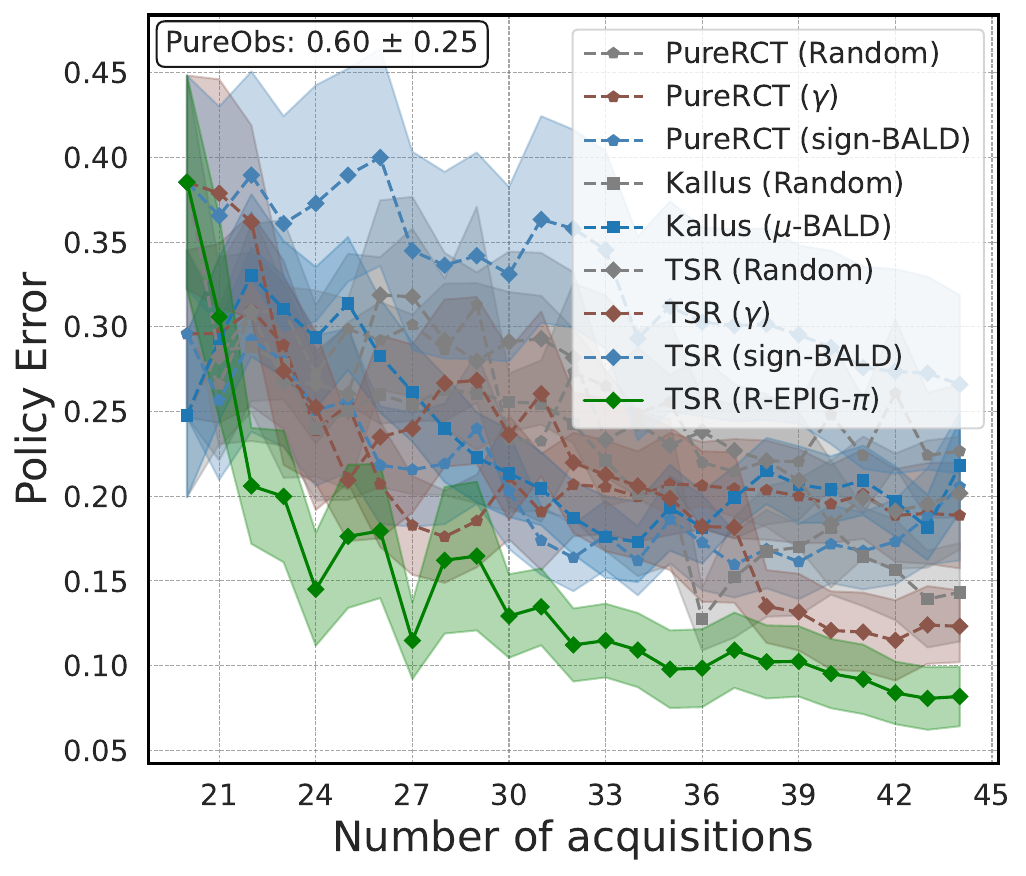}
    \end{minipage}
    \begin{minipage}{0.24\linewidth}
        \centering
        \includegraphics[width=\linewidth]{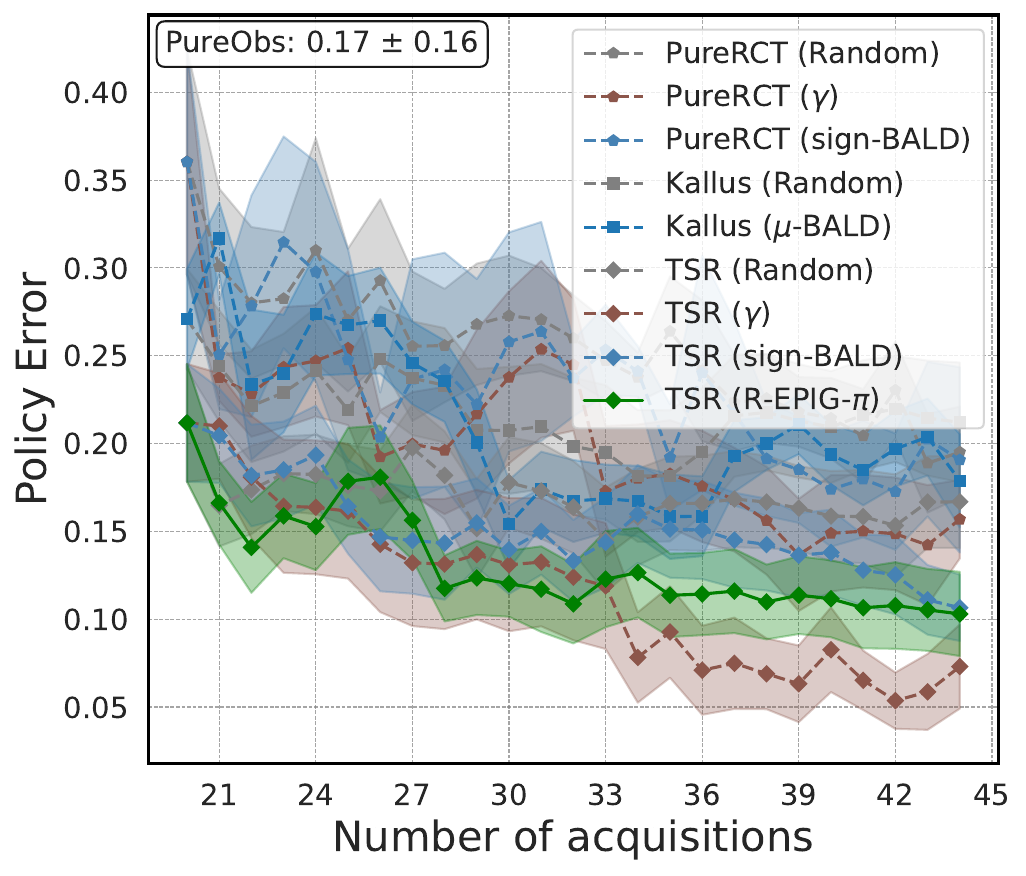}
    \end{minipage}
    \begin{minipage}{0.24\linewidth}
        \centering
        \includegraphics[width=\linewidth]{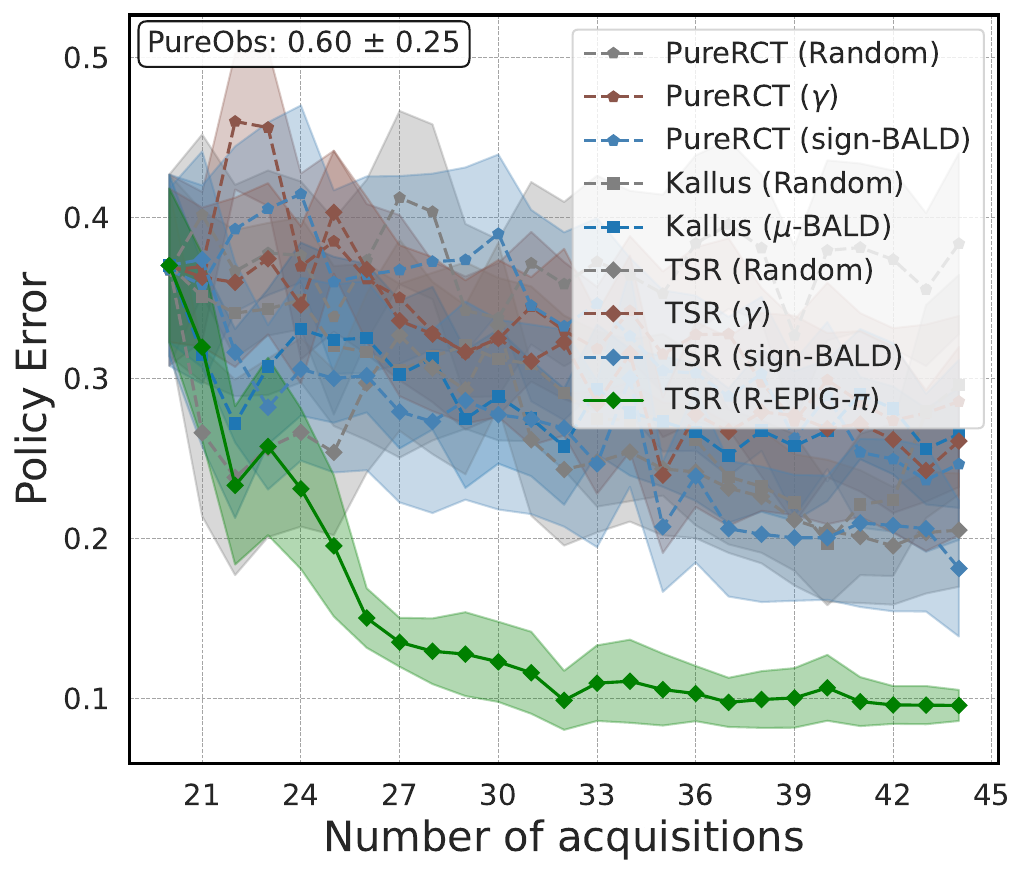}
    \end{minipage} 
    
    \caption{\textbf{Average Policy Error (APE) Convergence.} Comparison across eight univariate simulation benchmarks. The top row displays results for scenarios 1--4, and the bottom row for scenarios 5--8. Lower APE indicates higher accuracy in identifying the optimal treatment arm.}
    \label{fig:univaraite_8sim_ape}
\end{figure*}

\begin{figure*}[h]
    \centering   
    \begin{minipage}{0.24\linewidth}
        \centering
        \includegraphics[width=\linewidth]{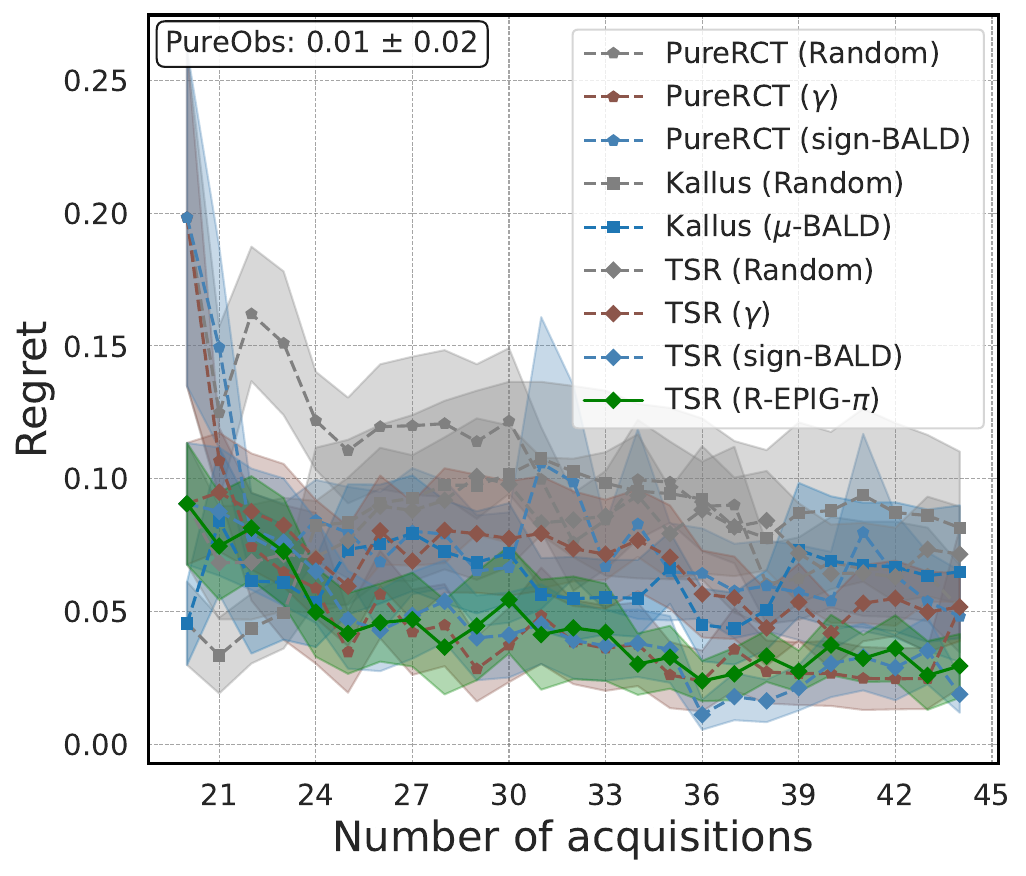}
    \end{minipage}
    \begin{minipage}{0.24\linewidth}
        \centering
        \includegraphics[width=\linewidth]{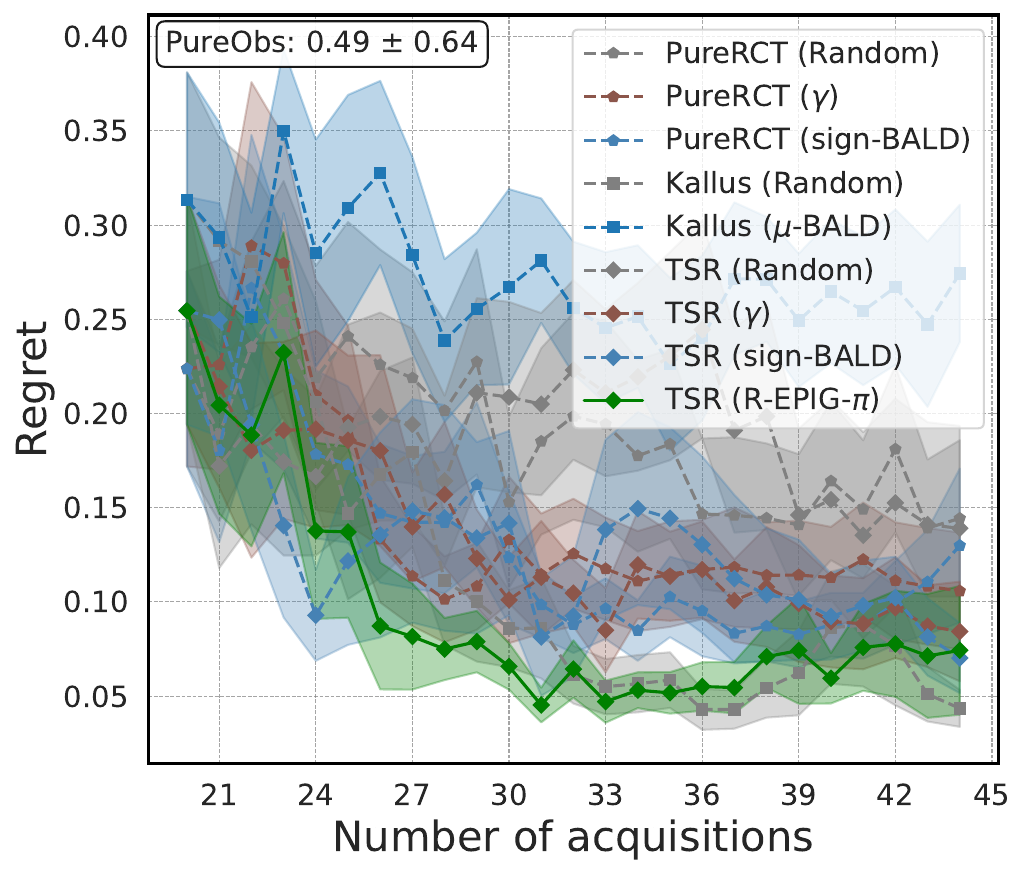}
    \end{minipage}
    \begin{minipage}{0.24\linewidth}
        \centering
        \includegraphics[width=\linewidth]{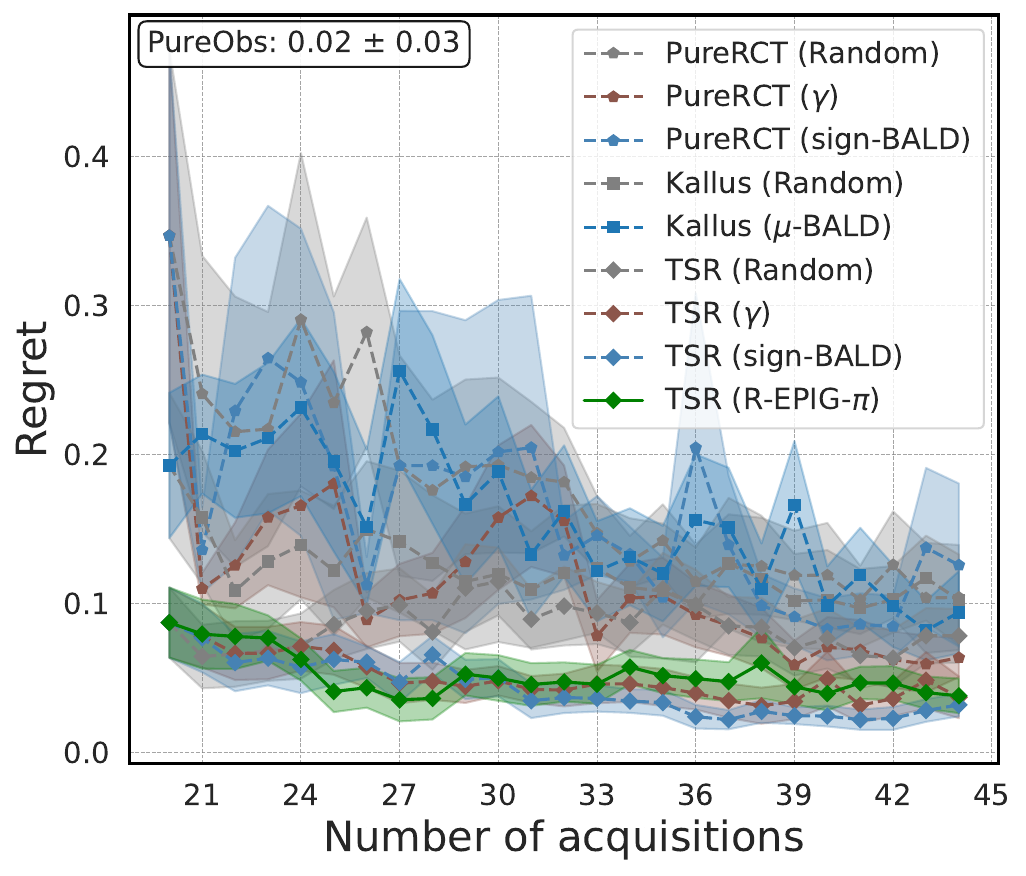}
    \end{minipage}
    \begin{minipage}{0.24\linewidth}
        \centering
        \includegraphics[width=\linewidth]{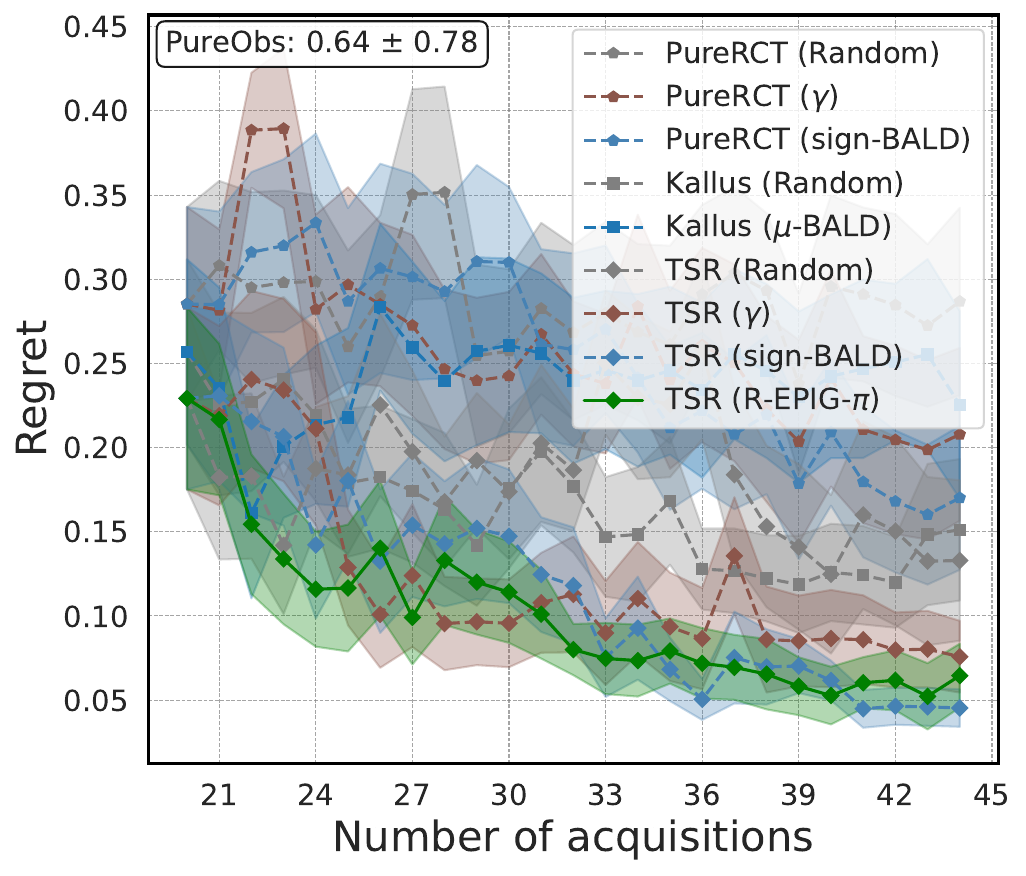}
    \end{minipage} \\
    
    \begin{minipage}{0.24\linewidth}
        \centering
        \includegraphics[width=\linewidth]{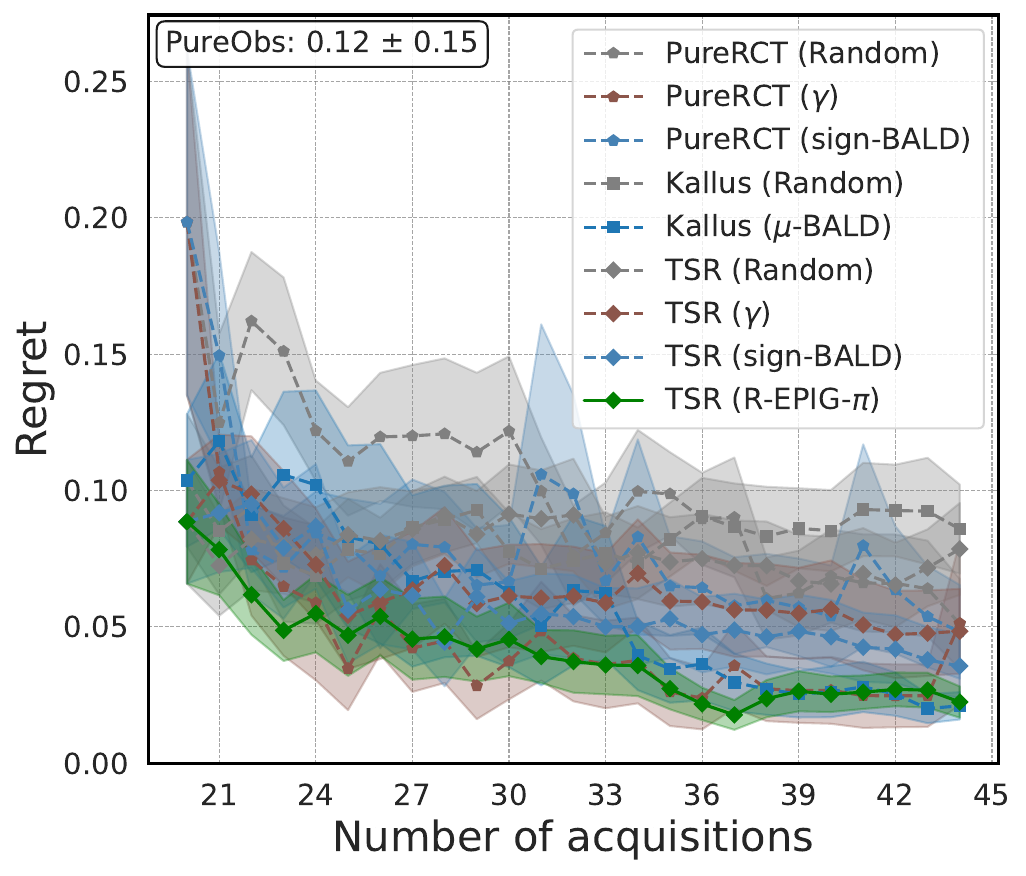}
    \end{minipage}
    \begin{minipage}{0.24\linewidth}
        \centering
        \includegraphics[width=\linewidth]{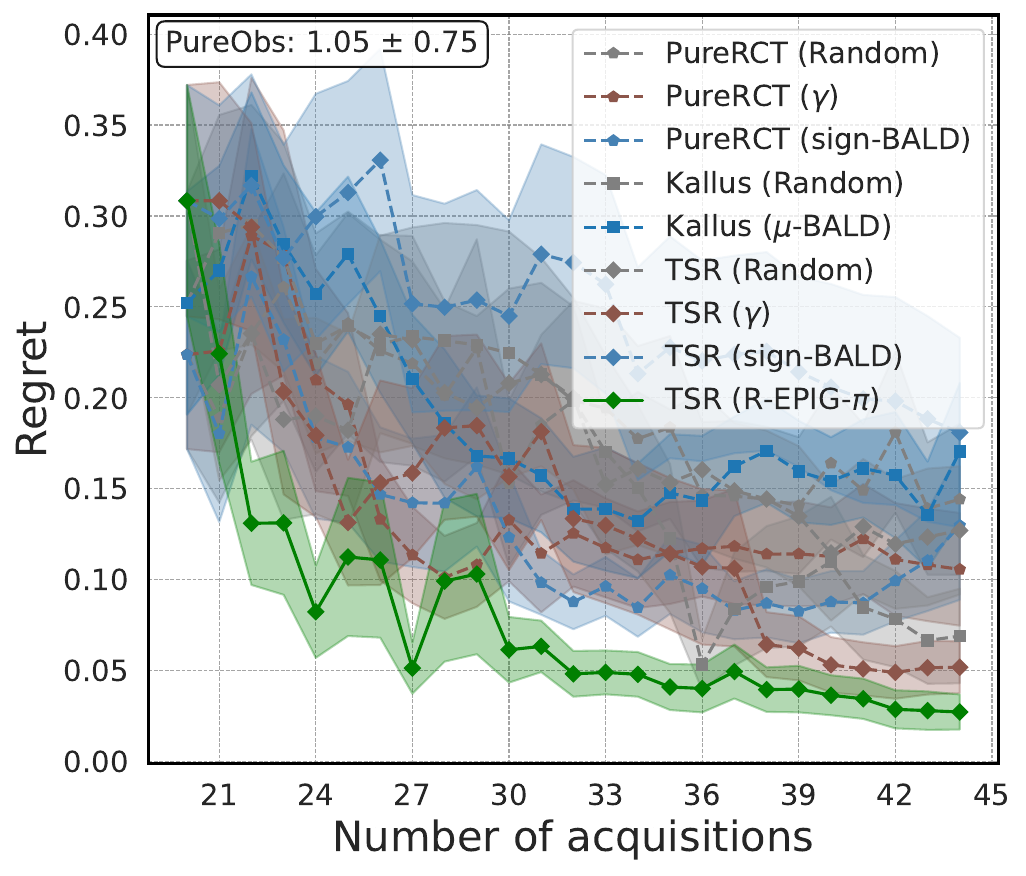}
    \end{minipage}
    \begin{minipage}{0.24\linewidth}
        \centering
        \includegraphics[width=\linewidth]{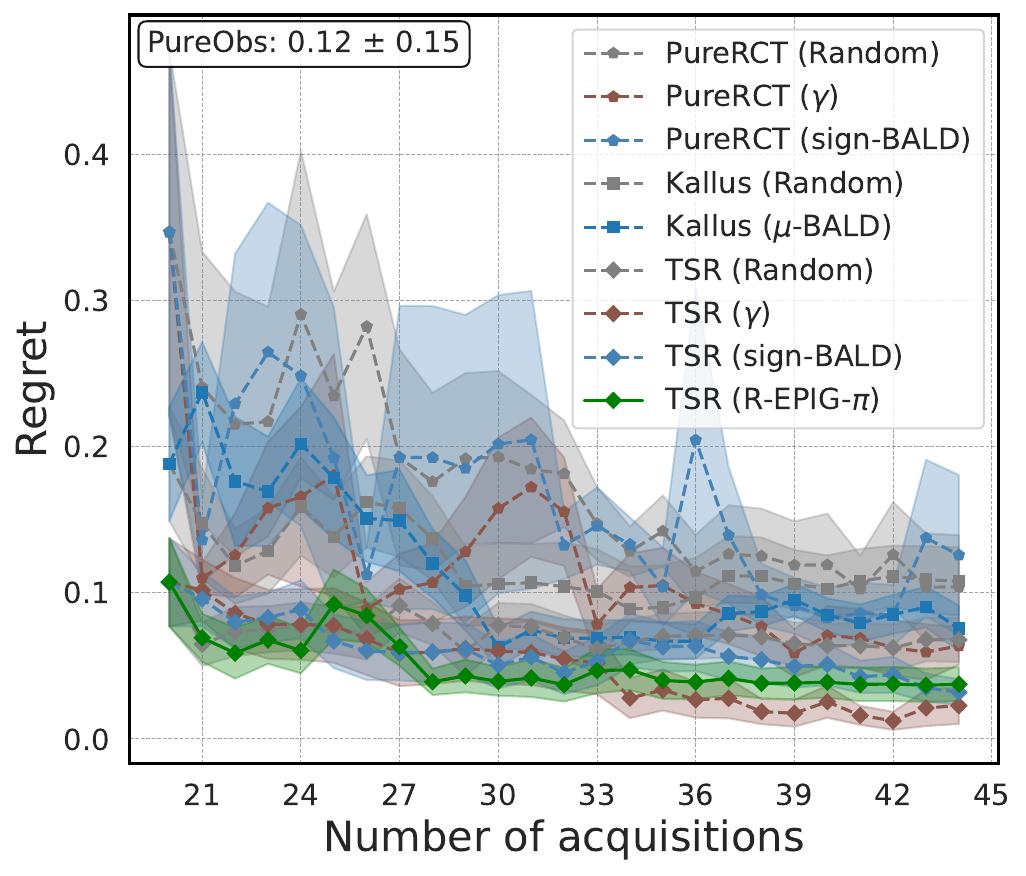}
    \end{minipage}
    \begin{minipage}{0.24\linewidth}
        \centering
        \includegraphics[width=\linewidth]{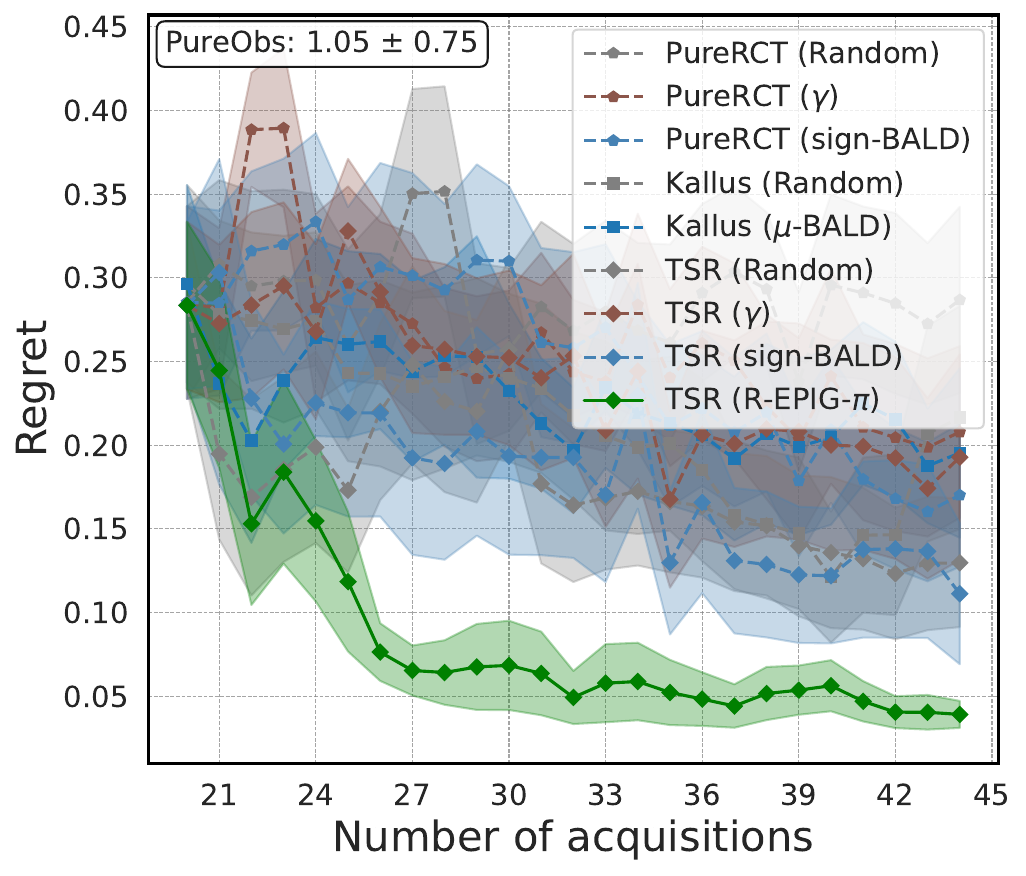}
    \end{minipage} 
    
    \caption{\textbf{Policy Regret Convergence.} Comparison across eight univariate simulation benchmarks. The top row displays results for scenarios 1--4, and the bottom row for scenarios 5--8. R-EPIG-$\pi$ achieves the fastest convergence to zero regret in the majority of tasks.}
    \label{fig:univaraite_8sim_regret}
\end{figure*}

\begin{figure*}[h]
    \centering   
    \begin{minipage}{0.24\linewidth}
        \centering
        \includegraphics[width=\linewidth]{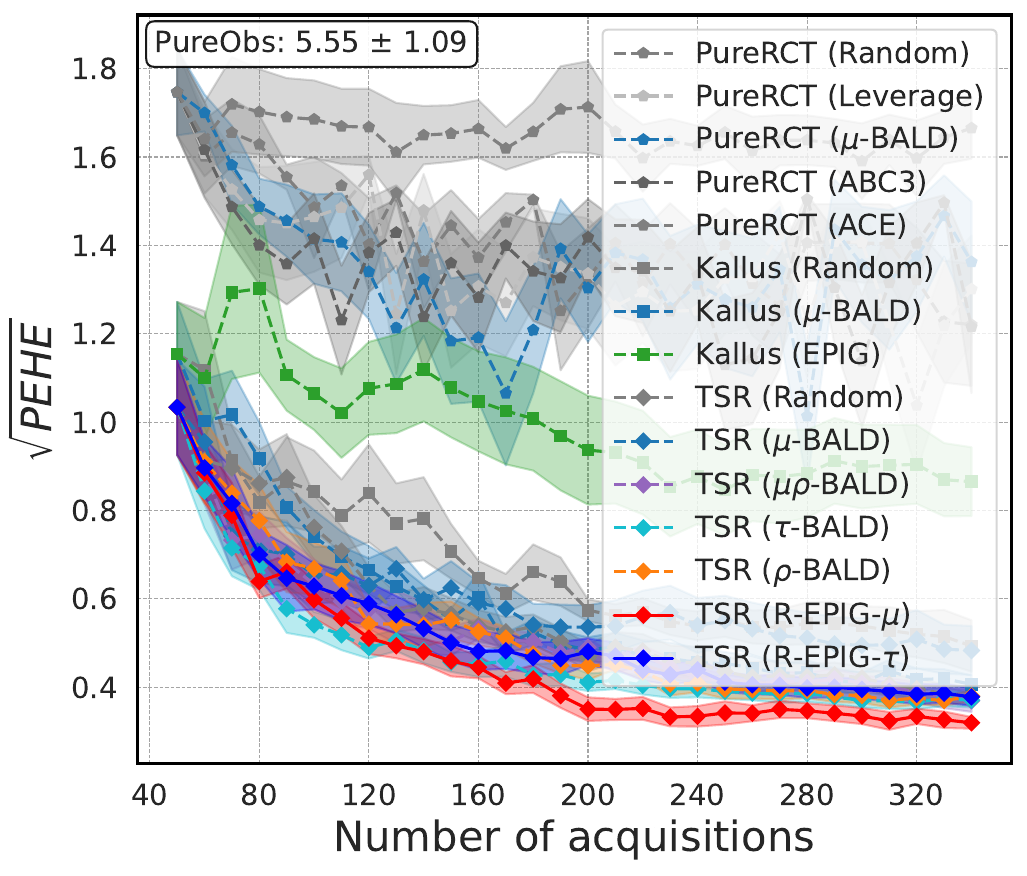}
    \end{minipage}
    \begin{minipage}{0.24\linewidth}
        \centering
        \includegraphics[width=\linewidth]{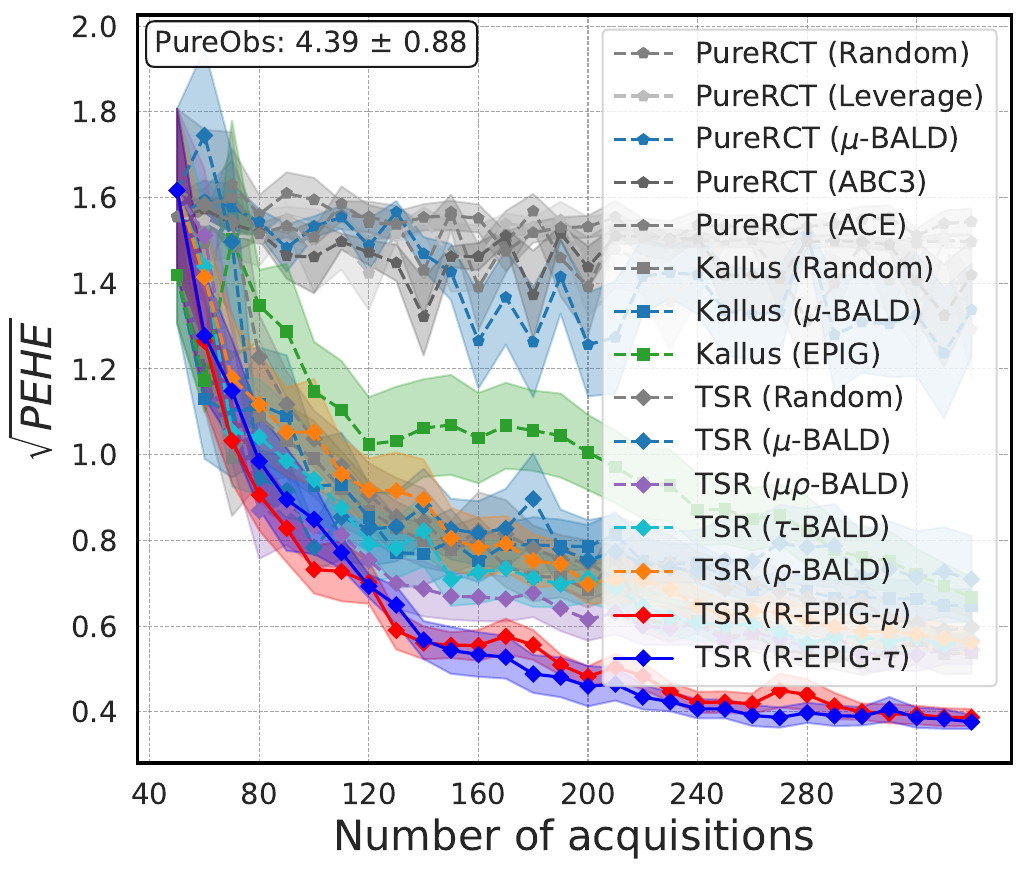}
    \end{minipage}
    \begin{minipage}{0.24\linewidth}
        \centering
        \includegraphics[width=\linewidth]{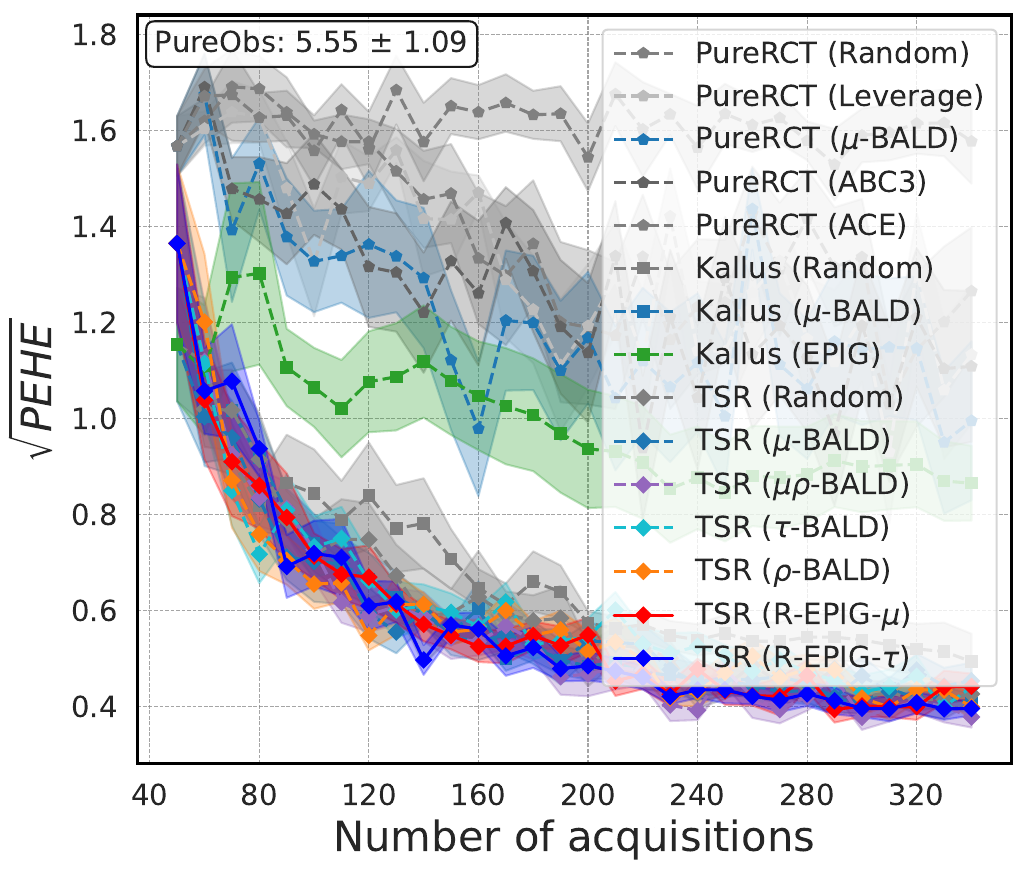}
    \end{minipage}
    \begin{minipage}{0.24\linewidth}
        \centering
        \includegraphics[width=\linewidth]{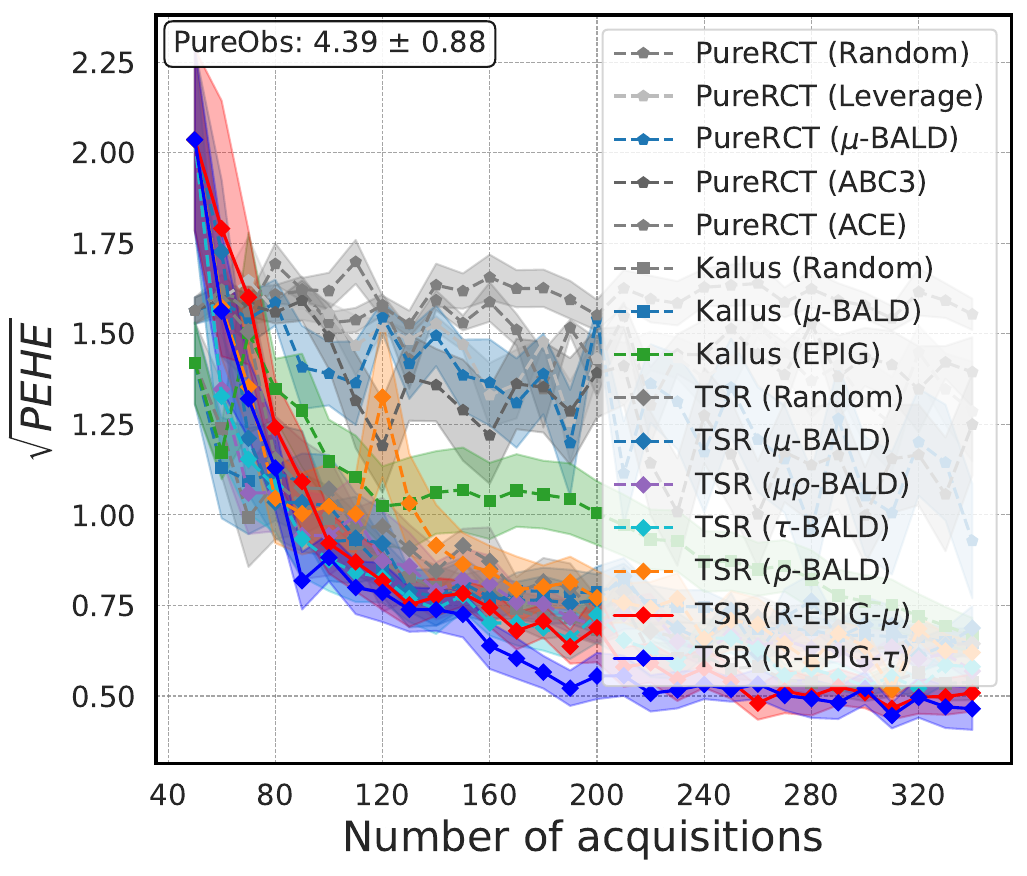}
    \end{minipage} 
    
    \caption{Comparison of PEHE over dim 6 datasets with different trail models. From left to right: CMGP on normal dataset, CMGP on heavy covaraite shift, NSGP on normal dataset, NSGP on heavy covariate shift.}
    \label{fig:dim6_full}
\end{figure*}

\begin{figure*}[h]
    \centering   
    \begin{minipage}{0.24\linewidth}
        \centering
        \includegraphics[width=\linewidth]{Figures/synthetic/dim6/dm/cmgp/dim6_cmgp_all_policy_error_test.pdf}
    \end{minipage}
    \begin{minipage}{0.24\linewidth}
        \centering
        \includegraphics[width=\linewidth]{Figures/synthetic/dim6/dm/cmgp/dim6_cmgp_all_regret.pdf}
    \end{minipage}
    \begin{minipage}{0.24\linewidth}
        \centering
        \includegraphics[width=\linewidth]{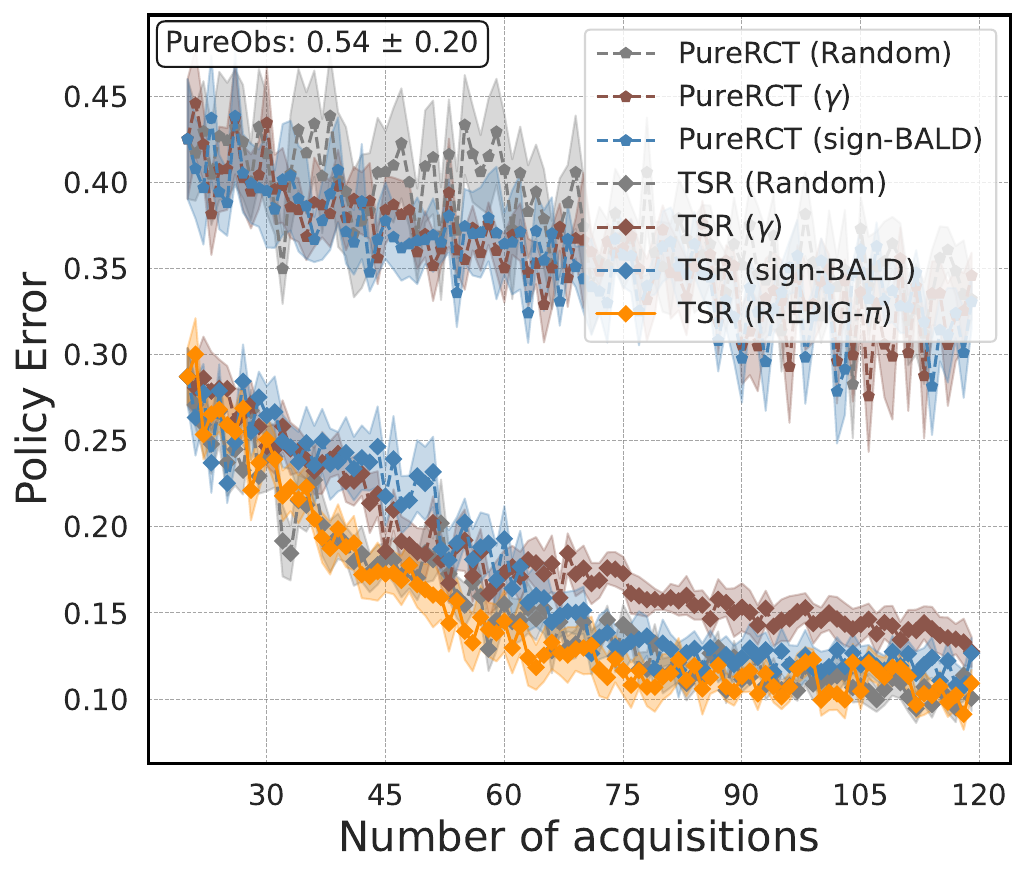}
    \end{minipage}
    \begin{minipage}{0.24\linewidth}
        \centering
        \includegraphics[width=\linewidth]{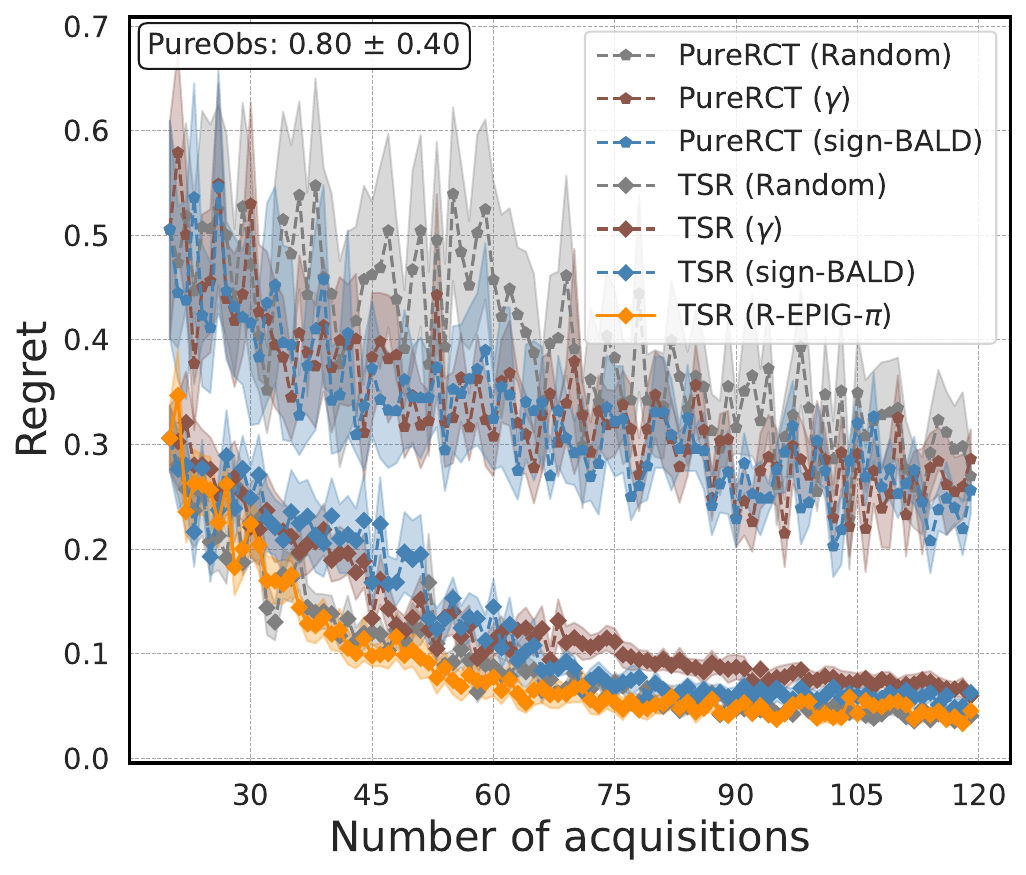}
    \end{minipage} 
    
    \caption{Comparison of decision making performance over the dim 6 dataset with different trail models. From left to right: CMGP (APE), CMGP (AR), NSGP (APE), NSGP (AR).}
    \label{fig:dim6_full_dm}
\end{figure*}

%% file: Pages/Appendix/fig_include/ablation.tex
\begin{figure*}[h]
    \centering   
    \begin{minipage}{0.24\linewidth}
        \centering
        \includegraphics[width=\linewidth]{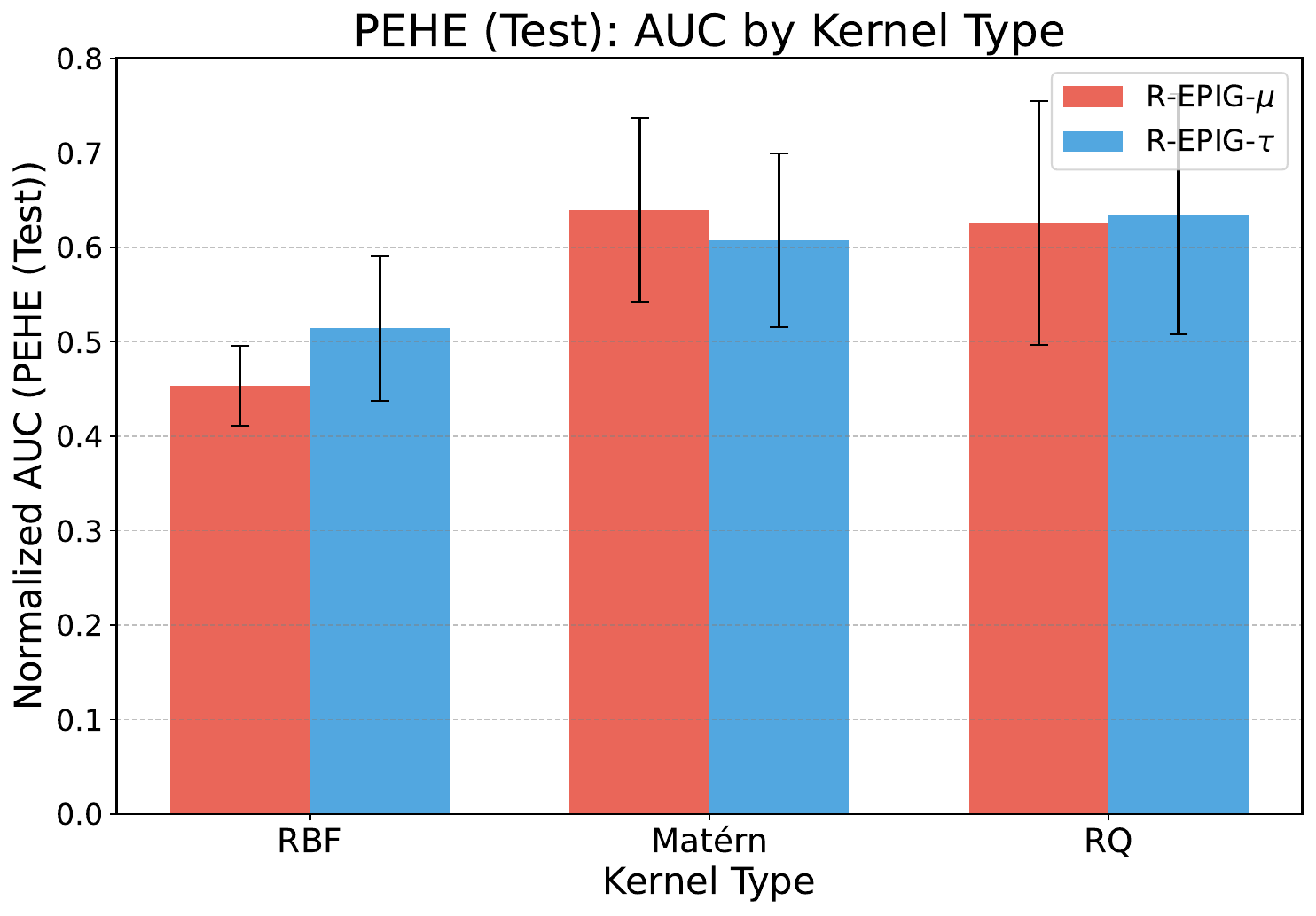}
    \end{minipage}
    \begin{minipage}{0.5\linewidth}
        \centering
        \includegraphics[width=\linewidth]{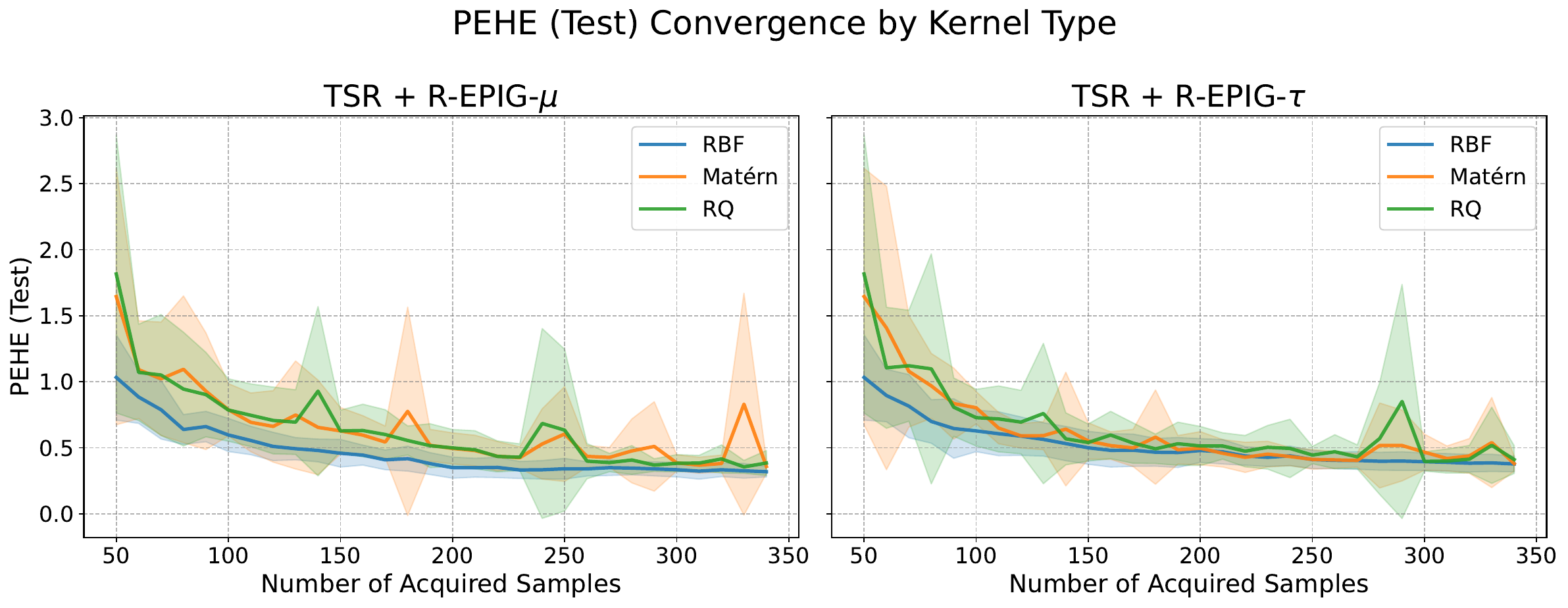}
    \end{minipage}
    
    \caption{Comparison of CATE estimation performance with different kernels.}
    \label{fig:different_kernels}
\end{figure*}

\begin{figure*}[h]
    \centering   
    \begin{minipage}{0.24\linewidth}
        \centering
        \includegraphics[width=\linewidth]{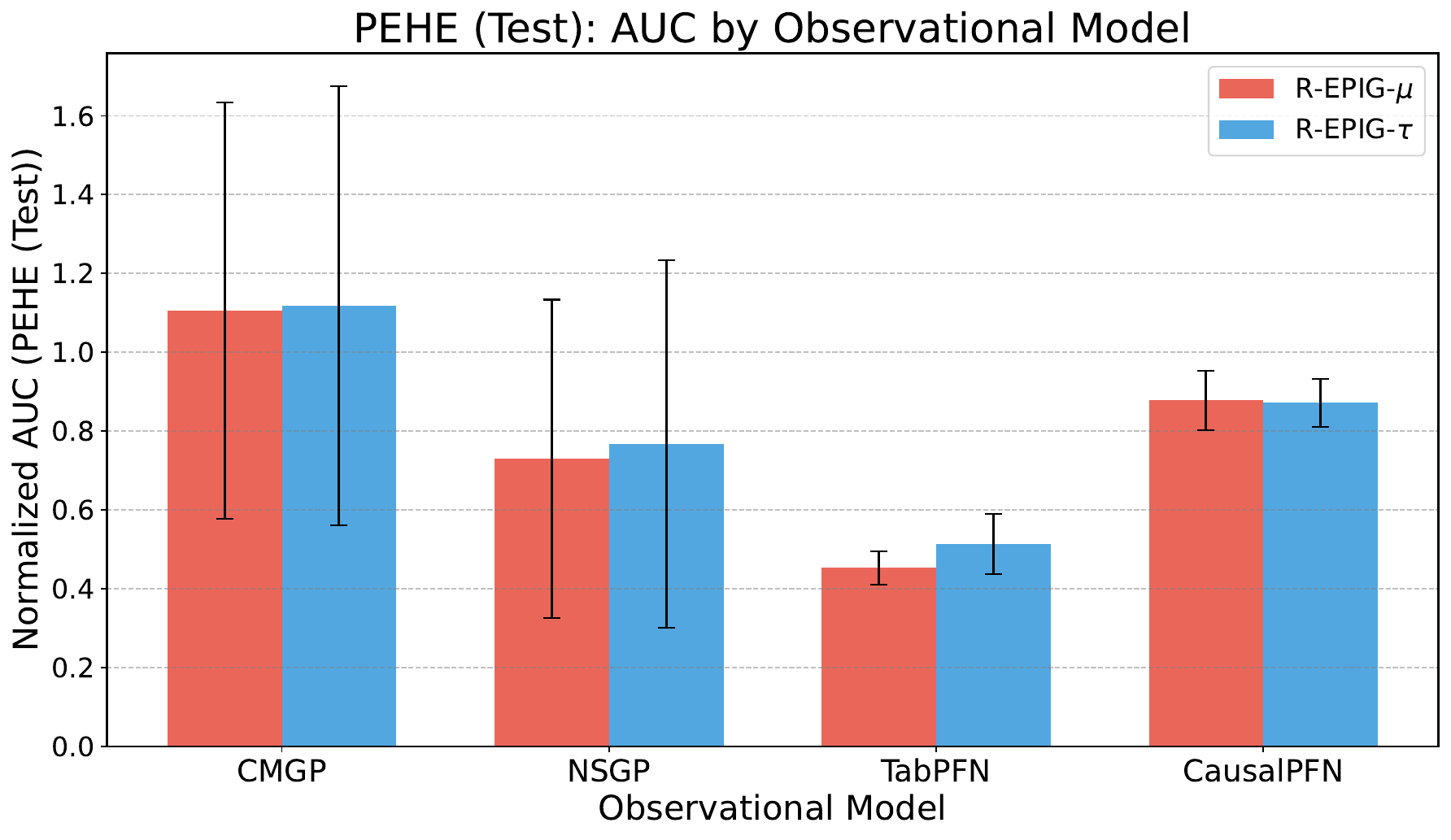}
    \end{minipage}
    \begin{minipage}{0.5\linewidth}
        \centering
        \includegraphics[width=\linewidth]{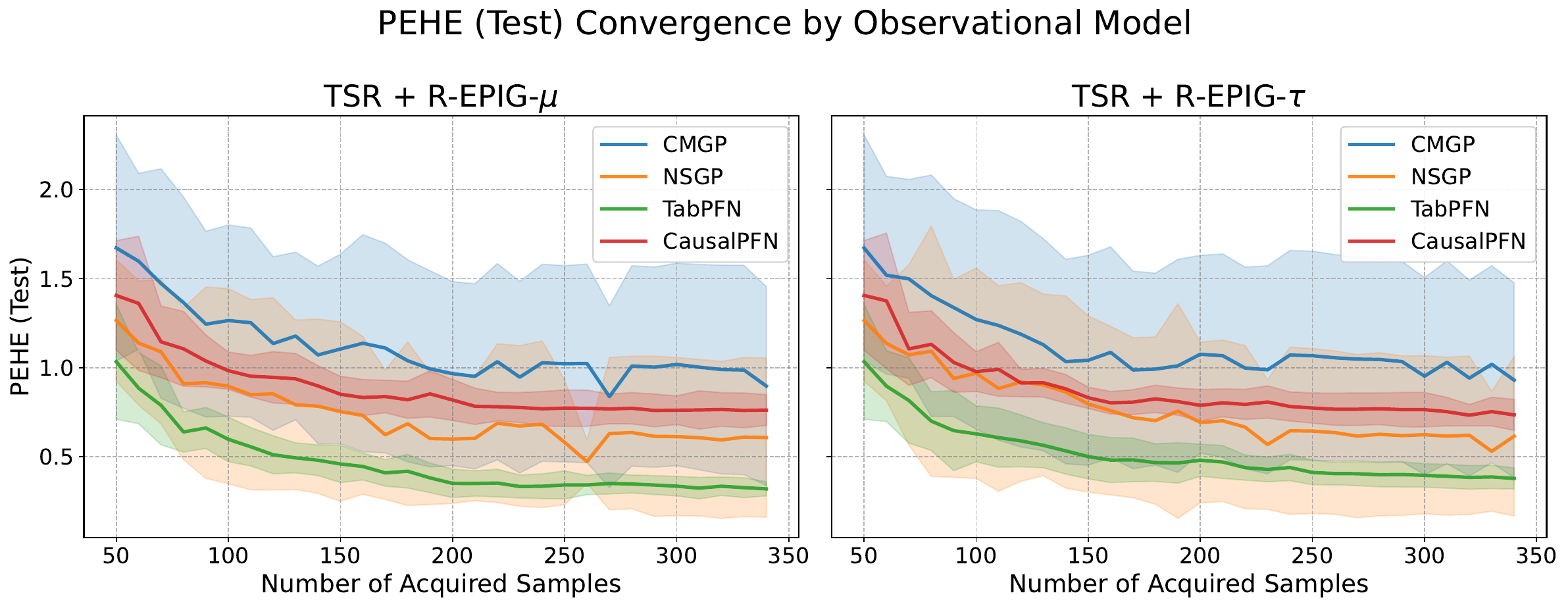}
    \end{minipage}
    
    \caption{Comparison of CATE estimation performance with different Observational models.}
    \label{fig:different_obs_models}
\end{figure*}

\begin{figure*}[h]
    \centering   
    \begin{minipage}{0.24\linewidth}
        \centering
        \includegraphics[width=\linewidth]{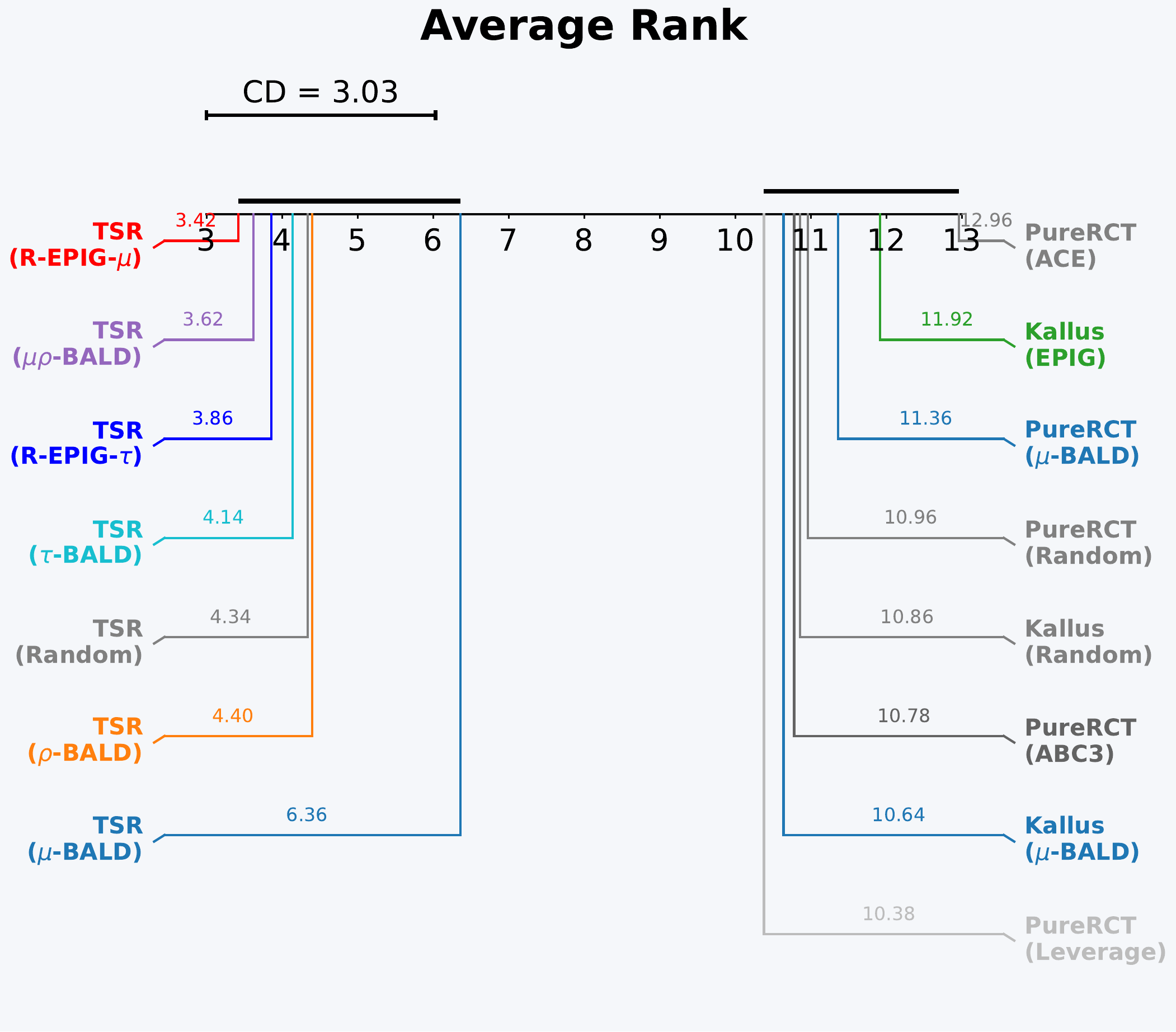}
    \end{minipage}
    \begin{minipage}{0.24\linewidth}
        \centering
        \includegraphics[width=\linewidth]{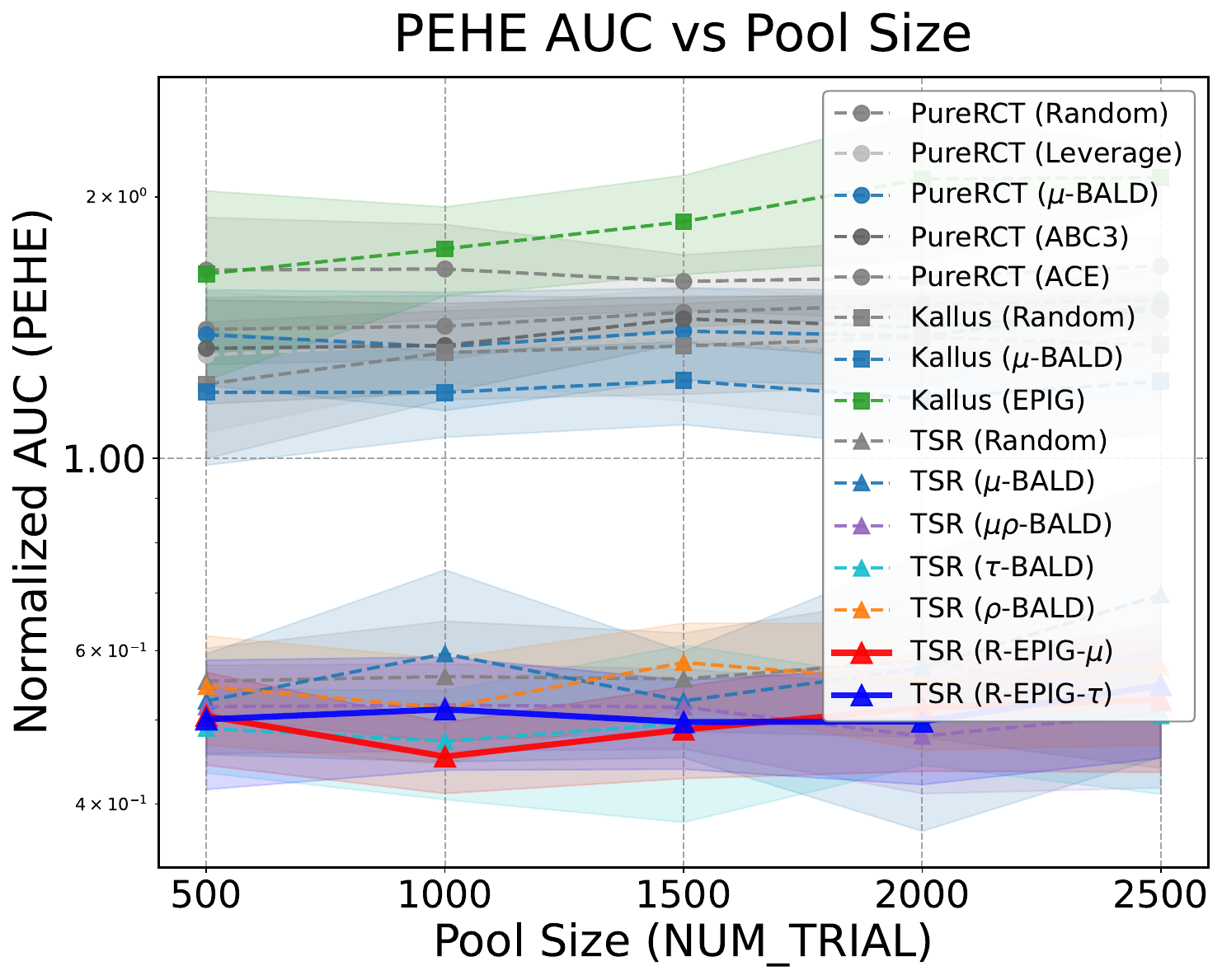}
    \end{minipage}
    \begin{minipage}{0.48\linewidth}
        \centering
        \includegraphics[width=\linewidth]{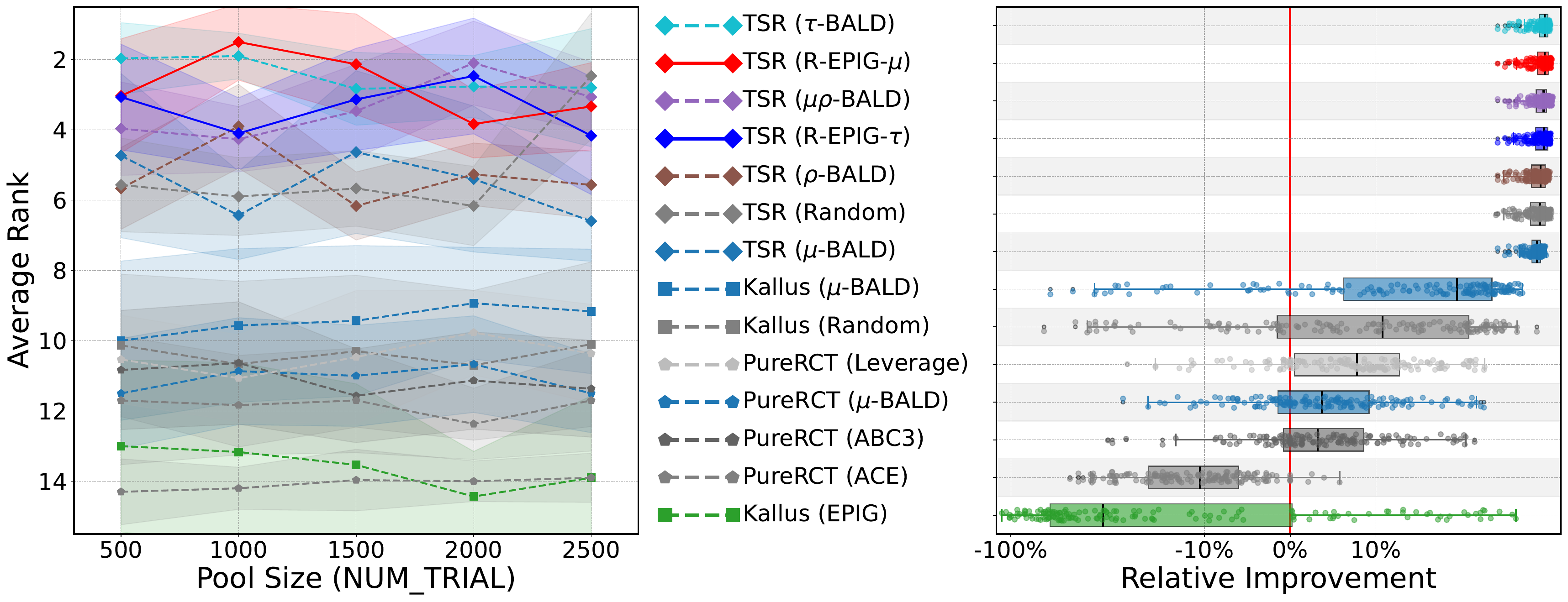}
    \end{minipage}
    
    \caption{Comparison of CATE estimation performance with different pool sizes.}
    \label{fig:different_pool_sizes}
\end{figure*}

\begin{figure*}[h]
    \centering   
    \begin{minipage}{0.24\linewidth}
        \centering
        \includegraphics[width=\linewidth]{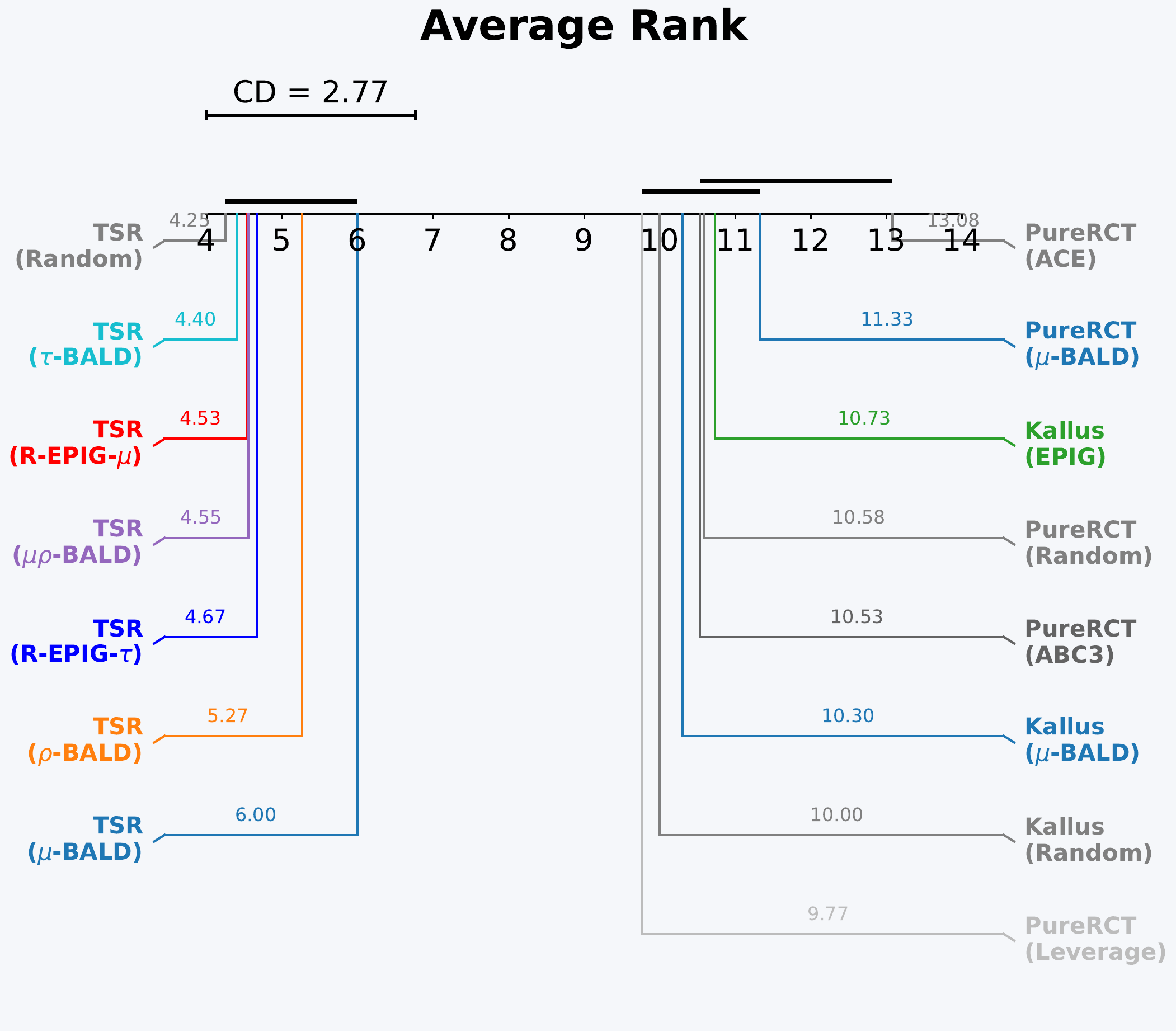}
    \end{minipage}
    \begin{minipage}{0.24\linewidth}
        \centering
        \includegraphics[width=\linewidth]{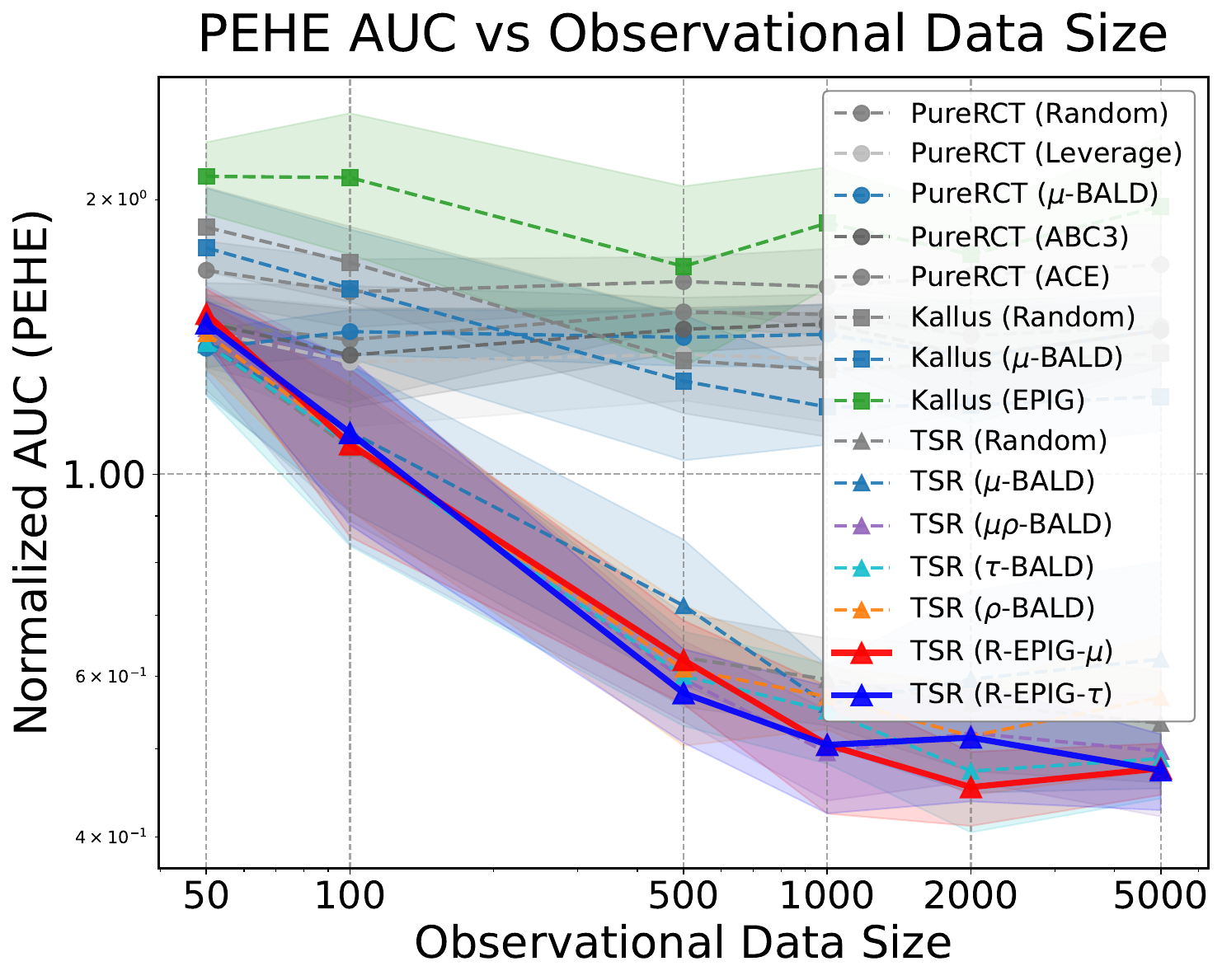}
    \end{minipage}
    \begin{minipage}{0.48\linewidth}
        \centering
        \includegraphics[width=\linewidth]{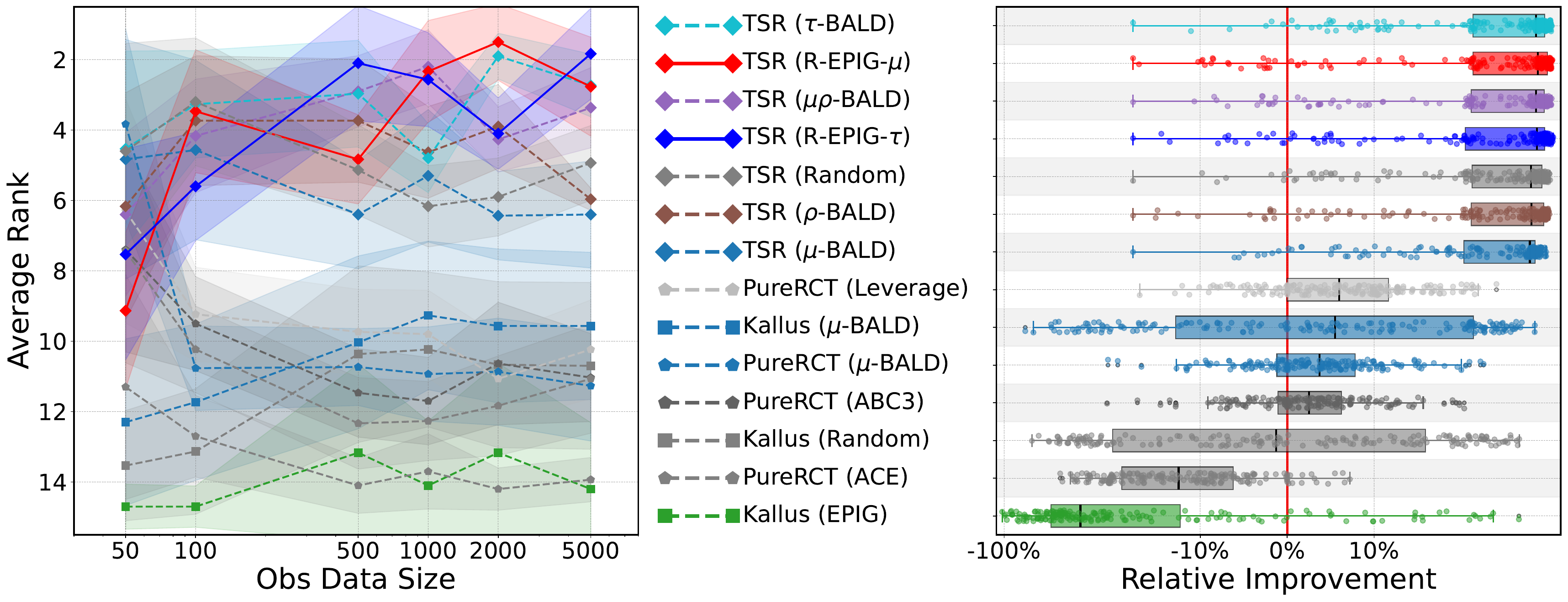}
    \end{minipage}
    
    \caption{Comparison of CATE estimation performance with different observational  sizes.}
    \label{fig:different_prior_sizes}
\end{figure*}

\begin{figure*}[h]
    \centering   
    \begin{minipage}{0.24\linewidth}
        \centering
        \includegraphics[width=\linewidth]{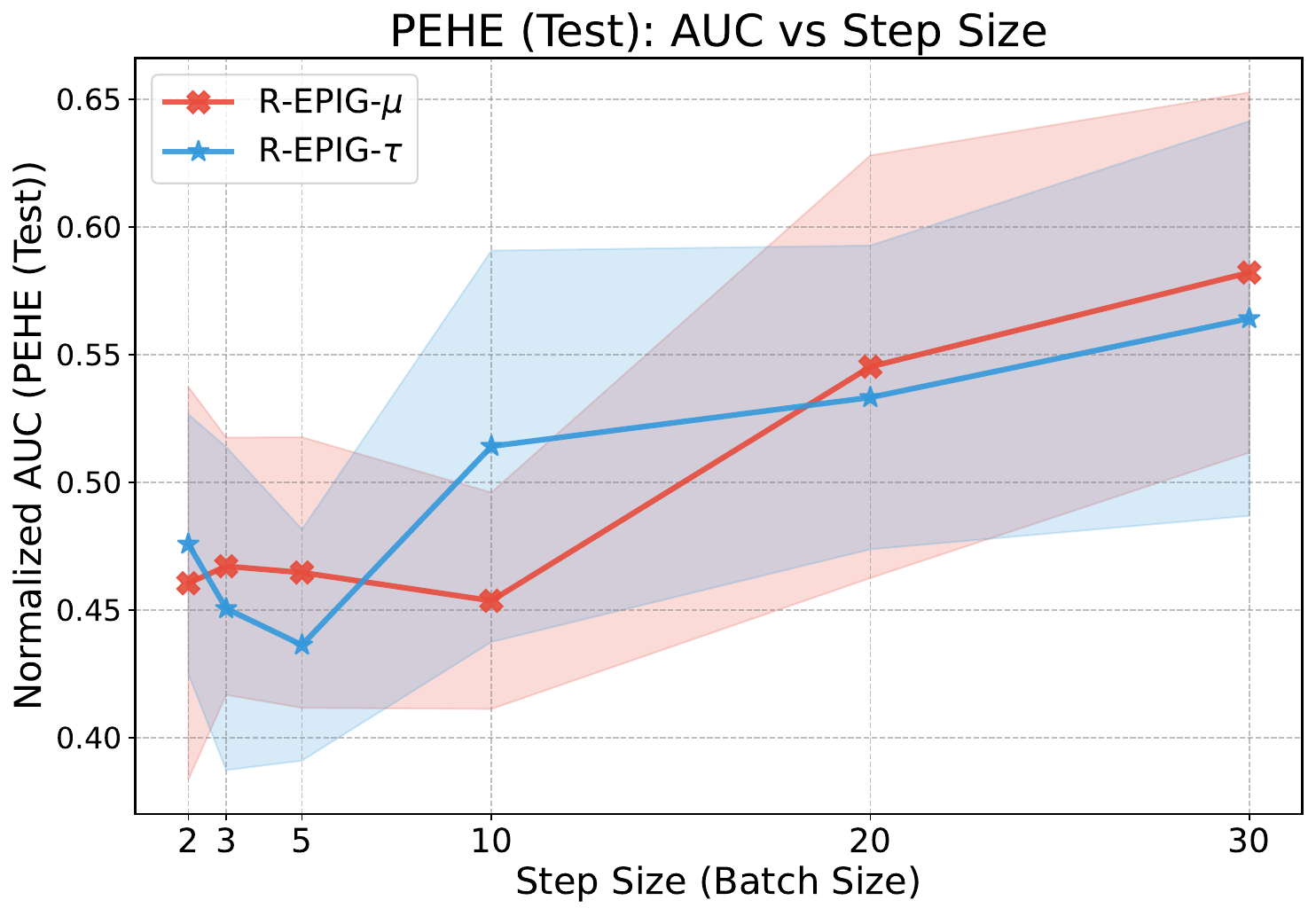}
    \end{minipage}
    \begin{minipage}{0.5\linewidth}
        \centering
        \includegraphics[width=\linewidth]{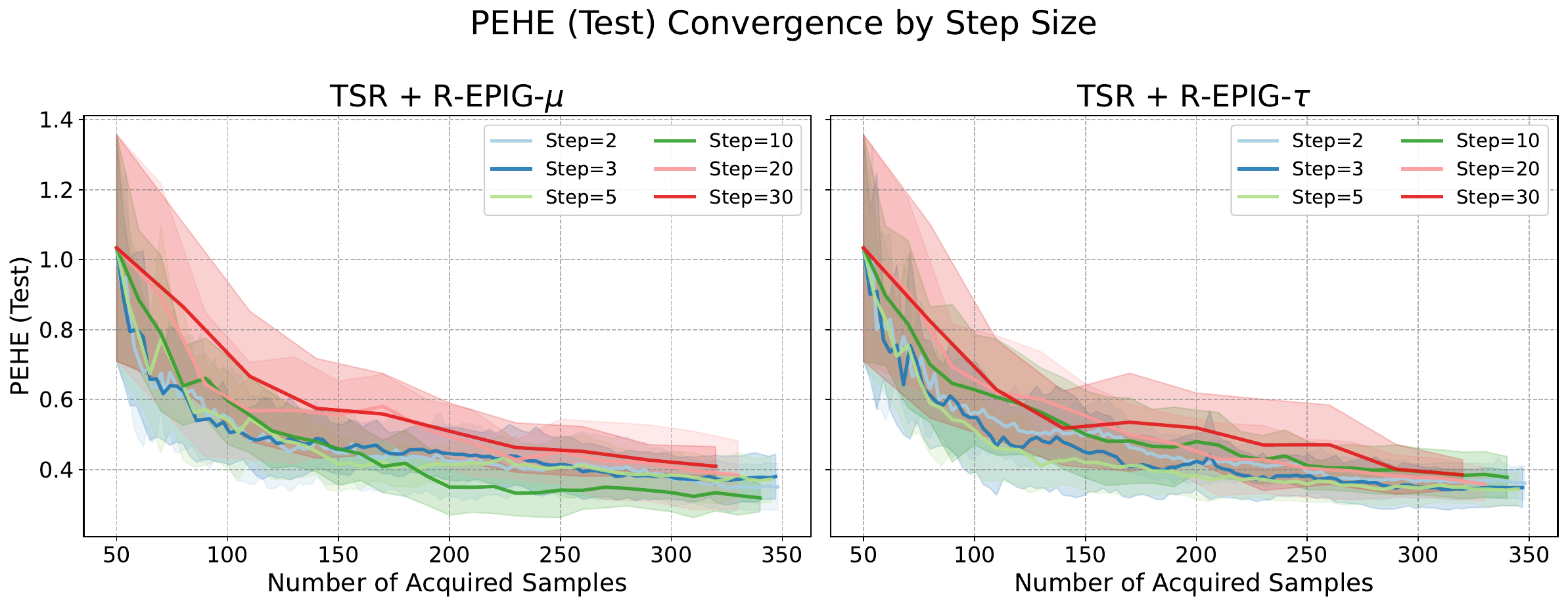}
    \end{minipage}
    
    \caption{Comparison of CATE estimation performance with different batch sizes.}
    \label{fig:different_step_sizes}
\end{figure*}

\begin{figure*}[h]
    \centering   
    \begin{minipage}{0.24\linewidth}
        \centering
        \includegraphics[width=\linewidth]{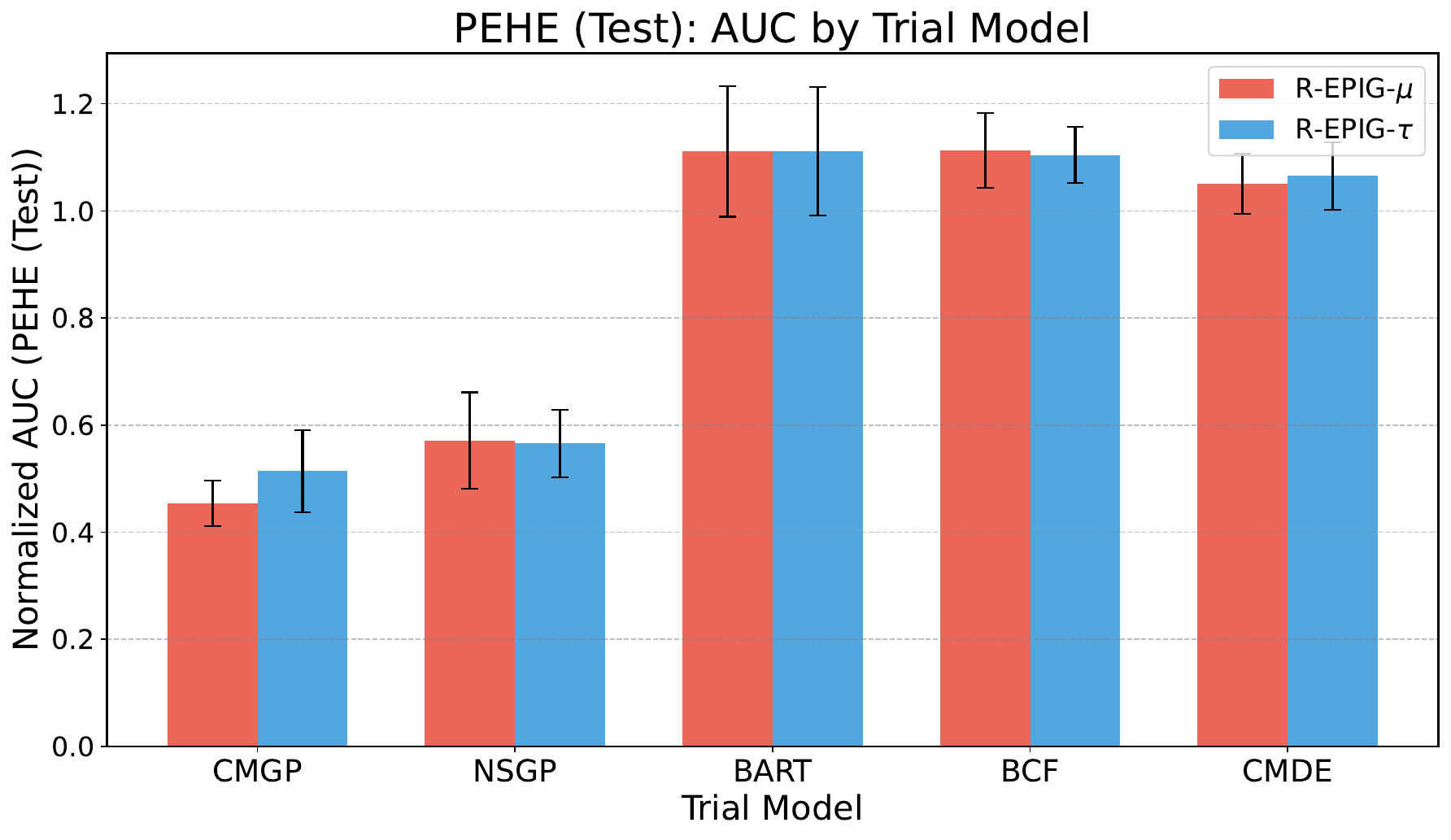}
    \end{minipage}
    \begin{minipage}{0.48\linewidth}
        \centering
        \includegraphics[width=\linewidth]{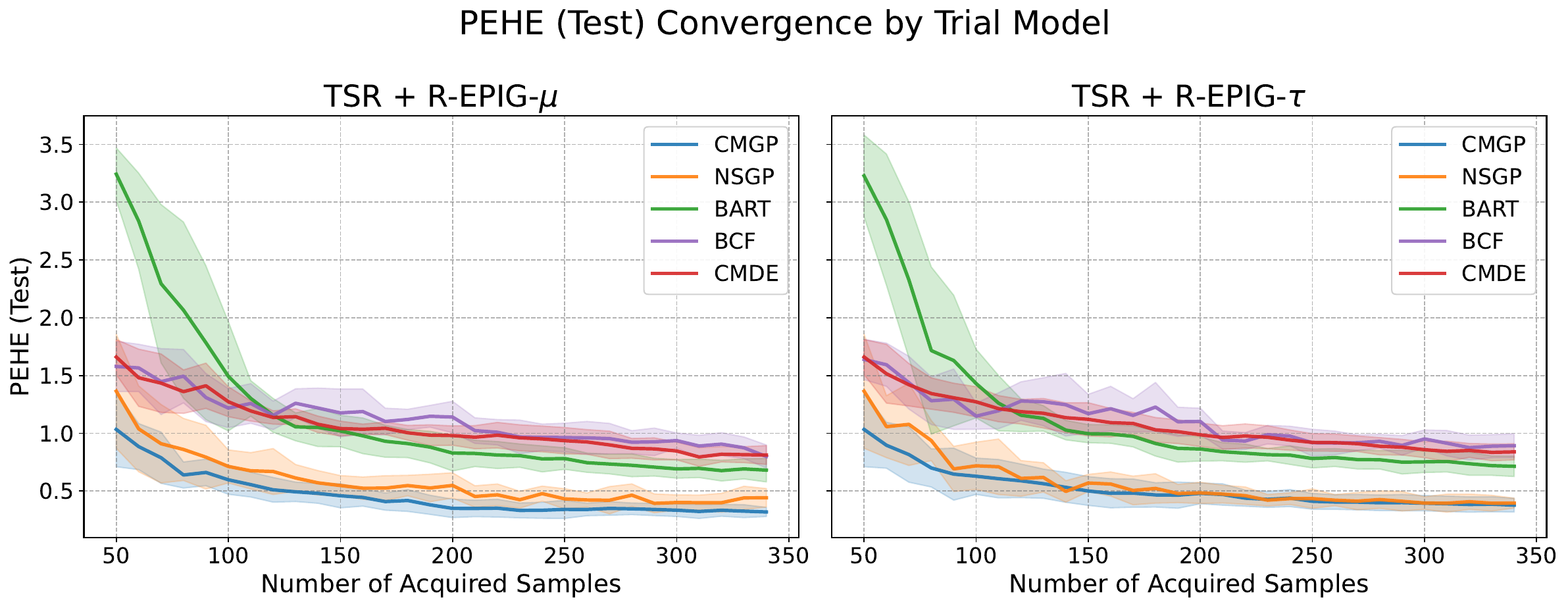}
    \end{minipage}
    \begin{minipage}{0.24\linewidth}
        \centering
        \includegraphics[width=\linewidth]{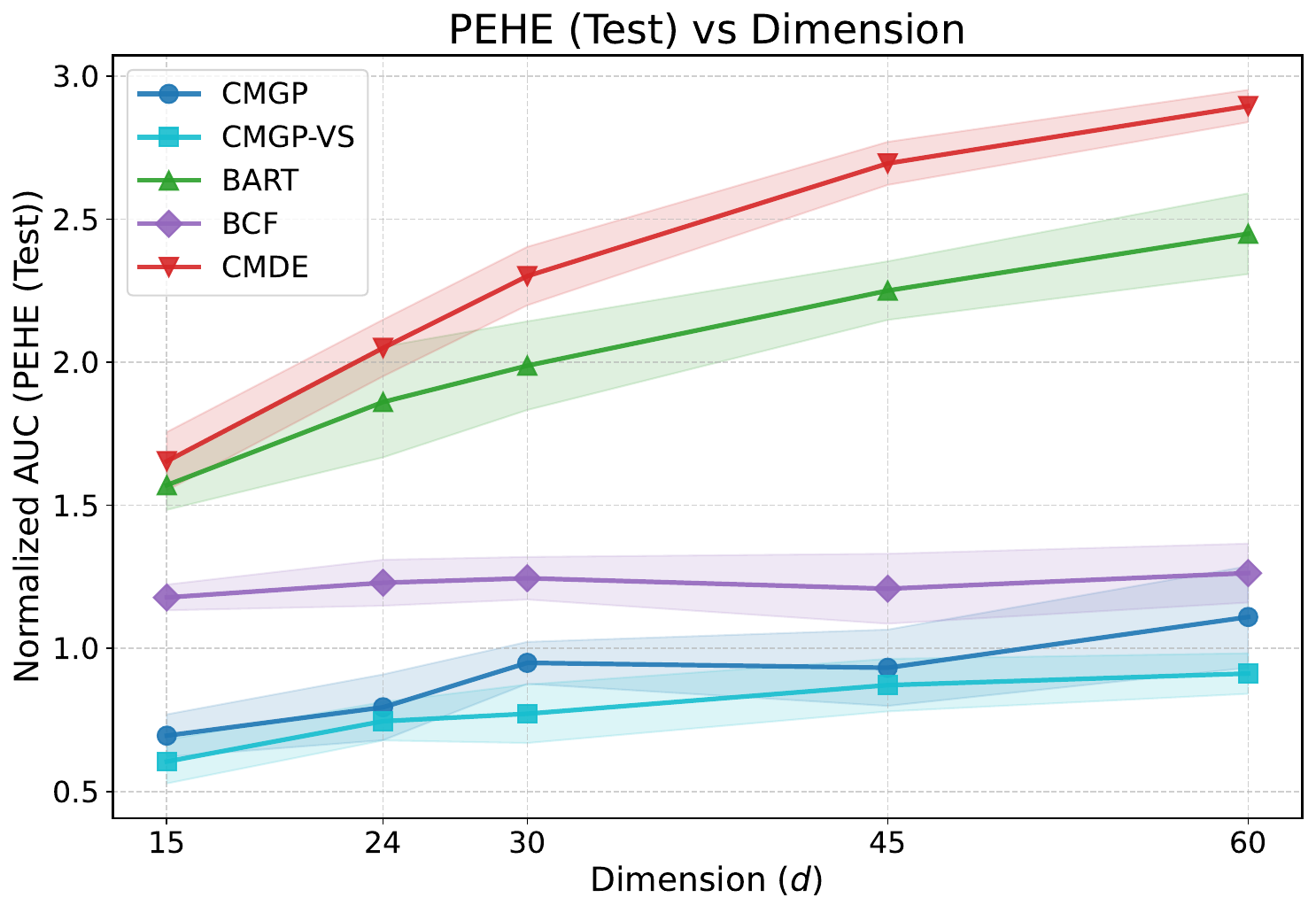}
    \end{minipage}
    \caption{Comparison of CATE estimation performance with different trial models.}
    \label{fig:different_trial_models}
\end{figure*}

%% file: Pages/Appendix/fig_include/real.tex
\begin{figure*}[h]
    \centering   
    \begin{minipage}{0.27\linewidth}
        \centering
        \includegraphics[width=\linewidth]{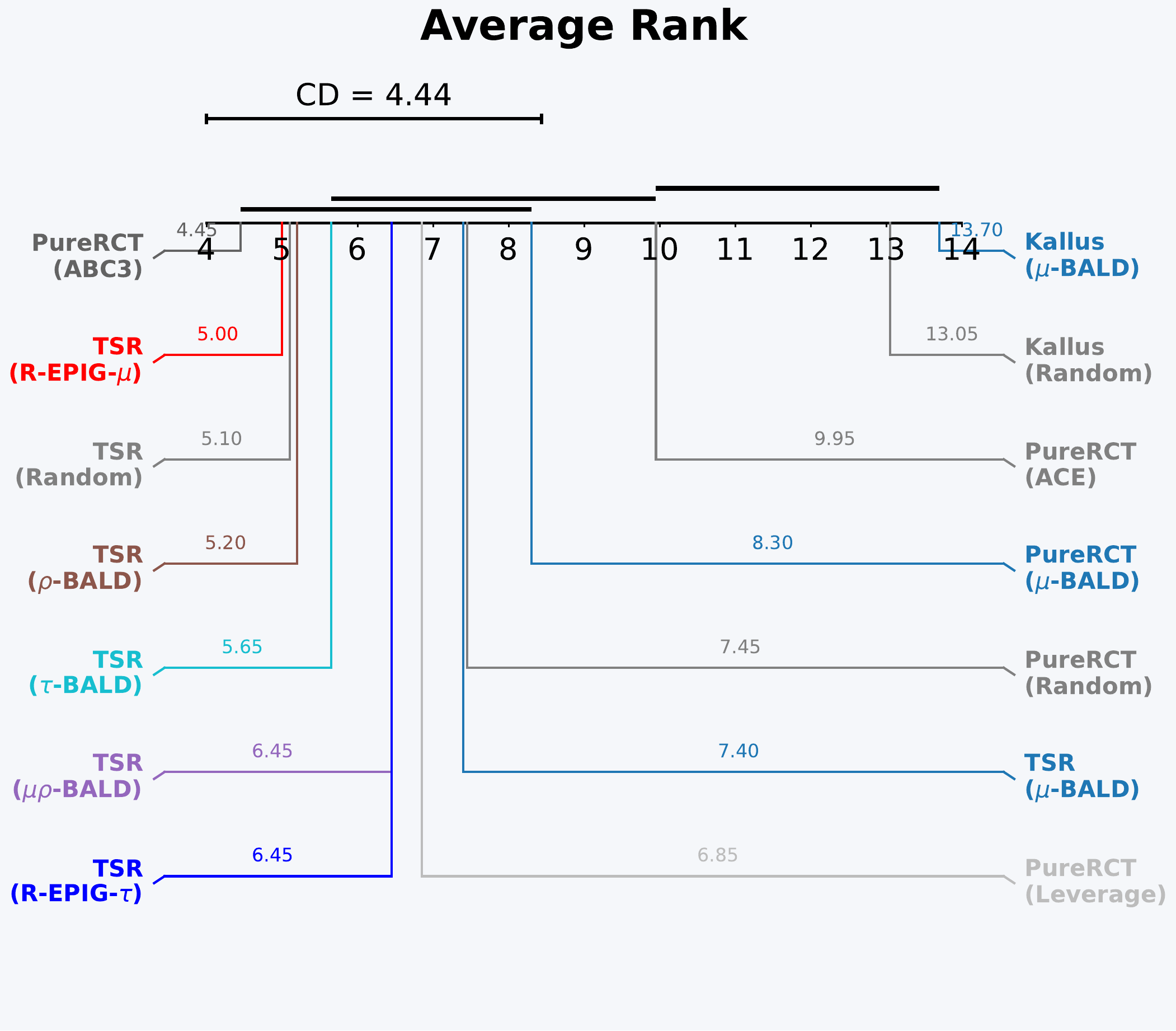}
    \end{minipage}
    \begin{minipage}{0.27\linewidth}
        \centering
        \includegraphics[width=\linewidth]{Figures/synthetic/notebook/real/ihdp/cate/all_methods/relative_pehe_test_boxplot.pdf}
    \end{minipage}
    \begin{minipage}{0.27\linewidth}
        \centering
        \includegraphics[width=\linewidth]{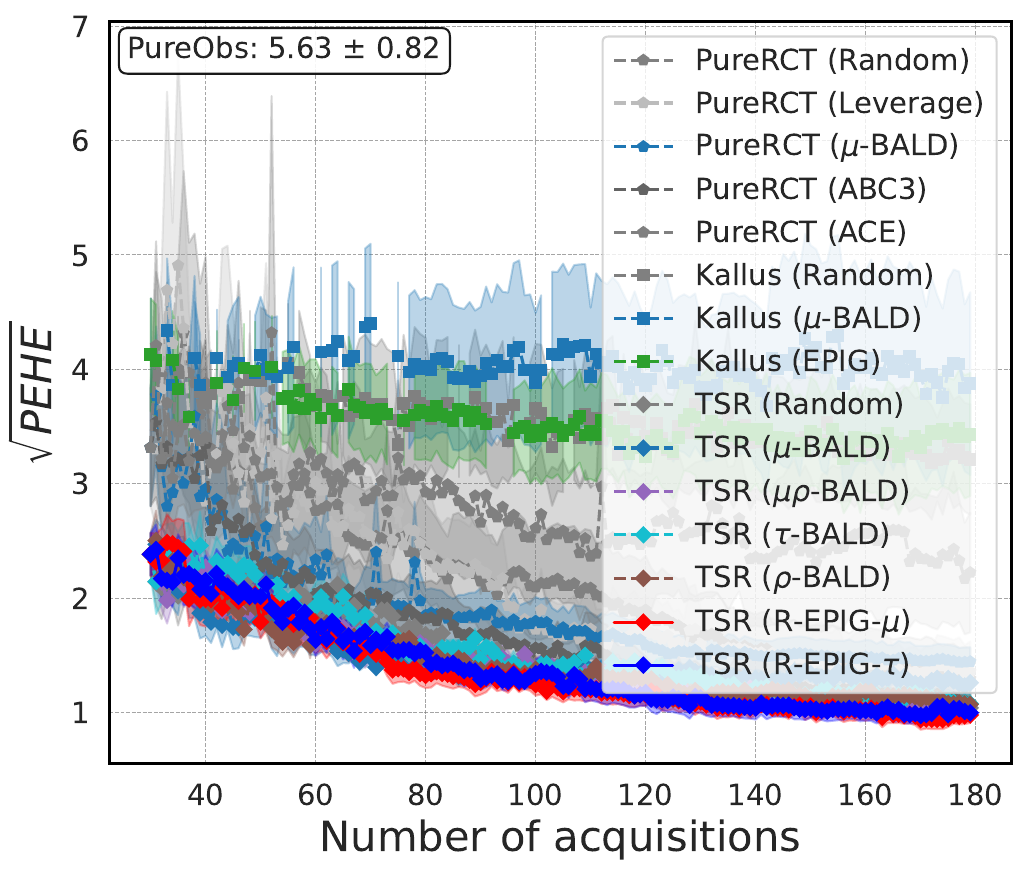}
    \end{minipage}
    \caption{Comparison of CATE estimation performance on IHDP dataset.}
    \label{fig:ihdp_cate}
\end{figure*}

\begin{figure*}[h]
    \centering   
    \begin{minipage}{0.27\linewidth}
        \centering
        \includegraphics[width=\linewidth]{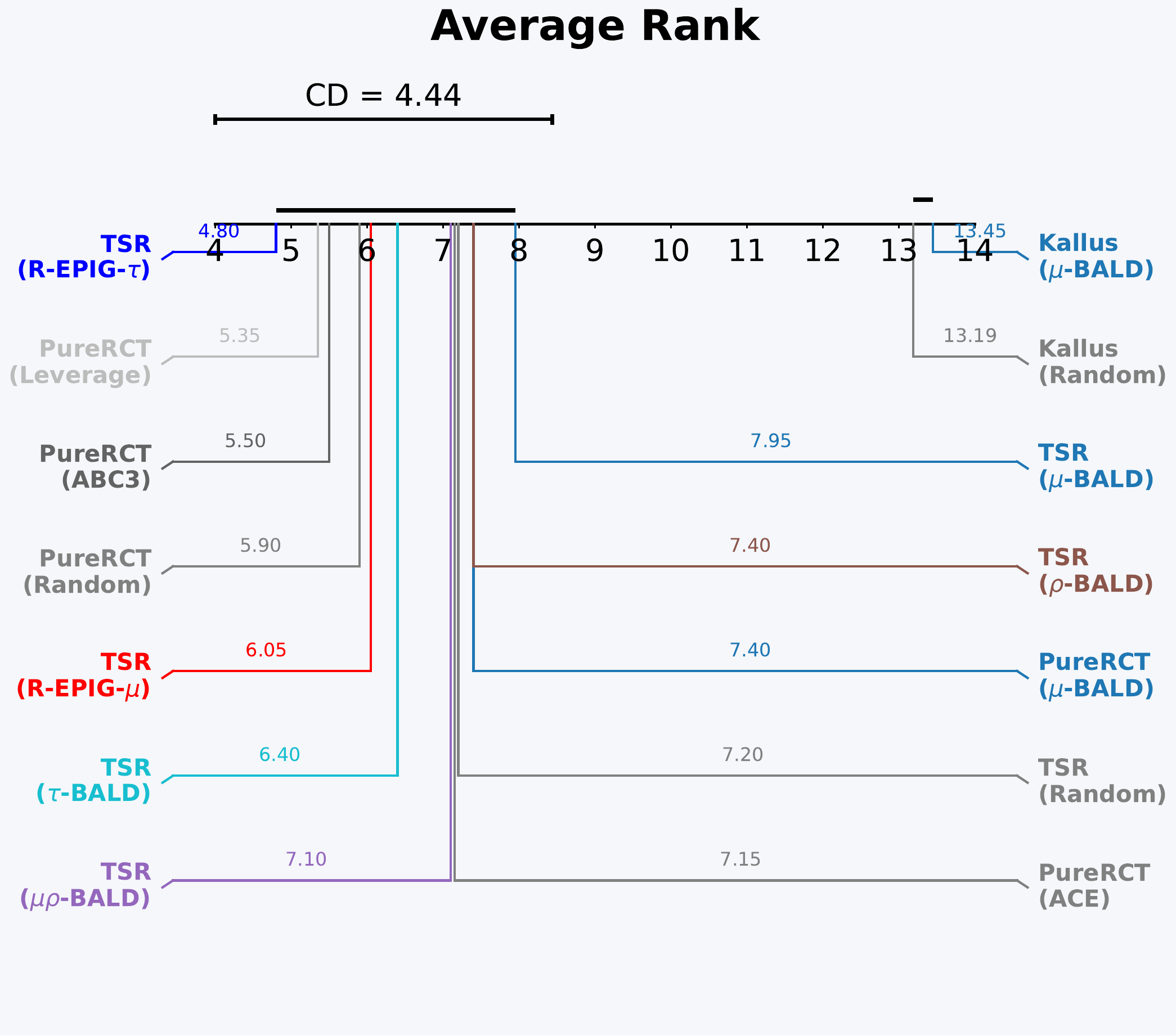}
    \end{minipage}
    \begin{minipage}{0.27\linewidth}
        \centering
        \includegraphics[width=\linewidth]{Figures/synthetic/notebook/real/actg/cate/all_methods/relative_pehe_test_boxplot.pdf}
    \end{minipage}
    \begin{minipage}{0.27\linewidth}
        \centering
        \includegraphics[width=\linewidth]{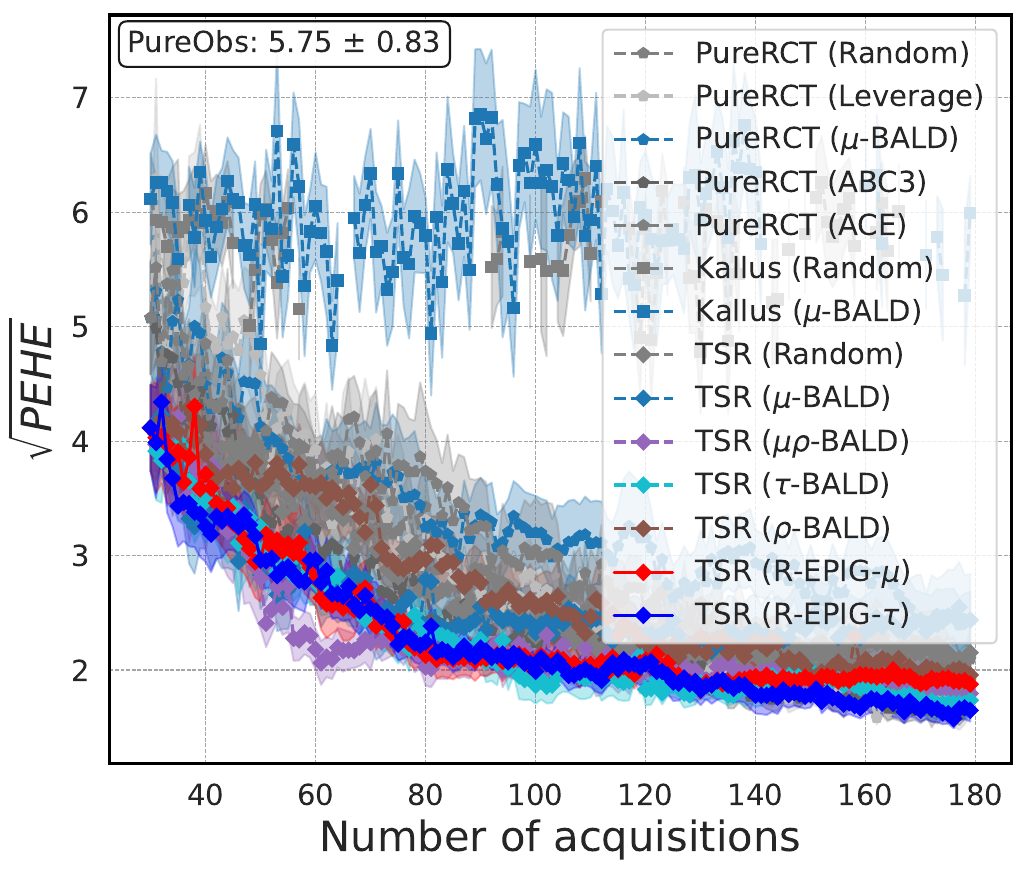}
    \end{minipage}
    \caption{Comparison of CATE estimation performance on ACTG-175 dataset.}
    \label{fig:actg_cate}
\end{figure*}

\begin{figure*}[h]
    \centering   
    \begin{minipage}{0.24\linewidth}
        \centering
        \includegraphics[width=\linewidth]{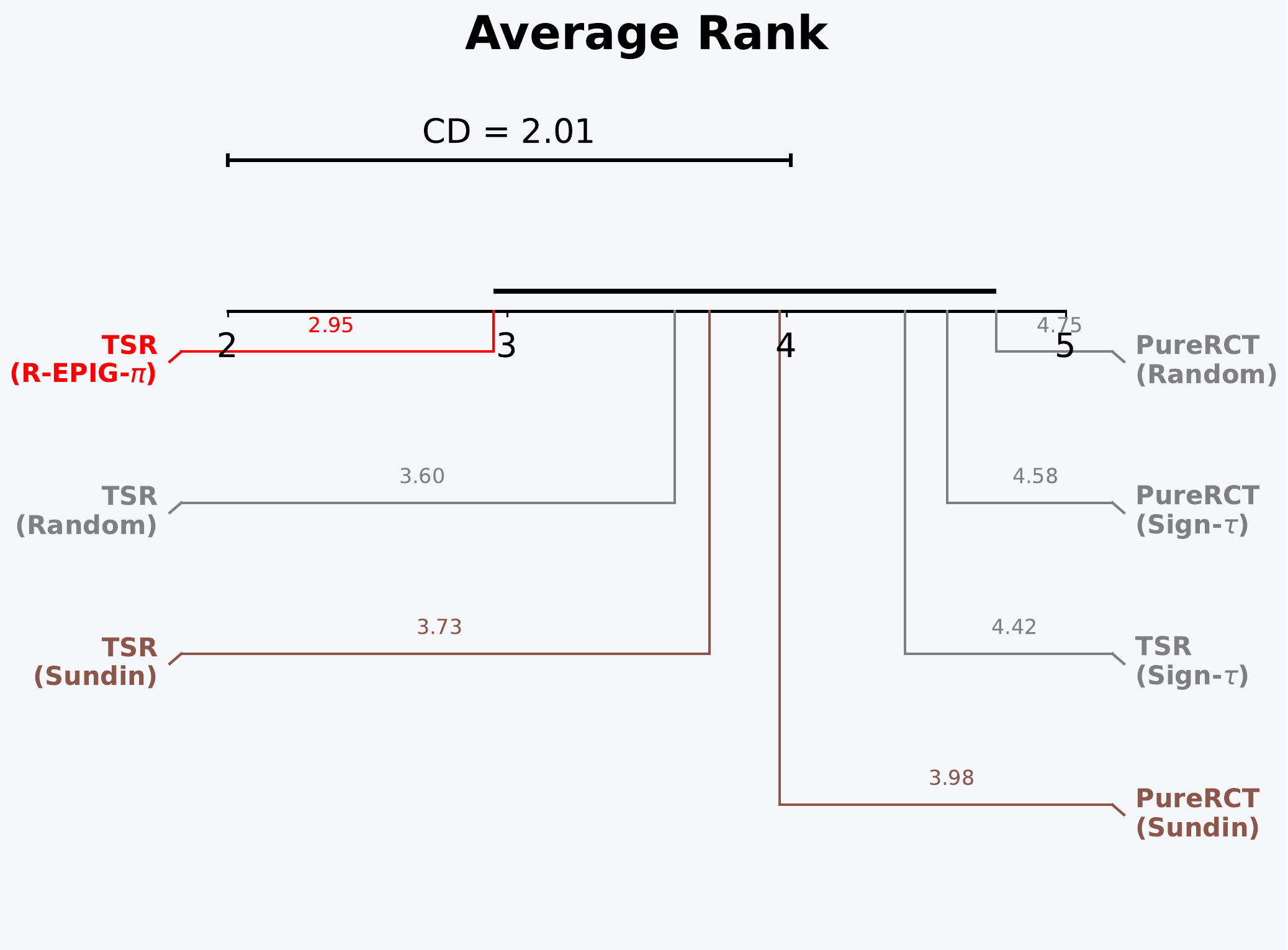}
    \end{minipage}
    \begin{minipage}{0.24\linewidth}
        \centering
        \includegraphics[width=\linewidth]{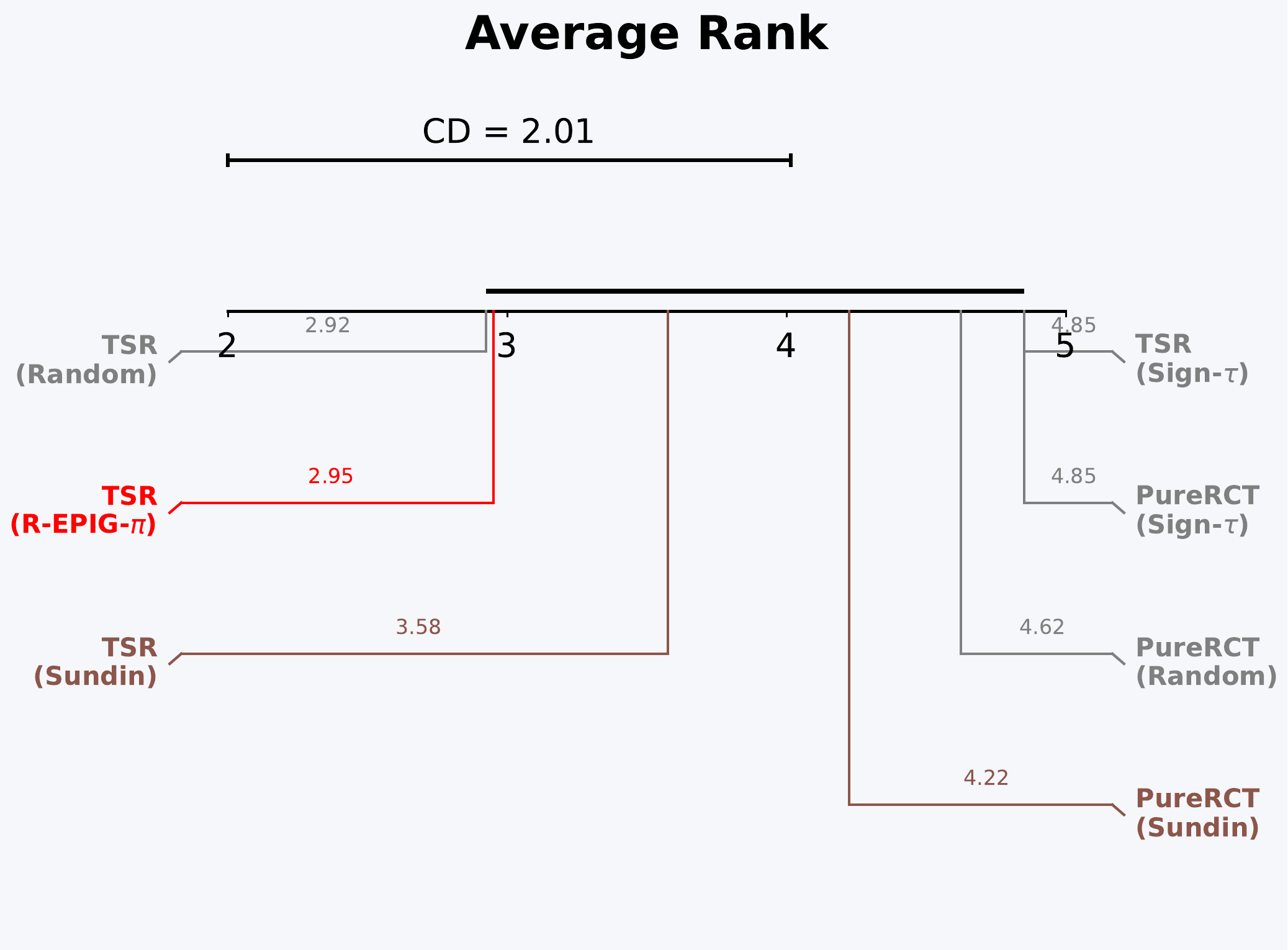}
    \end{minipage}
    \begin{minipage}{0.24\linewidth}
        \centering
        \includegraphics[width=\linewidth]{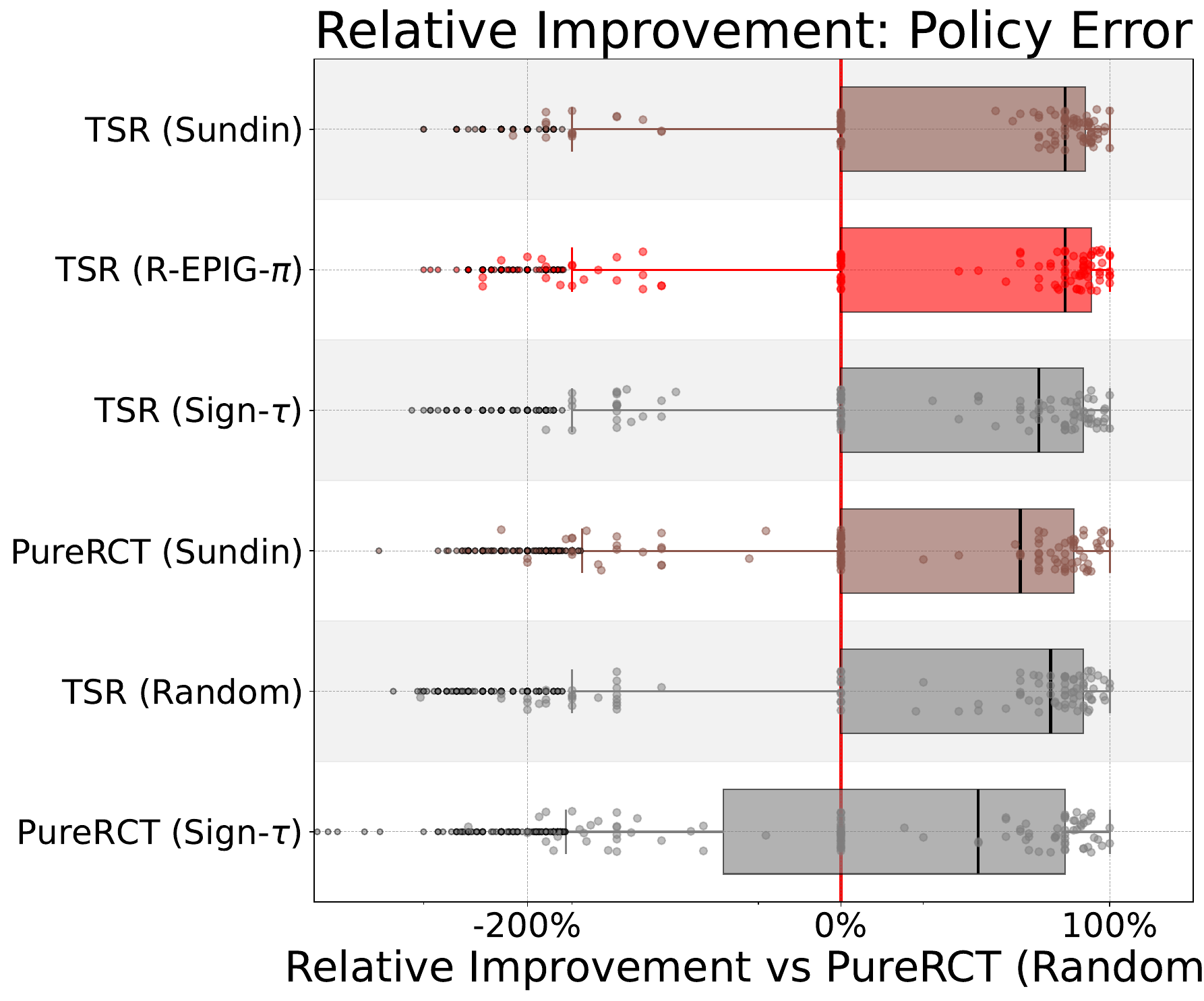}
    \end{minipage}
    \begin{minipage}{0.24\linewidth}
        \centering
        \includegraphics[width=\linewidth]{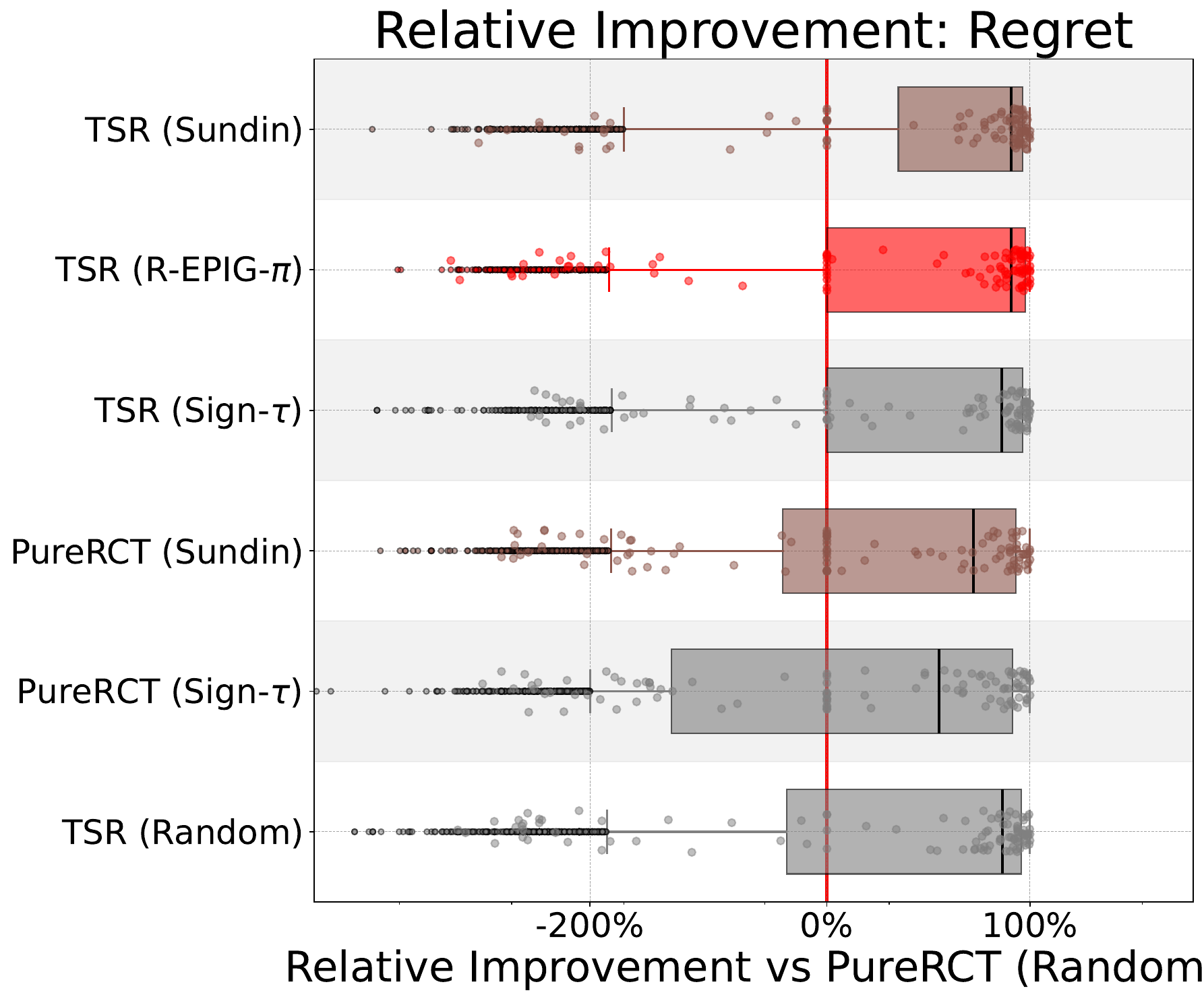}
    \end{minipage}
    \caption{Comparison of policy estimation performance on IHDP dataset.}
    \label{fig:ihdp_dm}
\end{figure*}

\begin{figure*}[h]
    \centering   
    \begin{minipage}{0.24\linewidth}
        \centering
        \includegraphics[width=\linewidth]{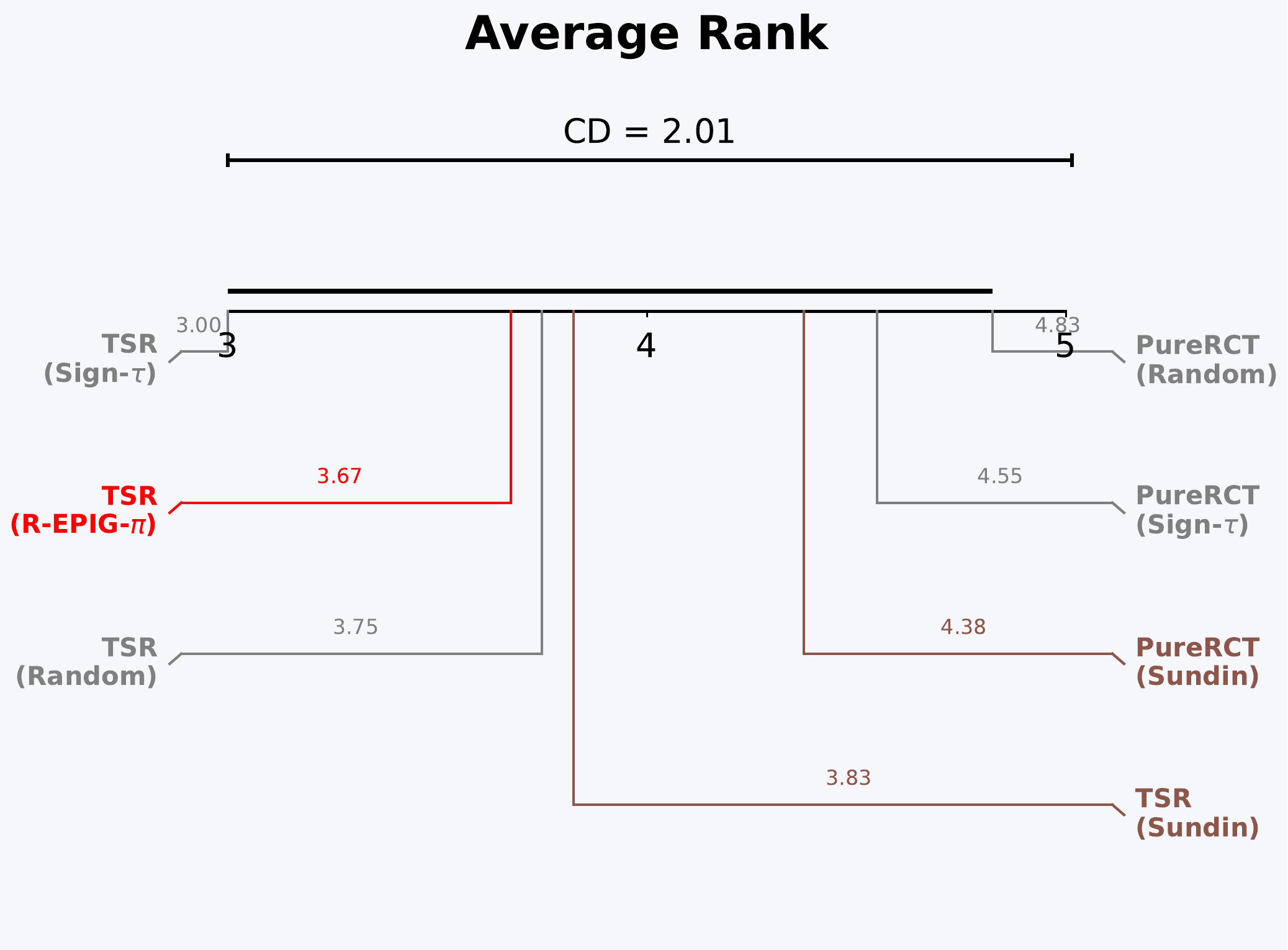}
    \end{minipage}
    \begin{minipage}{0.24\linewidth}
        \centering
        \includegraphics[width=\linewidth]{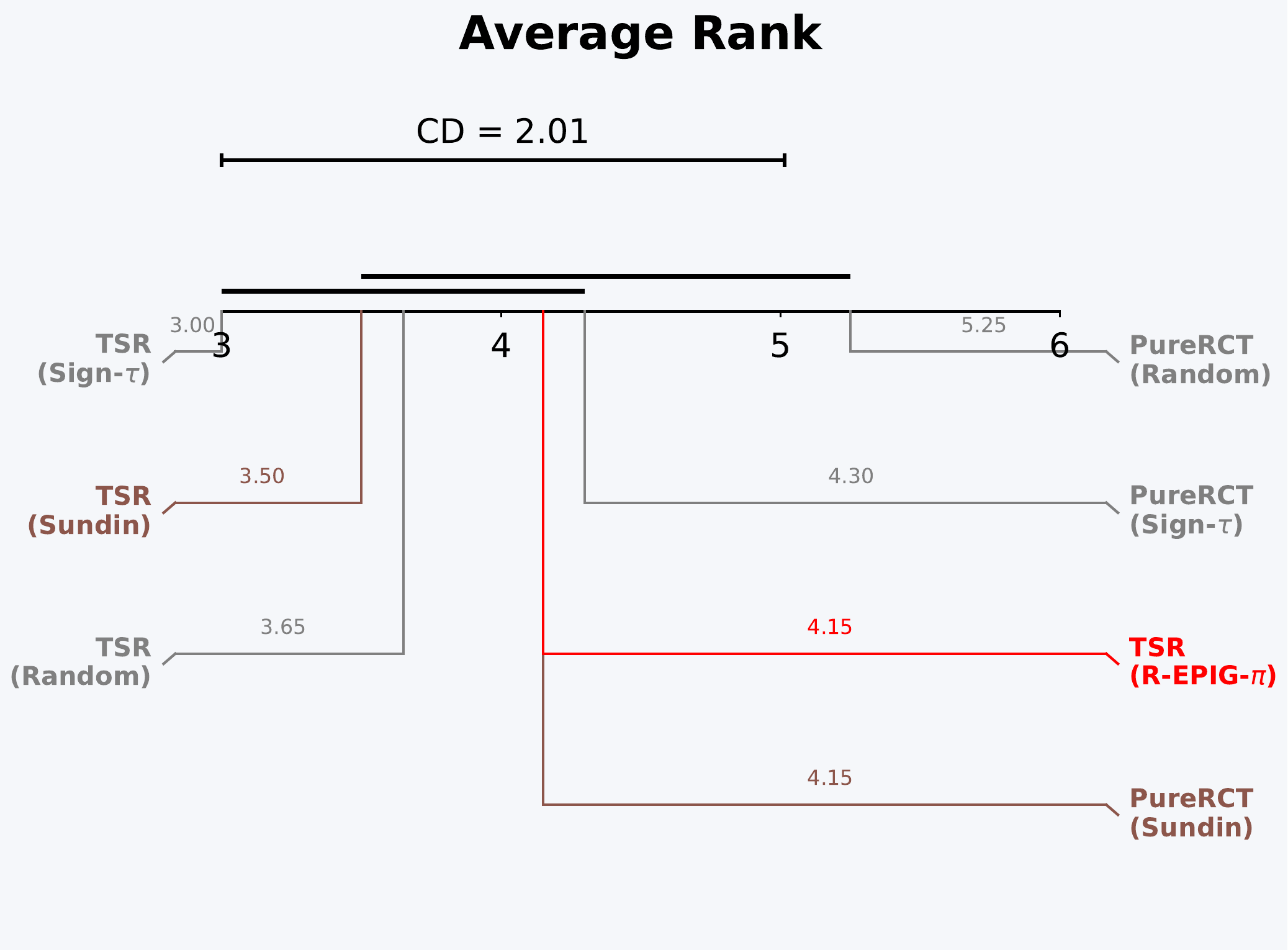}
    \end{minipage}
    \begin{minipage}{0.24\linewidth}
        \centering
        \includegraphics[width=\linewidth]{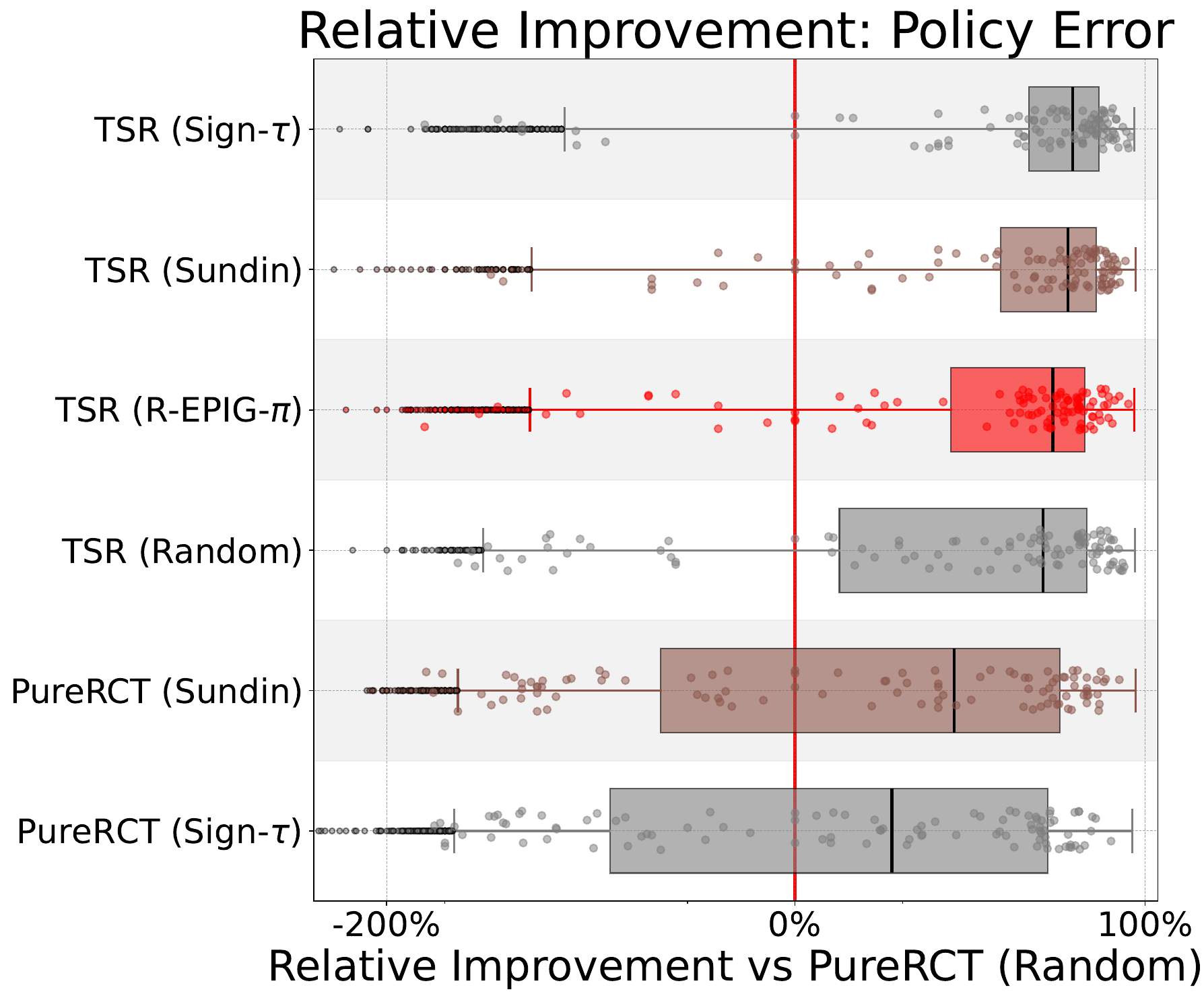}
    \end{minipage}
    \begin{minipage}{0.24\linewidth}
        \centering
        \includegraphics[width=\linewidth]{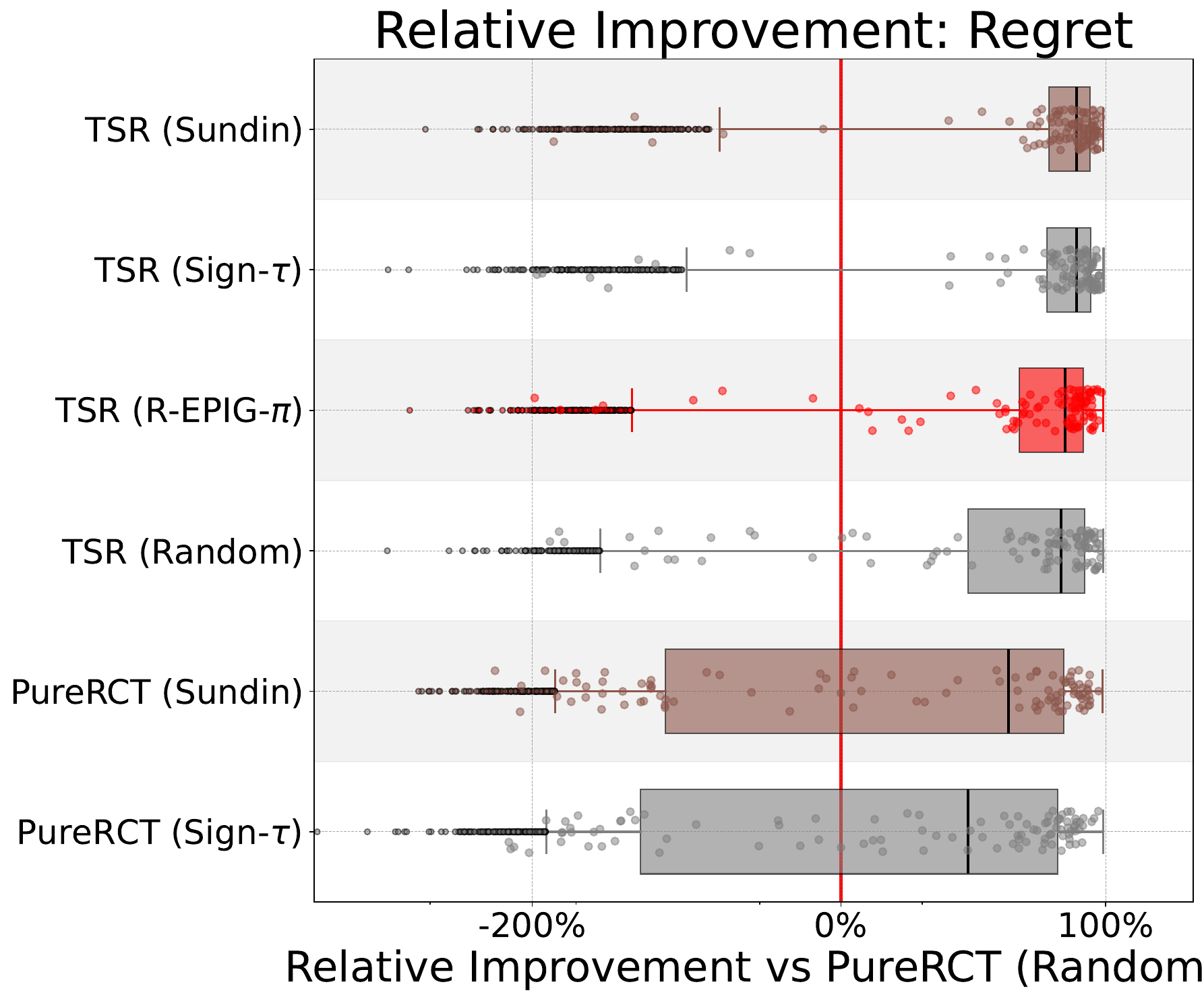}
    \end{minipage}
    \caption{Comparison of policy estimation performance on ACTG-175 dataset.}
    \label{fig:actg_dm}
\end{figure*}

%% file: Pages/Appendix/Discussions.tex
\section{Discussion: Joint Model or TSR?}
\label{app_sec:discussions}

\begin{figure*}[h]
    \centering   
    \begin{minipage}{0.7\linewidth}
        \centering
        \includegraphics[width=\linewidth]{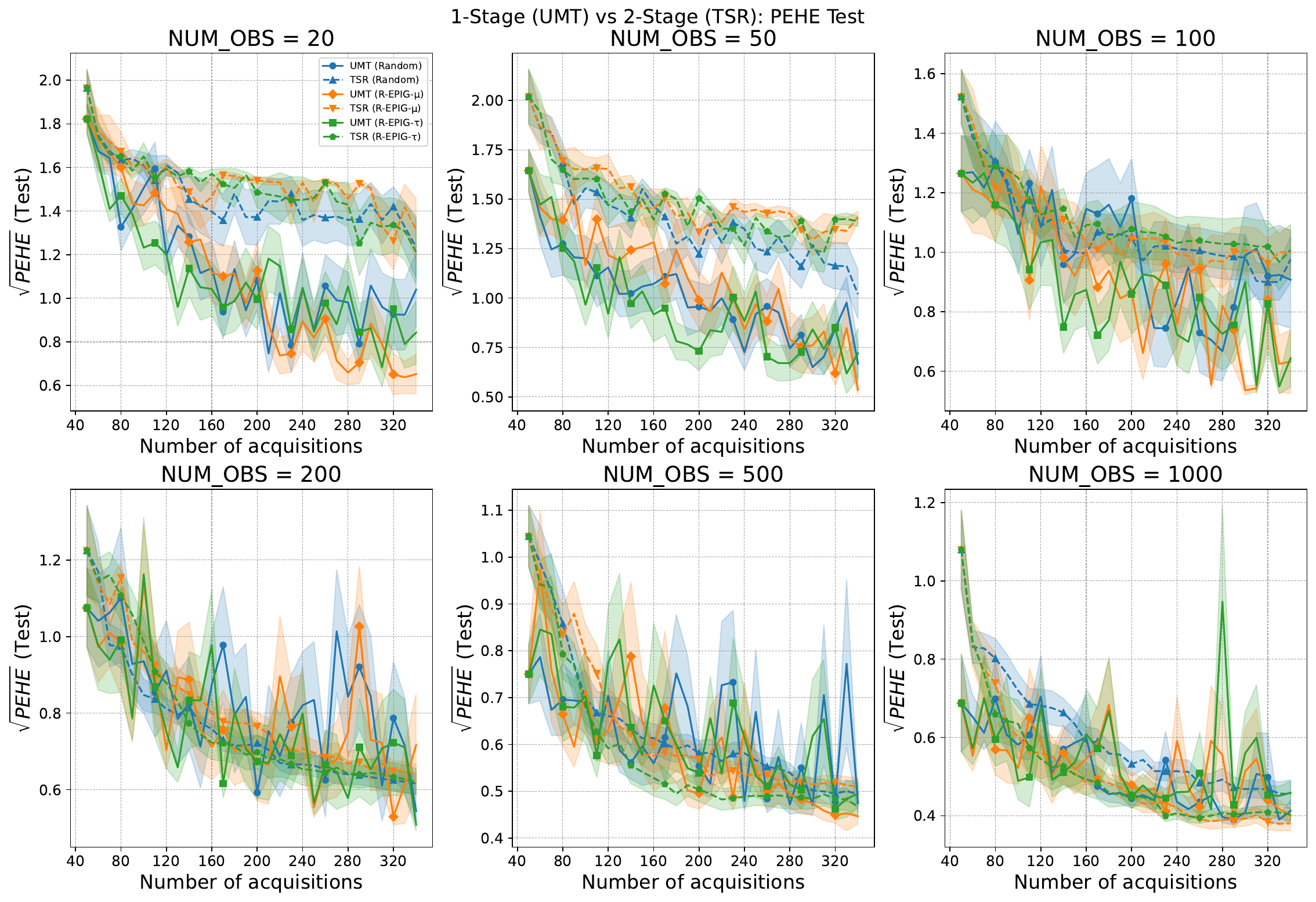}
    \end{minipage}
    \begin{minipage}{0.24\linewidth}
        \centering
        \includegraphics[width=\linewidth]{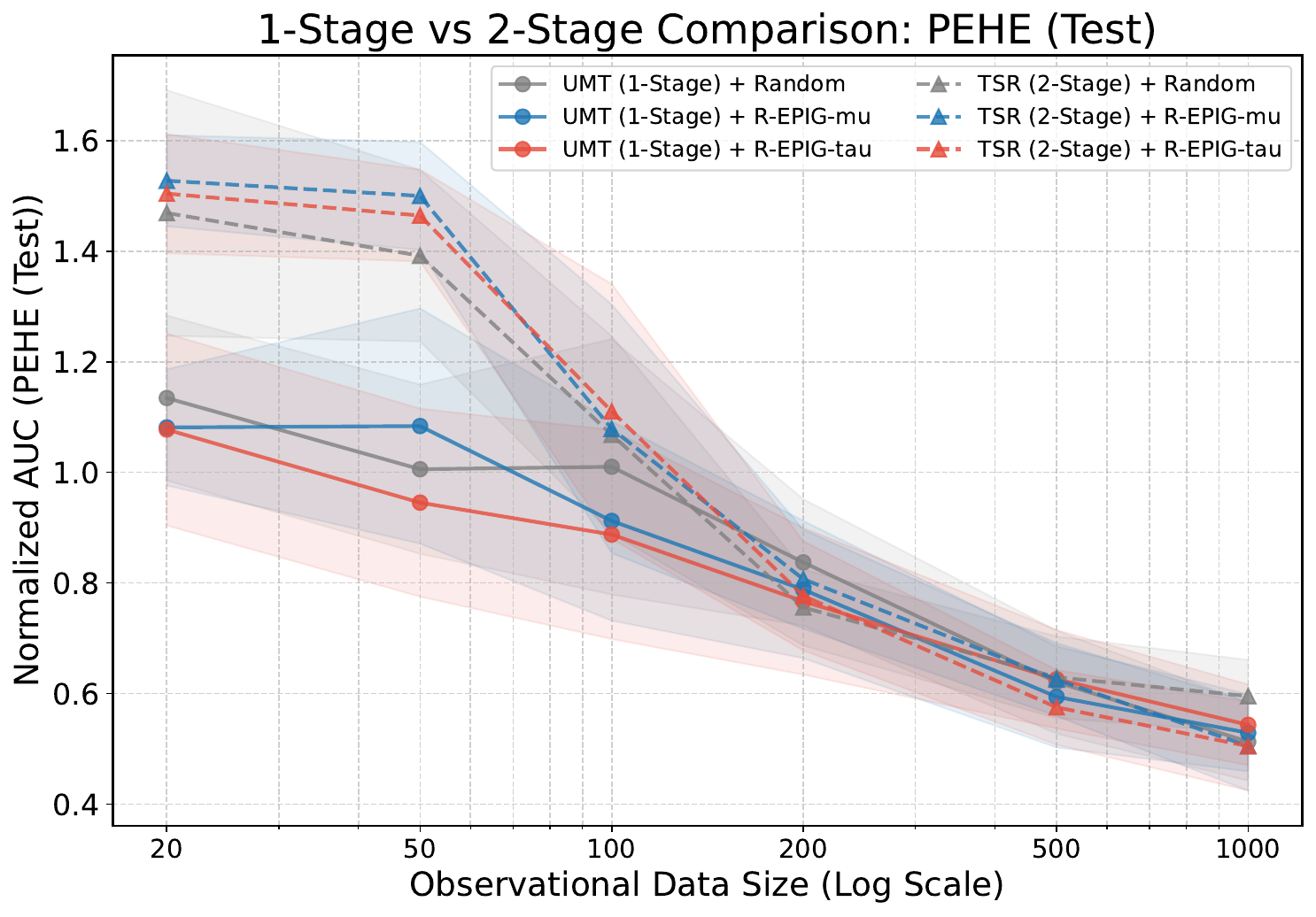}
    \end{minipage}
    \caption{Comparison of CATE estimation performance with two different strategies.}
    \label{fig:umt_and_tsr}
\end{figure*}

In the main text, we presented R-Design using the TSR framework. However, an alternative strategy exists: jointly modeling the observational and experimental data in a unified probabilistic framework, which follows the very similar trend as CausalICM~\citep{dimitriou2024data}. However, the CausalICM only model the correlation within each group, saying using two CMGPs for control and treatment group, respectively. This is for leveraging the information of large observational data to help predict the RCT data. In this part, we discuss a more general case that models all cared outputs together using one model and formally define this Unified Multi-Task (UMT) architecture and provide a comprehensive comparison with TSR regarding computational complexity, error propagation, and sample efficiency.

\subsubsection{The UMT Framework}

Unlike the sequential approach of R-Design, which freezes the observational model parameters before the stage-2 adaptive experiential design, the UMT approach treats the observational and experimental outcomes as outputs of a single, jointly trained Gaussian Process.

\textbf{Formulation.}
We construct a MTGP over the augmented input space $\tilde{\gX} = \gX \times \{0, 1\} \times \{s_o, s_e\}$, where $s$ denotes the data source. The joint model $f_{\text{joint}}(\vx, t, s)$ captures the correlations between the biased observational source and the unbiased experimental source. The kernel function is typically defined using an Intrinsic Coregionalization Model (ICM):
\begin{equation}
    k_{\text{joint}}((\vx, t, s), (\vx', t', s')) = k_{\text{input}}((\vx, t), (\vx', t')) \otimes \mathbf{B}_{s, s'},
\end{equation}
where $\mathbf{B}$ is a $2 \times 2$ coregionalization matrix capturing the covariance between sources $s_o$ and $s_e$.
In this framework, the active acquisition step involves:
1.  Training the global GP on $\gD_O \cup \gD_E^{(k)}$.
2.  Computing the acquisition score (e.g., R-EPIG) using the posterior of the experimental task $f(\cdot, \cdot, s_e)$.
3.  Updating the \textit{entire} model (including source correlations) upon observing a new sample $y_{new}$.

\subsubsection{Theoretical Comparison}

The choice between 1-Stage (UMT) and 2-Stage (TSR) involves a fundamental trade-off between information flow and computational tractability.

\textbf{1. Information Flow and Error Propagation.}
\begin{itemize}[leftmargin=*]
    \item \textbf{UMT (1-Stage):} Allows for bidirectional information flow. Experimental data can retrospectively update the belief about the observational parameters. This is advantageous in \textbf{low-data regimes}, where the observational dataset $\gD_O$ is too small to form a reliable prior. In such cases, joint training prevents the "frozen error" problem, where a poor Stage 1 model permanently handicaps the residual learner.
    \item \textbf{TSR (2-Stage):} Imposes a hard structural constraint. By fixing $\hat{\mu}_o$, it forces the Stage 2 model to explain all discrepancies as residuals. This isolates the experimental signal, preventing it from being overwhelmed by the observational likelihood when $n_O$ is large.
\end{itemize}

\textbf{2. Computational Complexity.}
This is the critical bottleneck for Active Learning.
\begin{itemize}[leftmargin=*]
    \item \textbf{UMT:} The computational cost scales cubically with the \textit{total} data size: $\gO((n_O + n_E)^3)$. For large observational studies (e.g., $n_O > 10^4$), re-training the joint model and computing acquisition scores at every step is computationally prohibitive.
    \item \textbf{TSR:} The active loop scales only with the experimental data size: $\gO(n_E^3)$. The massive observational data is processed once offline. This makes TSR scalable to large-scale real-world settings.
\end{itemize}

\subsubsection{Empirical Analysis}

We validate these theoretical insights through an ablation study comparing UMT and TSR. We vary the observational data size $n_O$ while keeping the experimental budget fixed, reporting the PEHE performance.

\begin{table*}[h]
    \centering
    \caption{Comparison of normalized PEHE (Test) AUC for 1-Stage (UMT) and 2-Stage (TSR) methods across varying observational data sizes ($n_O$). Both methods utilize the R-EPIG-$\mu$ acquisition function.}
    \label{tab:1vs2}
    \begin{tabular}{lcccccc}
        \toprule
        \textbf{Method} & \multicolumn{6}{c}{\textbf{Observational Data Size ($n_O$)}} \\
        \cmidrule(lr){2-7}
        & $n_O=20$ & $n_O=50$ & $n_O=100$ & $n_O=200$ & $n_O=500$ & $n_O=1000$ \\
        \midrule
        UMT (1-Stage) & \textbf{1.08} & \textbf{1.08} & \textbf{0.91} & \textbf{0.79} & \textbf{0.59} & 0.53 \\
        TSR (2-Stage) & 1.53 & 1.50 & 1.08 & 0.81 & 0.63 & \textbf{0.50} \\
        \bottomrule
    \end{tabular}
\end{table*}

\textbf{Discussion of Results.}
The results in Tab.~\ref{tab:1vs2} align perfectly with our theoretical analysis:

\begin{enumerate}[leftmargin=*]
    \item \textbf{Low-Data Regime ($n_O < 200$):} The UMT approach significantly outperforms TSR. When observational data is scarce, the Stage 1 estimator in TSR is highly inaccurate. Since TSR freezes this "garbage" model, the Stage 2 learner faces the difficult task of learning a highly complex, high-variance residual surface to compensate. UMT avoids this by jointly smoothing both sources, leveraging the experimental data to stabilize the observational estimates.
    
    \item \textbf{Transition and Crossover:} As $n_O$ increases, the Stage 1 model in TSR becomes reliable, and the performance gap narrows rapidly. The "residual" becomes simpler and smoother (matching the structural assumption of Lemma~\ref{lemma:complexity_decomp_hoelder}).
    
    \item \textbf{High-Data Regime ($n_O \ge 1000$):} A crossover occurs at $n=1000$, where TSR (0.50) surpasses UMT (0.53). This highlights the "Bias Isolation" advantage of the 2-Stage architecture. In the joint model, as $n_O$ becomes dominant, the likelihood function may prioritize fitting the massive observational data over the sparse experimental points, potentially diluting the bias correction. TSR, by treating the observational model as a fixed anchor, ensures that the experimental budget is strictly dedicated to refining the difference, yielding superior asymptotic performance when a strong prior is available.
\end{enumerate}

\textbf{Conclusion.} While UMT is robust in data-scarce settings, R-Design (TSR) is the optimal choice for the problem setting of this paper, where we assume access to a large, pre-existing observational dataset ($n_O \gg n_E$). TSR offers superior scalability and asymptotic precision by effectively decoupling bias estimation from residual correction.